\documentclass[letterpaper, 10 pt, journal]{IEEEtran}

\IEEEoverridecommandlockouts                             
\usepackage{comment}
\usepackage[binary-units]{siunitx}
\usepackage{relsize}
\usepackage{ifthen}
\usepackage[colorinlistoftodos]{todonotes}

\usepackage[vlined,ruled,linesnumbered]{algorithm2e}
\usepackage{graphics} 
\usepackage{rotating}
\usepackage{color}
\usepackage{enumerate}
\usepackage[T1]{fontenc}
\usepackage{psfrag}
\usepackage{epsfig} 
\usepackage{booktabs}
\usepackage{graphicx,url}
\usepackage{multirow}
\usepackage{array}
\usepackage{latexsym}
\usepackage{amsfonts}
\usepackage{amsmath}
\usepackage{amssymb}
\usepackage{xstring}
\usepackage[noend]{algorithmic}
\usepackage{multirow}
\usepackage{xcolor}
\usepackage{prettyref}
\usepackage{flexisym}
\usepackage{bigdelim}
\usepackage{breqn} 
\usepackage{listings}
\let\labelindent\relax
\usepackage{enumitem}
\usepackage{xspace}
\usepackage{bm}
\graphicspath{{./figures/}}
\usepackage{tikz}
\usetikzlibrary{matrix,calc}

\usepackage{mdwlist}

\let\itemize\undefined
\makecompactlist{itemize}{stditemize}

\usepackage{amsthm,amsmath}
\newtheoremstyle{mystyle}%
  {}%
  {}%
  {\itshape}%
  {}%
  {\bfseries}%
  {.}%
  { }%
  {\thmname{#1}\thmnumber{ #2}\thmnote{ (#3)}}%
\theoremstyle{mystyle}

\newrefformat{prob}{Problem\,\ref{#1}}
\newrefformat{def}{Definition\,\ref{#1}}
\newrefformat{sec}{Section\,\ref{#1}}
\newrefformat{sub}{Section\,\ref{#1}}
\newrefformat{prop}{Proposition\,\ref{#1}}
\newrefformat{app}{Appendix\,\ref{#1}}
\newrefformat{alg}{Algorithm\,\ref{#1}}
\newrefformat{cor}{Corollary\,\ref{#1}}
\newrefformat{thm}{Theorem\,\ref{#1}}
\newrefformat{lem}{Lemma\,\ref{#1}}
\newrefformat{fig}{Fig.\,\ref{#1}}
\newrefformat{tab}{Table\,\ref{#1}}

\newtheorem{theorem}{Theorem}
\newtheorem{problem}{Problem}

\newtheorem{corollary}[theorem]{Corollary}

\newtheorem{lemma}[theorem]{Lemma}

\newtheorem{definition}[theorem]{Definition}
\newtheorem{proposition}[theorem]{Proposition}
\newtheorem{remark}[theorem]{Remark}

\newcommand{\cf}{\emph{cf.}\xspace}

\newcommand{\bdmath}{\begin{dmath}}
\newcommand{\edmath}{\end{dmath}}
\newcommand{\beq}{\begin{equation}}
\newcommand{\eeq}{\end{equation}}
\newcommand{\bdm}{\begin{displaymath}}
\newcommand{\edm}{\end{displaymath}}
\newcommand{\bea}{\begin{eqnarray}}
\newcommand{\eea}{\end{eqnarray}}
\newcommand{\beal}{\beq \begin{array}{ll}}
\newcommand{\eeal}{\end{array} \eeq}
\newcommand{\beas}{\begin{eqnarray*}}
\newcommand{\eeas}{\end{eqnarray*}}
\newcommand{\ba}{\begin{array}}
\newcommand{\ea}{\end{array}}
\newcommand{\bit}{\begin{itemize}}
\newcommand{\eit}{\end{itemize}}
\newcommand{\ben}{\begin{enumerate}}
\newcommand{\een}{\end{enumerate}}

\newcommand{\calA}{{\cal A}}
\newcommand{\calB}{{\cal B}}

\newcommand{\calE}{{\cal E}}

\newcommand{\calG}{{\cal G}}

\newcommand{\calI}{{\cal I}}

\newcommand{\calM}{{\cal M}}
\newcommand{\calN}{{\cal N}}

\newcommand{\calP}{{\cal P}}

\newcommand{\calS}{{\cal S}}

\newcommand{\calV}{{\cal V}}
\newcommand{\calW}{{\cal W}}

\newcommand{\calY}{{\cal Y}}

\newcommand{\etal}{\emph{et~al.}\xspace}
\newcommand{\setal}{~\emph{et~al.}\xspace}
\newcommand{\eg}{\emph{e.g.,}\xspace}
\newcommand{\ie}{\emph{i.e.,}\xspace}
\newcommand{\myParagraph}[1]{{\bf #1.}\xspace}

\newcommand{\M}[1]{{\bm #1}} 
\renewcommand{\boldsymbol}[1]{{\bm #1}}

\newcommand{\LC}[1]{{\color{red} \textbf{LC}: #1}}

\newcommand{\hide}[1]{}
\newcommand{\wrt}{w.r.t.\xspace}

\newcommand{\grayout}[1]{{\color{gray} #1}}

\newcommand{\hiddenText}{{\color{gray} hidden text.}}
\newcommand{\hideWithText}[1]{\hiddenText}

\newcommand{\kron}{\otimes}

\newcommand{\one}{ {\mathbf{1}} }
\newcommand{\subject}{\text{ subject to }}
\DeclareMathOperator*{\argmax}{arg\,max}
\DeclareMathOperator*{\argmin}{arg\,min}

\newcommand{\norm}[1]{\left\| #1 \right\|}

\newcommand{\prob}[1]{{\mathbb P}\left(#1\right)}
\newcommand{\tran}{^{\mathsf{T}}}

\newcommand{\trace}[1]{\mathrm{tr}\left(#1\right)}

\newcommand{\rank}[1]{\mathrm{rank}\left(#1\right)}

\newcommand{\inv}{^{-1}}

\newcommand{\ones}{{\mathbf 1}}
\newcommand{\zero}{{\mathbf 0}}
\newcommand{\eye}{{\mathbf I}}
\newcommand{\vect}[1]{\left[\begin{array}{c}  #1  \end{array}\right]}
\newcommand{\matTwo}[1]{\left[\begin{array}{cc}  #1  \end{array}\right]}

\newcommand{\Real}[1]{ { {\mathbb R}^{#1} } }

\newcommand{\setdef}[2]{ \{#1 \; {:} \; #2 \} }

\newcommand{\SEthree}{\ensuremath{\mathrm{SE}(3)}\xspace}

\newcommand{\SOthree}{\ensuremath{\mathrm{SO}(3)}\xspace}

\newcommand{\MA}{\M{A}}
\newcommand{\MB}{\M{B}}
\newcommand{\MC}{\M{C}}

\newcommand{\MG}{\M{G}}
\newcommand{\MM}{\M{M}}

\newcommand{\MP}{\M{P}}
\newcommand{\MQ}{\M{Q}}

\newcommand{\MR}{\M{R}}

\newcommand{\MI}{\M{I}}

\newcommand{\MH}{\M{H}}

\newcommand{\MX}{\M{X}}
\newcommand{\MY}{\M{Y}}
\newcommand{\MW}{\M{W}}

\newcommand{\va}{\boldsymbol{a}} 
\newcommand{\vh}{\boldsymbol{h}} 
\newcommand{\vb}{\boldsymbol{b}}
\newcommand{\vc}{\boldsymbol{c}}

\newcommand{\ve}{\boldsymbol{e}}
\newcommand{\vf}{\boldsymbol{f}}
\newcommand{\vg}{\boldsymbol{g}}

\newcommand{\vn}{\boldsymbol{n}}
\newcommand{\vo}{\boldsymbol{o}}
\newcommand{\vp}{\boldsymbol{p}}

\newcommand{\vr}{\boldsymbol{r}}
\newcommand{\vs}{\boldsymbol{s}}
\newcommand{\vu}{\boldsymbol{u}}
\newcommand{\vv}{\boldsymbol{v}}
\newcommand{\vt}{\boldsymbol{t}}
\newcommand{\vxx}{\boldsymbol{x}} 
\newcommand{\vy}{\boldsymbol{y}}

\newcommand{\vtheta}{\boldsymbol{\theta}}

\newcommand{\vepsilon}{\boldsymbol{\epsilon}}

\newcommand{\scenario}[1]{{\smaller \sf#1}\xspace}

\newcommand{\blue}[1]{{\color{blue}#1}}

\newcommand{\red}[1]{{\color{red}#1}}

\newcommand{\linkToPdf}[1]{\href{#1}{\blue{(pdf)}}}
\newcommand{\linkToPpt}[1]{\href{#1}{\blue{(ppt)}}}
\newcommand{\linkToCode}[1]{\href{#1}{\blue{(code)}}}
\newcommand{\linkToWeb}[1]{\href{#1}{\blue{(web)}}}
\newcommand{\linkToVideo}[1]{\href{#1}{\blue{(video)}}}
\newcommand{\linkToMedia}[1]{\href{#1}{\blue{(media)}}}
\newcommand{\award}[1]{\xspace} 

\usepackage{capt-of}
\usepackage[caption=false]{subfig}
\usepackage{overpic}
\usepackage{multicol}

\newcommand{\JS}[1]{#1}
\newcommand{\JSTwo}[1]{#1}
\renewcommand{\LC}[1]{#1}

\renewcommand{\subject}{\text{s.t.}}

\newcommand{\nchoosek}[2]{\left( \substack{#1 \\ #2}\right)}
\newcommand{\TLS}{\scenario{TLS}}
\newcommand{\tls}{\scenario{TLS}}
\newcommand{\gm}{\scenario{GM}}
\newcommand{\GNC}{\scenario{GNC}}
\newcommand{\gnc}{\scenario{GNC}}
\newcommand{\BnB}{\scenario{BnB}}

\newcommand{\ransac}{\scenario{RANSAC}}

\newcommand{\vectorize}[1]{\ensuremath{\mathrm{vec}\left(#1 \right) }}

\newcommand{\hatvt}{\hat{\vt}}

\newcommand{\sym}[1]{\mathcal{S}^{#1}}

\newcommand{\measured}[1]{\vy_{#1}}
\newcommand{\params}[1]{\vtheta_{#1}}
\newcommand{\nrMeasurements}{N}

\newcommand{\proj}{\text{proj}}

\newcommand{\parentheses}[1]{\left( #1 \right)}

\renewcommand{\norm}[1]{\left\| #1 \right\|}

\newcommand{\omitted}[1]{}

\newcommand{\bmat}{\left[ \begin{array}}
\newcommand{\emat}{\end{array}\right]}

\newcommand{\teaserpp}{\scenario{TEASER++}}

\newcommand{\robin}{\scenario{ROBIN}}

\newrefformat{eq}{Eq.\,\ref{#1}}
\newcommand{\subMeas}[1]{\calM} %

\newcommand{\domainX}{\mathbb{X}}

\newcommand{\compatible}{compatible\xspace}

\newcommand{\compatibility}{compatibility\xspace}
\newcommand{\Compatibility}{Compatibility\xspace}

\newcommand{\abs}[1]{\left|#1\right|}

\newcommand{\apollo}{\scenario{ApolloScape}}
\newcommand{\apolloCar}{\scenario{ApolloCar3D}}
\newcommand{\pascal}{\scenario{PASCAL3D+}}

\newcommand{\keypointnet}{\scenario{KeypointNet}}

\newcommand{\vertex}{node\xspace}

\newcommand{\robinLong}{Reject Outliers Based on INvariants} %
\newcommand{\invariant}{invariant\xspace}
\newcommand{\invariants}{invariants\xspace}
\newcommand{\Invariant}{Invariant\xspace}

\newcommand{\barvy}{\bar{\vy}}

\newcommand{\barMB}{\bar{\MB}}
\newcommand{\vrhomo}{\tilde{\vr}}
\newcommand{\vF}{\boldsymbol{F}}

\newcommand{\measTwo}[1]{\vp_{\text{\scriptsize{2D}},#1}}   %
\newcommand{\bmeasTwo}[1]{\bar{\vp}_{\text{\scriptsize{2D}},#1}}   %
\newcommand{\tmeasTwo}[1]{\tilde{\vp}_{\text{\scriptsize{2D}},#1}}
\newcommand{\measThree}[1]{\vp_{\text{\scriptsize{3D}},#1}} %
\newcommand{\epsThree}[1]{\vepsilon_{\text{\scriptsize{3D}},#1}} %
\newcommand{\epsTwo}[1]{\vepsilon_{\text{\scriptsize{2D}},#1}} %
\newcommand{\basis}[2]{\vb^{#1}_{#2}}
\newcommand{\shapeParam}[1]{c_{#1}}
\newcommand{\shapeVector}{\vc}

\newcommand{\invFunThree}{\boldsymbol{F}_{3D}}
\newcommand{\invFunTwo}{\boldsymbol{F}_{2D}}
\newcommand{\invfunThree}{\boldsymbol{f}_{3D}}
\newcommand{\invfunTwo}{\boldsymbol{f}_{2D}}

\newcommand{\normal}[2]{\vn^{#1}_{#2}}

\DeclareMathOperator{\sgn}{sgn}

\newcommand{\vis}{visibility\xspace}
\newcommand{\Vis}{Visibility\xspace}
\newcommand{\covis}{covisibility\xspace}
\newcommand{\Covis}{Covisibility\xspace}
\newcommand{\region}{region\xspace}
\newcommand{\Region}{Region\xspace}
\newcommand{\regions}{regions\xspace}
\newcommand{\Regions}{Regions\xspace}

\newcommand{\CovisSet}{\mathcal{C}\xspace}

\newcommand\scalemath[2]{\scalebox{#1}{\mbox{\ensuremath{\displaystyle #2}}}}

\newcommand{\nrShapes}{K}

\newcommand{\inthrThree}{\beta_{3D}}
\newcommand{\inthrTwo}{\beta_{2D}}
\newcommand{\barcsqThree}{\inthrThree^2}
\newcommand{\barcsqTwo}{\inthrTwo^2}
\newcommand{\MHtl}{\bar{\MH}}

\newcommand{\PACETwo}{\scenario{PACE2D$^\star$}}
\newcommand{\PACETwoLZero}{\scenario{PACE2D$^\star \text{(OH)}$}}
\newcommand{\PACETwoLTwo}{\scenario{PACE2D$^\star$}}
\newcommand{\PACEThree}{\scenario{PACE3D$^\star$}}
\newcommand{\PACEThreeLong}{\emph{Pose and shApe estimation for  3D-3D Category-level pErception}}%
\newcommand{\PACETwoLong}{\emph{Pose and shApe estimation for  2D-3D Category-level pErception}}%

\newcommand{\PACErobust}{\scenario{PACE\#}}
\newcommand{\PACErobustTwo}{\scenario{PACE2D\#}}
\newcommand{\PACErobustTwoLTwo}{\scenario{PACE2D\#}}

\newcommand{\PACETwoGNC}{\scenario{GNC-PACE2D$^{\star}$}}
\newcommand{\PACETwoGNCLTwo}{\scenario{GNC-PACE2D$^{\star}$}}

\newcommand{\PACEThreeGNC}{\scenario{GNC-PACE3D$^{\star}$}}
\newcommand{\PACErobustThree}{\scenario{PACE3D\#}}

\newcommand{\cliquePACETwoLTwo}{\scenario{Clique-PACE2D$^\star$}}
\newcommand{\cliquePACEThree}{\scenario{Clique-PACE3D$^\star$}}
\newcommand{\ransacPACEThree}{\scenario{RANSAC-PACE3D$^\star$}}

\newcommand{\kpnp}{\scenario{KPnP}}

\newcommand{\kpnpRobust}{\scenario{KPnP-Robust}}

\newcommand{\nameRobust}{\scenario{PACE\#}}

\newcommand{\altern}{\scenario{Altern}}
\newcommand{\meanPnP}{\scenario{MS-PnP}}
\newcommand{\ransacMeanPnP}{\scenario{RANSAC-MS-PnP}}
\newcommand{\cliqueMeanPnP}{\scenario{Clique-MS-PnP}}
\newcommand{\kostas}{\scenario{Zhou}}
\newcommand{\kostasRobust}{\scenario{Zhou-Robust}}
\newcommand{\shapestar}{\scenario{Shape$^{\star}$}}

\newcommand{\irlsgm}{\scenario{IRLS-GM}}
\newcommand{\irlstls}{\scenario{IRLS-TLS}}

\newcommand{\aka}{\emph{a.k.a.}}

\newcommand{\toCheckTwo}[1]{{#1}}
\newcommand{\outPaceSharp}{$70-90\%$\xspace}
\newcommand{\outPaceSharpTwo}{$10\%$\xspace}

\newcommand{\outGNC}{$50-60\%$\xspace}
\newcommand{\outGNCTwo}{$10\%$\xspace}
\newcommand{\weight}{\omega}

\newcommand{\best}[1]{{\bf #1}}
\newcommand{\secondBest}[1]{{#1}}

\newcommand{\hypergraph}{hypergraph\xspace}
\newcommand{\Hypergraph}{Hypergraph\xspace}
\newcommand{\Hypergraphs}{Hypergraphs\xspace}
\newcommand{\hyperclique}{hyperclique\xspace}
\newcommand{\hypercliques}{hypercliques\xspace}
\newcommand{\Hyperclique}{Hyperclique\xspace}
\newcommand{\Hypercliques}{Hypercliques\xspace}

\newcommand{\tldMW}{\widetilde{\MW}}

\newcommand{\inprod}[2]{\left\langle #1, #2 \right\rangle}
\newcommand{\fstar}{f^\star}
\newcommand{\pstar}{p^\star}
\newcommand{\MXstar}{\MX^\star}
\newcommand{\vxxstar}{\vxx^\star}
\newcommand{\MRstar}{\MR^\star}
\newcommand{\vcstar}{\vc^\star}
\newcommand{\vtstar}{\vt^\star}
\newcommand{\hatMR}{\widehat{\MR}}
\newcommand{\hatvc}{\widehat{\vc}}
\newcommand{\hatp}{\widehat{p}}

\newcommand{\constraintc}{\ones\tran \vc = 1}

\newcommand{\urlPACE}{{{\tt \smaller\url{https://github.com/MIT-SPARK/PACE}}}}

\usepackage[bookmarks=true]{hyperref}
\usepackage{adjustbox}
\usepackage{nicefrac}

\usepackage{import}
\usepackage{xifthen}
\usepackage{pdfpages}
\usepackage{transparent}

\usepackage{rotating}

\usepackage{microtype}

\newcommand{\isExtended}[2]{#1}

\newcommand{\supplementary}[1]{App.~\ref{#1} in~\cite{Shi22arxiv-PACE}\xspace}

\renewcommand{\baselinestretch}{0.95}
\title{\LARGE \bf
 Optimal and Robust Category-level Perception: Object Pose \\  and Shape Estimation from 2D and 3D Semantic Keypoints
}

\author{Jingnan Shi, Heng Yang, Luca Carlone
\thanks{J.\,Shi (\texttt{jnshi@mit.edu}) and L.\,Carlone (\texttt{lcarlone@mit.edu}) are with the Laboratory for
  Information \& Decision Systems (LIDS), Massachusetts Institute of Technology, Cambridge, MA 02139, USA.
  H.\,Yang (\texttt{hankyang@seas.harvard.edu}) is with the School of Engineering and Applied Sciences at Harvard University, Cambridge, MA 02139, USA.
  Corresponding author: Jingnan Shi .}
\thanks{The authors would like to thank the editor and the anonymous reviewers for their constructive feedback, Rajat Talak for discussion about pose estimation networks, and Charleen Tan for labeling keypoints. This work was partially funded by ARL DCIST CRA W911NF-17-2-0181, ONR RAIDER N00014-18-1-2828, Lincoln Laboratory ``Resilient Perception in Degraded Environments’’, an Amazon Research Award, and Carlone’s NSF CAREER.
}
}

\begin{document}

\makeatletter
\let\@oldmaketitle\@maketitle%
\renewcommand{\@maketitle}{
\@oldmaketitle%
\vspace{2mm}
 \begin{minipage}{\textwidth}
    \begin{center}
        \begin{tabular}{cc}
        \hspace{-2mm}%
        \includegraphics[width=\columnwidth]{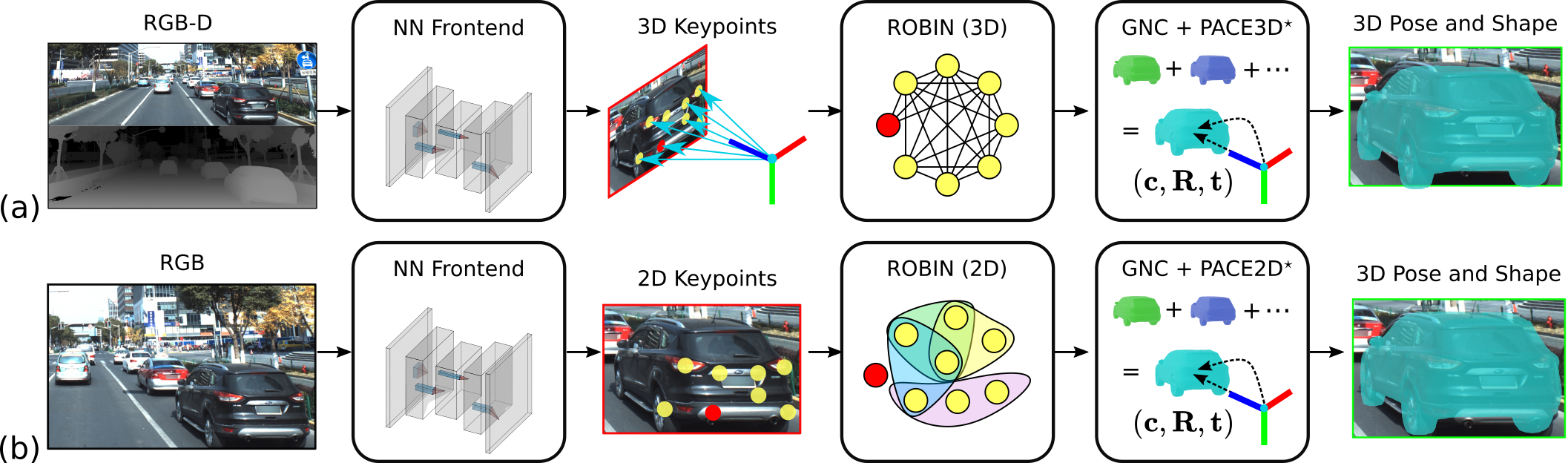}
        \end{tabular}
    \end{center}
    \vspace{-4mm}
  \end{minipage}
  \\

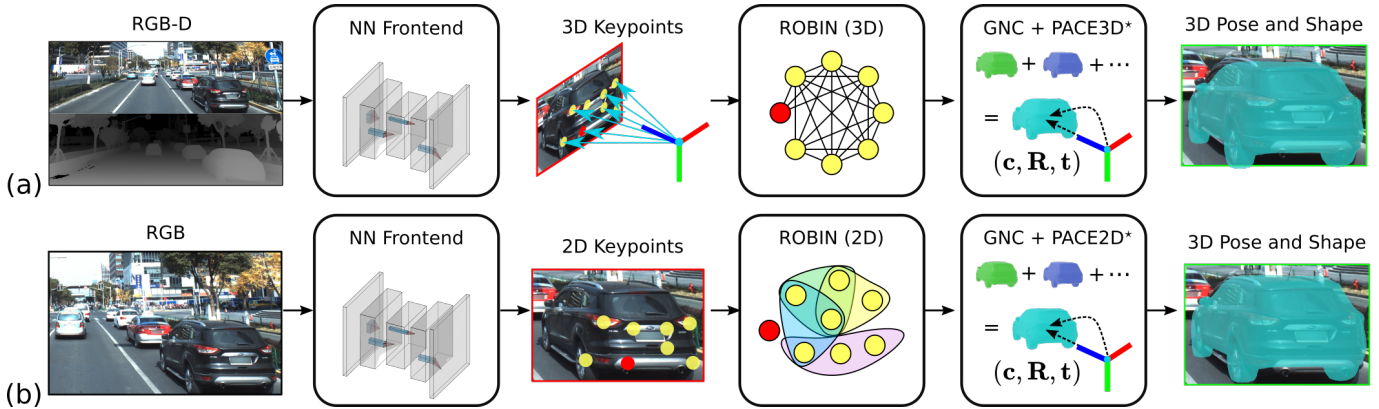
\captionof{figure}{ We develop algorithms for 
3D-3D and 2D-3D category-level perception, which estimate 3D pose and shape of an object from 3D and 2D sensor data, respectively. 
    (a) \PACErobustThree performs 3D-3D category-level perception and works on RGB-D inputs.
    We pass the sensor data through a neural network front-end to obtain 3D keypoint detections. We then use \robin with 3D-3D compatibility (hyper)graphs (Section~\ref{sec:robin-category-level}) to prune outliers (outliers are shown as red dots, while inliers are yellow). We finally pass the resulting measurements to our optimal solver (\PACEThree, Section~\ref{sec:optimalSolver-3d3d}) wrapped in a graduated non-convexity (\GNC) scheme (Section~\ref{sec:gnc-category-level}), which estimates the object pose and shape coefficients. %
   (b) \PACErobustTwo performs 2D-3D category-level perception and works on RGB inputs. The images are passed through a neural network to obtain 2D keypoints. We then use \robin with 2D-3D compatibility hypergraphs (Section~\ref{sec:robin-category-level}) to prune outliers.
    We finally pass the resulting measurements to our optimal solver (\PACETwo, Section~\ref{sec:optimalSolver-2d3d}) with \GNC to obtain a pose and shape estimate. 
    \label{fig:method-overview}
 \vspace{-8mm}}
}
\makeatother

\maketitle
\begin{tikzpicture}[overlay, remember picture]
\path (current page.north east) ++(-4.0,-0.4) node[below left] {
This paper has been accepted for publication at the IEEE Transactions on Robotics, 2023.
};
\end{tikzpicture}
\begin{tikzpicture}[overlay, remember picture]
\path (current page.north east) ++(-2.1,-0.8) node[below left] {
Please cite the paper as: J. Shi, H. Yang, and L. Carlone, ``Optimal and Robust Category-level Perception: Object Pose
};
\end{tikzpicture}
\begin{tikzpicture}[overlay, remember picture]
\path (current page.north east) ++(-4.0,-1.2) node[below left] {
and Shape Estimation from 2D and 3D Semantic Keypoints,'' \emph{IEEE Trans. Robotics}, 2023.
};
\end{tikzpicture}

\begin{abstract}

We consider a \emph{category-level perception} problem, where one is given
2D or 3D sensor data picturing an object of a given category (\eg a car), and has to reconstruct
 the 3D pose and shape of the object despite intra-class variability (\ie different car models have different shapes).
 We consider an \emph{active shape model}, where ---for an
object category--- we are given a library of potential CAD models describing objects in that category,
and
we adopt a standard formulation where pose and shape are estimated from 2D or 3D keypoints via
non-convex optimization.
Our first contribution is to develop \PACEThree and \PACETwo,
 the first \emph{certifiably optimal} solvers for pose and shape estimation using 3D and 2D keypoints, respectively.
Both solvers rely on the design of tight (\ie exact) semidefinite relaxations.
Our second contribution is to develop outlier-robust versions of both solvers, named \PACErobustThree and \PACErobustTwo.
Towards this goal, we propose \robin, a general graph-theoretic framework to prune outliers, which uses
\emph{compatibility hypergraphs} to model measurements' compatibility.
We show that in category-level perception problems these hypergraphs can be built from the winding orders of the keypoints (in 2D)
or their convex hulls (in 3D), and many outliers can be filtered out via maximum \hyperclique computation.
The last contribution is an extensive experimental evaluation.
Besides providing an ablation study on
 simulated datasets and on the \pascal dataset, we combine our solver
with a deep keypoint detector, and {show that \PACErobustThree  improves over the state of the art in vehicle pose estimation in the \apollo~datasets, and its runtime is compatible with practical applications.} We release our code at \urlPACE.
\end{abstract}

\begin{keywords}
\LC{category-level object perception, RGB-D perception, outliers-robust estimation, certifiable algorithms.}
\end{keywords}

\section{Introduction}
\label{sec:intro}
Robotics applications, from self-driving cars to domestic robotics, demand
robots to be able to identify and estimate the
pose and shape \LC{of objects in the environment.
The perception system of a self-driving car needs to}
estimate the poses of other vehicles, identify traffic lights and signs, and
detect pedestrians. Similarly, domestic applications require estimating the location and shape of objects to support effective interaction and manipulation~\cite{Manuelli19-kpam,Gao21-kpam2,Pavlakos17icra-semanticKeypoints}.
Object pose estimation is made harder by the large intra-class shape variability of common objects:
for instance, the shape of a car largely varies depending on the model (\eg take a van versus a sedan).

Despite the fast-paced progress, reliable 3D object pose estimation remains a challenge,
as witnessed by recent self-driving car accidents caused by
misdetections~\cite{McCausland-UberCrash}. %
Deep learning has been making great strides in enabling robots to detect objects;
popular tools such as YOLO~\cite{Redmon16cvpr} and Mask-RCNN~\cite{He17iccv-maskRCNN} have made object detection possible on commodity hardware and with reasonable performance for in-distribution test data.
However, detections are typically at the level of \mbox{categories (\eg car) rather than}
at the level of instances (\eg a specific car model).
In turn, category-level perception renders the use of standard tools for pose estimation (from point cloud registration~\cite{Yang20tro-teaser,Horn87josa,Bustos18pami-GORE} to the Perspective-n-Point problem~\cite{Zheng2013ICCV-revisitPnP,Kneip2014ECCV-UPnP,Schweighofer2008bmvc-SOSforPnP}) ineffective, since they rely on the knowledge of the shape of the  object. %

These limitations have triggered robotics and computer vision
research on category-level 3D object pose and shape estimation (see Section~\ref{sec:relatedWork} for an in-depth review).
Traditional methods include the popular
\emph{active shape model}~\cite{Cootes95cviu,Zhou15cvpr,Yang20cvpr-shapeStar}, where
one attempts to estimate the pose and shape of an object given a large database of 3D CAD models.
Despite its popularity (\eg the model is also used in human shape estimation and face detection~\cite{Zhou15cvpr}), 
 estimation with active shape models leads to a non-convex optimization
problem and local solvers get stuck in poor solutions, and are sensitive to outliers~\cite{Zhou15cvpr,Yang20cvpr-shapeStar}.
More recently, research effort has been devoted to end-to-end learning-based 3D pose estimation
with encouraging results in human pose estimation~\cite{Kolotouros19cvpr-shapeRec} and vehicle pose estimation~\cite{Chabot17-deepMANTA,Ke20-gsnet,Lopez19-vehicle,Kundu18-3dRCNN,Suwajanakorn18-latent3Dkeypoints};
these approaches still require a large amount of 3D labeled data, \mbox{which is time-consuming (or expensive) to obtain in the wild.}

\myParagraph{Contribution}
We address the shortcomings of existing approaches for pose and shape estimation based on the active shape model and
propose the first approaches that can compute optimal estimates and are robust to outliers.
We consider a \emph{category-level perception} problem, where one is given
keypoint detections of an object belonging to a given category (\eg detections of the wheels, rear-view mirrors, and other interest points of a car), and has to reconstruct
the pose and shape of the object.
 We assume the availability of
a library of CAD models of objects in that category;
such a library is typically available, since CAD models are extensively used in the design, manufacturing, and simulation of 3D objects.
\JS{
Note that the definition of category-level perception is ambiguous in the literature:
we follow the setup in~\cite{Sahin18eccv-categoryFromDepth, Manuelli19-kpam, Wang19-normalizedCoordinate, Li20-categoryArticulated, Wang21iros-categoryObjectCascade, Gao21-kpam2}, while some authors use the same name for problems
where the objects seen at test-time are different from the ones seen during training~\cite{Li21nips-LeveragingSE}, while being in the same category.}

Our first contribution is to develop the first \emph{certifiably optimal} solvers for pose and shape estimation using 3D and 2D keypoints. %
In the 3D case, we show that ---despite the non-convexity of the problem--- rotation estimation can be decoupled from the estimation of object translation and shape, and
we demonstrate that (i) the optimal object rotation can be computed via a tight (small-size) semidefinite relaxation,
and (ii) the translation and shape parameters can be computed in closed form given the rotation.
We call the resulting solver \PACEThree (\PACEThreeLong).
In the 2D case, we formulate pose and shape estimation using an algebraic point-to-line cost,
and leverage \emph{Lasserre's hierarchy of semidefinite relaxations}~\cite{Lasserre01siopt-LasserreHierarchy} to solve the
problem to certifiable global optimality.
We call the resulting solver \PACETwo (\PACETwoLong).
Contrarily to \PACEThree, \PACETwo leads to semidefinite {relaxations whose size increases with the 
number of CAD models in the active shape model.}

Our second contribution is to develop an outlier rejection scheme applicable to both \PACEThree and \PACETwo.
 Towards this goal, we introduce a general framework for graph-theoretic outlier pruning, named \robin, which 
 generalizes our previous work~\cite{Yang20tro-teaser, Shi21icra-robin} to use hypergraphs.
 \robin models compatibility between subset of measurements 
using a \emph{compatibility hypergraph}. 
We show that the compatibility hypergraph can be efficiently constructed by inspecting the 
winding orders of the keypoints in 2D, or the convex hulls of the keypoints in 3D.
We then prove that all the inliers are contained in a single \hyperclique of 
the compatibility hypergraph and can be typically found within the maximum \hyperclique.
\robin is able to remove a large fraction of outliers. %
The resulting measurements are then passed to our optimal solvers (\PACEThree and \PACETwo), that we also wrap in 
a standard \emph{graduated non-convexity}~\cite{Yang20ral-GNC} scheme 
to mitigate the impact of outliers that survived \robin.
The resulting outlier-robust approaches are named \PACErobustThree and \PACErobustTwo 
and are illustrated in Fig.~\ref{fig:method-overview}.

Our last contribution is an extensive experimental evaluation in both synthetic experiments and real datasets~\cite{Wang19pami-apolloscape}. %
We provide an ablation study on
simulated datasets and on the \pascal dataset, and show that
\JS{(i) \PACEThree is more accurate than state-of-the-art iterative solvers,
(ii) \PACETwo is more accurate than baseline local solvers and convex relaxations based on the weak perspective projection model~\cite{Zhou15cvpr,Yang20cvpr-shapeStar},
(iii) \PACErobustThree~\toCheckTwo{dominates} other robust solvers and is robust to \outPaceSharp outliers, 
and (iv) \PACErobustTwo is robust to \outPaceSharpTwo outliers.} %
Finally, we integrate our solvers in a realistic system ---including a deep keypoint detector--- and apply it to vehicle pose and shape estimation in the \apollo~\cite{Wang19pami-apolloscape} driving datasets. 
{While \PACErobustTwo is currently slow and suffers from the low quality of the deep keypoint detections, 
\PACErobustThree largely outperforms the state of the art and a non-optimized implementation runs in a fraction of a second. We also show that \robin is even able to detect mislabeled keypoints 
used to train the keypoint detector in the \apollo dataset~\cite{Wang19pami-apolloscape}.}

\myParagraph{Novelty with Respect to~\cite{Shi21icra-robin, Shi21rss-pace}}
In our previous works,
we introduced \robin~\cite{Shi21icra-robin}, a graph theoretic outlier rejection framework,
and two solvers~\cite{Shi21rss-pace} (\PACEThree and \PACErobustThree) for pose and shape estimation from 3D keypoints.
The present paper %
(i) \LC{allows \robin~to handle more general {compatibility tests},} and
(ii) extends the concept of inlier selection to maximum \hypercliques on \emph{compatibility hypergraphs}.
This paper also develops 
\PACETwo and \PACErobustTwo, which estimate pose and shape from only 2D (instead of 3D) keypoints.
In addition, we report a more comprehensive experimental evaluation in simulation and on the \apollo dataset.

\myParagraph{Paper Structure}
Section~\ref{sec:problemStatement} formulates the category-level perception problem.
Section~\ref{sec:approachOverview} provides a brief overview of the proposed approaches (also summarized in Fig.~\ref{fig:method-overview}). 
Sections~\ref{sec:robin-formulation} and~\ref{sec:robin-category-level} present our graph-theoretic outlier 
pruning (\robin) and its application to category-level perception. 
Section~\ref{sec:optimalSolvers-category-level} introduces our certifiably optimal solvers for category-level perception 
and Section~\ref{sec:gnc-category-level} recalls how to wrap the solvers in a graduated non-convexity scheme.
Section~\ref{sec:experiments} discusses experimental results.
Section~\ref{sec:relatedWork} provides an in-depth review of related work.
Section~\ref{sec:conclusion} concludes the paper. %

\newcommand{\mpwthree}{5.9cm}
\newcommand{\mpwthreetwo}{11.4cm}

\begin{figure*}[ht!]
    \centering%
    \includegraphics[width=\textwidth]{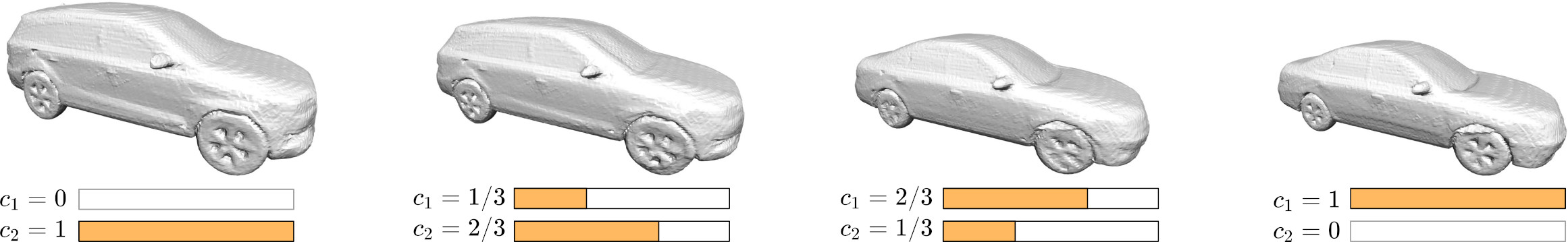}
	\caption{Illustration of a simple active shape model with two CAD models: a Hyundai Sonata (with shape coefficient $c_{1}$) and an Audi Q7 (with shape coefficient $c_{2}$).
      From left to right, we show four convex combinations of the CAD models, with $c_{1}$ increasing and $c_{2}$ decreasing while keeping $c_{1} + c_{2} = 1$.
      The shape variations between the four combinations show the expressivity of the active shape model and its capability to interpolate between shapes.
    \label{fig:active-shape-demo}}
    \vspace{-5mm}
\end{figure*}

\section{Problem Statement: \\ 3D-3D and 2D-3D Category-Level Perception}
\label{sec:problemStatement}

In this section, we formulate the 3D-3D and 2D-3D category-level perception problems.
The goal is to compute the 3D pose and shape of an object, given 3D or 2D sensor data.
We focus on a \emph{multi-stage} setup (\cf Section~\ref{sec:relatedWork}) where a front-end is used to extract 2D or 3D semantic keypoints from the sensor data,
which are then used by a back-end solver to estimate the object's 3D pose $(\MR,\vt)$ and shape,
where $\MR \in \SOthree$ and $\vt \in \Real{3}$ are the 3D rotation and translation, respectively.
The front-end is typically implemented using standard deep networks~\cite{Pavlakos17icra-semanticKeypoints,Schmeckpeper22arxiv-singleRGBpose}, while our goal here is to design more accurate and robust back-ends.

In the following, we first introduce a standard parametrization of the object shape (Section~\ref{sec:activeShape}), 
and then formalize the 3D-3D category-level perception problem (Section~\ref{sec:3D3Dformulation}) 
and its 2D-3D counterpart (Section~\ref{sec:2D3Dformulation}).

\subsection{Active Shape Model}\label{sec:activeShape}

We assume the object shape to be partially specified: we are given a library
of 3D CAD models $\calB^{k}$, $k=1,\ldots,\nrShapes$, and assume that the unknown object shape $\calS$
(modeled as a collection of 3D points) can be written
as a combination of the points on the given CAD models.
More formally,
each point $\vs_{i}$ of the
shape $\calS$ can be written as:
\bea
\label{eq:activeShape}
\vs_i = \textstyle\sum_{k=1}^{\nrShapes} \shapeParam{k} \basis{k}{i}
\eea
where $\basis{k}{i}$ is a given point on the surface of the CAD model $\calB^{k}$;
the \emph{shape parameters} $\shapeVector \triangleq [\shapeParam{1} \ldots \shapeParam{\nrShapes}]\tran$ are unknown, and
the entries of $\shapeVector$ are assumed to be non-negative and sum up to 1,
\ie $\vc$ belongs to the $\nrShapes$-dimensional probability simplex $\Delta_K := \{ \vc \in \Real{\nrShapes} \mid \vc \geq \zero, \sum_{k=1}^{\nrShapes} c_k = 1 \}$.
For instance, if ---upon estimation---
the vector $\shapeVector$ has the $l$-th entry equal to 1 and the remaining entries equal to zero in~\eqref{eq:activeShape}, then the 
estimated shape matches the $l$-th CAD model in the library; therefore, the estimation of the 
shape parameters $\shapeVector$ can be understood as a fine-grained classification of the object among the instances in the library. However, the model is even more expressive, since it allows the object shape to be a convex 
combination of CAD models, which enables the active shape model~\eqref{eq:activeShape} to interpolate between different shapes in the library; see Fig.~\ref{fig:active-shape-demo}. %

\subsection{3D-3D Category-Level Perception}\label{sec:3D3Dformulation}

In the 3D-3D category-level perception problem, the goal is to
estimate an object's pose and shape, given a set of $N$ 3D keypoint detections, typically obtained
using learning-based keypoint detectors.
Such detectors are
trained to detect semantic features of the 3D object (\eg wheels of a car),
and can be
applied to RGB-D or RGB+Lidar data (\eg~\cite{Pavlakos17icra-semanticKeypoints}).

 We assume each \textbf{3D measurement} $\measThree{i}$ ($i=1, \ldots, N$) is a noisy measurement of a keypoint $\vs_{i}$ 
 of our target object, in the coordinate frame of the sensor.
More formally, each $\measThree{i}$ is described by the following generative model:
\begin{equation}
\label{eq:generativeModel-3d}
\measThree{i}=\MR \textstyle\sum_{k=1}^{\nrShapes}\shapeParam{k} \basis{k}{i} + \vt + \epsThree{i} \qquad i=1, \ldots, N
\end{equation}
where the measurement $\measThree{i}$ pictures a 3D point on the object (written as a linear combination $\sum_{k=1}^{\nrShapes}\shapeParam{k} \basis{k}{i}$ of the shapes in the library as in~\eqref{eq:activeShape}), after the point is rotated and translated
according to the 3D pose $(\MR,\vt)$ of the object, and where $\epsThree{i}$ represents measurement noise.
Intuitively, each measurement corresponds to a noisy measurement of a semantic feature of the object
(\eg wheel center or rear-view mirrors of a car) and each $\basis{k}{i}$ corresponds to the corresponding feature (\eg wheel or mirror) location for a specific CAD model.

\begin{problem}[Robust 3D-3D Category-Level Perception]\label{prob:3d3d-statement}
Compute the 3D pose $(\MR,\vt)$ and shape $(\shapeVector)$ of an object given $N$ 3D keypoint measurements in the
form~\eqref{eq:generativeModel-3d}, possibly corrupted by outliers, \ie measurements with large error $\epsThree{i}$.
\end{problem}

\subsection{2D-3D Category-Level Perception}\label{sec:2D3Dformulation}
In the 2D-3D category-level perception problem,
we want to estimate an object's 3D pose and shape, given only
2D projections of keypoints.
In this case,
we describe each \textbf{2D measurement} using the following generative model:
\begin{equation}
\label{eq:generativeModel-2d}
{\measTwo{i}} = \pi\left( \MR \textstyle\sum_{k=1}^{\nrShapes}\shapeParam{k} \basis{k}{i} + \vt \right) + \epsTwo{i} \qquad i=1, \ldots, N
\end{equation}
where 
$\measTwo{i}$ represents a 2D (pixel) measurement,
 $\pi(\cdot)$ is the canonical perspective projection,\footnote{For a 3D vector $\vp \!=\!\! \matTwo{\!\!p_x\!\! \\ \!\!p_y\!\! \\ \!\!p_z\!\!}$, 
the canonical projection is $\pi(\vp) \!=\!\! \matTwo{\!\! \nicefrac{p_x}{p_z}\!\! \\\!\! \nicefrac{p_y}{p_z} \!\!}$.} 
and $\epsTwo{i}$ is the measurement noise. %
Intuitively, the measurements in~\eqref{eq:generativeModel-2d} correspond to pixel projections 
of the object keypoints onto an image.  

\begin{problem}[Robust 2D-3D Category-Level Perception]\label{prob:2d3d-statement}
Compute the 3D pose $(\MR,\vt)$ and shape $(\shapeVector)$ of an object given $N$ 2D keypoint measurements in the
form~\eqref{eq:generativeModel-2d}, possibly corrupted by outliers, \ie measurements with large error $\epsTwo{i}$.
\end{problem}

\section{Overview of \PACErobust: Pose and Shape Estimation for Robust Category-level Perception}
\label{sec:approachOverview}

  Our approach, named \PACErobust, 
  is summarized in Fig.~\ref{fig:method-overview}, for both the 3D-3D case 
  (\PACErobustThree, Fig.~\ref{fig:method-overview}(a)) and the 2D-3D case (\PACErobustTwo, Fig.~\ref{fig:method-overview}(b)).
  We assume access to a perception front-end that detects semantic keypoints given sensor data.
  Our work forms the back-end, and consists of two stages.
In the first stage,
we employ a graph-theoretic framework, named \robin, to pre-process the keypoints and
 prune gross outliers without explicitly solving the underlying estimation problem.
We then pass the filtered measurements to the second stage,
where an optimal solver (wrapped in a graduated non-convexity scheme) computes a pose and shape estimate.\footnote{We use a two-stage approach (\ie \robin followed by graduated non-convexity) for two reasons.
While our solver with graduated non-convexity is robust against \outGNCTwo of outliers in 2D-3D problems and \outGNC of outliers in 3D-3D problems,
our goal is to further increase its robustness. Indeed we show that \robin
 boosts robustness to \outPaceSharpTwo in 2D-3D problems and \outPaceSharp in 3D-3D problems.
In addition, since the first stage prunes outliers independently from the solver,
\robin can be used in a plug-and-play manner with other existing solvers to boost their robustness.}

In the following we introduce Stage 1 by first presenting \robin, a general framework for outlier pruning (Section~\ref{sec:robin-formulation})  
and then discussing its application to category-level perception (Section~\ref{sec:robin-category-level}); 
we then discuss Stage 2
by presenting our optimal solvers for 3D-3D (\PACEThree, Section~\ref{sec:optimalSolver-3d3d}) and 2D-3D (\PACETwo, Section~\ref{sec:optimalSolver-2d3d}) category-level perception, and a brief review of graduated non-convexity~\cite{Yang20ral-GNC} (Section~\ref{sec:gnc-category-level}).

\section{Stage 1: Graph-theoretic Outlier Pruning With \robin}
\label{sec:robin-formulation}

This section develops a general framework to prune gross outliers from a set of measurements without
explicitly computing an estimate for the variables of interest.
In particular, we introduce the notion of \emph{$n$-invariant} to check if a  subset of measurements contains outliers. 
 We then use these checks to construct 
\emph{\compatibility hypergraphs} that describe mutually compatible measurements,
 and show how to reject outliers by computing maximum \hypercliques of these graphs.
Combining these insights, we obtain
\robin (\emph{\robinLong}), our graph-theoretic algorithm for pruning outliers.

This section presents our framework in full generality %
and then we tailor it to category-level perception in Section~\ref{sec:robin-category-level}.
In particular, here we consider a more general measurement model 
that relates measurements $\measured{i}$ to 
the to-be-estimated variable  $\vxx \in \domainX$ (where $\domainX$ is the domain of $\vxx$, \eg the set of 3D poses) and a given model $\vtheta_{i}$ (\eg our CAD models):
\beq
\label{eq:robin-measurements}
\measured{i} = \vh(\vxx, \params{i}, \vepsilon_i), \qquad i=1,\ldots,\nrMeasurements
\eeq
where $\vepsilon_i$ denotes the measurement noise. 
Clearly, eqs.~\eqref{eq:generativeModel-3d} and~\eqref{eq:generativeModel-2d} can be understood as 
special instances of~\eqref{eq:robin-measurements}, where $\vxx$ includes the unknown pose and shape of the 
object, and the (given) model $\vtheta$ corresponds to the CAD models.

\subsection{From Measurements to Invariants}

This section formalizes the concepts of $n$-\invariant and generalized $n$-\invariant, which 
are the building blocks of our outlier pruning framework.
  The main motivation is to use invariance to establish checks on the (inlier) measurements that hold true
regardless of the state under estimation; we are later going to use these checks to detect outliers.

Let us consider the measurements in eq.~\eqref{eq:robin-measurements} and
denote the indices of the measurements as $\calY \doteq \{1,\ldots,N\}$.
For a given integer $n \leq N$, let $\subMeas{n} \subset \calY$ be a subset of $n$ indices in $\calY$, %
and denote with $\subMeas{n}_j$ the $j$-th element of this subset (with $j=1,\ldots,n)$.
Then, we use the following notation:
\beq
\measured{\subMeas{n}} = \matTwo{ \measured{\subMeas{n}_1} \\ \measured{\subMeas{n}_2} \\ \vdots \\\measured{\subMeas{n}_n}   },
\;\;
\params{\subMeas{n}} = \matTwo{ \params{\subMeas{n}_1} \\ \params{\subMeas{n}_2} \\ \vdots \\\params{\subMeas{n}_n}   },
\;\;
\vepsilon_{\subMeas{n}} = \matTwo{ \vepsilon_{\subMeas{n}_1} \\ \vepsilon_{\subMeas{n}_2} \\ \vdots \\ \vepsilon_{\subMeas{n}_n}   }
\eeq
which is simply stacking together measurements $\vy_i$,
parameters $\vtheta_i$, and noise $\vepsilon_i$ for the subset of measurements $i \in \subMeas{n}$.

Let us now formalize the notion of noiseless invariance.
\begin{definition}[Noiseless $n$-\Invariant]\label{def:n-invariant}
Consider the generative model~\eqref{eq:robin-measurements} and assume there is no measurement noise (\ie 
$\measured{i} = \vh(\vxx, \params{i})$).
Then a function $\vf$ is called a noiseless
\emph{$n$-\invariant} if {for any arbitrary $\subMeas{n} \!\subset\!\calY$ of size $n$}, 
the following relation holds, regardless of the choice of $\vxx$:
  \begin{equation}
\label{eq:n-invariant}
\vf( \measured{\subMeas{n}}  ) = \vf(\params{\subMeas{n}})
  \end{equation}
\end{definition}

Intuitively, an $n$-\invariant function $\vf$ takes a subset of models $\params{\subMeas{n}}$ and computes
a quantity $\vf(\params{\subMeas{n}})$ that remains constant
when the models are transformed by the measurement model $\vh$ to generate the measurements $\measured{\subMeas{n}}$.

\myParagraph{Example} To concretely understand invariance,
consider a simpler instance of~\eqref{eq:generativeModel-3d} where there is a single shape:
\begin{equation}
  \label{eq:p-reg-gen-model}
  \measured{i}=\MR \textstyle \params{i} + \vt + \vepsilon_{i} \qquad i=1, \ldots, N
\end{equation}
Eq.~\eqref{eq:p-reg-gen-model} is also known as the
\emph{point cloud registration} problem~\cite{Arun87pami,Horn87josa} and consists in finding a rigid body transformation $(\MR,\vt)$ (where $\MR \in \SOthree$ and $\vt \in \Real{3}$) that
aligns two sets of 3D points $\measured{i} \in \Real{3}$ and $\params{i} \in \Real{3}$, with $i=1,\ldots,N$.
The corresponding measurement model can again be seen to be an instance of the general model~\eqref{eq:robin-measurements}.

In the noiseless case $(\vepsilon_{i} = \zero)$, it follows from~\eqref{eq:p-reg-gen-model} that
\begin{align}\label{eq:p-reg-noiseless-invariant}
\| \measured{j} - \measured{i}\| &= \| (\MR \textstyle \params{j} + \vt) - (\MR \textstyle \params{i} + \vt)\| = \\ 
 &= \| \MR (\params{j} - \params{i}) \| = \| \params{j} - \params{i} \| \nonumber
\end{align}
for any pair of measurements $i,j$, where $\|\cdot\|$ is the 2-norm and we used the fact that  
the 2-norm is invariant under rotation.
This is an example of noiseless 2-invariant, $\vf(\measured{i}, \measured{j}) \doteq \| \measured{j} - \measured{i} \| = \vf(\params{i}, \params{j}) $, which relates measurements and model regardless of the choice of $\vxx$,
and formalizes the intuition that 
the distance between pairs of points in a point cloud is invariant
under rigid transformations.

While Definition~\ref{def:n-invariant} provides a general definition of noiseless invariance,
 practical problems always involve noise.
Therefore we need to generalize the notion of invariance as follows.

\begin{definition}[Generalized $n$-\Invariant]\label{def:generalized-n-invariant}
Given eq.~\eqref{eq:robin-measurements} and assuming the measurement noise is bounded $\|\vepsilon_{i}\| \leq \beta$ ($i=1,\ldots,N$), a pair of functions $(\vf, \vF)$ is called a
\emph{generalized $n$-\invariant} if {for any arbitrary $\subMeas{n} \!\subset\!\calY$ of size $n$},
the following relation holds, regardless of the choice of $\vxx$:
  \begin{equation}
\label{eq:generalized-n-invariant}
\vf( \measured{\subMeas{n}}) \in \vF(\vtheta_{\subMeas{n}}, \beta)
  \end{equation}
  where $\vF(\vtheta_{\subMeas{n}}, \beta)$ is a set-valued function (independent on $\vxx$).
\end{definition}

Intuitively, because of the noise, now the measurements can produce a number of different invariants 
$\vf( \measured{\subMeas{n}})$, but we can still define a set $\vF(\vtheta_{\subMeas{n}}, \beta)$ that contains all potential invariants produced by the measurements. An example is in order.

\myParagraph{Example}
Going back to the example of point cloud registration in eq.~\eqref{eq:p-reg-gen-model},
with $\|\vepsilon_{i}\| \leq \beta$ and $\|\vepsilon_{j}\| \leq \beta$, we have
\begin{align}
  \| \measured{j}\!-\!\measured{i} \| &=  \| \MR (\params{j}  - \params{i}) + \vepsilon_j - \vepsilon_i\|  
\end{align}
If we apply the triangle inequality to the right-hand-side of~\eqref{eq:p-reg-noisy-invariant} we obtain:

\vspace{-8mm}
\begin{align}
 \label{eq:test-p-ref2}
 \| \params{j}  - \params{i}\| - \| \vepsilon_j - \vepsilon_i\|  &\leq&  
 \overbrace{ 
 \| \MR (\params{j}  - \params{i}) + \vepsilon_j - \vepsilon_i)\| }^{ = \| \measured{j}\!-\!\measured{i} \| \text{ per eq.}~\eqref{eq:p-reg-noisy-invariant}}
 \\  
 &\leq&  \| \params{j}  - \params{i}\| + \| \vepsilon_j - \vepsilon_i\| \nonumber
\end{align}
Now $\|\vepsilon_i\|\leq\!\beta$ and $\|\vepsilon_j\|\leq\!\beta$ imply
$\| \vepsilon_j - \vepsilon_i \| \leq 2\beta$. Substituting in~\eqref{eq:test-p-ref2} we obtain:
 \beq
 \label{eq:test-p-ref3}
 \| \params{j} - \params{i}\| - 2\beta  \leq  \| \measured{j}\!-\!\measured{i} \|  \leq  \| \params{j}  - \params{i}\| + 2\beta
 \eeq
 or, in other words:
 \beq
 \label{eq:p-reg-noisy-invariant}
\overbrace{\| \measured{j}\!-\!\measured{i} \|}^{\vf( \measured{\subMeas{n}})} 
\in  
\overbrace{
\Big[ \| \params{j}  - \params{i}\| - 2\beta, \| \params{j}  - \params{i}\| + 2\beta \Big]
}^{ \vF(\vtheta_{\subMeas{n}}, \beta) }
 \eeq
 which corresponds to our definition of generalized $n$-invariant (with $n=2$).
 Geometrically, eq.~\eqref{eq:p-reg-noisy-invariant} states the distances between pairs of measured points 
 ($\measured{j}$ and $\measured{i}$) must match
 corresponding distances between points in our model ($\params{j}$ and $\params{i}$) up to noise. 
Note that when $\beta = 0$ (noiseless case), eq.~\eqref{eq:p-reg-noisy-invariant} reduces back to eq.~\eqref{eq:p-reg-noiseless-invariant}, as the set 
$\vF(\vtheta_{\subMeas{n}}, \beta)  \doteq \big[ \| \params{j}  - \params{i}\| - 2\beta, \| \params{j}  - \params{i}\| + 2\beta \big]$ reduced to the singleton $\| \params{j}  - \params{i}\|$.

Definition~\ref{def:generalized-n-invariant} generalizes Definition~\ref{def:n-invariant} to account for the presence of noise. Moreover, as we will see in Section~\ref{sec:robin-category-level}, we can develop generalized $n$-invariants for category-level perception even when noiseless invariants are difficult to find.
In other words, while it may be difficult to pinpoint a noiseless invariant function, it is often easier to 
find a set of values that a suitable function $\vf( \measured{\subMeas{n}})$ must belong to. 
In the following sections, unless otherwise specified, we use the term $n$-\invariants to refer to 
generalized $n$-\invariants.

\subsection{From Invariants to \Compatibility Tests for Outlier Pruning}
\label{sec:compatibilityTests}

While the previous section developed invariants without distinguishing inliers from outliers, this section shows that
the resulting invariants can be directly used to check if a subset of measurements contains an outlier.
Towards this goal, we formalize the notion of inlier and outlier. %

\begin{definition}[Inliers and Outliers]\label{def:inlier-outlier}
Given measurements~\eqref{eq:robin-measurements} and a threshold $\beta {\geq} 0$,
a measurement $i$ is an \emph{inlier} if the corresponding noise satisfies $\|\vepsilon_i\| \leq \!\beta$
and is an \emph{outlier} otherwise.
\end{definition}

The notion of invariants introduced in the previous section allowed us to obtain relations that
depend on the measurements and a noise bound, 
but are independent on $\vxx$, see eq.~\eqref{eq:generalized-n-invariant}.
Therefore, we can directly use these relations to check if a subset of $n$ measurements contains outliers:
if eq.~\eqref{eq:generalized-n-invariant} 
 is not satisfied by a subset of measurements $\measured{\subMeas{n}}$
then the corresponding subset of measurements \emph{must} contain an outlier.
We call the corresponding check a \emph{\compatibility test}.
In the following, we provide an example of \compatibility test for point cloud registration.
The reader can find more examples of compatibility tests for other applications in~\cite{Shi21icra-robin}.

\myParagraph{Example of \Compatibility Test}
For our point cloud registration example,
 eq.~\eqref{eq:p-reg-noisy-invariant} states that 
 any pair of measurements $\measured{i}$ and $\measured{j}$
  with noise $\|\vepsilon_i\|\leq\!\beta$ and $\|\vepsilon_j\|\leq\!\beta$
 must satify:
  \beq
\| \measured{j}\!-\!\measured{i} \|
\in  
\Big[ \| \params{j}  - \params{i}\| - 2\beta, \| \params{j}  - \params{i}\| + 2\beta \Big]
 \eeq
 If the relation is satisfied, we say that $\measured{i}$ and $\measured{j}$ are \compatible with each other 
 (\ie they can potentially be both inliers); however, if the relation is \emph{not} satisfied, then one of the measurements \emph{must} be an outlier.
Generalizing this example, we obtain the following general definition of \emph{\compatibility test}.

\begin{definition}[\Compatibility Test]
  \label{def:comp-test}
Given a subset of $n$ measurements and the corresponding $n$-\invariant, a \compatibility test
 is a binary condition (computed using the invariant), such that if the condition fails, then the set of measurements
 \emph{must} contain at least {one} outlier.
\end{definition}

Note that we require the test to be \emph{sound} (\ie it does not detect outliers when testing a set of inliers), but
may not be \emph{complete} (\ie the test might pass even in the presence of {outliers}).
This property is important since our goal is to prune as many outliers as we can, while preserving the inliers.
Also note that the test detects if the set contains {outliers}, but does not provide information on \emph{which} measurements are outliers. We are going to fill in this gap below.

\subsection{From \Compatibility Tests to \Compatibility \Hypergraph}
\label{sec:inv-graph}

For a problem with an $n$-\invariant,
we describe the results of the compatibility tests for all subsets of $n$ measurements using a
\emph{\compatibility hypergraph}.

\begin{definition}[Compatibility \Hypergraph]\label{def:comp-hypergraph}
  Given a compatibility test with $n$ measurements, define the compatibility \hypergraph $\calG(\calV,\calE)$ as
  an $n$-uniform undirected hypergraph,\footnote{In an $n$-uniform hypergraph, each hyperedge involves exactly $n$ nodes.} where
each \vertex $v$ in the \vertex set $\calV$ is associated to a measurement in~\eqref{eq:robin-measurements}, and
an hyperedge $e$ (connecting a subset of $n$ measurements) %
belongs to 
the edge set $\calE$ if and only if its subset of measurements passes the compatibility test.
\end{definition}

Note that in the case where $n=2$,
the above definition reduces to a regular undirected graph.
Building the \compatibility graph requires looping over all subsets of $n$ measurements and, whenever the subset
passes the \compatibility test, adding a hyperedge between the corresponding $n$ nodes in the graph.
Note that these checks are very fast and easy to parallelize since they only involve checking boolean conditions (\eg~\eqref{eq:test-p-ref3})  without computing an estimate (as opposed to \ransac).

\myParagraph{Inlier Structures in \Compatibility \Hypergraphs} 
Here we show that the inliers in the set of measurements~\eqref{eq:robin-measurements} are contained in a 
single hyperclique of the \compatibility \hypergraph. 
Let us start with some definitions.

\begin{definition}[\Hypercliques and Maximum \Hyperclique in Hypergraphs]
A \hyperclique of an $n$-uniform hypergraph $\calG$ is a set of vertices such that any subsets of $n$  vertices is connected by an hyperedge in $\calG$.
The \emph{maximum} \hyperclique of $\calG$ is the \hyperclique with the largest number of vertices.
\end{definition}

Again, in the case where $n=2$, the above definition reduces to the usual clique and maximum clique definition.
Given a \compatibility \hypergraph $\calG$, the following result relates the set of inliers in~\eqref{eq:robin-measurements}
with \hypercliques in $\calG$ ({proof in \supplementary{sec:app-proofGraph}}).

\begin{theorem}[Inliers and \Hypercliques]\label{thm:inliers-form-clique}
Assume~we are given measurements~\eqref{eq:robin-measurements} (with inlier noise bound $\beta$) and the corresponding $n$-\invariants;
call $\calG$ the corresponding \compatibility hypergraph.
  Then, assuming there are at least $n$ inliers, all inliers belong to a single \hyperclique in $\calG$. %
\end{theorem}

Theorem~\ref{thm:inliers-form-clique} implies that we can look for inliers by computing \hypercliques in the \compatibility graph.
Since we expect to have more \compatible inliers than outliers, here we propose to compute the maximum \hyperclique;
this approach is shown to work extremely well in practice in Section~\ref{sec:experiments}.
\supplementary{sec:app-milp} describes an algorithm to find
the maximum \hyperclique. %

\JS{
\myParagraph{Comparison with~\cite{Shi21icra-robin}}
  In our previous work~\cite{Shi21icra-robin} ---where we first proposed \robin--- we defined compatibility graphs as ordinary graphs, and inliers structures as maximum cliques.
  In the present paper, we define compatibility graphs as hypergraphs (Definition~\ref{def:comp-hypergraph}), and inlier structures as maximum hypercliques (Theorem~\ref{thm:inliers-form-clique}).
  This new formulation is equivalent to~\cite{Shi21icra-robin} for $2$-invariants, but different otherwise.
  Topologically, the compatibility graphs in~\cite{Shi21icra-robin} can be seen as \emph{clique-expanded} compatibility hypergraphs, where each hyperedge on a subset of $n$ nodes is substituted with pairwise edges between all nodes in the subset (\ie a clique in the graph).
  Compared to~\cite{Shi21icra-robin}, our new formulation leads to pruning a larger number of outliers (see \supplementary{sec:app-hypergraph-v-graph} for a concrete example).
}

\setlength{\textfloatsep}{0pt}%
\begin{algorithm}[t]
{\footnotesize
\SetAlgoLined
\textbf{Input:} \ set of measurements $\calY$ and model~\eqref{eq:robin-measurements}; $n$-\invariant functions $(\vf,\vF)$ (for some $n$); inlier noise bound $\beta$\;
\textbf{Output:} \  subset $\calY^\star \subset \calY$ \;
\% Initialize \compatibility graph \label{line:startGraph} \\
$\calV  = \calY$; \!\!\quad  \% each node is a measurement \label{line:vertices}\\
$\calE  = \emptyset$; \quad \% start with empty hyperedge set \label{line:emptyEdges}\\
\% Perform \compatibility tests \\
\For{{\bf all subsets } $\subMeas{n} \subset \calY$ {\bf of size $n$}}{
  \If{ {\rm \texttt{testCompatibility}}($\subMeas{n},\vf,\vF$) = {\rm pass} }{
    add a hyperedge $e = \subMeas{n}$ to $\calE$;   \label{line:addEdges} \\
  }
} \label{line:endGraph}
\% Find \compatible measurements \label{line:graphTheory} \\
$\calY^\star = {\texttt{max\_hyperclique}}(\calV,\calE)$\;
 \textbf{return:} $\calY^\star$. \label{line:return}
 \caption{\robin \label{alg:robin}}
}
\end{algorithm}

\subsection{\robin: Graph-theoretic Outlier Rejection}
This section summarizes all the findings above into a single algorithm for graph-theoretic outlier pruning, named
\robin (\emph{\robinLong}).
\robin's pseudocode is given in Algorithm~\ref{alg:robin}.
The algorithm takes a set of measurements $\calY$ in input, and outputs a subset $\calY^\star \subset \calY$ from which many outliers have been pruned.
\LC{Given an $n$-invariant,} \robin first performs \compatibility tests on all subsets of $n$ measurements
and builds the corresponding \compatibility hypergraph (lines~\ref{line:startGraph}-\ref{line:endGraph}).
Then, it uses a maximum \hyperclique solver %
\LC{to compute} the subset of measurements surviving outlier pruning %
(lines~\ref{line:graphTheory}-\ref{line:return}).
We have implemented the maximum hyperclique solver described in~\supplementary{sec:app-milp} using CVXPY~\cite{Diamond16cvxpy};
in the special case where $n=2$, \ie the \compatibility graph is an ordinary graph,
we use the parallel maximum clique solver from~\cite{Rossi15parallel}.
We remark that \robin is not guaranteed to reject all outliers, \ie some outliers may still pass the \compatibility 
tests and end up in the maximum \hyperclique.
 Indeed,  as mentioned in the introduction, \robin is designed to be a preprocessing
step to prune gross outliers and enhance the robustness of existing robust estimators.
  Note that Algorithm~\ref{alg:robin} can be applied to various estimation problems,
  as long as suitable \compatibility tests are defined (\ie \texttt{testCompatibility} function).
  In~\cite{Shi21icra-robin}, we report applications to many geometric perception problems, including point cloud registration,
   point-with-normal registration, and camera pose estimation. 
  In the following section, we develop \compatibility tests for 3D-3D and 2D-3D category-level perception problems.

\section{Stage 1 (continued): Application to Category-Level Perception}
\label{sec:robin-category-level}

This section tailors \robin to category-level perception.
Specifically, we develop \emph{\compatibility tests}
for Problem~\ref{prob:3d3d-statement} (Section~\ref{sec:3d3d-comp-test}) and Problem~\ref{prob:2d3d-statement} (Section~\ref{sec:2d3d-comp-test}).

\subsection{3D-3D Category-level Compatibility Test}
\label{sec:3d3d-comp-test}

We develop a 3D-3D category-level compatibility test to check if a pair of 3D keypoint measurements are mutually compatible; our test generalizes results  
 on instance-level perception, where the object shape is known~\cite{Yang20tro-teaser,Enqvist09iccv,Shi21icra-robin}.

\begin{figure}[t]
    \centering
    \includegraphics[width=0.7\columnwidth]{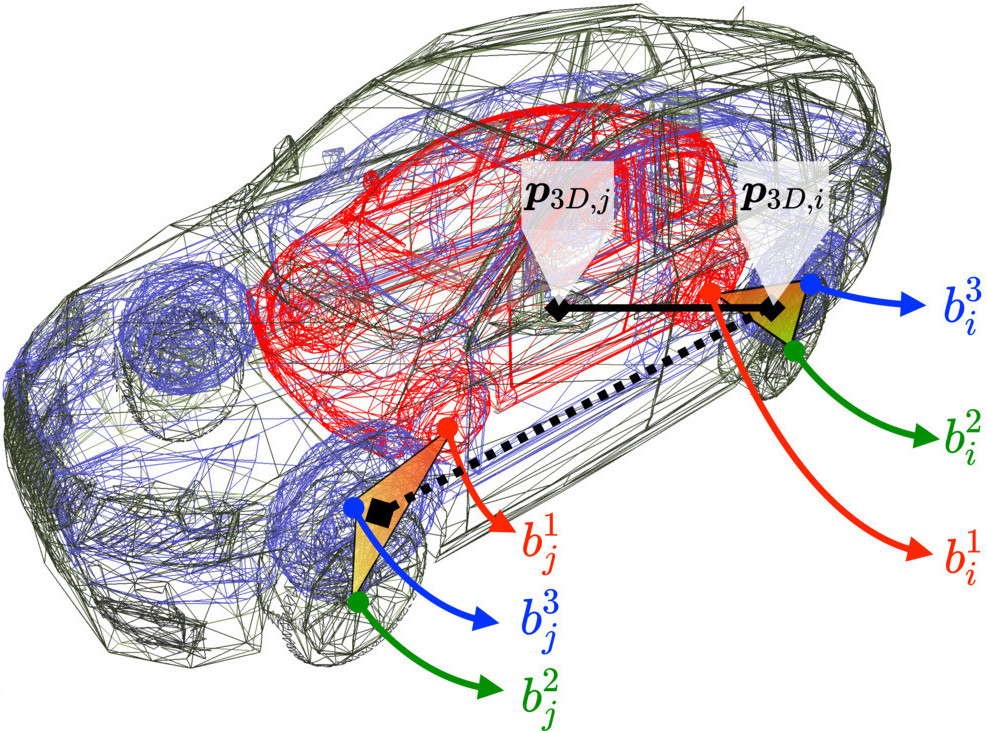} \vspace{-3mm}
    \caption{Example of compatibility test %
    with 3 CAD models of cars (red, dark green, blue, indexed from 1 to 3). (Noiseless) inliers (\eg the detection of the back wheel $\measThree{i}$ in the figure)
    must fall in the convex hull of the corresponding points on the CAD models (\eg triangle $\basis{1}{i}-\basis{2}{i}-\basis{3}{i}$ encompassing the back wheel positions across CAD models).
    This restricts the relative distance between two inliers and allows filtering out outliers.
     For instance, the dashed black line shows a distance that is compatible with the location of the convex hulls, while
     the solid black line is too short compared to the relative position of the wheels (for any car model) and
     allows pointing out that there is an outlier (\ie $\measThree{j}$ in the figure).
     \label{fig:3d-compatibility} }
\end{figure}

We derive a pairwise \invariant (\ie a 2-invariant) according to Definition~\ref{def:generalized-n-invariant}.
The challenge is to develop a set-valued function $\vF$ 
that does not depend on the pose and shape parameters, which are unknown.
Towards this goal, we show how to manipulate~\eqref{eq:generativeModel-3d} to obtain a function $\vF$ that does not depend on $\MR$, $\vt$, and $\vc$.
Taking the difference between measurement $i$ and $j$ in~\eqref{eq:generativeModel-3d} leads to:
\bea
\measThree{j} - \measThree{i}=
\MR \textstyle\sum_{k=1}^{\nrShapes}c^{k} ( \basis{k}{j} - \basis{k}{i}) + (\epsThree{j} - \epsThree{i} ) \nonumber
\eea
where the translation cancels out in the subtraction.
Now taking the $\ell_2$ norm of both members we obtain: %
\bea
\| \measThree{j} - \measThree{i} \|=
\big\| \MR \textstyle\sum_{k=1}^{\nrShapes}c^{k} ( \basis{k}{j} - \basis{k}{i}) + (\epsThree{j} - \epsThree{i}) \big\| \nonumber
\eea
Using the triangle inequality, we have
\begin{align}
  - \| \epsThree{j} - \epsThree{i} \| &\leq \|  \measThree{j} - \measThree{i} \| - 
  \bigg\| \MR \textstyle\sum_{k=1}^{\nrShapes}c_{k} ( \basis{k}{j} - \basis{k}{i}) \bigg\| \nonumber \\
  &\leq \| \epsThree{j} - \epsThree{i} \|
\end{align}
Now observing that the $\ell_2$ norm is invariant to rotation and rearranging the terms:
\begin{align}
  \bigg\| \textstyle\sum_{k=1}^{\nrShapes}&c^{k} ( \basis{k}{j} - \basis{k}{i}) \bigg\| - \| \epsThree{j} - \epsThree{i} \| 
  \leq \| \measThree{j} - \measThree{i} \| 
  \nonumber \\
  &
  \leq \bigg\| \textstyle\sum_{k=1}^{\nrShapes}c^{k} ( \basis{k}{j} - \basis{k}{i}) \bigg\| + \| \epsThree{j} - \epsThree{i} \| 
\end{align}
Taking the extreme cases over the  possible shape coefficients:

\vspace{-5mm}
\begin{align} \label{eq:definebminbmax}
  &\overbrace{\min_{ \vc \geq 0, \ones\tran \vc = 1} \bigg\| \textstyle \sum_{k=1}^{\nrShapes}c^{k} ( \basis{k}{j} - \basis{k}{i}) \bigg\|}^{b_{ij}^\min} - \| \epsThree{j} - \epsThree{i} \| \\
  &\leq \| \measThree{j} - \measThree{i} \| \nonumber \\
  &\leq \underbrace{ \max_{\vc \geq 0, \ones\tran \vc = 1} \bigg\| \textstyle \sum_{k=1}^{\nrShapes}c^{k} ( \basis{k}{j} - \basis{k}{i}) \bigg\|}_{ b_{ij}^\max } + \| \epsThree{j} - \epsThree{i} \| \nonumber
\end{align}
Since $\sum_{k=1}^{\nrShapes}c^{k} \basis{k}{j}$ is a convex combinations of the points $\basis{k}{j}$ ($k=1\ldots,\nrShapes$) and hence lies in the convex hull of such points, the term $\|  \sum_{k=1}^{\nrShapes}c^{k} ( \basis{k}{j} - \basis{k}{i}) \|$ represents the distance between two (unknown) points in the two
convex hulls defined by the set of points $\basis{k}{j}$ and $\basis{k}{i}$ ($k=1\ldots,\nrShapes$) (Fig.~\ref{fig:3d-compatibility}).
 The minimum $b_{ij}^\min$ and the maximum
$b_{ij}^\max$ over the convex hulls can be easily computed, either in closed form or via small convex programs (see details in \supplementary{sec:app-bmin-bmax}).
Accordingly,
\begin{align} \label{eq:3d3d-pt-dist-bound}
  \| \measThree{j}& - \measThree{i} \| \\
  &\in \left[ b_{ij}^\min - \| \epsThree{j} - \epsThree{i} \|, b_{ij}^\max + \| \epsThree{j} - \epsThree{i} \| \right] \nonumber
\end{align}
Note that $b_{ij}^\min$ and $b_{ij}^\max$ only depend on the \LC{given CAD library}, and are independent on $(\MR, \vt, \vc)$. Therefore, they can be pre-computed.
We can now define the pairwise invariant for Problem~\ref{prob:3d3d-statement} with generative model defined in eq.~\eqref{eq:generativeModel-3d}:
\begin{proposition}[3D-3D Category-level Pairwise \Invariant and \Compatibility Test]\label{prop:3d3d-invariant}
  Assume bounded noise $\|\epsThree{i} \| \leq \inthrThree$ for $i=1,\ldots,N$.
  The function $ \invfunThree(\measThree{i}, \measThree{j}) \doteq \| \measThree{j} - \measThree{i} \|$ is a pairwise invariant for eq.~\eqref{eq:generativeModel-3d}, with
  \begin{align}
    \invFunThree(\vtheta_{i}, &\vtheta_{j}, \inthrThree )= \left[ b_{ij}^\min - 2\inthrThree, b_{ij}^\max + 2\inthrThree \right] \nonumber
  \end{align}
  where $\vtheta_{i} \!=\! \{ \basis{k}{i} \!\mid\! k=1,\ldots,K\}$ and $\vtheta_{j} \!=\! \{ \basis{k}{j} \!\mid\! k=1, \ldots, K \}$. Therefore, two measurements $\measThree{i}$ and $\measThree{j}$ are compatible if $\invfunThree(\measThree{i}, \measThree{j}) \in \invFunThree(\vtheta_{i}, \vtheta_{j}, \inthrThree )$.
\end{proposition}

The proof of the proposition trivially follows from~eq.~\eqref{eq:3d3d-pt-dist-bound} and from the 
observation that $\|\epsThree{i}\|\leq\!\inthrThree$ and $\|\epsThree{j}\|\leq\!\inthrThree$ imply
$\| \epsThree{j} - \epsThree{i} \| \leq 2\inthrThree$.

Proposition~\ref{prop:3d3d-invariant} provides a compatibility test according to Definition~\ref{def:comp-test}.
In words, a pair of measurements is mutually compatible if their distance 
$\| \measThree{j} - \measThree{i} \|$ matches the corresponding distances in the CAD models
(lower-bounded by $b_{ij}^\min$ and upper-bounded by $b_{ij}^\max$) up to measurement noise $\inthrThree$.
If a pair of measurements fails the compatibility test, then one of them must be an outlier.
A geometric interpretation of the compatibility test (for $\inthrThree = 0$) is given in
Fig.~\ref{fig:3d-compatibility}.

\subsection{2D-3D Category-level Compatibility Test} \label{sec:2d3d-comp-test}

The pairwise invariant presented in the previous section was inspired by the fact that the distance between 
pairs of points is (a noiseless) invariant to rigid transformations. When it comes to 2D-3D problems, it is known that there is no 
(noiseless) invariant for 3D points in generic configurations
under perspective projection~\cite{Mundy92book}. 
One option would be to use invariants for special configurations of points, such as cross ratios for collinear points~\cite{Shi21icra-robin}; however, this would not be generally applicable to our problem, where 3D keypoints are arbitrarily distributed on the CAD models.
Here, we take an alternative route and we directly design a {generalized} 3-invariant for 
generic 3D point projections. %

Our 2D-3D category-level compatibility test draws inspiration from back-face culling in computer graphics~\cite{Eberly06book-gameEngineDesign}.
The key idea is that when observing an object, the \emph{winding order}  of 
triplet of keypoints seen in the image (roughly speaking: if the points are arranged in clockwise or counterclockwise order) must be consistent with the arrangement of the 
corresponding triplet of keypoints in the CAD models.
Therefore, 
we formulate a test by checking whether the observed winding order
matches our expectation from the CAD models.
In this section, we first define the notion of 2D and 3D winding orders, 
as well as the \vis and \covis \regions (camera locations where triplets of points are covisible);
then we introduce a $3$-invariant involving winding orders;
finally, we combine winding orders and \covis \regions to develop a %
2D-3D category-level compatibility test.

\renewcommand{\mpwthree}{5.9cm}
\renewcommand{\mpwthreetwo}{11.4cm}

\begin{figure}[t!]
    \centering%
    \includegraphics[width=0.7\columnwidth]{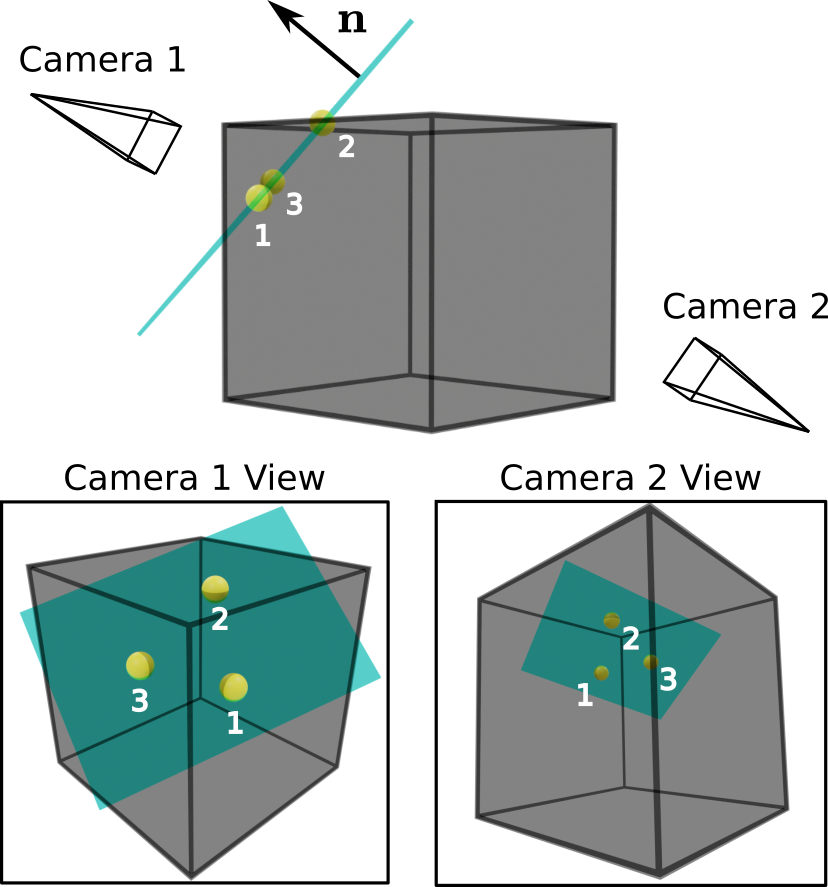}
	\caption{Illustration of the concept of winding orders on a cube with three keypoints (in yellow) on its surface.
	 The cube is opaque, but shown as semi-transparent for visualization.
	 The plane cutting through all three keypoints (blue) has its normal $\vn$ pointing outwards from the cube,
     following eq.~\eqref{eq:single-shape-triplet-normal}.
	 Camera 1 is in front of the plane, and the observed 2D winding order (Camera 1 View) is counterclockwise (enumerating in the order $1 \rightarrow 2 \rightarrow 3$),
     following the right-hand rule with thumb pointing towards the viewer.
	 Camera 2 is behind the plane, and would observe a clockwise winding order.
     The only feasible winding order is counterclockwise as observed by Camera 1, since the
      keypoints would be occluded by the cube in Camera 2.
   \label{fig:2d3d-winding-cube}}
\end{figure}

\myParagraph{2D Winding Orders}
Winding orders refer to whether an ordered triplet of projected points are arranged in a clockwise or counterclockwise order.
For example, if we enumerate the points in Fig.~\ref{fig:2d3d-winding-cube} in ascending order of their indices (\ie 1, 2, 3),
camera 1 sees points in counterclockwise order, while camera 2 sees them in clockwise order.
\omitted{In this section, we adopt the convention where the camera frame is right-handed, with $z$-axis pointing forward, $x$-axis pointing right and $y$-axis pointing downwards.
The projected points lie on the calibrated image plane at $z=1$, with the 2D image frame's $x$-axis pointing right and $y$-axis pointing down.}

\begin{definition}[2D Winding Order] \label{def:winding-order}
Given three 2D image points $\measTwo{i}$, $\measTwo{j}$, and $\measTwo{m}$ where $i < j < m$,
their winding order
is the orientation \{clockwise, counterclockwise\} of points
when enumerating them in the order $i \rightarrow j \rightarrow m$.
\end{definition}

The following proposition allows computing the 2D winding order algebraically (proof in \supplementary{sec:app-compute-winding-order}). %
\begin{proposition}[2D Winding Order Computation]\label{prop:compute-winding-order}
Assume three non-collinear 2D image points $\measTwo{i}$, $\measTwo{j}$, and $\measTwo{m}$,
then the
winding order $W$ can be computed as:
\begin{align}
  W = \begin{cases}
\text{clockwise} &  \text{if } V > 0 \\
\text{counterclockwise} &  \text{otherwise}
\end{cases}
\end{align}
where $V = \det \left( \begin{bmatrix}
        \measTwo{j} - \measTwo{i} \;\;
        \measTwo{m} - \measTwo{i}
   \end{bmatrix}   \right)$ and 
$\det(\cdot)$ denotes the matrix determinant.
\end{proposition}

{In words, Proposition~\ref{prop:compute-winding-order} states that the 2D
winding order can be calculated from the signed area of the parallelogram formed by $\measTwo{j} - \measTwo{i}$ and $\measTwo{m} - \measTwo{i}$.}{}

\myParagraph{Half-spaces and 3D Winding Orders}
While Proposition~\ref{prop:compute-winding-order} provides a simple way to compute the winding order for
a triplet of 2D image points, towards our 2D-3D category-level invariant,
we need to establish a notion of winding order also for the 3D shape keypoints on a CAD model.
In the following we show that the winding order for a triplet of 3D points can be uniquely
determined given the location of the camera. 
To develop some intuition, consider Fig.~\ref{fig:2d3d-winding-cube}, where we have three 3D keypoints on the faces of a cube.
The plane passing across the triplet of 3D keypoints divides the space into two half-spaces.
Theorem~\ref{thm:winding-order-half-space} below shows that, whenever the camera lies within one of the  half-spaces, only one winding order is possible. Therefore, if the keypoints are only covisible by camera locations in one of the half-spaces, their winding order is uniquely determined. 

Let us formally define the half-spaces induced by the triplet plane.
Let $\basis{k}{i}$, $\basis{k}{j}$, and $\basis{k}{m}$ ($i < j < m$) be three model keypoints on the $k$-th CAD model.
Define the triplet normal vector in the model's frame as follows (\cf $\normal{}{}$ in Fig.~\ref{fig:2d3d-winding-cube}):
\begin{align}
  \label{eq:single-shape-triplet-normal}
\normal{k}{i,j,m} &= (\basis{k}{j} - \basis{k}{i}) \times (\basis{k}{m} - \basis{k}{i})
\end{align}
So the triplet plane equation is:
$(\vo - \basis{k}{i}) \cdot \normal{k}{i,j,m} = 0$ for any $\vo \in \Real{3}$.
If $(\vo - \basis{k}{i}) \cdot \normal{k}{i,j,m} > 0$ (resp. $(\vo - \basis{k}{i}) \cdot \normal{k}{i,j,m} < 0$), $\vo$ lies in the positive (resp. negative) half-space. 
In this section, one can think about $\vo$ as the optical center of the camera in the CAD model's frame, hence 
the inequalities above capture which half-space the camera is located in.

The next theorem connects winding order with the two half-spaces created by the triplet plane (proof in~\supplementary{sec:app-winding-order-half-space}).
\begin{theorem}[Half-spaces and 3D Winding Orders]\label{thm:winding-order-half-space}
Under perspective projection per eq.~\eqref{eq:generativeModel-2d} with zero noise ($\epsTwo{i} = \epsTwo{j} = \epsTwo{m} = 0$), the following equality holds:
\begin{align}\label{eq:winding-order-eq-half-space}
  &\sgn{((\vo - \basis{k}{i}) \cdot \normal{k}{i,j,m})} \nonumber \\
  =& -\sgn{\left( \det \begin{bmatrix}  \measTwo{j} - \measTwo{i} & \measTwo{m} - \measTwo{i} \end{bmatrix} \right)}
\end{align}
where $\sgn(\cdot)$ is the signum function.
\end{theorem}
The theorem states that the half-space the camera is located in (identified by $\sgn{((\vo - \basis{k}{i}) \cdot \normal{k}{i,j,m})}$)
uniquely determines the 2D winding order of the projection of the 3D points.
This is not informative if the camera can be anywhere, since both winding orders are possible.
 In the following, we use the idea of \covis \regions to restrict potential 
 locations of the camera, such that we can associate a possibly unique winding order to triplets of 3D keypoints.

\begin{figure*}[t!]
    \centering%
    \includegraphics[width=\textwidth]{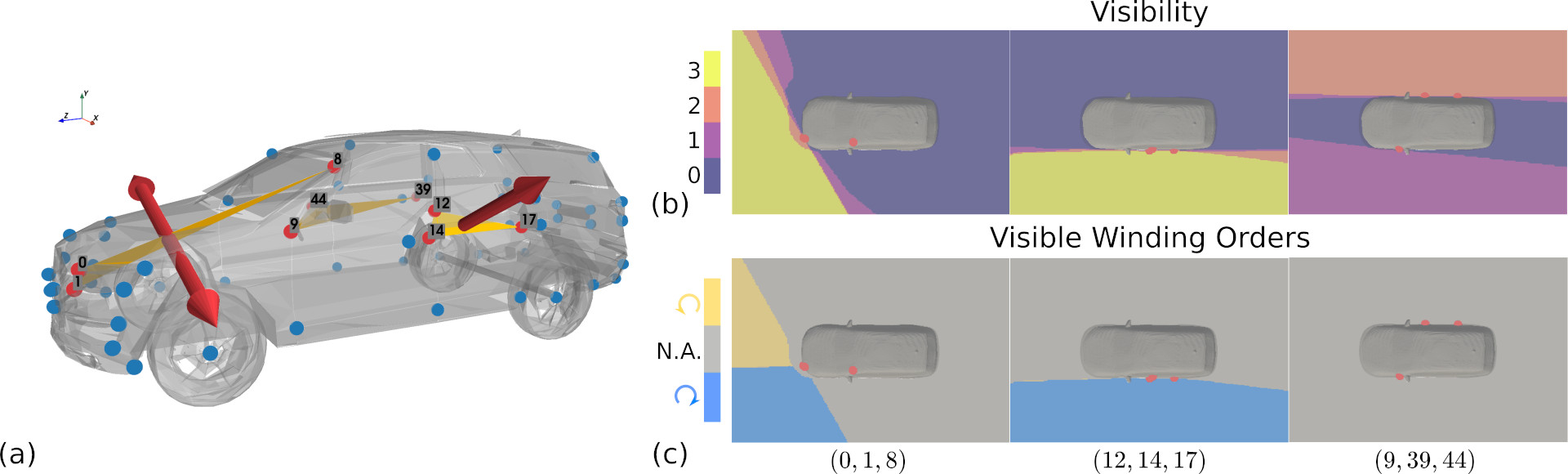}
    \vspace{-6mm}
	\caption{Visualization of 3D winding orders of three triplets of keypoints on the model \texttt{Q7-SUV} from the \apollo dataset.
    (a) CAD model with keypoints.
      Red points represent the selected  keypoint triplets, while yellow planes indicate the planes formed by the triplets.
      Arrows point towards the half-space in which the triplets are covisible.
      Triplet (0, 1, 8) can be viewed in both clockwise and counterclockwise winding order in their covisibility region.
      Triplet (12, 14, 17) can be only viewed in counterclockwise winding order.
      Triplet (9, 39, 44) cannot be viewed in either winding order, as their covisibility region is empty.
      (b-c) Winding orders and visibility of the triplets in a volume surrounding the car.
      The top row (b) shows visibility of the keypoints on a slice along the xz-plane. Color-coded values represent how many keypoints are visible in the triplet; covisibility regions are shaded in yellow.
      The bottom row (c) shows winding orders of the triplets in the covisibility region, with beige denoting counterclockwise, blue clockwise, and gray corresponding to cases where the \covis \region is empty.
      The visibility \region and winding order plots are generated following the ray tracing method described in \supplementary{sec:app-winding-order-dictionary}.
    }
    \label{fig:apollo-car-winding-order}
    \vspace{-5mm}
\end{figure*}

\myParagraph{\Vis~and~\Covis~\Regions}
\Vis and \covis \regions describe the set of camera locations from which keypoints are visible.
The 3D object observed by the camera is assumed to be opaque, hence a keypoint visible from one viewpoint, might not 
be visible from another, due to self-occlusions.
Fig.~\ref{fig:2d3d-winding-cube} demonstrates this concept:
due to self-occlusions,
the three keypoints on the cube are visible to Camera 1, but occluded (\ie not visible) in Camera 2.

We now define the \vis \region of keypoints on a shape,
following the standard definition~\cite{Hughes14book-computerGraphics}.%
\begin{definition}[\Vis \Region]\label{def:pt-visibility-region}
  The \vis~\region~of a keypoint $\basis{k}{i}$ %
  is the
  set of 3D points $\{\vo\}$ such that line segments connecting $\vo$ and $\basis{k}{i}$ do not intersect the $k$-th shape model.
\end{definition}
We also define \covis \regions for keypoints triplets,
which are the 3D space in which all three keypoints are visible.
\begin{definition}[\Covis \Region]\label{def:pt-covisibility-region}
  The \covis~\region~of a triplet of keypoints $\basis{k}{i}$, $\basis{k}{j}$, and $\basis{k}{m}$ is the intersection
  of the \vis \region of each keypoint.
\end{definition}

Fig.~\ref{fig:apollo-car-winding-order}(b) shows \vis and \covis \regions surrounding the \texttt{Q7-SUV} model from the \apollo dataset for selected triplets of keypoints.
For the triplet $(0, 1, 8)$, their \covis \regions are to the front of the car,
whereas for $(12, 14, 17)$, their \covis \regions are to the left of the car.
Notice how the relative positions between triplet half-spaces and \covis \regions affect the visible winding orders; see Fig.~\ref{fig:apollo-car-winding-order}(c).
The plane formed by $(0, 1, 8)$ cuts through their \covis \regions, hence both winding orders are visible.
For $(12, 14, 17)$, their \covis \regions are on one side of the keypoints plane, so the only visible winding order is clockwise. This observation is formalized below.

\begin{corollary}[\Covis-constrained 3D Winding Orders]\label{cor:winding_orders_feasibility}
  The projection of a triplet of 3D keypoints $\basis{k}{i}$, $\basis{k}{j}$, $\basis{k}{m}$ is arranged
  in \LC{counterclockwise winding order if and only if}
\begin{align}
  \{ \vo \in \Real{3} \mid  (\vo - \basis{k}{i}) \cdot \normal{k}{i,j,m}  > 0, \vo \in \CovisSet \} \neq \varnothing \label{eq:ijm_order_feasible}
\end{align}
where $\CovisSet$ is the \covis~\region of $\basis{k}{i}$, $\basis{k}{j}$, and $\basis{k}{m}$ on the $k$-th shape.
Similarly, \LC{the keypoints can be viewed}
in clockwise winding order if and only if
\begin{align}
  \{ \vo \in \Real{3} \mid  (\vo - \basis{k}{i}) \cdot \normal{k}{i,j,m}  < 0, \vo \in \CovisSet \} \neq \varnothing \label{eq:imj_order_feasible}
\end{align}
\end{corollary}
\noindent
This corollary follows directly from Proposition~\ref{prop:compute-winding-order},
Theorem~\ref{thm:winding-order-half-space}, and Definition~\ref{def:pt-covisibility-region}.
Remarkably, the two feasibility problems in eqs.~\eqref{eq:ijm_order_feasible} and~\eqref{eq:imj_order_feasible} and can be solved \emph{a priori}, since they only depend on the 
triplets of 3D keypoints and their normal.
For polyhedral shapes, the constraint $\vo \in \CovisSet$ can be expressed via linear constraints; 
in addition, if we replace the strict inequalities in eqs.~\eqref{eq:ijm_order_feasible} and~\eqref{eq:imj_order_feasible}
with non-strict ones \JS{and replace zeros with a small positive constant,} the problems can be solved via linear programming.
For complex shapes, we can check the conditions in eqs.~\eqref{eq:ijm_order_feasible} and~\eqref{eq:imj_order_feasible} 
by splitting the space around the CAD models into voxels (\ie we discretize the set of possible $\vo$) and 
ray tracing the keypoint to check visibility (see \supplementary{sec:app-winding-order-dictionary}).
In summary, given a CAD model and for each triplet of 3D keypoints, using eqs.~\eqref{eq:ijm_order_feasible} and~\eqref{eq:imj_order_feasible} we are able to predict if a covisible triplet
will produce clockwise or counterclockwise keypoint projections.

\begin{definition}[Feasible Winding Order Dictionary]\label{def:feasible_orders}
  For shape $k$, its feasible winding order dictionary $\calW_{k} : \{ 1, \ldots, N\}^{3} \to 
  \calP(\{-1,+1\})$ (where $\calP(\cdot)$ denotes the power set, \ie $\{+1, -1, \pm 1, \emptyset\}$) is defined as %
  \begin{align}\label{eq:feasible-winding-order-sets}
    \calW_{k} (i,j,m) = \{ \sgn{((\vo - \basis{k}{i}) \cdot \normal{k}{i,j,m})} \mid \forall \vo \in \CovisSet\}
  \end{align}
  where $\{ \sgn{((\vo \!-\! \basis{k}{i}) \!\cdot\! \normal{k}{i,j,m})} \!\mid\! \forall  \vo\in\! \CovisSet\}$
  is empty if both eq.~\eqref{eq:ijm_order_feasible} and eq.~\eqref{eq:imj_order_feasible} are false (\ie when the triplet is never covisible);
  it contains $+1$ if~\eqref{eq:ijm_order_feasible} is true;
  it contains $-1$ if~\eqref{eq:imj_order_feasible} is true; 
  it contains both $+1$ and $-1$ if both~\eqref{eq:ijm_order_feasible} and~\eqref{eq:imj_order_feasible} are true.
\end{definition}

We are now ready to define our generalized 3-invariant and the corresponding compatibility test.

\myParagraph{2D-3D Invariant and Compatibility Test}
We solve the two feasibility problems~\eqref{eq:ijm_order_feasible} and~\eqref{eq:imj_order_feasible}
for all triplets and CAD models and obtain a \emph{dictionary of feasible winding orders}.
In essence, each dictionary serves as a compatibility test for a single shape:
for observed keypoints $\measTwo{i}$, $\measTwo{j}$, and $\measTwo{m}$,
if the observed winding orders are not in $\calW_{k} (i,j,m)$, then the triplets are not compatible.
However, this dictionary is only for a single shape.
To formulate a 2D-3D category-level invariant,
we need to create a dictionary of feasible winding orders for all $K$ shapes.
We address this in Proposition~\ref{prop:2d3d-comp-test} below.

\begin{proposition}[2D-3D Category-level \Invariant and Compatibility Test]\label{prop:2d3d-comp-test}
  Assume the keypoints in eq.~\eqref{eq:generativeModel-2d} are generated by one of the shapes $\{1,\ldots,\nrShapes\}$, 
  that the reprojection noise is bounded by $\beta$ (\ie $\| \epsTwo{i} \| \leq\!\!\beta, \| \epsTwo{j} \| \leq \beta$, $\| \epsTwo{m} \| \leq \beta$), and that
  $\beta$ is small enough for Theorem~\ref{thm:winding-order-half-space} to hold true;
 then the functions $(\invfunTwo, \invFunTwo)$ is a $3$-\invariant for eq.~\eqref{eq:generativeModel-2d}, with
  \begin{align}
    \invfunTwo(&\measTwo{i}, \measTwo{j}, \measTwo{m}) \nonumber \\
               &\doteq \det \left( \begin{bmatrix} \measTwo{j} - \measTwo{i} & \measTwo{m} - \measTwo{i} \end{bmatrix}\right) \\
    \invFunTwo(&\vtheta_{i}, \vtheta_{j}, \vtheta_{m}) = \bigcup\limits_{k=1}^{K} \calW_{k}(i, j, m)
  \end{align}
  where $\vtheta_{\cdot} = \{ \basis{k}{\cdot} \mid k=1, \ldots, K \}$ and $\calW_{k}$ is the winding order dictionary for shape $k$, as per Definition~\ref{def:feasible_orders}.
  Therefore, a triplet of measurements $\measTwo{i}, \measTwo{j}$ and $\measTwo{m}$ is compatible if
   $\invfunTwo(\measTwo{i}, \measTwo{j}, \measTwo{m}) \in \invFunTwo(\vtheta_{i}, \vtheta_{j}, \vtheta_{m})$.
\end{proposition}

Intuitively, the proposition states that the observed winding order must match at least one of the winding orders contained  in the feasible winding order dictionary of all shapes.
The non-compatible triplets from Proposition~\ref{prop:2d3d-comp-test} 
are points with noise so large that the measured winding orders become inconsistent with the models.
Under technical conditions (discussed in \supplementary{sec:app-2d3d-invariant-valid-condition}), Proposition~\ref{prop:2d3d-comp-test} holds true even when keypoints are generated by convex combinations of $\basis{k}{i}$.

\section{Stage II: Certifiably Optimal Solvers for Category-Level Perception}
\label{sec:optimalSolvers-category-level}

While Stage 1 serves the purpose of filtering out a large fraction of gross outliers
(without even computing an estimate),
Stage 2 aims to use the remaining measurements to estimate the pose and shape parameters.
In this section, we develop \emph{certifiably optimal} solvers
for Problem~\ref{prob:3d3d-statement} and~\ref{prob:2d3d-statement} in the \emph{outlier-free} case (\ie assuming that 
\robin removed all the outliers).
In Section~\ref{sec:gnc-category-level},
we further improve robustness by incorporating a \emph{graduated non-convexity} scheme to handle potential 
left-over outliers in the measurements after \robin.

\subsection{Certifiably Optimal Solver for Outlier-free \\ 3D-3D Category-Level Perception}
\label{sec:optimalSolver-3d3d}

We show how to solve Problem~\ref{prob:3d3d-statement} in the outlier-free case, where
the noise $\epsThree{i}$ is assumed to follow a zero-mean Gaussian distribution.
In the outlier-free case, a standard formulation for the pose and shape estimation problem leads to a
\emph{regularized non-linear least squares} problem: %
\beq
\tag{3D-3D}
\label{eq:probOutFree-3D3Dcatlevel}
  \hspace{-5mm} \min_{\substack{\MR \in \SOthree, \\ \vt \in \Real{3}, \constraintc } }  \displaystyle
   \hspace{0mm} \sum_{i=1}^{N} w_i \left\| \measThree{i} - \MR \sum_{k=1}^{\nrShapes} c_{k} \basis{k}{i} - \vt \right\|^{2} + \lambda \norm{\vc}^2
\eeq
where the first summand in the objective minimizes the residual error \wrt the generative model~\eqref{eq:generativeModel-3d} ($w_i \geq 0, i=1,\dots,N$ are given weights), %
and the second term provides an $\ell_2$ regularization (\aka~\emph{Tikhonov regularization}~\cite{Tikhonov13book-numericalIllposed}) of the shape coefficients $\vc$
(controlled by the user-specified parameter $\lambda \geq 0$).
Note that for mathematical convenience we replaced the constraint $\vc \in \Delta_\nrShapes$ with
the constraint $\ones\tran \vc  = 1$, where~$\ones$ is a vector with all entries equal to 1;
in other words, we force the shape coefficients to sum-up to 1 but allow them to be negative.
\JS{
Note that 
  the regularization term serves the dual purpose of penalizing the occurrence of large negative entries in $\vc$ and
  keeping the problem well-posed regardless of
  the number of shapes in the library.
}
From the probabilistic standpoint, problem~\eqref{eq:probOutFree-3D3Dcatlevel} is a \emph{maximum a posteriori}
estimator assuming that the keypoints measurement noise follows a zero-mean Gaussian with covariance $\frac{1}{w_i} \eye_3$ (where $\eye_3$ is the 3-by-3 identity matrix) and we have a zero-mean Gaussian prior with covariance $\frac{1}{\lambda}$ over the shape parameters $\vc$ (see \supplementary{sec:app-mapOutlierFree}).
Problem~\eqref{eq:probOutFree-3D3Dcatlevel} is non-convex due to the product \LC{between $\MR$ and  $\vc$} in the objective, and due to the nonconvexity of the constraint set $\SOthree$ the rotation $\MR$
is required to belong to, see \eg~\cite{Rosen18ijrr-sesync}. Therefore, existing approaches based on
 local search~\cite{Lin14eccv-modelFitting,Gu06cvpr-faceAlignment,Ramakrishna12eccv-humanPose} are prone to converge to local minima corresponding to incorrect estimates.

\myParagraph{3D-3D solver overview}
 We develop the first certifiably optimal algorithm to solve~\eqref{eq:probOutFree-3D3Dcatlevel}.
Towards this goal we show that
(i) the translation $\vt$ in~\eqref{eq:probOutFree-3D3Dcatlevel} can be solved in closed form given the
rotation and shape parameters (Section~\ref{sec:translation-cat-level}),
(ii) the shape parameters $\vc$ can be solved in closed form given the rotation  (Section~\ref{sec:shape-cat-level}),
and (iii) the rotation can be estimated (independently on shape and translation) using a
tight semidefinite relaxation (Section~\ref{sec:rotation-cat-level}).

\subsubsection{Closed-form Translation Estimation}
\label{sec:translation-cat-level}

\LC{By inspection of~\eqref{eq:probOutFree-3D3Dcatlevel}, we observe that $\vt$ is unconstrained and appears quadratically in the cost. 
Hence, for any choice of $\MR$ and $\vc$, the optimal translation 
can be computed in closed-form:}
\bea
\label{eq:optTran}
\vt^{\star} (\MR,\vc) = \vy_w - \MR \textstyle\sum_{k=1}^{\nrShapes} c_{k} \vb_{k,w} 
\eea

\vspace{-2mm}
\noindent
where
\begin{align}
\label{eq:weightedCentroids}
\!\!\vy_w \!\triangleq \!\frac{1}{(\sum_{i=1}^N \!w_i)} \!\textstyle\sum_{i=1}^N \!w_i \measThree{i},
\;\;\;\; 
\vb_{k,w} \!\triangleq \!\frac{1}{(\sum_{i=1}^N \!w_i)} \!\sum_{i=1}^N  \!w_i \basis{k}{i}, \!\!\!
\nonumber
\end{align}

\normalsize
\noindent
are the weighted centroids of $\measThree{i}$ and $\basis{k}{i}$'s. 
This manipulation is common in related work, \eg~\cite{Zhou15cvpr,Yang20cvpr-shapeStar}.

\subsubsection{Closed-form Shape Estimation}
\label{sec:shape-cat-level}

Substituting the optimal translation~\eqref{eq:optTran} (as a function of $\MR$ and $\vc$) 
back into~\eqref{eq:probOutFree-3D3Dcatlevel}, we obtain an optimization that only 
depends on $\MR$ and $\vc$:
\bea\label{eq:translation-free-problem-1}
  \hspace{-8mm} \min_{\substack{\MR \in \SOthree \\ \constraintc} } \!\!&\!\! \sum_{i=1}^{N} \left\Vert \bar{\vy}_{i}  - \MR \sum_{k=1}^{\nrShapes} c_{k} \bar{\vb}^{k}_{i} \right\Vert^{2} + \lambda \norm{\vc}^2
\eea
where
\begin{align}
  \bar{\vy}_{i} \triangleq \sqrt{w_i} (\measThree{i} - \vy_w), \ \
  \bar{\vb}^{k}_{i} \triangleq \sqrt{w_i} (\basis{k}{i} - \vb_{k,w}),
\end{align}
are the (weighted) relative positions of $\measThree{i}$ and $\basis{k}{i}$ \wrt~their corresponding weighted centroids.
Using the fact that the $\ell_2$ norm is invariant to rotation,
problem~\eqref{eq:translation-free-problem-1} is equivalent to:
\bea\label{eq:translation-free-problem}
\hspace{-3mm} \min_{\substack{\MR \in \SOthree \\  \constraintc } } \!\!&\!\! \sum_{i=1}^{N} \left\Vert \MR\tran \bar{\vy}_{i}  - \sum_{k=1}^{\nrShapes} c_{k} \bar{\vb}^{k}_{i} \right\Vert^{2} + \lambda \norm{\vc}^2
\eea

We can further simplify the problem by adopting the following matrix notations:

\vspace{-5mm}
\begin{align}
\bar{\vy} &= \left( \bar{\vy}_{1}\tran, \ldots, \bar{\vy}_{N}\tran \right)\tran
\in \Real{3N} 
\\
\bar{\MB} &= \begin{bmatrix}
\bar{\vb}^{1}_{1} & \cdots & \bar{\vb}^{\nrShapes}_{1} \\
\vdots & \ddots & \vdots \\
\bar{\vb}^{1}_{N} & \cdots & \bar{\vb}^{\nrShapes}_{N}
\end{bmatrix} 
\in \Real{3N \times \nrShapes} 
\end{align}
which allows rewriting~\eqref{eq:translation-free-problem} in the following compact form:
\bea \label{eq:clsofc}
  \min_{\substack{\MR \in \SOthree \\ \constraintc } } & \left\Vert \barMB \vc -
(\eye_N \kron \MR\tran) \barvy
  \right\Vert^{2} + \lambda \norm{\vc}^2
\eea
Now the reader can again recognize that ---for any choice of $\MR$---
problem~\eqref{eq:clsofc} is a linearly-constrained linear least squares problem in $\vc$, 
which admits a closed-form solution.

\begin{proposition}[Optimal Shape]\label{prop:shapeEstimation}
For any choice of rotation $\MR$, the optimal shape parameters that solve~\eqref{eq:clsofc} 
can be 
computed in closed form as:
\bea \label{eq:optimalvcofR}
\vc^{\star} (\MR) = 2\MG \barMB\tran (\eye_N \kron \MR\tran) \barvy + \vg
\eea
where we defined the following constant matrices and vectors:
\bea
\MHtl \triangleq 2(\barMB\tran\barMB + \lambda \eye_K)  \label{eq:inverseofdensematrix}\\
\MG \triangleq \MHtl\inv - \frac{\MHtl\inv \ones \ones\tran \MHtl\inv }{\ones\tran \MHtl\inv \ones},\quad 
\vg \triangleq \frac{\MHtl\inv \ones}{\ones\tran \MHtl\inv \ones}
\eea
\end{proposition}

\subsubsection{Certifiably Optimal Rotation Estimation}
\label{sec:rotation-cat-level}

Substituting the optimal shape parameters~\eqref{eq:optimalvcofR} (as a function of $\MR$) 
back into~\eqref{eq:clsofc}, we obtain an optimization  that only 
depends on $\MR$:
\bea \label{eq:nonconvexR}
\min_{\MR \in \SOthree}  \norm{\MM (\eye_N \kron \MR\tran)\barvy + \vh}^2
\eea
where the matrix $\MM \in \Real{(3N+\nrShapes)\times 3N}$ and 
vector $\vh \in \Real{3N+\nrShapes}$ are defined as:
\bea
\MM \triangleq \bmat{c}
2\barMB \MG \barMB\tran - \eye_{3N} \\
2 \sqrt{\lambda} \MG \barMB\tran
\emat \qquad
\vh \triangleq \bmat{c}
\barMB \vg \\ \sqrt{\lambda} \vg
\emat
\eea
Problem~\eqref{eq:nonconvexR} is a quadratic optimization over the non-convex set $\SOthree$. 
It is known that $\SOthree$ can be described as a set of quadratic equality constraints, see \eg~\cite{Tron15rssws3D-dualityPGO3D,Rosen18ijrr-sesync} or~\cite[Lemma 5]{Yang20cvpr-shapeStar}. 
Therefore, we can succinctly rewrite~\eqref{eq:nonconvexR} as a 
\emph{quadratically constrained quadratic program} (QCQP).

\begin{proposition}[Optimal Rotation]\label{prop:optRotation}
The category-level rotation estimation problem~\eqref{eq:nonconvexR} can be equivalently written as a
\emph{quadratically constrained quadratic program} (QCQP):
\bea
\label{eq:categoryQCQP}
\min_{\vrhomo \in \Real{10}} & \vrhomo\tran \MQ \vrhomo \\
\subject & \vrhomo\tran \MA_i \vrhomo = 0, \forall i = 1,\dots, 15 \nonumber
\eea
where $\vrhomo \triangleq [1,\vectorize{\MR}\tran]\tran \in \Real{10}$ is a vector stacking all the entries of the unknown 
rotation $\MR$ in~\eqref{eq:nonconvexR} (with an additional unit element), 
$\MQ \in \sym{10}$ is a symmetric constant matrix (expression given in \supplementary{sec:app-rotEst-details}), and
$\MA_i \in \sym{10}, i=1,\dots,15$ are the constant matrices that define the quadratic constraints 
describing the set $\SOthree$~\cite[Lemma 5]{Yang20cvpr-shapeStar}.
\end{proposition}

While a QCQP is still a non-convex problem, it admits a standard semidefinite relaxation, described below.

\begin{corollary}[Shor's Semidefinite Relaxation]\label{cor:optRotation-relax}
The following semidefinite program (SDP) is a convex relaxation of~\eqref{eq:categoryQCQP}:
\bea
\label{eq:categoryQCQPrelax}
\min_{\MX \in \sym{10}} & \trace{\MQ \MX} \\
\subject & \trace{\MA_0 \MX} = 1,  \nonumber\\
& \trace{\MA_i \MX} = 0, \forall i=1,\dots,15  \nonumber\\
 & \MX \succeq 0  \nonumber
\eea
Moreover, when the optimal solution $\MX^\star$ of~\eqref{eq:categoryQCQPrelax} has rank 1, it can 
be factorized as $\MX^\star = \vect{1 \\ \vectorize{\MR^\star}} [1 \; \vectorize{\MR^\star}]$ where $\MR^\star$ 
is the optimal rotation minimizing~\eqref{eq:nonconvexR}. 
\end{corollary}
{
The relaxation entails solving a small SDP ($10 \times 10$ matrix size, and $16$ linear equality constraints), hence it can be solved in milliseconds 
using standard interior-point methods (\eg~MOSEK~\cite{mosek}). 
Similar to related quadratic problems over \SOthree~\cite{Rosen18ijrr-sesync,Yang20tro-teaser,Yang19iccv-QUASAR,Briales16iros,Eriksson18cvpr-strongDuality}, the relaxation~\eqref{eq:categoryQCQPrelax} empirically produces rank-1 ---and hence \emph{optimal}--- solutions in common problems.  
Even when the relaxation is not tight, the relaxation allows computing an estimate and a bound on its suboptimality 
(see \supplementary{sec:app-roundingAndSuboptimality3D3D}). 
The proposed solution falls in the class of \emph{certifiable algorithms}~(see \cite{Bandeira15arxiv}  
and Appendix A in~\cite{Yang20tro-teaser}), since it allows solving a hard (non-convex) problem efficiently and with 
provable a posteriori guarantees.

\subsubsection{Summary}
\label{sec:summary-cat-level}

The results in this section suggest a simple algorithm to compute a certifiably optimal solution to the 
 pose and shape estimation problem~\eqref{eq:probOutFree-3D3Dcatlevel}: 
(i) we first estimate the rotation $\MR^\star$ using the semidefinite relaxation~\eqref{eq:categoryQCQPrelax}
(and compute the corresponding suboptimality gap $\eta$ as described in~\supplementary{sec:app-roundingAndSuboptimality3D3D}); 
(ii) we retrieve the optimal shape $\vc^{\star} (\MR^\star)$ given the 
optimal rotation using~\eqref{eq:optimalvcofR}; 
(iii) we retrieve the optimal translation $\vt^{\star} (\MR^\star,\vc^\star)$ %
using~\eqref{eq:optTran}.
If $\eta=0$, we certify that $(\hatMR,\hatvt,\hatvc)$ is a globally optimal solution to \eqref{eq:probOutFree-3D3Dcatlevel}.
We call the resulting algorithm \PACEThree (\emph{\PACEThreeLong}).

\subsection{Certifiably Optimal Solver for Outlier-free \\ 2D-3D Category-Level Perception}
\label{sec:optimalSolver-2d3d}
We now show how to solve Problem~\ref{prob:2d3d-statement} in the outlier-free case, 
where the noise $\epsTwo{i}$ in the generative model \eqref{eq:generativeModel-2d} follows a zero-mean isotropic Gaussian distribution. %
In such a case, the maximum a posteriori estimator for Problem~\ref{prob:2d3d-statement} becomes:
\begin{equation}\label{eq:probOutFree-2D3Dcatlevel}
\min_{\substack{\MR \in \SOthree \\ \vt \in \Real{3},  \vc \in \Delta_K}} \sum_{i=1}^{N} 
w_i \norm{ \measTwo{i}\! -\! \pi\! \parentheses{\MR \sum_{k=1}^{\nrShapes} c_k \basis{k}{i} + \vt} }^2 + \lambda \norm{\vc}^2 %
\end{equation}
where $\lambda \norm{\vc}^2$ with $\lambda \geq 0$ is an $\ell_2$ regularization on the shape parameters, and $w_i$ are given non-negative weights.
Eq. \eqref{eq:probOutFree-2D3Dcatlevel} minimizes the \emph{geometric reprojection} error and
belongs to a class of optimization problems known as \emph{fractional programming} because the objective in \eqref{eq:probOutFree-2D3Dcatlevel} is a sum of \emph{rational} functions.
Unfortunately, it is generally intractable to obtain a globally optimal solution for fractional programming \cite{Schaible03OMS-fractional}.
In fact, even if $\vc$ is known in \eqref{eq:probOutFree-2D3Dcatlevel}, searching for the optimal $(\MR,\vt)$ typically
 resorts to branch-and-bound \cite{Hartley09ijcv-globalRotationRegistration,Olsson06ICPR-optimalPnP},
which runs in worst-case exponential time.

To circumvent the difficulty of fractional programming caused by the geometric reprojection error, we adopt an 
\emph{algebraic}  reprojection error that minimizes the \emph{point-to-line} %
distance between each 3D keypoint (\ie $\MR \sum_{k=1}^{\nrShapes} c_k \basis{k}{i} + \vt$) and the \emph{bearing vector} emanating from the camera center to the 2D keypoint (\ie %
the bearing vector $\vv_i = \tmeasTwo{i}/\norm{\tmeasTwo{i}}$, where $\tmeasTwo{i} \doteq [\measTwo{i}\tran,1]\tran$ is the homogenization of $\measTwo{i}$):
\begin{equation}\label{eq:2d3dp2l}
\tag{2D-3D}
\min_{\substack{\MR \in \SOthree \\ \vt \in \Real{3}, \vc \in \Delta_K}} \sum_{i=1}^N 
w_i  \norm{\MR \sum_{k=1}^{\nrShapes} c_k \basis{k}{i} + \vt  }_{\MW_i}^2 + \lambda \norm{\vc}^2 \!\!\!  %
\end{equation}
where $\MW_i := \eye_3 - \vv_i \vv_i\tran$ and $\norm{\vp}^2_{\MW_i} := {\vp\tran \MW_i \vp}$ computes the distance from a given 3D point $\vp$ to a bearing vector $\vv_i$.
The point-to-line objective in \eqref{eq:2d3dp2l} has been adopted in other works,
including \cite{Schweighofer2008bmvc-SOSforPnP,Yang21iccv-damp}.
Below, we omit the weights $w_i$, since they can be directly included in the matrix $\MW_i$.

{
\myParagraph{2D-3D solver overview}
We develop the first certifiably optimal algorithm to solve problem \eqref{eq:2d3dp2l}. We show that
(i) the translation $\vt$ in \eqref{eq:2d3dp2l} can be solved in closed form given the
rotation and shape parameters (Section~\ref{sec:optimal2d3dtranslation}),
and 
(ii) the shape parameters $\vc$ and the rotation $\MR$ can be estimated using a
tight semidefinite relaxation (Section~\ref{sec:optimal2d3dshaperotation}).
}

\subsubsection{Closed-form Translation Estimation}
\label{sec:optimal2d3dtranslation}
{Similar to Section \ref{sec:translation-cat-level}, our first step is to algebraically eliminate the translation $\vt$ in \eqref{eq:2d3dp2l}. By deriving the gradient of the objective of \eqref{eq:2d3dp2l} and setting it to zero, we obtain that}
\bea\label{eq:closedformtranslation}
\vtstar = - \sum_{i=1}^N \tldMW_i \MR \parentheses{ \sum_{k=1}^{\nrShapes} c_k \basis{k}{i} }
\eea
where $\MW \!\doteq\! \sum_{i=1}^N \MW_i$ and $\tldMW_i \!\doteq\!  \MW\inv \MW_i$. Inserting~\eqref{eq:closedformtranslation} back into \eqref{eq:2d3dp2l}, we get the following translation-free problem
\bea \label{eq:2d3dp2ltf}
\min_{\substack{\MR \in \SOthree \\ \vc \in \Delta_K} } \sum_{i=1}^N \norm{ \MR \vs_i(\vc) - \sum_{j=1}^N \tldMW_j \MR \vs_j(\vc) }_{\MW_i}^2 + \lambda \norm{\vc}^2
\eea
where we compactly wrote $\vs_i(\vc) = \sum_{k=1}^{\nrShapes} c_k \basis{k}{i}$.

\subsubsection{Certifiably Optimal Shape and Rotation Estimation}
\label{sec:optimal2d3dshaperotation}
{
In this case, it is not easy to decouple rotation and shape parameters as we did in Section \ref{sec:shape-cat-level}.
Therefore, we adopt a
more advanced machinery to globally optimize $\MR$ and $\vc$ together. Towards this goal,} we observe that problem \eqref{eq:2d3dp2ltf} is in the form of a \emph{polynomial optimization problem} (POP), \ie an optimization problem
where both objective and constraints can be written using polynomials: 
\bea
 \label{eq:popgeneral}
\min_{\vxx \in \Real{d}} & p(\vxx) \\
\subject & h_i(\vxx) = 0, i=1,\dots,l_h \nonumber \\
& g_j(\vxx) \geq 0, j=1,\dots,l_g \nonumber 
\eea
where $\vxx = [\vectorize{\MR}\tran,\vc\tran]\tran \in \Real{d}, d=K+9$, denotes the vector of unknowns, $p(\vxx)$ is a degree-4 objective polynomial corresponding to the objective in \eqref{eq:2d3dp2ltf}, $h_i(\vxx),i=1,\dots,l_h=16$ are polynomial equality constraints (including $\MR \in \SOthree$ and $\one\tran \vc = 1$), and $g_j(\vxx),j=1,\dots,l_g=K+1$ are polynomial inequality constraints (\eg $c_k \geq 0$ for all $k$). %
{$p(\vxx)$ in \eqref{eq:popgeneral} has degree four, which prevents us from applying Shor's semidefinite relaxation as in Corollary \ref{cor:optRotation-relax}.} {However, writing \eqref{eq:2d3dp2ltf} in the form~\eqref{eq:popgeneral} allows us to leverage a more general semidefinite relaxation technique called \emph{Lasserre's hierarchy of moment and sum-of-squares relaxations}~\cite{Lasserre01siopt-LasserreHierarchy,Parrilo03mp-sos} to solve \eqref{eq:2d3dp2ltf} to certifiable global optimality.}

\begin{corollary}[Order-2 Lasserre's Moment Relaxation]\label{cor:lasserre-relax}
The following multi-block SDP is a convex relaxation of~\eqref{eq:categoryQCQP}:
\bea\label{eq:denserelax}
\min_{\substack{ \MX = (\MX_0,\MX_1,\dots,\MX_{K+1}) \\ \in \sym{n_0} \times \sym{n_1} \times \dots \times \sym{n_{K+1}} } } & \inprod{\MC}{\MX} \\
\subject & \calA(\MX) = \vb, \quad \MX \succeq 0 \nonumber
\eea
where $\MX_0$ is the \emph{moment matrix} of size $n_0 = \nchoosek{K+11}{2}$, 
$\MX_1,\dots,\MX_{K+1}$ are the so called \emph{localizing matrices} (of size $n_i < n_0$ for $i\geq 1$) 
arising from the inequality constraints $g_j$ in \eqref{eq:popgeneral} and the additional constraint $\vc\tran \vc \leq 1$, $\MX \succeq 0$ indicates each element of $\MX$, \ie $\MX_i,i=0,\dots,K+1$, is positive semidefinite, $\MC = (\MC_0,\MC_1,\dots,\MC_{K+1})$ are known matrices with $\MC_i = \zero$ for $i\geq 1$, and $\calA(\MX) = \vb$ 
collects all %
linear equality constraints on $\MX$ (each scalar constraint is written as $\sum_{i=0}^{K+1} \inprod{\MA_{ij}}{\MX_i} = b_j $ for $j=1,\dots,m$). 

Moreover, when the optimal solution $\MX^\star$ of~\eqref{eq:denserelax} is such that 
$\MXstar_0$ has rank 1, then $\MXstar_0$ can be factorized as $\MXstar_0 = [\vxxstar]_2[\vxxstar]_2\tran$, where $\vxxstar = [\vectorize{\MRstar}\tran,(\vcstar)\tran]\tran$ is a globally optimal for \eqref{eq:2d3dp2ltf}, and 
$[\vxxstar]_2$ denotes the vector of monomials in the entries of $\vxxstar$ of degree up to 2. 
\end{corollary}
We refer the interested reader to \cite[Section 2]{Yang22pami-certifiablePerception} for details about how to generate $(\calA,\vb,\MC)$ from the POP formulation \eqref{eq:popgeneral}. In practice, there exists an efficient Matlab implementation\footnote{\url{https://github.com/MIT-SPARK/CertifiablyRobustPerception}} that automatically generates the SDP relaxation \eqref{eq:denserelax} given a POP \eqref{eq:popgeneral}.
{The reader can observe that the SDP \eqref{eq:denserelax} is conceptually the same as the SDP \eqref{eq:categoryQCQPrelax} except that \eqref{eq:denserelax} has more than one positive semidefinite matrix decision variable}. 
\supplementary{sec:app-roundingAndSuboptimality2D3D} contains more details about Lasserre's hierarchy optimality certificates and 
the rounding procedure to extract an estimate from the solution of the SDP~\eqref{eq:denserelax}.

\subsubsection{Summary} 
We solve problem \eqref{eq:2d3dp2l} %
by (i) solving the SDP \eqref{eq:denserelax} using MOSEK \cite{mosek} and obtaining 
an estimate $(\hatMR,\hatvc)$ and the corresponding suboptimality gap $\eta$, using the rounding procedure in \supplementary{sec:app-roundingAndSuboptimality2D3D};
then (ii) we compute $\hatvt$ from \eqref{eq:closedformtranslation}.
If $\eta=0$, we certify that $(\hatMR,\hatvt,\hatvc)$ is a globally optimal solution to \eqref{eq:2d3dp2l}.
We call this approach \PACETwo (\PACETwoLong).

\JS{
\begin{remark}[Distribution of $\vc$] \label{remark:one-hot-c}
The active shape model allows the entries of $\vc$ to be scalars in $[0,1]$, which enables the model to interpolate between shapes; see Fig.~\ref{fig:active-shape-demo}. 
We note that if the vector $\vc$ is assumed to be a one-hot vector (\ie it has a single entry equal to 1 and all other entries equal to zero), then there exist trivial solvers for~\eqref{eq:2d3dp2l} and~\eqref{eq:probOutFree-3D3Dcatlevel}. Namely, one can run $K$ times an instance-based solver (\eg PnP or 3D registration), once for each shape in the CAD library, and then pick the solution attaining the lowest cost. 
Such an approach sacrifices the interpolation power of the active shape model and in the experiments
 we show that it leads to large errors when $\vc$ is not a one-hot vector.
\end{remark}
}

\section{Stage II (continued): Further Robustness Through Graduated Non-Convexity}
\label{sec:gnc-category-level}

\JS{
While in principle we can use \PACEThree and \PACETwo
directly after Stage 1,
the measurements selected by \robin potentially still contain a few outliers.
These outliers could still hinder the quality of the pose and shape estimates.
}

\JS{To compute accurate estimates in the face of those remaining outliers,}
we add a robust loss function ---in particular, a truncated least square loss--- to problems~\eqref{eq:probOutFree-3D3Dcatlevel} and~\eqref{eq:2d3dp2l},
and solve the resulting optimization using a standard graduated non-convexity (GNC)~\cite{Blake1987book-visualReconstruction,Yang20ral-GNC} approach. 
At each iteration, GNC alternates between solving a weighted least squares problem in the form~\eqref{eq:probOutFree-3D3Dcatlevel} and~\eqref{eq:2d3dp2l} (these can be solved to certifiable optimality using \PACEThree and \PACETwo) and updating the weights for each measurement (which can be computed in closed form~\cite{Yang20ral-GNC}).
 The interested reader can find more details in \supplementary{sec:app-gnc-category}.
\JS{
 We release the proposed solvers at \urlPACE.
}

\section{Experiments}
\label{sec:experiments}

This section presents a comprehensive evaluations of the proposed approaches.
First, \LC{we showcase the optimality of \PACEThree and robustness of \PACErobustThree through experiments on synthetic data, \pascal~\cite{Xiang2014WACV-PASCAL+} and \keypointnet~\cite{You20cvpr-KeypointNetLargescale} (Section~\ref{sec:exp-optimality-robustness-3d}).
 Then, we demonstrate the optimality of \PACETwo and robustness of \PACErobustTwo through experiments on synthetic datasets (Section~\ref{sec:exp-optimality-robustness-2d}).}
Finally, we test both \PACErobustThree and \PACErobustTwo on \apollo, a self-driving dataset~\cite{Song19-apollocar3d}.
\JS{
Experiments in Section~\ref{sec:exp-optimality-robustness-3d} and \ref{sec:exp-apollo} are run with an Intel i9-9920X CPU at 3.5 GHz with 128 GB RAM.
Section~\ref{sec:exp-optimality-robustness-2d} experiments are run on MIT SuperCloud~\cite{Reuther18hpec-supercloud} Xeon Platinum 8260 cluster, using 6 threads and 24 GB RAM per run.
}

\newcommand{\mpwfour}{4.4cm}
\newcommand{\myhspace}{\hspace{-3.5mm}}
\begin{figure*}[ht]
	\begin{center}
	\begin{minipage}{\textwidth}
	\begin{tabular}{cccc}%
		\myhspace \hspace{-3mm}
			\begin{minipage}{\mpwfour}%
			\centering%
			\includegraphics[width=\columnwidth]{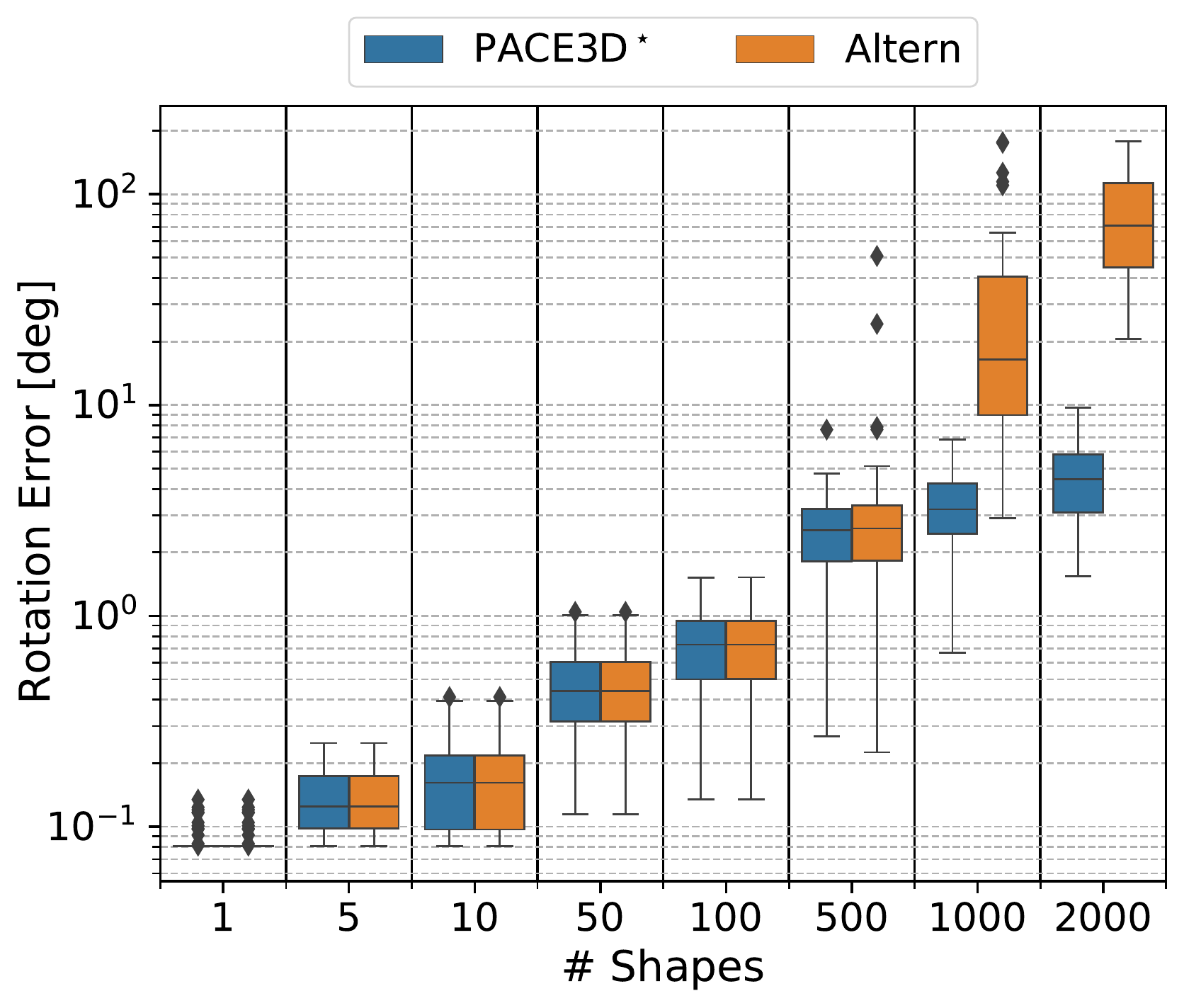}
			\end{minipage}
		&   \myhspace
			\begin{minipage}{\mpwfour}%
			\centering%
			\includegraphics[width=\columnwidth]{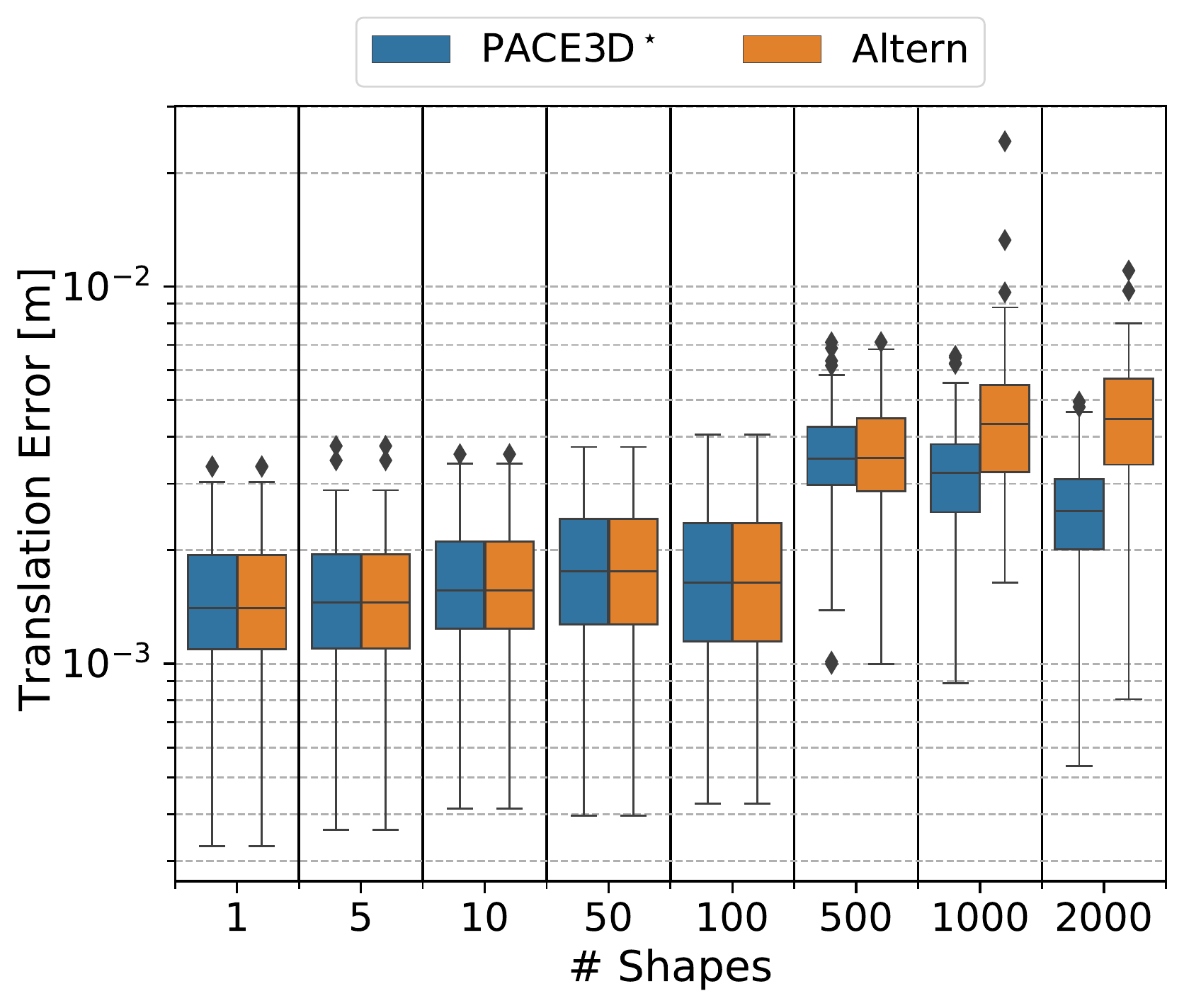}
			\end{minipage}
		&   \myhspace
			\begin{minipage}{\mpwfour}%
			\centering%
			\includegraphics[width=\columnwidth]{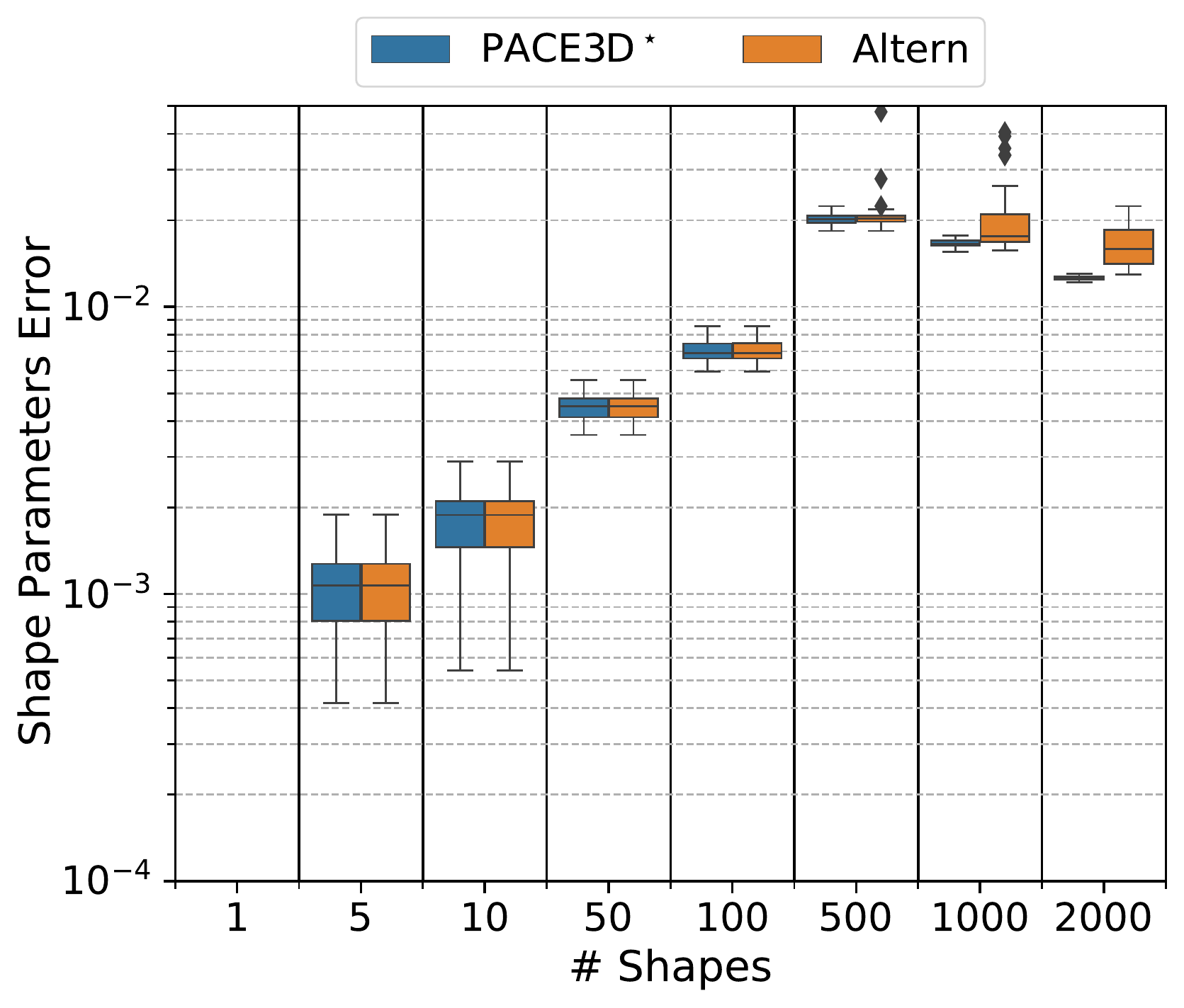}
			\end{minipage}
		&   \myhspace
			\begin{minipage}{\mpwfour}%
			\centering%
			\includegraphics[width=\columnwidth]{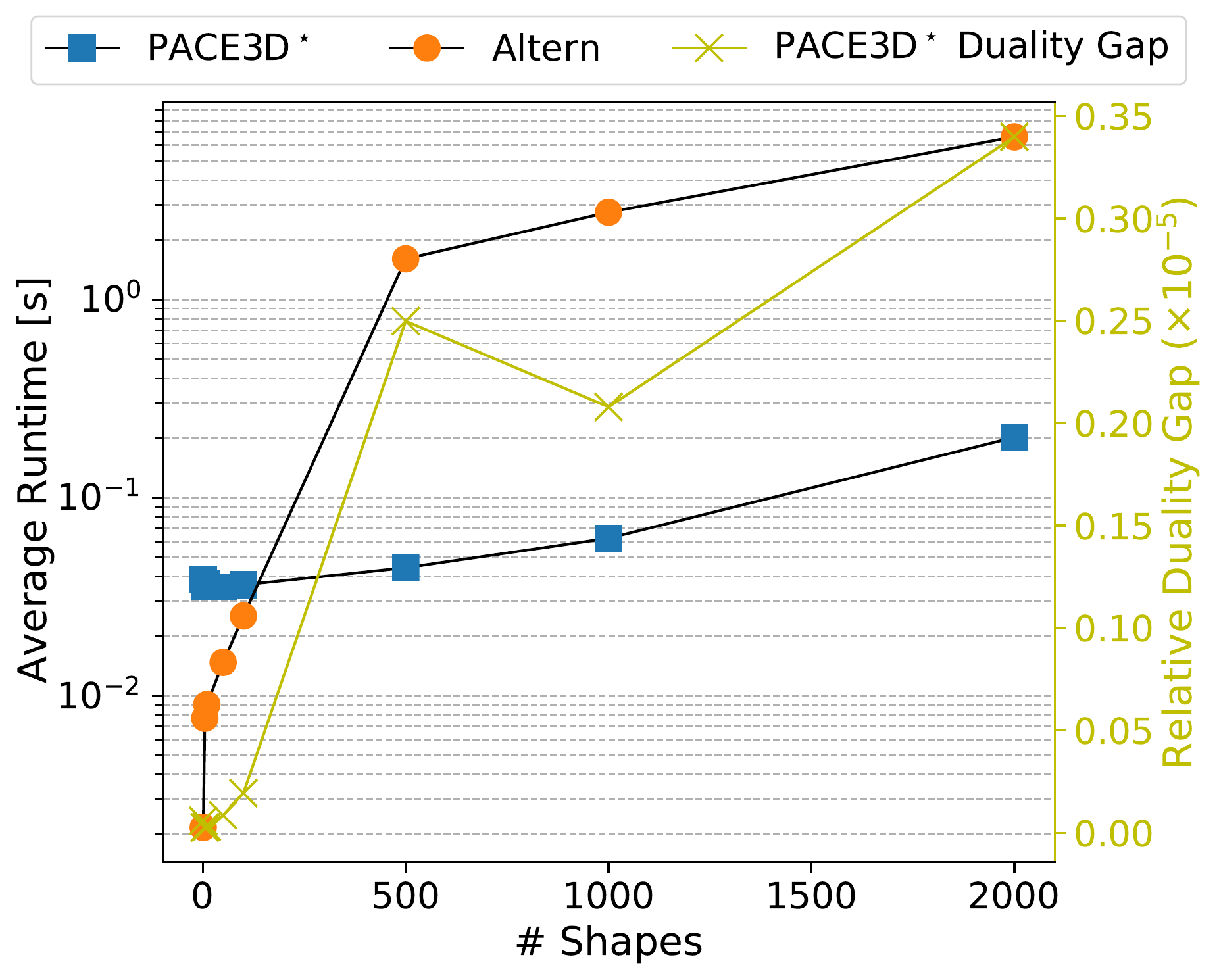}
			\end{minipage}
		\\
		\multicolumn{4}{c}{\smaller (a) Performance of \PACEThree~on outlier-free synthetic data: $N=100$. }
		\\
		\myhspace \hspace{-3mm}
			\begin{minipage}{\mpwfour}%
			\centering%
			\includegraphics[width=\columnwidth]{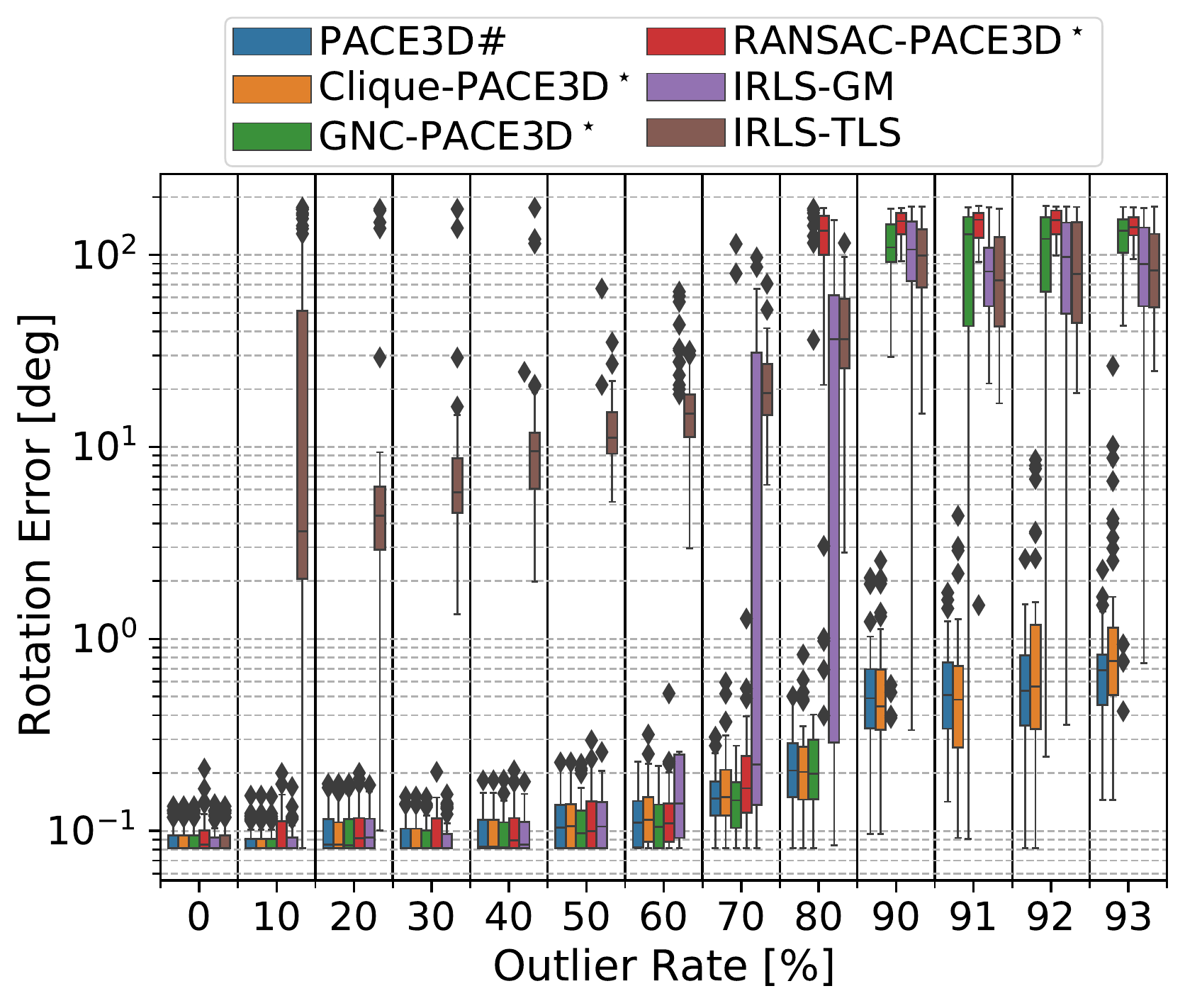}
			\end{minipage}
		&   \myhspace
			\begin{minipage}{\mpwfour}%
			\centering%
			\includegraphics[width=\columnwidth]{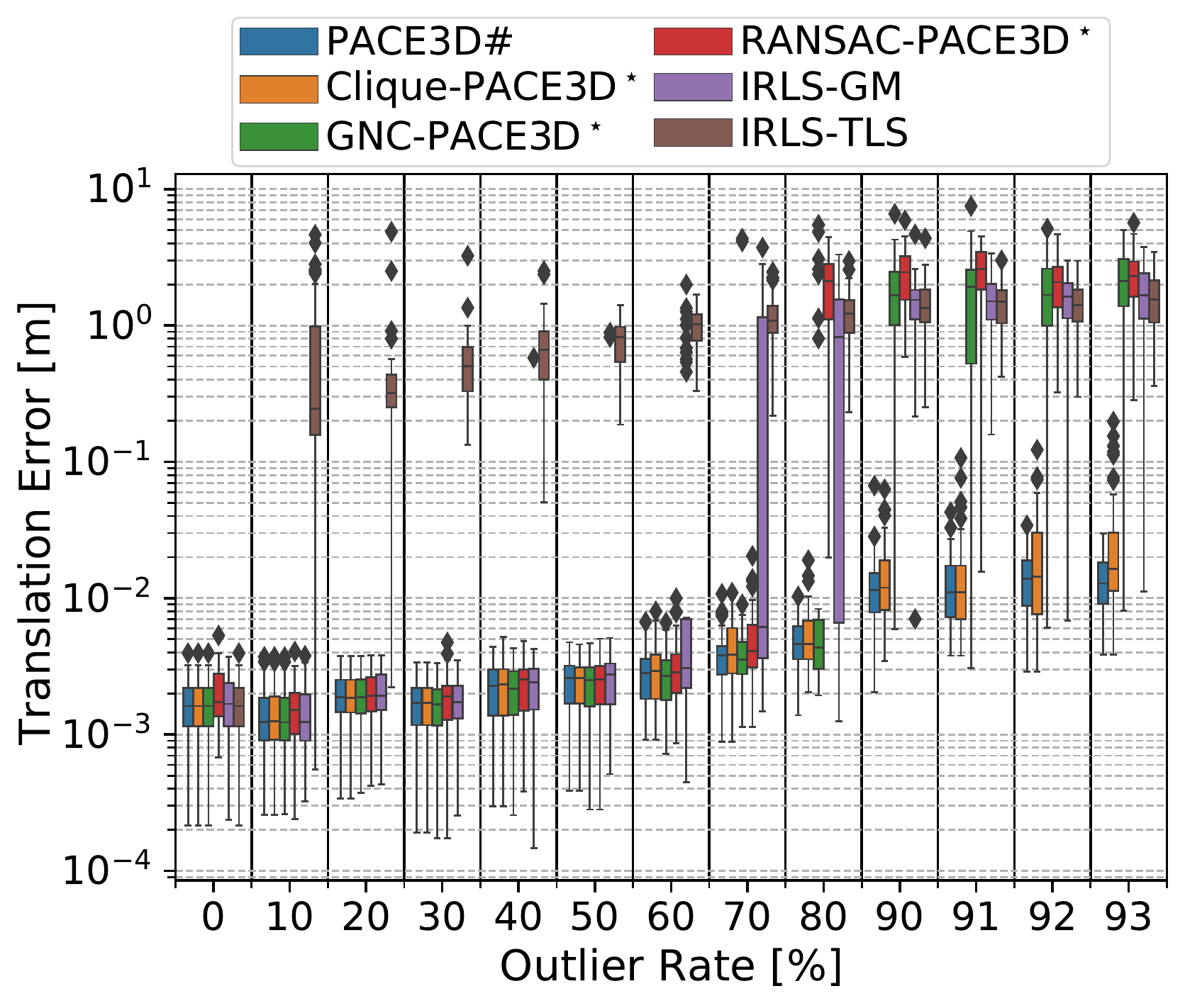}
			\end{minipage}
		&   \myhspace
			\begin{minipage}{\mpwfour}%
			\centering%
			\includegraphics[width=\columnwidth]{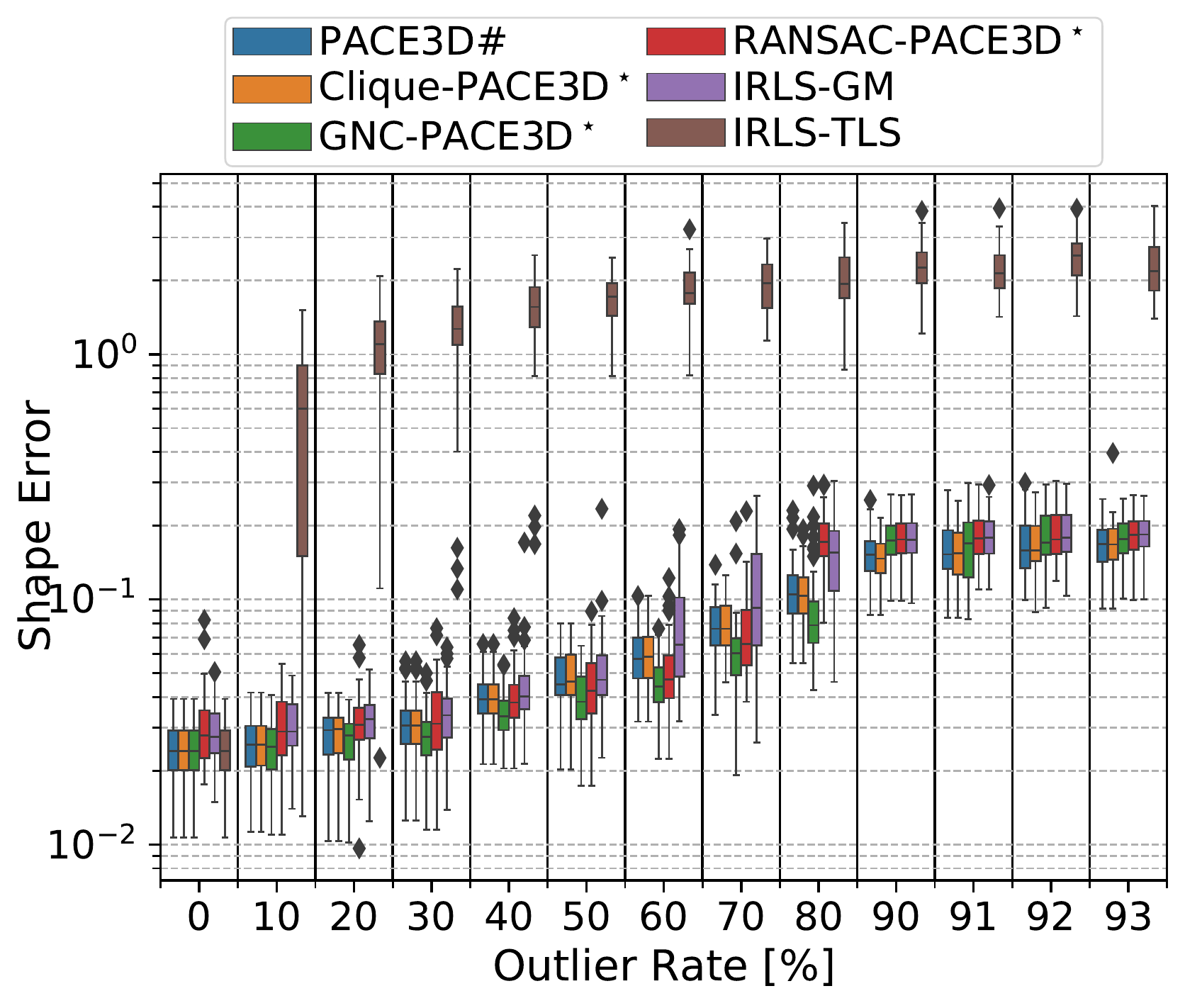}
			\end{minipage}
		&   \myhspace
			\begin{minipage}{\mpwfour}%
			\centering%
			\includegraphics[width=\columnwidth]{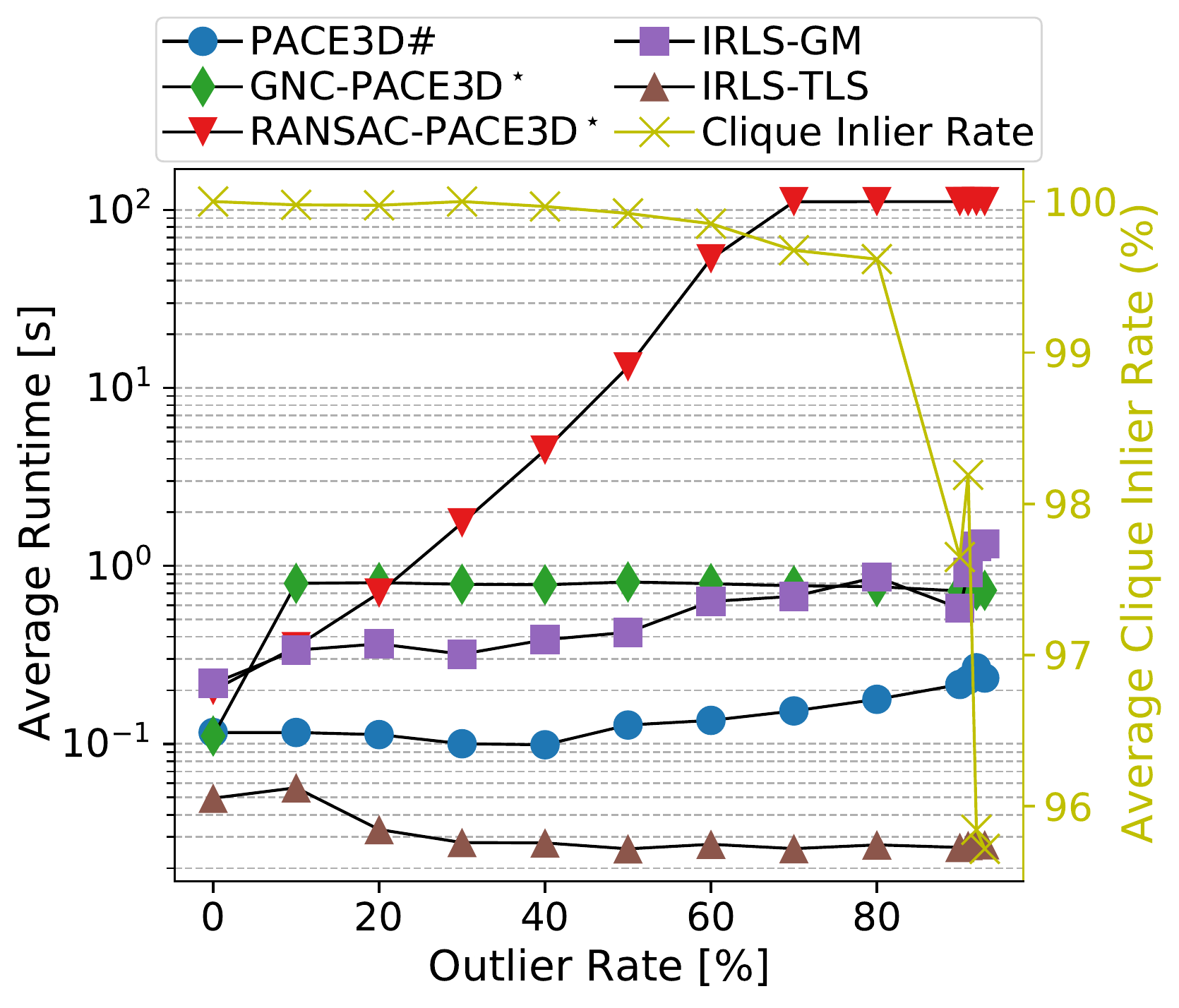}
			\end{minipage}
		\\
		\multicolumn{4}{c}{\smaller (b) Robustness of \PACErobustThree against increasing outliers on synthetic data: $N=100$, $K=10$, $r=0.1$.}
		\\
		\myhspace \hspace{-3mm}
			\begin{minipage}{\mpwfour}%
			\centering%
			\includegraphics[width=\columnwidth]{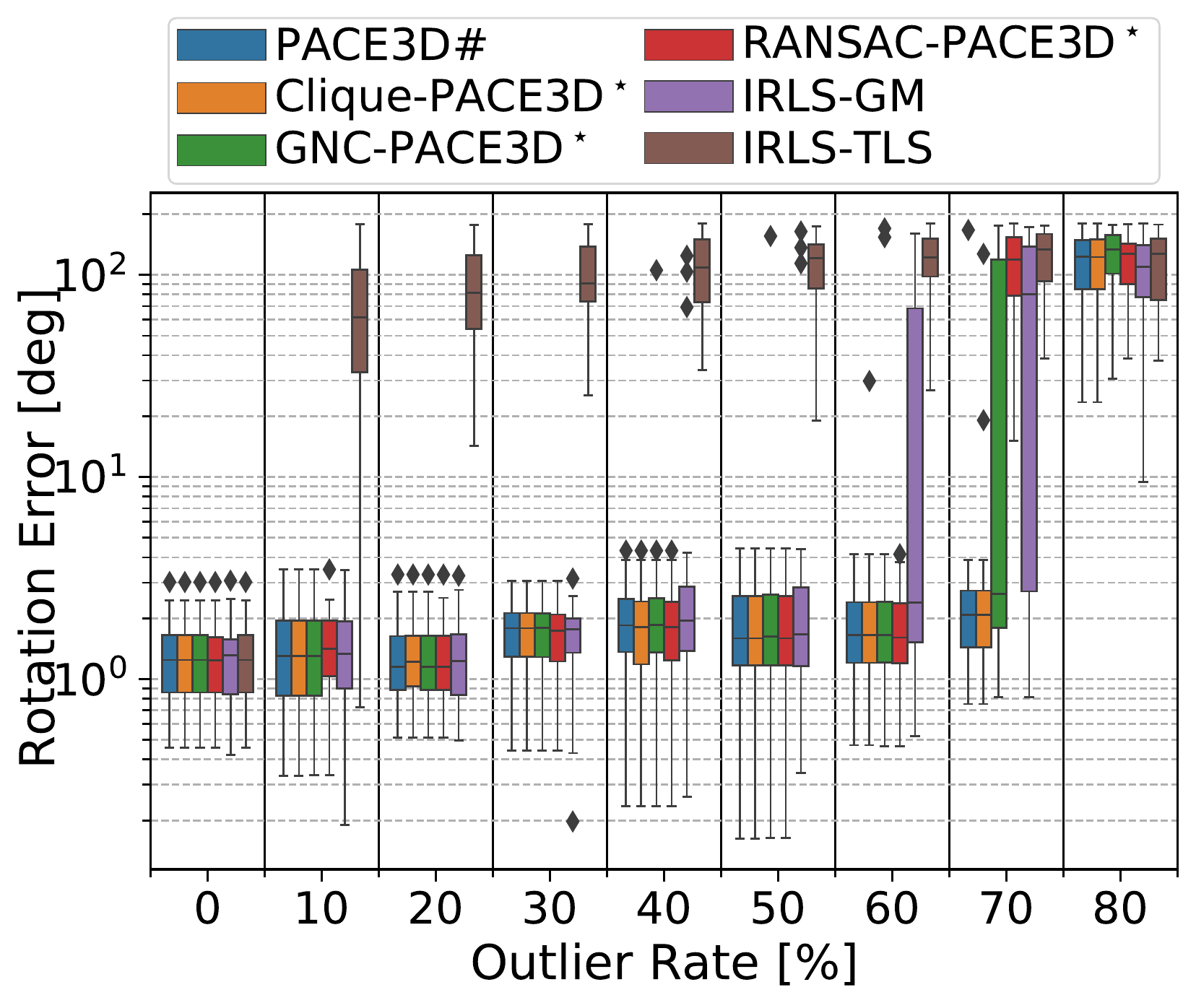}
			\end{minipage}
		&   \myhspace
			\begin{minipage}{\mpwfour}%
			\centering%
			\includegraphics[width=\columnwidth]{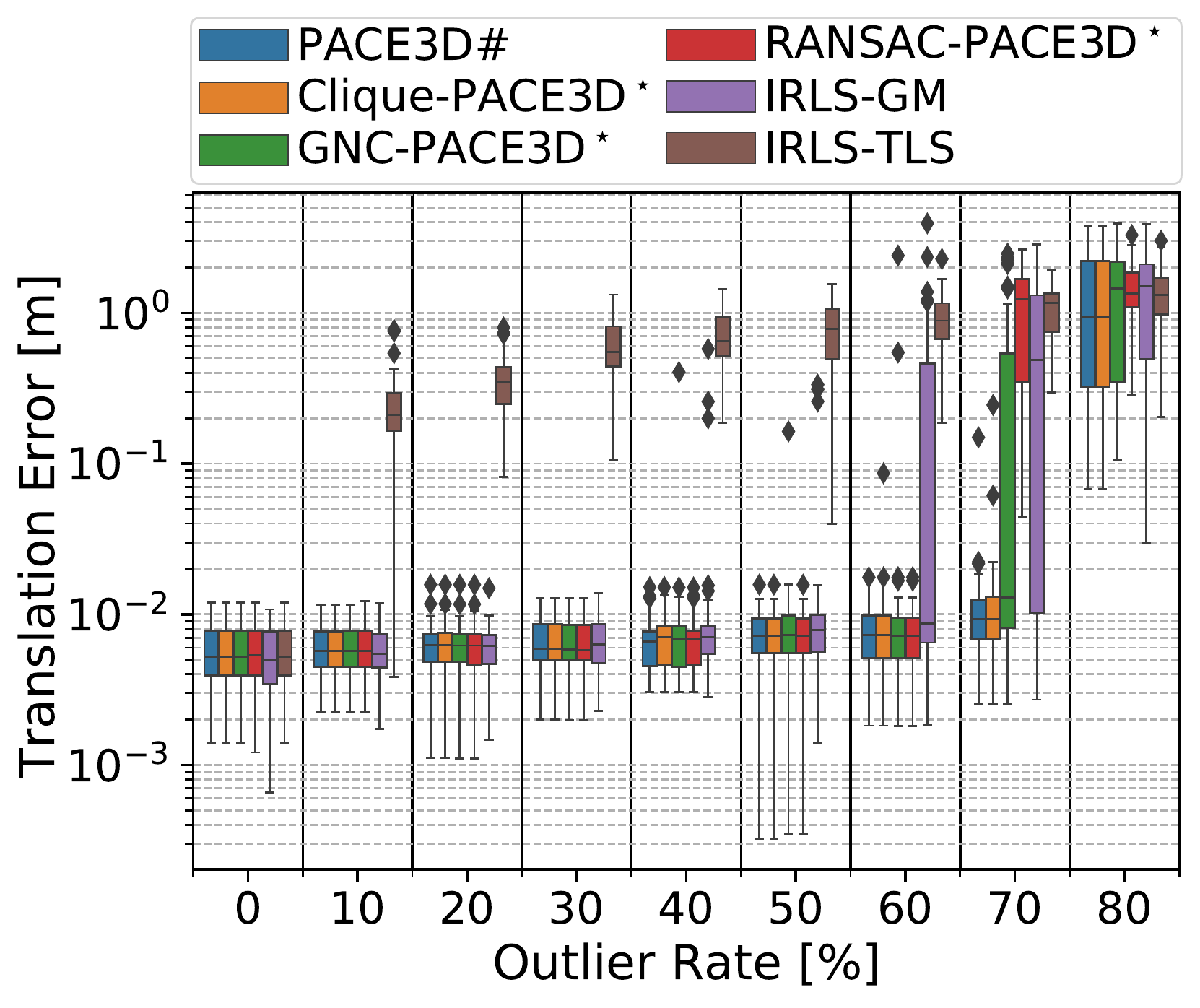}
			\end{minipage}
		&   \myhspace
			\begin{minipage}{\mpwfour}%
			\centering%
			\includegraphics[width=\columnwidth]{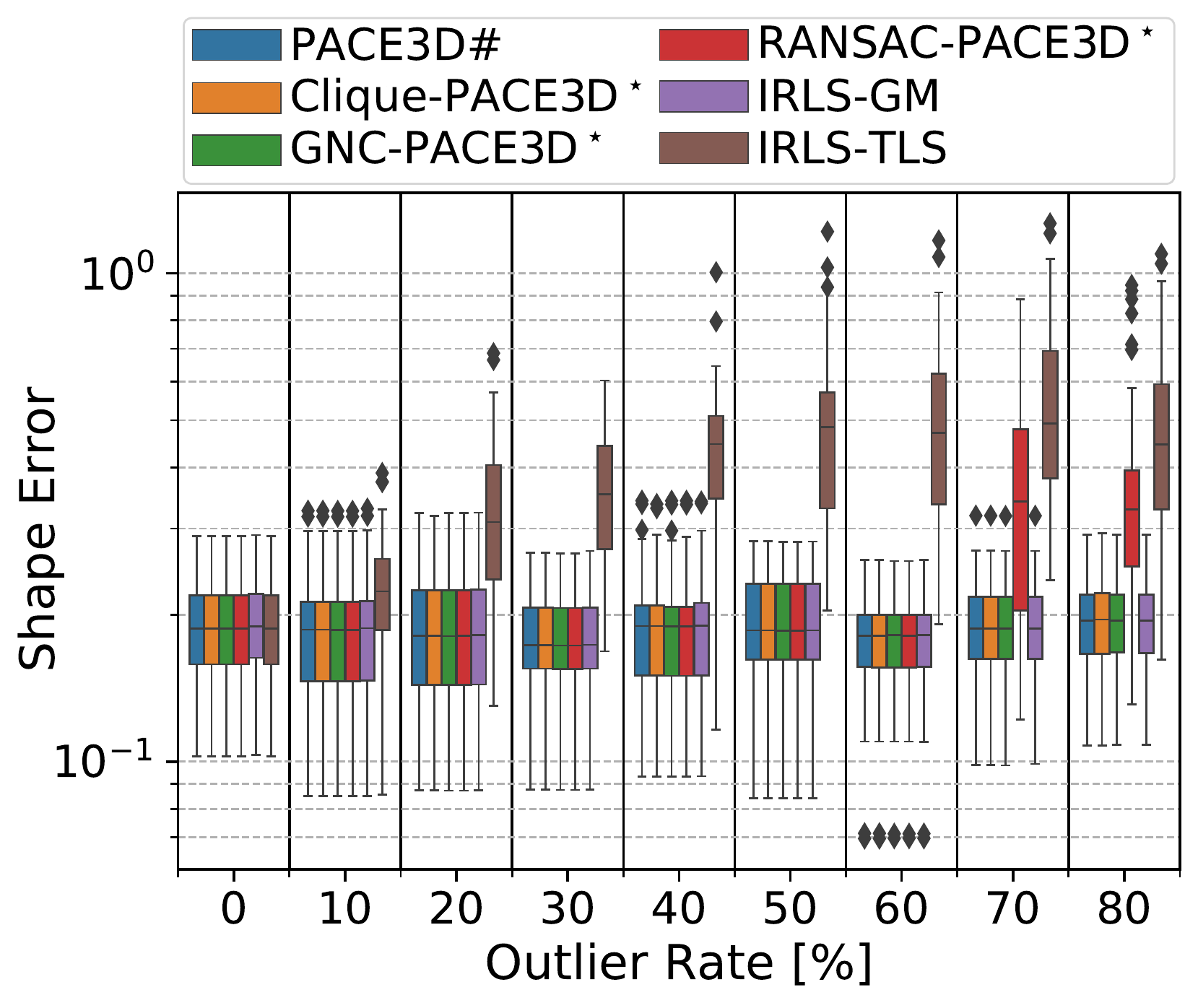}
			\end{minipage}
		&   \myhspace
			\begin{minipage}{\mpwfour}%
			\centering%
			\includegraphics[width=\columnwidth]{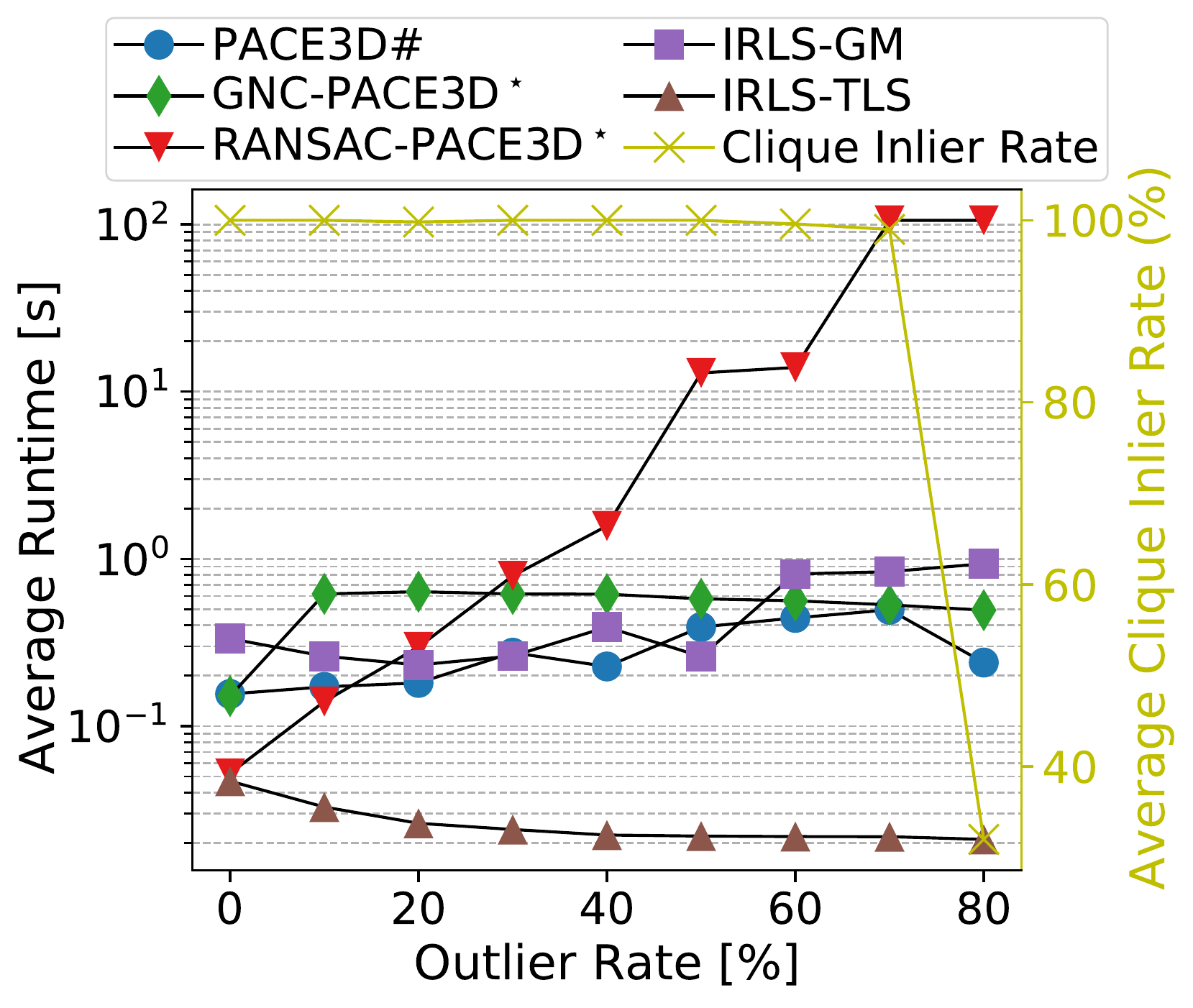}
			\end{minipage}
		\\
		\multicolumn{4}{c}{\smaller (c) Robustness of \PACErobustThree against increasing outliers on the \emph{car} category in the \pascal~dataset~\cite{Xiang2014WACV-PASCAL+}: $N=12$, $K=9$.}
		\\
		\myhspace
			\begin{minipage}{\mpwfour}%
			\centering%
			\includegraphics[width=\columnwidth]{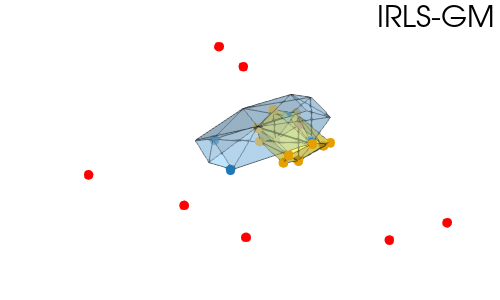}
			\end{minipage}
		&   \myhspace
			\begin{minipage}{\mpwfour}%
			\centering%
			\includegraphics[width=\columnwidth]{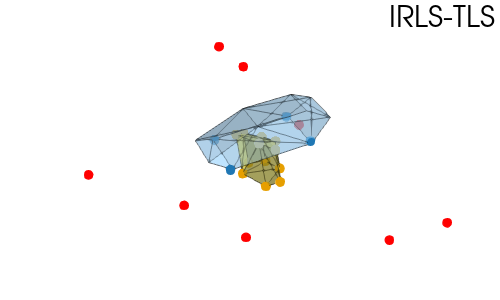}
			\end{minipage}
		&   \myhspace
			\begin{minipage}{\mpwfour}%
			\centering%
			\includegraphics[width=\columnwidth]{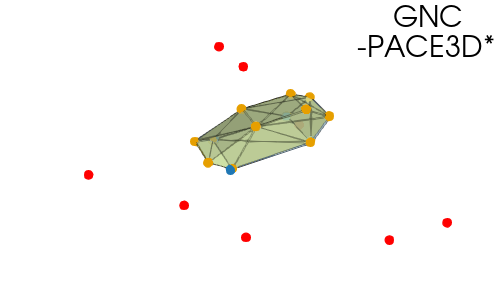}
			\end{minipage}
		&   \myhspace
			\begin{minipage}{\mpwfour}%
			\centering%
			\includegraphics[width=\columnwidth]{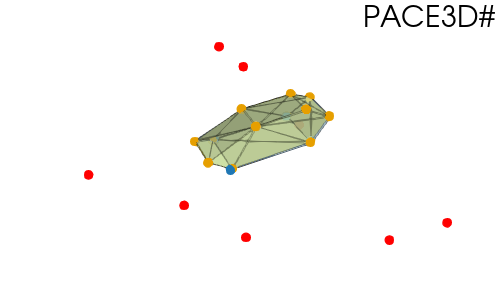}
			\end{minipage}
		\\
	  \multicolumn{4}{c}{\smaller (d) Qualitative results of \irlsgm, \irlstls, \gnc, and \PACErobustThree on a \pascal  instance with 70\% outlier rate. }
	\end{tabular}
	\end{minipage}
	\caption{Performance of \PACEThree~and \PACErobustThree~compared with baselines in simulated experiments. (a) \PACEThree~compared with alternating minimization (\altern) on synthetic outlier-free data with $N=100$ and $K$ increasing from $10$ to $2000$; (b) \PACErobustThree and variants (\cliquePACEThree~and \gnc), compared with two variants of iterative reweighted least squares (\irlsgm~and \irlstls)~\cite{MacTavish15crv-robustEstimation} on synthetic outlier-contaminated data with $N=100$, $K=10$, and outlier rates up to $93\%$; (c) same as (b) but using the \emph{car} category CAD models from the \pascal~dataset~\cite{Xiang2014WACV-PASCAL+}, with $N=12$, $K=9$, and outlier rates up to $80\%$. Each boxplot and lineplot summarizes $50$ Monte Carlo random runs.
	  (d) Qualitative results of \irlsgm, \irlstls, \gnc, and \PACErobustThree on a \pascal instance with 70\% outlier rate. Blue meshes represent the ground-truth shape, and yellow meshes represent the pose and shape estimated by each model. %
	  Red points represent outliers.
	  In this case, both \gnc and \PACErobustThree succeeded, while \irlsgm and \irlstls failed.
	\label{fig:ablations-3d}}
    \vspace{-5mm}
	\end{center}
\end{figure*}

\subsection{Optimality and Robustness of \PACEThree and \PACErobustThree}
\label{sec:exp-optimality-robustness-3d}

\myParagraph{Optimality of \PACEThree} To evaluate the performance of \PACEThree in solving the outlier-free problem~\eqref{eq:probOutFree-3D3Dcatlevel}, we randomly simulate $K$ shape models $\calB_k$ whose points $\basis{k}{i}$'s
are drawn from an i.i.d. Gaussian distribution $\calN(\zero,\eye_3)$. We sample
shape parameters $\vc$ uniformly at random in $[0,1]^\nrShapes$,
and normalize $\vc$ such that $\ones\tran\vc = 1$. Then we
\JS{uniformly sample $\MR$ from \SOthree and $\vt$ from $\calN(\zero,\eye_3)$}
and generate the measurements $\measThree{i}$ according to the model~\eqref{eq:generativeModel-3d}, where the noise $\epsThree{i}$ follows $\calN(\zero,\sigma^2\eye_3)$ with standard deviation $\sigma = 0.01$. \JS{We fix $N=100$, and increase $K$ from $1$ up to $\num{2000}$, to stress-test the algorithms.}
We set the regularization factor $\lambda = \sqrt{K/N}$ so that larger regularization \LC{is imposed when the problem becomes more ill-posed.} We compare \PACEThree with a baseline approach based on alternating minimization~\cite{Lin14eccv-modelFitting,Gu06cvpr-faceAlignment,Ramakrishna12eccv-humanPose} (details in \supplementary{sec:app-alternationApproach}) that offers no optimality guarantees (label: \altern).

Fig.~\ref{fig:ablations-3d}(a) plots the statistics of rotation error (angular distance between estimated and ground-truth rotations), translation error, shape parameters error ($\ell_2$ distance between estimated and ground-truth translation/shape parameters), as well as average runtime and relative duality gap (see also {\supplementary{sec:app-sdp-relaxation-gap}} for a formal definition).
We make the following observations: (i) \PACEThree returns accurate pose and shape estimates up to $K=2000$, while \altern starts failing at $K=500$. (ii) Although \altern is faster than \PACEThree for small $K$ (\eg~$K<200$), \PACEThree is orders of magnitude faster than \altern for large $K$.  
\LC{\PACEThree solves a fixed-size SDP regardless of $K$ and the slight runtime increase  is due to inversion of the dense matrix in~\eqref{eq:inverseofdensematrix}.} (iii) The relaxation~\eqref{eq:categoryQCQPrelax} is empirically tight (\JS{relative duality gap $< 10^{-5}$}), certifying global optimality of the solution returned by \PACEThree.
\JS{
  In \supplementary{sec:app-experiments-pace-3d}, we show extra results with different noise levels and $N$.}

\myParagraph{Robustness of \PACErobustThree}
To test the robustness of \PACErobustThree on outlier-contaminated data, we follow the same data generation protocol as before, except that (i) when generating the CAD models, we follow a more realistic active shape model~\cite{Cootes95cviu} where we first generate a mean shape $\calB$ whose points $\vb_{i}$'s are i.i.d. Gaussian $\calN(\zero,\eye_3)$, and then each CAD model is generated from the mean shape by: $\basis{k}{i} = \vb_{i} + \vv_{i}$, where $\vv_{i}$ follows $\calN(\zero,r^2 \eye_3)$ and represents the \emph{intra-class variation} of semantic keypoints with variation  radius $r$; (ii) we replace a fraction of the measurements $\measThree{i}$ with arbitrary 3D points sampled according to $\calN(\zero,\eye_3)$ and violating the generative model~\eqref{eq:generativeModel-3d}. We compare \PACErobustThree with two variants: \cliquePACEThree
(where, after pruning outliers using maximum clique, \PACEThree is applied \emph{without} \gnc) and \PACEThreeGNC (where \gnc is applied \emph{without} any outlier pruning), as well as two variants of the popular \emph{iterative reweighted least squares} method: \irlstls and \irlsgm, where \tls and \gm denote the truncated least squares cost function and the Geman-McClure cost function~\cite{MacTavish15crv-robustEstimation}. 
\JS{
Moreover, we compare against \ransacPACEThree, a 5-point \ransac scheme.
We use \PACEThree in the inner iterations of \PACErobustThree, \PACEThreeGNC, \irlstls, \irlsgm, and \ransacPACEThree. We set $\inthrThree = 0.05$ for outlier pruning in \gnc, \irlstls, and \irlsgm.
}
\JSTwo{
  \ransacPACEThree uses a maximum of 5000 iterations, whereas both
  \irlstls and \irlsgm use a maximum of $10^{3}$ iterations.
}
Fig.~\ref{fig:ablations-3d}(b) plots the results under increasing outlier rates up to $93\%$ when $N=100$, $K=10$, and $r=0.1$. We make the following observations:
(i) \irlstls quickly fails (at $10\%$ outlier rate) due to the highly non-convex nature of the \tls cost, \LC{while \irlsgm starts breaking at {$60\%$ outliers}.}
(ii) \PACEThreeGNC alone already \toCheckTwo{outperforms} \irlstls and \irlsgm and is robust to $60\%$ outliers.
\LC{(iii) \ransacPACEThree is also robust to $60\%$ outliers.}
(iv) With our maximum-clique outlier pruning, the robustness of \PACErobustThree is boosted to $93\%$, a level that is comparable to cases when the shapes are known (\eg~\cite{Yang20tro-teaser}).
In addition, outlier pruning speeds up the convergence of \PACEThreeGNC (\cf~{the runtime plot for \gnc and \PACErobustThree in Fig.~\ref{fig:ablations-3d}(b))}.
(v) Even without \gnc, the outlier pruning is so effective that \PACEThree alone is able to succeed with up to $92\%$ outliers, albeit the estimates are typically less accurate than \PACErobustThree.
In fact, looking at the clique inlier rate plot ({yellow} lineplot in Fig.~\ref{fig:ablations-3d}(b)), the reader sees that the set of measurements after maximum clique pruning is almost free of outliers, explaining the performance of \cliquePACEThree. In \supplementary{sec:app-experiments-pace-3d}, we show extra results
for $r=0.2$
which further confirm \PACErobustThree's robustness with up to $90\%$ outliers.

\myParagraph{Robustness on \pascal} For a simulation setup that is closer to realistic scenarios, we use the CAD models from the \emph{car} category in the \pascal dataset~\cite{Xiang2014WACV-PASCAL+}, which contains $K=9$ CAD models with $N=12$ semantic keypoints. We randomly sample $(\MR,\vt,\vc)$ and add noise and outliers as before, and compare the performance of \PACErobustThree with other baselines, as shown in Fig.~\ref{fig:ablations-3d}(c).
The dominance of \PACErobustThree over other baselines, and the effectiveness of outlier pruning is clearly seen across the plots. \PACErobustThree is robust to $70\%$ outliers (see Fig.~\ref{fig:ablations-3d}(d) for a qualitative example), while other baselines break at a much lower outlier rate. Note that at $80\%$ outlier rate, there are only two inlier semantic keypoints, making it pathological to estimate shape and pose.

\begin{figure}[t!]
  \centering
  \includegraphics[width=0.75\columnwidth]{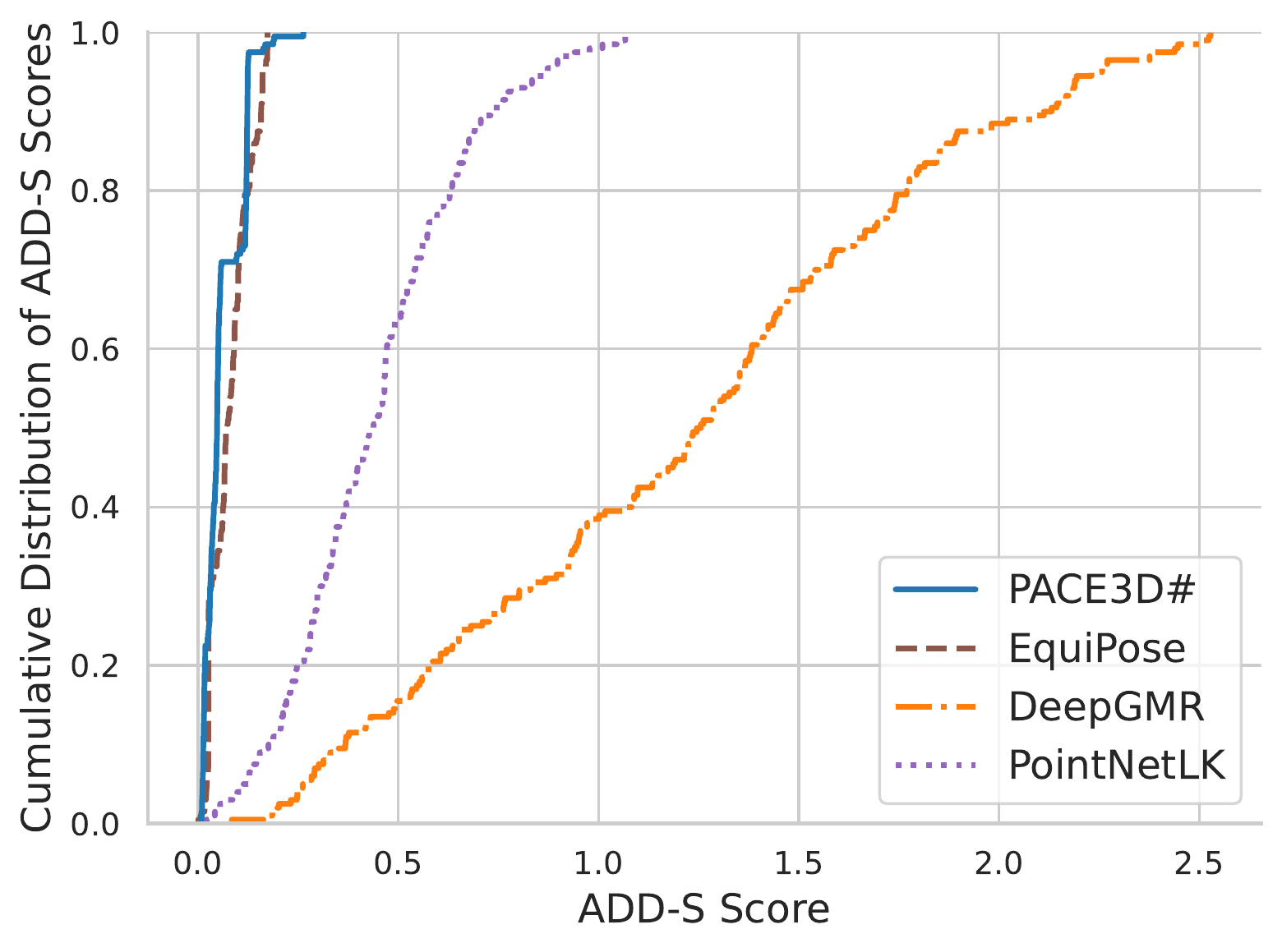}
  \vspace{-3mm}
  \caption{Cumulative distribution of ADD-S score for \PACErobustThree, EquiPose~\cite{Li21nips-LeveragingSE}, PointNetLK~\cite{Li21cvpr-PointNetLKRevisited}, and DeepGMR~\cite{Yuan20arxiv-DeepGMRLearning} in the \keypointnet experiment,
    averaged across 4 categories: airplane, bottle, car, and chair.\label{fig:shapenet-equipose-objs}}
\end{figure}

\myParagraph{Performance on \keypointnet} \JS{We conduct experiments on \keypointnet, a large-scale 3D keypoint dataset built from ShapeNetCore~\cite{Chang15arxiv-shapenet}, containing 8329 3D models from 16 object categories~\cite{You20cvpr-KeypointNetLargescale}.
  For each object category, we select one object and render depth point clouds using Open3D~\cite{Zhou18arxiv-open3D}.
  We apply a random translation bounded within $[0, 1]^{3}$, normalized by the objects' diameters, and apply a uniformly sampled rotation to the object.
  We generate 5000 samples as the training set, and 50 samples for test and validation sets each.
  We compare the following methods:
  (i) \PACErobustThree, where we first use a network based on PointTransformer~\cite{Zhao21iccv-PointTransformer} to detect keypoints, and then run \PACErobustThree,
  (ii) EquiPose~\cite{Li21nips-LeveragingSE}, (iii) PointNetLK~\cite{Li21cvpr-PointNetLKRevisited}, and (iv) DeepGMR~\cite{Yuan20arxiv-DeepGMRLearning}.
  Pretrained models provided by the authors (and trained on the same \keypointnet dataset) are used for EquiPose, PointNetLK, and DeepGMR.
  Fig.~\ref{fig:shapenet-equipose-objs} shows the cumulative distribution of ADD-S score~\cite{Xiang17rss-posecnn} over four categories (airplane, bottle, car, and chair), for which pretrained models are available for EquiPose.
  \PACErobustThree outperforms DeepGMR and PointNetLK by a large margin. Only EquiPose achieves comparable performance with \PACErobustThree.
  Additional experimental details and results on all \keypointnet categories are given in~\supplementary{sec:app-experiments-keypointnet}.}

\subsection{Optimality and Robustness of \PACETwo and \PACErobustTwo}
\label{sec:exp-optimality-robustness-2d}

\renewcommand{\mpwfour}{4.6cm}
\renewcommand{\myhspace}{\hspace{-3.5mm}}
\begin{figure*}[h!]
	\begin{center}
	\begin{minipage}{\textwidth}
	\begin{tabular}{cccc}%
		\myhspace \hspace{-3mm}
			\begin{minipage}{\mpwfour}%
			\centering%
			\includegraphics[width=\columnwidth]{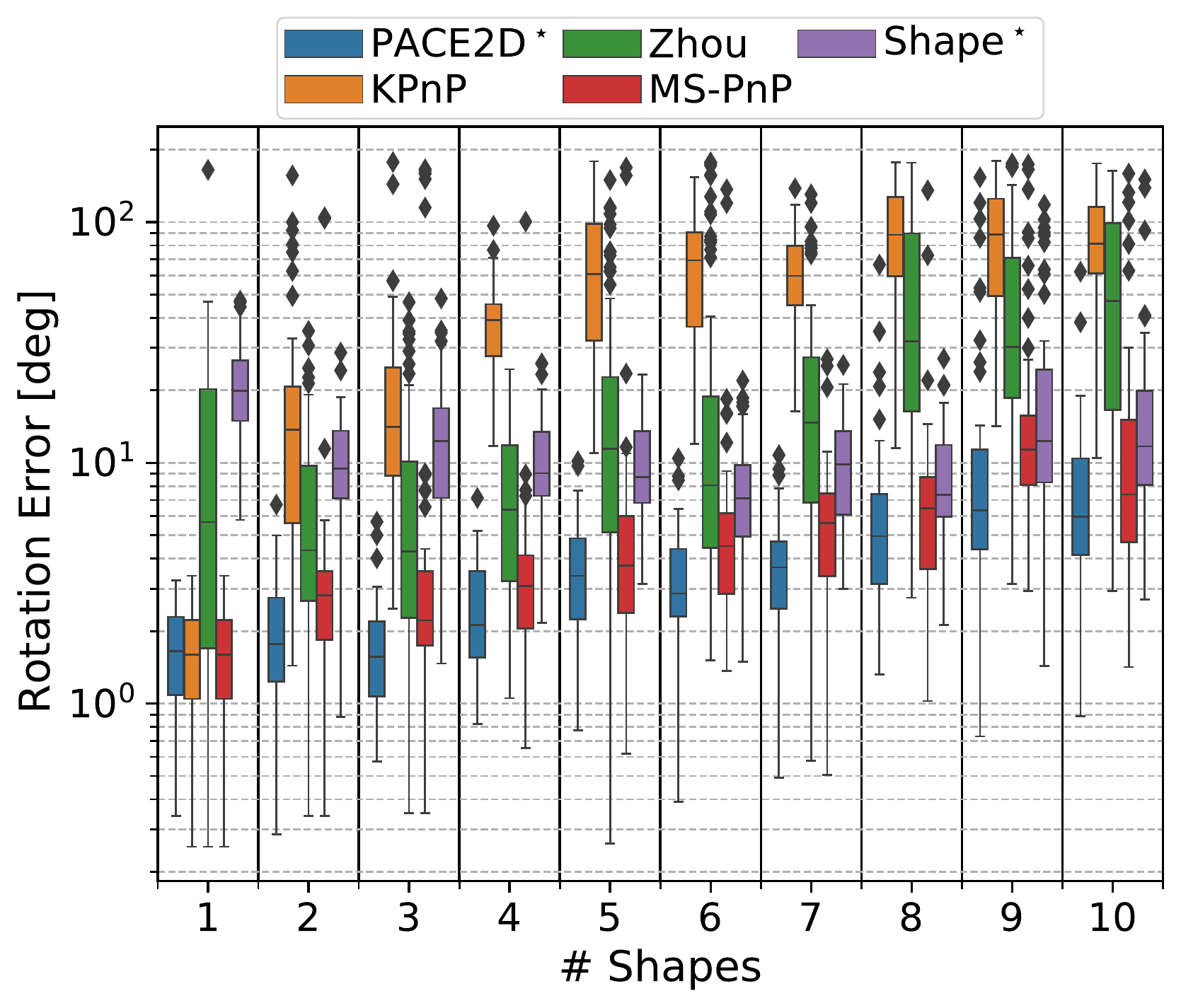}
			\end{minipage}
		&   \myhspace
			\begin{minipage}{\mpwfour}%
			\centering%
			\includegraphics[width=\columnwidth]{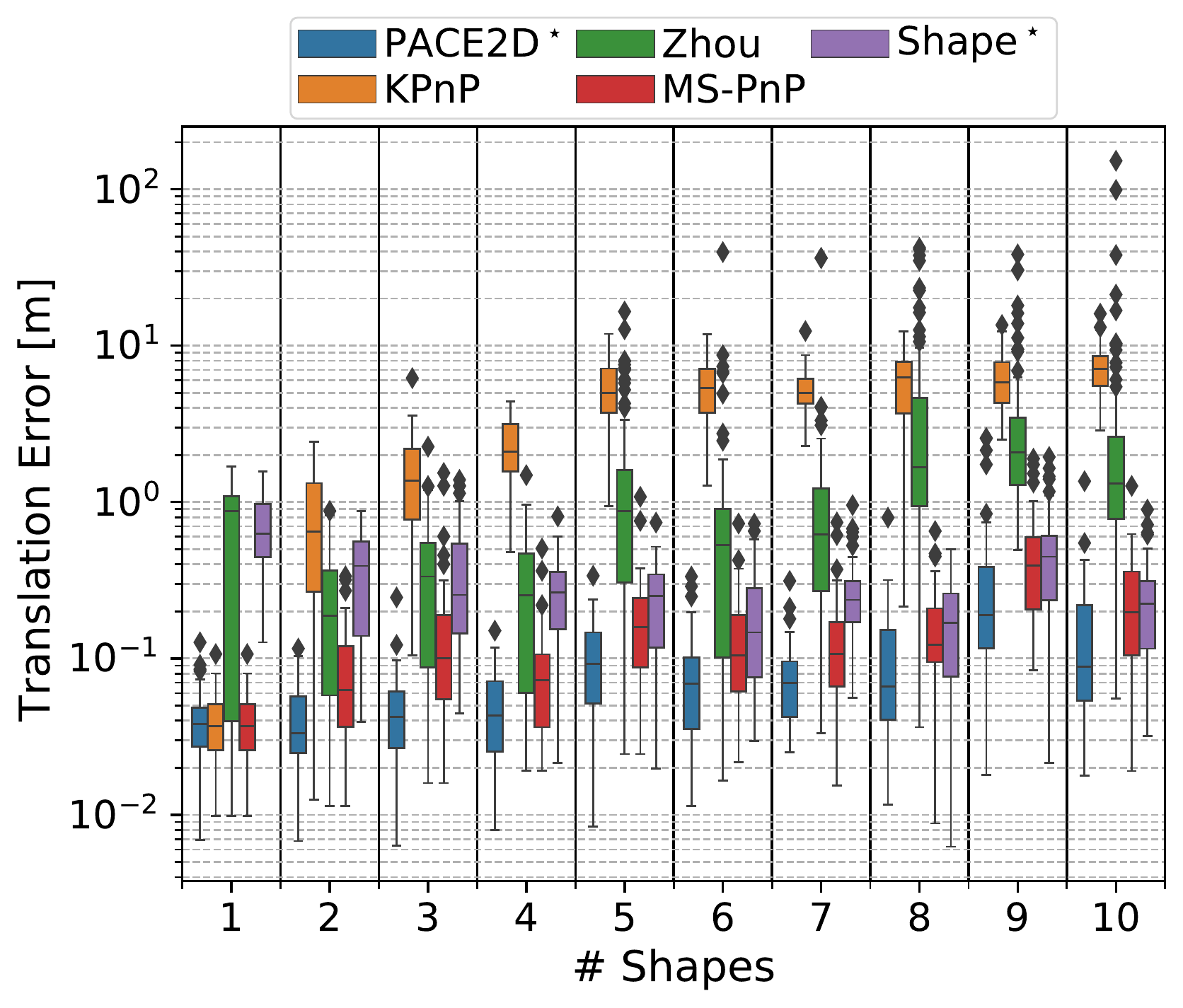}
			\end{minipage}
		&   \myhspace
			\begin{minipage}{\mpwfour}%
			\centering%
			\includegraphics[width=\columnwidth]{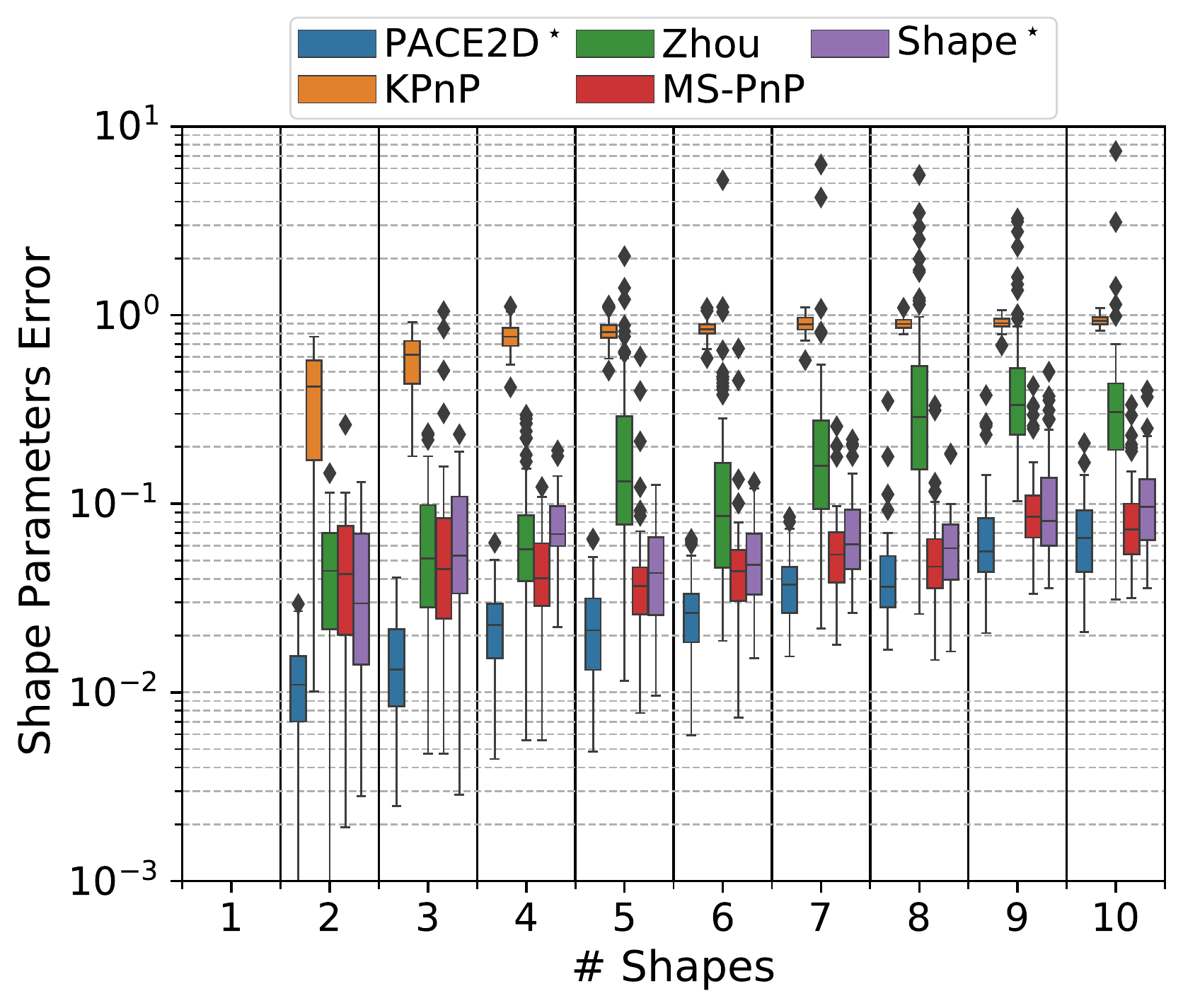}
			\end{minipage}
		&   \myhspace
			\begin{minipage}{\mpwfour}%
			\centering%
			\includegraphics[width=\columnwidth]{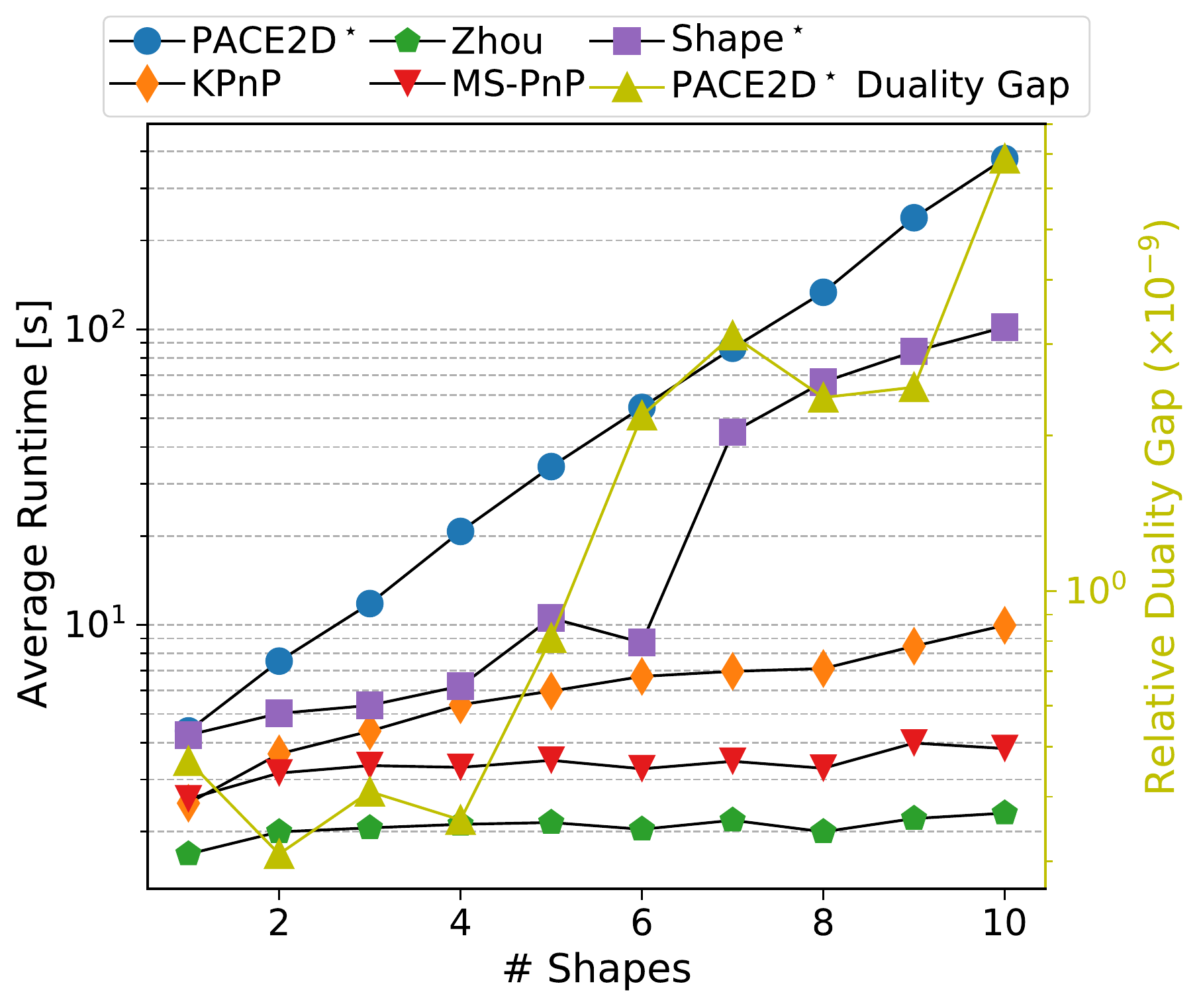}
			\end{minipage}
		\\
		\multicolumn{4}{c}{\smaller (a) Performance of \PACETwoLZero and \PACETwoLTwo on outlier-free synthetic data: $N=8$; $\vc$ sampled from $\Delta_{K}$ uniformly at random. \vspace{1mm}}
		\\
		\myhspace \hspace{-3mm}
			\begin{minipage}{\mpwfour}%
			\centering%
			\includegraphics[width=\columnwidth]{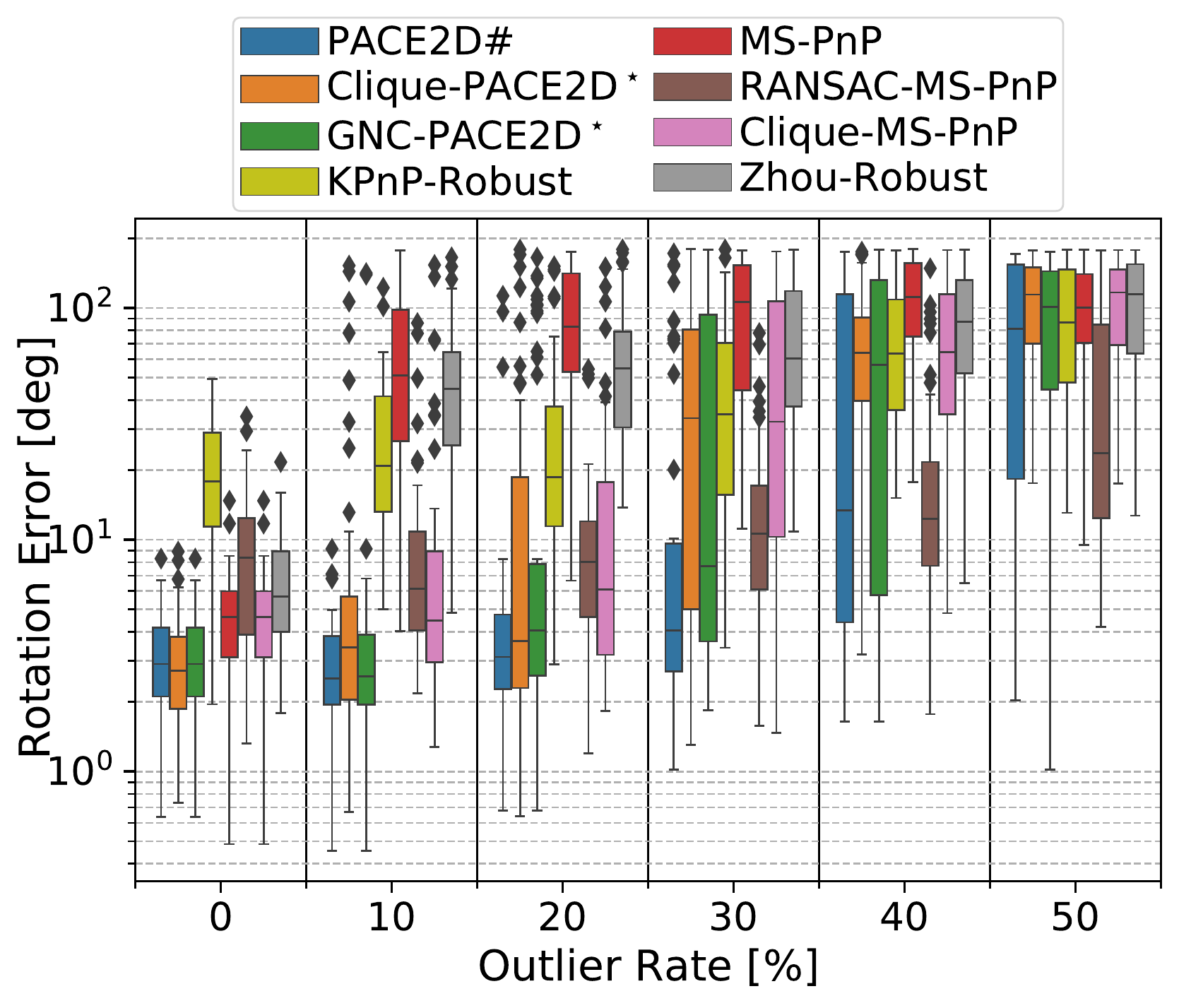}
			\end{minipage}
		&   \myhspace
			\begin{minipage}{\mpwfour}%
			\centering%
			\includegraphics[width=\columnwidth]{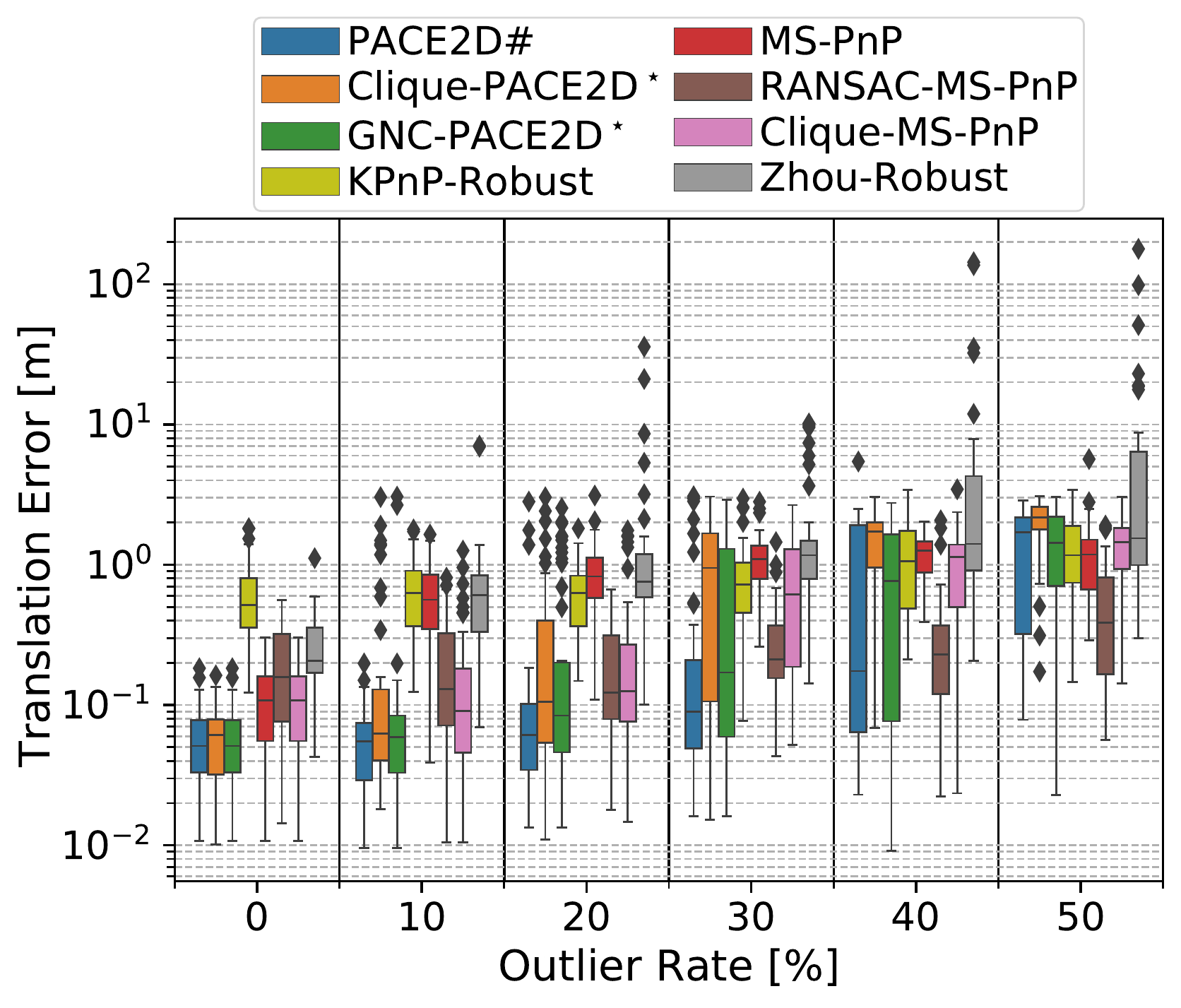}
			\end{minipage}
		&   \myhspace
			\begin{minipage}{\mpwfour}%
			\centering%
			\includegraphics[width=\columnwidth]{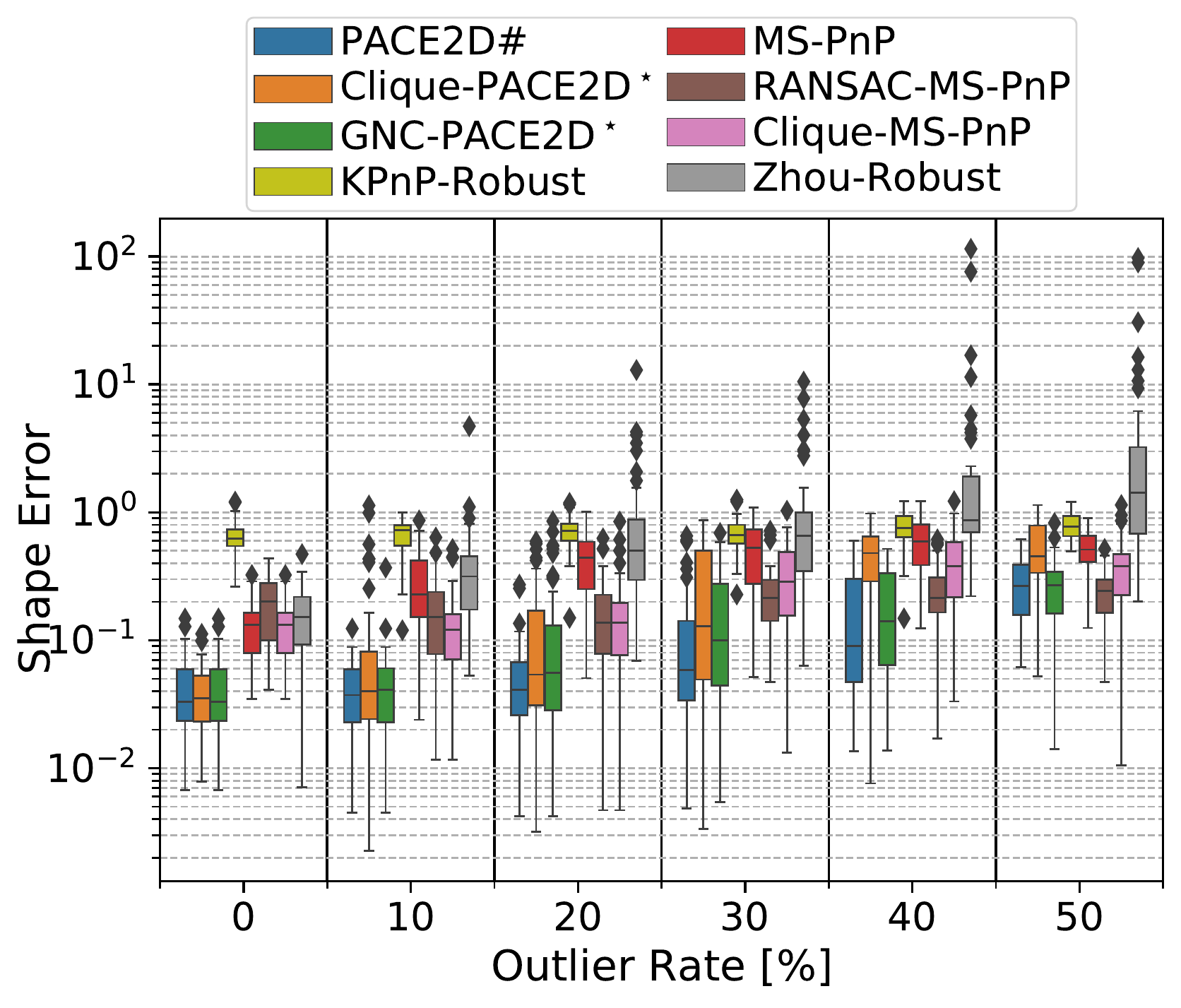}
			\end{minipage}
		&   \myhspace
			\begin{minipage}{\mpwfour}%
			\centering%
			\includegraphics[width=\columnwidth]{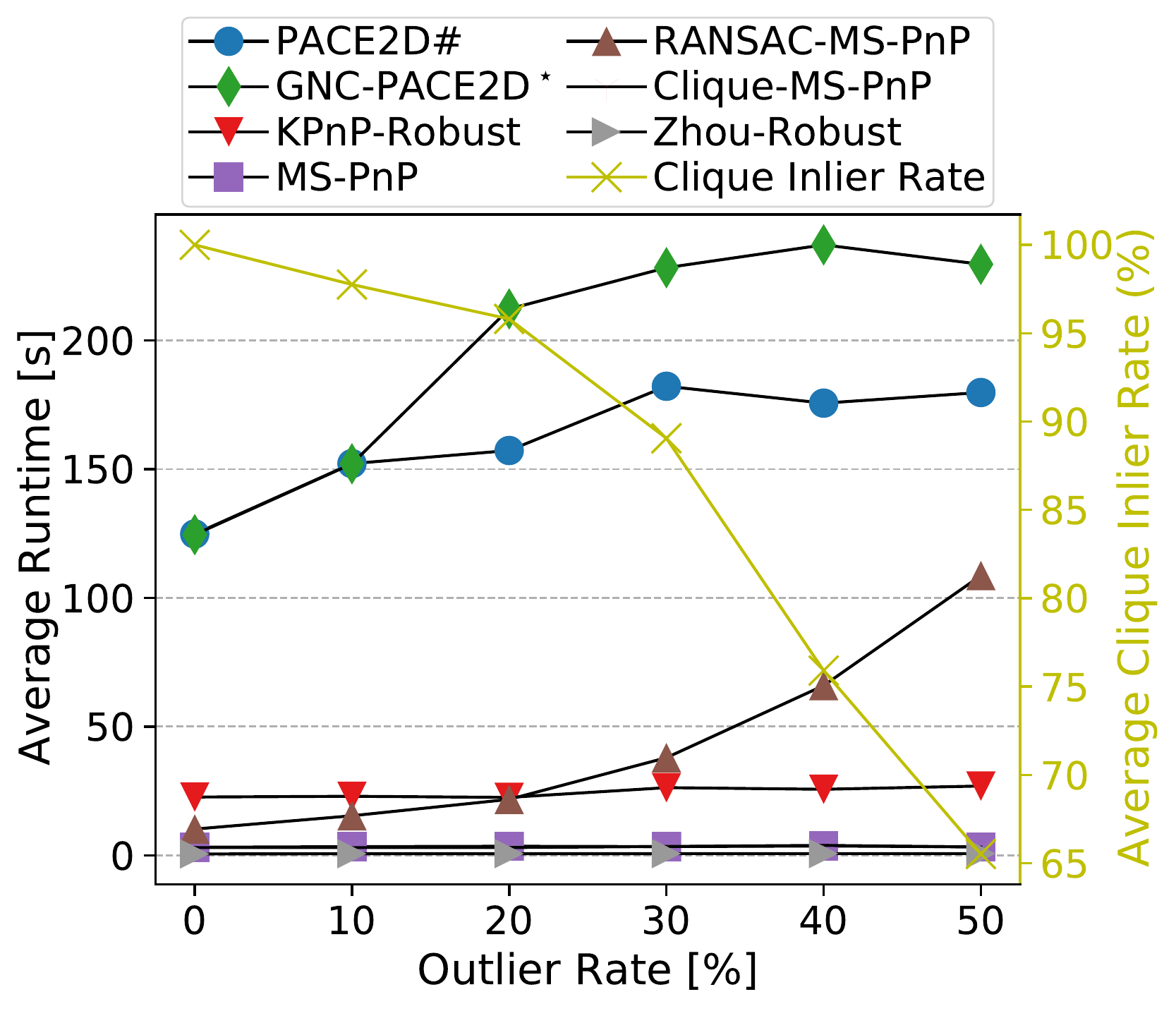}
			\end{minipage}
		\\
		\multicolumn{4}{c}{\smaller (b) Robustness of \PACErobustTwoLTwo against increasing outliers on synthetic data: $N=10$, $K=3$; $\vc$ sampled from $\Delta_{K}$ uniformly at random. \vspace{1mm}}
	  \\
		\myhspace \hspace{-3mm}
			\begin{minipage}{\mpwfour}%
			\centering%
			\includegraphics[width=\columnwidth]{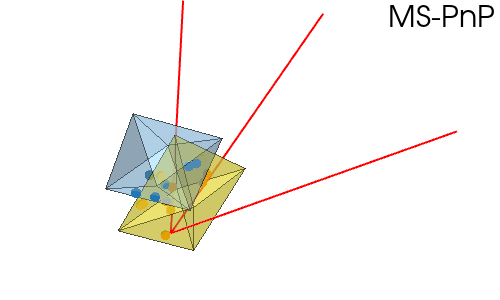}
			\end{minipage}
		&   \myhspace
			\begin{minipage}{\mpwfour}%
			\centering%
			\includegraphics[width=\columnwidth]{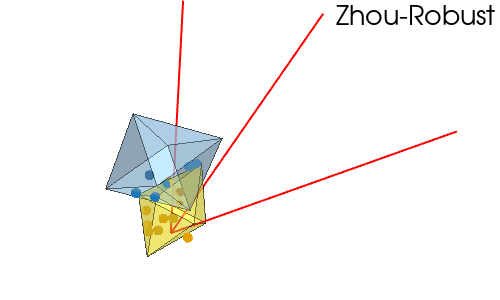}
			\end{minipage}
		&   \myhspace
			\begin{minipage}{\mpwfour}%
			\centering%
			\includegraphics[width=\columnwidth]{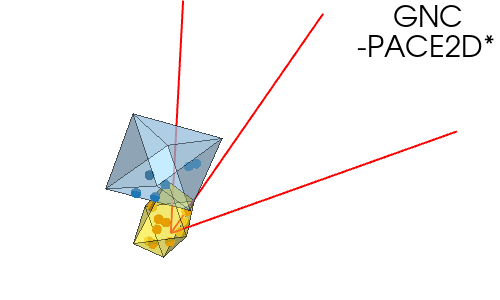}
			\end{minipage}
		&   \myhspace
			\begin{minipage}{\mpwfour}%
			\centering%
			\includegraphics[width=\columnwidth]{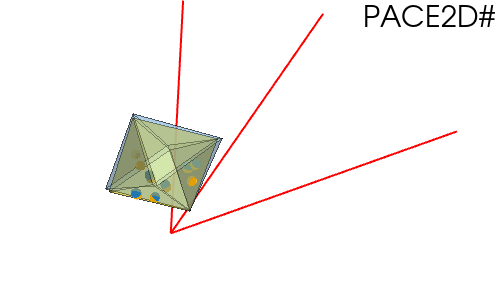}
			\end{minipage}
		\\
		\multicolumn{4}{c}{\smaller (c) Qualitative results of \meanPnP, \kostasRobust, \PACETwoGNC, and \PACErobustTwo on a test instance with 30\% outlier rate.}
	\end{tabular}
	\end{minipage}
	\caption{Performance of \PACETwoLTwo~and \PACErobustTwo~compared with baselines in simulated experiments.
	  (a) \PACETwoLTwo compared with \kpnp, \meanPnP, \kostas, and \shapestar on synthetic outlier-free data with varying number of shapes, where $\vc$ is sampled uniformly at random from $\Delta_{K}$.
	  (b) \PACErobustTwoLTwo and variants compared with \kpnpRobust, \meanPnP, \ransacMeanPnP, \cliqueMeanPnP, and \kostasRobust on synthetic outlier-contaminated data with varying outlier rates.
	  (c) Qualitative results of \meanPnP, \kostasRobust, \PACETwoGNC, and \PACErobustTwoLTwo on an instance with 30\% outlier rate. Blue meshes represent the ground-truth shape, and yellow meshes represent the pose and shape estimated by each model. Red rays indicate outliers (bearing vectors originated from the camera center).
	  In this case, \PACErobustTwoLTwo succeeds while the other methods fail.
	  \label{fig:ablations-2d}}
	\vspace{-6mm}
	\end{center}
\end{figure*}

\myParagraph{Optimality of \PACETwo}
To evaluate the performance of \PACETwo in solving the outlier-free problem~\eqref{eq:probOutFree-2D3Dcatlevel},
we randomly simulate $K$ shapes $\calB_k$ whose points $\basis{k}{i}$'s
are drawn from an i.i.d. Gaussian distribution $\calN(\zero,\eye_3)$.
\JS{
We sample $\vc$ uniformly at random in the probability simplex $\Delta_{\nrShapes}$.
We draw random poses $(\MR,\vt)$  %
such that the camera lies on a sphere with radius equals to $3$ centered at origin,
}
and generate the measurements $\measTwo{i}$ according to the model~\eqref{eq:generativeModel-2d},
where the noise $\epsTwo{i}$ follows $\calN(\zero,\sigma^2\eye_2)$ with standard deviation $\sigma = 0.01$.
We compare \PACETwoLTwo against  (i)~a baseline approach based on
a local solver optimizing~\eqref{eq:probOutFree-2D3Dcatlevel},
starting from an initial guess obtained by running EPnP~\cite{Lepetit09-epnp} on the mean shape (label: \meanPnP),
(ii)~a solver based on a convex relaxation using the weak perspective camera model~\cite{Zhou17pami-shapeEstimationConvex} (label: \kostas),
(iii) a solver based on a tighter relaxation with the weak perspective model~\cite{Yang20cvpr-shapeStar} (label: \shapestar),
\JS{
and (iv) a solver that solves a PnP problem for each of the $K$ shapes using SQPnP~\cite{Terzakis20eccv-sqpnp} and picks the one with the lowest cost (label: \kpnp), see Remark~\ref{remark:one-hot-c}.
}

\JS{
In Fig.~\ref{fig:ablations-2d}(a),
we report statistics for the rotation,
translation, and shape errors.
We also show the average runtime and relative duality gap. %
We observe that \PACETwoLTwo consistently outperforms all baselines.
Predictably, \kpnp fails when $K \geq 2$, as it assumes a one-hot $\vc$, while the ground truth shape is a generic point in the probability simplex.
In \supplementary{sec:app-experiments-pace-2d}, we show that in the case of one-hot $\vc$, \kpnp can indeed obtain good performance.
}
While \kostas and \shapestar perform similarly at $K=1$, \shapestar has lower errors at higher shape counts.
Both \kostas and \shapestar perform significantly worse than \PACETwoLTwo, as they use the weak perspective projection model
 to approximate the actual (fully perspective) camera.
The relative duality gap for \PACETwoLTwo stays below $10^{-8}$, indicating an empirically tight relaxation.

\myParagraph{Robustness of \PACErobustTwo}
To test the robustness of \PACErobustTwo, we use a different data generation procedure to enable the use of \robin.
We first generate $K$ octahedra, center-aligned at the origin, with
 vertices sampled component-wise in $[0.5, 2]$ m.
The use of octahedra (convex shapes with known face planar equations)
allows us to solve for sets of feasible winding orders ---following Corollary~\ref{cor:winding_orders_feasibility}--- using linear programs (see \supplementary{sec:app-winding-order-dictionary}).
\JS{
We sample
shape parameters $\vc$
uniformly at random in $[0,1]^\nrShapes$,
Then we draw random poses $(\MR,\vt)$ %
such that the resulting camera locations lie on a sphere centered at the origin with radius of $3$ m.
}
For each camera location,
we randomly sample $\basis{k}{i}, \; i = 1 \ldots N$ from each octahedron
such that $\lfloor r \cdot N \rfloor$ of them lie on the
weighted octahedron's visible faces where $r$ is the outlier ratio.
For the remaining $\basis{k}{i}$,
we sample them from the nonvisible faces of the octahedron.
We generate the measurements $\measTwo{i}$ according to eq.~\eqref{eq:generativeModel-2d},
where the noise $\epsTwo{i}$ follows $\calN(\zero,\sigma^2\eye_2)$ with $\sigma = 0.01$.
For the $\measTwo{i}$ generated by nonvisible $\basis{k}{i}$'s,
we replace their noise term $\epsTwo{i}$
with arbitrary 2D points violating the generative model~\eqref{eq:generativeModel-2d}.
The regularization factor $\lambda$ is set to $0.01$.

\JS{
We consider \PACErobustTwoLTwo, as well as two variants:
\PACETwoGNCLTwo (where \gnc is applied to \PACETwoLTwo without \robin), and
\cliquePACETwoLTwo (where \PACETwoLTwo is applied after \robin without \gnc).
For \PACErobustTwoLTwo and \PACETwoGNCLTwo, inlier threshold $\inthrTwo$ is set to be $0.05$.
We also compare against
(i) \kostasRobust, which is a robust version of \kostas's solver from~\cite{Zhou17pami-shapeEstimationConvex}, (ii) \meanPnP,
(iii) \ransacMeanPnP, where we wrap \meanPnP in a 4-point \ransac loop with inlier threshold of $0.05$, (iv) \cliqueMeanPnP, where we apply \robin before \meanPnP, and (v) \kpnpRobust, where we apply \robin and \gnc before \kpnp.
}
Fig.~\ref{fig:ablations-2d}(b) plots the results under increasing outlier rates. %
\JS{
\PACErobustTwoLTwo is robust to $10$\% of outliers, and achieves low median errors for outlier rates below $30$\%, while \PACETwoGNCLTwo, \meanPnP, \ransacMeanPnP, and \kostasRobust already exhibit large median errors at $10$\% outlier rates.
}
Interestingly, at around $10$\%, \cliquePACETwoLTwo remains robust while \PACETwoGNCLTwo starts to show failures.
 This remarks the effectiveness of \robin in filtering out outliers,
also shown in the clique inlier rate plot in Fig.~\ref{fig:ablations-2d}(c).
Similar to \PACErobustThree, \robin improves the convergence rate of \gnc; see runtime curves of \PACErobustTwoLTwo and \PACETwoGNCLTwo in Fig.~\ref{fig:ablations-2d} (c).
\JS{
  \kpnpRobust fails to obtain low median errors due to its assumption of one-hot $\vc$.
\ransacMeanPnP, while degrading more gracefully at high outlier rates, is unable to achieve low median errors comparable to \PACErobustTwoLTwo.
Its inner 4-pt \meanPnP solver produces estimates that have large residuals for inliers (note the discrepancy between \meanPnP and \ransacMeanPnP at $0$\% outlier rate), hence affecting its inlier set estimation.
}
Fig.~\ref{fig:ablations-2d}(d) reports qualitative results on an simulated instance with 30\% outlier rate. %
These results underline that 2D-3D category-level perception is a much harder problem compared to its 3D-3D counterpart.
Our algorithms, while being competitive against baselines, are still slow and only robust to a small fraction of outliers.
\JS{
One may consider using \ransac with \PACETwo; unfortunately the slow runtime of \PACETwo prohibits such implementation: a mere 100 iterations of \PACETwo at $K=5$, $N=8$ takes around 50 minutes.
This highlights another benefit of \robin: it does not require solving the underlying problem, hence it can improve robustness at a significant runtime advantage, more so if the underlying solver is slow.
}
\JS{
  We also report extra results in \supplementary{sec:app-experiments-pace-2d}.}

\vspace{-3mm}
\subsection{Vehicle Pose and Shape Estimation on \apollo}
\label{sec:exp-apollo}

\myParagraph{Setup and Baselines}
We evaluate \PACErobustThree and \PACErobustTwo on the \apollo dataset~\cite{Wang19pami-apolloscape,Song19-apollocar3d}.
The \apollo self-driving dataset is a large collection of multi-modal data collected in four different cities in China under varying lighting and road conditions~\cite{Wang19pami-apolloscape}.
For our experiments, we specifically use the subset of \apollo named \apolloCar.
\apolloCar consists of high-resolution (3384 $\times$ 2710) images taken from the main \apollo dataset,
with additional 2D annotations of semantic keypoints, ground truth poses, and 3D CAD models of car instances in each frame.
The dataset contains a total of 5277 images, with an average of 11.7 cars per image, and a total of 79 ground-truth CAD models~\cite{Song19-apollocar3d}.
For each car, a total of 66 semantic keypoints were labeled on 2D images.
In addition, the ground truth CAD model (out of the 79 models) is provided for each vehicle.
Note this corresponds to having a one-hot vector for the ground-truth $\vc$ in eqs. \eqref{eq:generativeModel-3d} and \eqref{eq:generativeModel-2d}.

\begin{figure}[ht!]
    \centering%
    \includegraphics[width=0.8\columnwidth]{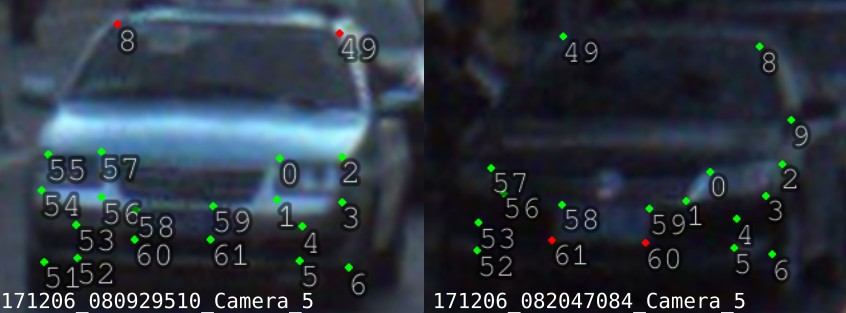}
	\caption{
      We pass ground-truth annotated keypoints to \robin, using a winding order dictionary generated from ray tracing.
      Green dots represent keypoints inliers accepted by \robin. Red dots represent outliers rejected by \robin.
      In these four examples, \robin correctly rejects mislabelled keypoints.
      Left: \#8 and \#49 are switched (with respect to the CAD models). Right: \#61 and \#60 are switched.
      See more examples in \supplementary{sec:apollo-gt-mislabel}.
      \label{fig:apolloscape-wrong-gt}\vspace{0mm}}

\end{figure}

\JS{
We compare \PACErobustThree and \PACErobustTwoLTwo against DeepMANTA~\cite{Chabot17-deepMANTA}, 3D-RCNN~\cite{Kundu18-3dRCNN}, and GSNet~\cite{Ke20-gsnet}, three recent state of the art methods for 3D vehicle pose estimation.
}
For our experiments, we use the official splits of the \apolloCar dataset.
Namely, we use the validation split (200 images) for all the quantitative experiments shown below, consistent with the evaluation setups reported in other baseline methods.
We use the 2D semantic keypoints extracted by GSNet~\cite{Ke20-gsnet} as measurements for \PACErobustTwo.
We use the pretrained weights from~\cite{Ke20-gsnet} \toCheckTwo{and reject keypoints with confidence less than $0.05$}.
\JS{For \PACErobustThree and \irlsgm, we additionally retrieve the corresponding depth from the depth images provided by \apollo for each 2D semantic keypoint.}
For \PACErobustTwo, we
\JS{
  construct the dictionary of feasible winding orders by recording the observed winding orders for all keypoint triplets in the training set,
  as well as by performing ray tracing to keypoints in a volume surrounding each CAD model
}
(see \supplementary{sec:app-winding-order-dictionary}).
Notably, when we applied \robin on the ground-truth annotations, we discovered multiple mislabeled keypoints (see Fig.~\ref{fig:apolloscape-wrong-gt} and \supplementary{sec:apollo-gt-mislabel}).
\JS{This suggests \robin may also be helpful in terms of verifying datasets.}

While \apolloCar provides 2D semantic keypoint annotations, it does not provide the corresponding 3D keypoint annotations on the CAD models.
\JS{Hence we manually labeled the 66 3D semantic keypoints on the 79 models, and provide them as the shape library to \PACErobustThree.}
For \PACErobustTwo, we instead select 3 random models, including the ground truth model, as a shape library,
since using the entire set leads to prohibitive runtime.
We use $\lambda = 0.5$ and \toCheckTwo{$\inthrThree = 0.15$} in \PACErobustThree,
and $\lambda = 0.001$ and $\inthrTwo = 0.01$ in \PACErobustTwo. \irlsgm uses an inlier threshold of $0.15$, same as \PACErobustThree.

\begin{table}[t!]
  \centering
    \begin{minipage}[b]{0.99\linewidth}
    \centering
      \begin{tabular}{llllrrr}
      \toprule
                            & \multicolumn{3}{c}{A3DP-Rel $\uparrow$}    & \multicolumn{3}{c}{A3DP-Abs $\uparrow$}          \\
      \midrule
                            & mean & c-l  & c-s               & mean & c-l & c-s \\
      \midrule
      DeepMANTA~\cite{Chabot17-deepMANTA}  & 16.0 & 23.8 & 19.8              & \secondBest{20.1} & 30.7          & \secondBest{23.8}            \\
      3D-RCNN~\cite{Kundu18-3dRCNN}        & 10.8 & 17.8 & 11.9              & 16.4 & 29.7          & 19.8            \\
      GSNet~\cite{Ke20-gsnet}              & \secondBest{20.2} & \best{40.5} & \secondBest{19.9}              & 18.9 & \best{37.4} & 18.4            \\
      \irlsgm    & \JS{13.6} & \JS{20.7} &  \JS{17.4}                         &
      \JS{11.9}   & \JS{20.0}  & \JS{16.7}    \\
      \PACErobustTwo    & \JS{11.0} & \JS{26.0} &  \JS{17.9}                         &
      \JS{7.2}   & \JS{17.5}  & \JS{11.3}    \\
      \PACErobustThree    & \best{28.5} & \secondBest{37.4} &  \best{35.8}                         &
      \best{24.0}   & \secondBest{36.4}  & \best{33.3}    \\
      \bottomrule
      \end{tabular}%
      \end{minipage} %
      \vspace{-1mm}
      \caption{Evaluation of \nameRobust on \apollo. Results for DeepMANTA,3D-RCNN, and GSNet are taken from~\cite{Ke20-gsnet}. 
      The best result for each column is highlighted in boldface. \label{tab:apollo-stats}\vspace{-2mm}}   
\end{table}

\myParagraph{Results} %
Table~\ref{tab:apollo-stats} shows the performance of \PACErobustThree and \PACErobustTwo against various baselines (qualitative results can be found in \supplementary{sec:app-apollo-vis}).
We use two metrics called A3DP-Rel and A3DP-Abs (for both, the higher the better) following the same definitions in~\cite{Song19-apollocar3d}.
They are measures of precision with thresholds jointly considering translation, rotation, and 3D shape similarity between estimated cars and ground truth.
A3DP-Abs uses absolute translation thresholds, whereas A3DP-Rel uses relative ones.
\JS{
  A total of 10 thresholds are used, of which
  \textit{c-l} represents a loose criterion ($2.8$ m, $\pi/6$ rad, and $0.5$ for translation, rotation and shape) and
  \textit{c-s} represents a strict criterion ($1.4$ m, $\pi/12$ rad, and $0.75$ for translation, rotation and shape).
  The \textit{mean} column represents the average A3DP-Abs/Rel over all thresholds.
}

Overall, \PACErobustTwo achieves performance comparable but slightly inferior to learning-based approaches in A3DP-Rel, while \PACErobustThree excels in both A3DP-Rel and A3DP-Abs.
The main failure mode of \PACErobustTwo is in its translation estimation: over 98\% of the failures do not meet the translation threshold only, and %
\JSTwo{the translation estimation accuracy degrades with the increase in distance between the vehicle and the camera.}
\PACErobustThree outperforms the baselines in terms of the \textit{mean} and  \textit{c-s} criteria;
this is partially expected since we use depth information, which is not available to the other methods at test time.
In terms of the strict criterion \textit{c-s}, \PACErobustThree outperforms competitors by a large amount, confirming that it can retrieve highly accurate estimates. When using the loose criterion \textit{c-l}, GSNet is slightly better than \PACErobustThree, suggesting learning-based techniques may have more graceful failure modes.

\myParagraph{Runtime}
Table~\ref{tab:timing} shows the timing breakdown for \PACErobustThree and \PACErobustTwo.
We also report the timing for the GSNet keypoint detection from~\cite{Ke20-gsnet} for completeness.
For \PACErobustThree, the max-clique pruning is written in C++ and has negligible runtime,
while \gnc is implemented in Python.
For \PACErobustTwo, both the maximum hyperclique estimation and \gnc are written in Python.
\PACErobustThree is significantly faster than \PACErobustTwo thanks to \PACEThree's compact semidefinite relaxation.
While \PACErobustTwo is currently slow for real-world applications,
an optimized implementation of \PACErobustThree is amenable to practical applications.

\begin{table}[t!]
  \centering
    \begin{minipage}[b]{0.99\linewidth}\centering
        \begin{tabular}{c|c|c|c|c}
            \toprule
             \multirow{2}{2cm}{GSNet Keypoint Detection}  & \multicolumn{2}{c}{\PACErobustThree} &  \multicolumn{2}{c}{\PACErobustTwo}\\
              											&  Max-clique & \gnc  & Max-clique & \gnc \\
            \midrule
            0.45 \si{s}   &  2 \si{ms} & 0.45 \si{s} & \JS{2.13} \si{s} & \JS{201.21} \si{s}  \\
            \bottomrule
        \end{tabular}
      \caption{Average Runtime Breakdown for \PACErobustTwo and \PACErobustThree.\vspace{-4mm}}
      \label{tab:timing}
    \end{minipage}%
\end{table}

\vspace{-2mm}
\section{Related Work}
\label{sec:relatedWork}

This section reviews related work on
\emph{category-level perception}
and \emph{outlier-robust estimation}. %
\vspace{-2mm}

\subsection{Category-level Perception}

{\bf Early approaches}  for category-level perception focus on 2D problems, where one has to locate objects
---from human faces~\cite{Pantic00-facialExp} to resistors~\cite{Cootes95cviu}--- in images.
Classical approaches include \emph{active contour models}~\cite{Kass87ijcv,Chan99-activeContour} and
\emph{active shape models}~\cite{Cootes95cviu,Belongie02-shapeContext}.
 These works use techniques like PCA to build a library of 2D landmarks from training data, and then use iterative optimization
 algorithms to estimate the 2D object locations in the images, rather than estimating 3D poses.

The landscape of category-level perception has been recently reshaped by the rapid adoption of
{\bf convolutional networks}~\cite{Lecun98IEEE-CNN,Krizhevsky12nips-alexNet,Simonyan15-vgg}.
Pipelines using deep learning have seen great successes in areas such as human pose estimation~\cite{Toshev14-deepPose,Newell16-stackedHourglass,Tompson14-cnnHumanPose,He17iccv-maskRCNN, Martinez17-3DHumanPose},
and pose estimation of household objects~\cite{Manuelli19-kpam,Gao21-kpam2,Pavlakos17icra-semanticKeypoints}.
With the growing interest in self-driving vehicles, %
 research has also focused on jointly estimating vehicle pose and shape~\cite{Chabot17-deepMANTA,Ke20-gsnet,Lopez19-vehicle,Kundu18-3dRCNN,Suwajanakorn18-latent3Dkeypoints}.

For methods that aim to recover both the pose and shape of objects, a common paradigm is
to use {\bf end-to-end methods}.
Usually, an encoder-decoder network is used to first convert input images to some latent representations,
and then map the latent representation back to 3D estimates (\eg 3D bounding boxes, or pose and shape estimates).
For example, Richter\setal~\cite{Richter18CVPR-MatryoshkaNetwork} predict 3D shapes through an efficient 2D encoding.
Groueix\setal~\cite{Groueix18arxiv-atlasnet} represent shapes as collections of parametric surface elements.
Tatarchenko\setal~\cite{Tatarchenko17iccv-OGN} generate 3D shapes through an octree representation.
 Burchfiel\setal~\cite{Burchfiel19-probabilisticCategory} train CNNs with generative representations of 3D objects to predict probabilistic distributions of object poses.
An additional alignment loss can also be incorporated into the network to directly regress a pose~\cite{Avetisyan19ICCV-e2eCADAlign, Manhardt19-2dlifting, Manhardt20-cps}.
Wen\setal~\cite{Wen20iros-seTrack} design a network with a loss function over $\SEthree$ to regress relative poses.
One drawback of such approaches is that it is difficult for neural networks to learn the necessary 3D structure of the object on a per-pixel basis; moreover, end-to-end approaches typically require 3D pose labels that might be difficult (or expensive) to obtain for real data.
As shown by Tatarchenko\setal~\cite{Tatarchenko19CVPR-singleViewReconLimitation}, such networks can be outperformed by methods trained on model recognition and retrieval only.
Alternative methods circumvent pose and shape estimation and directly
regress 3D semantic keypoints for manipulation~\cite{Manuelli19-kpam} or dense visual descriptors~\cite{Florence18corl-denseObjectNet}. 

{\bf Multi-stage methods} constitute another major paradigm for category-level perception.
Such approaches commonly include a neural-network-based front-end that extracts features from
 input data (such as RGB or RGB-D images)~\cite{Pavlakos17icra-semanticKeypoints,Deng22ral-icaps},
and an optimization-based back-end that recovers the 3D pose of the object
given the features~\cite{Pavlakos17icra-semanticKeypoints,Oberweger18-heatmapPose,Peng19CVPR-PVNet,Mo19iros-orcVIO,Hou20-mobilepose}.
The front-ends may predict positions of semantic keypoints~\cite{Pavlakos17icra-semanticKeypoints},
or feature embeddings~\cite{Deng22ral-icaps},
and generate correspondences from those features.
In early works,
Lim\setal~\cite{Lim13iccv} establish 2D-3D correspondences between images and textureless CAD models by using HOG descriptors on images and rendered edgemaps of the CAD models.
Chabot\setal~\cite{Chabot17-deepMANTA} regress a set of 2D part coordinates, and then choose the best corresponding 3D template.
Pavlakos\setal~\cite{Pavlakos17icra-semanticKeypoints} use a stacked hourglass neural network~\cite{Newell16-stackedHourglass} for 2D semantic keypoint detection.
In other works, a canonical category-level coordinate space is predicted for each detection,
from which correspondences are  generated~\cite{Wang19-normalizedCoordinate,Feng20-convCategory,Li20-categoryArticulated,Chen20-learnCanonicalShape}.
Our work belongs to the class of multi-stage methods.
In particular, we use~\cite{Ke20-gsnet} as our front-end,
and develop optimal and robust back-end solvers. %

Research effort has also been devoted to developing more robust and efficient {\bf back-end solvers} given
2D or 3D features extracted by the front-end.
The back-end solvers 
recover the 3D pose (and possibly the shape) of the object
by solving an optimization problem~\cite{Oberweger18-heatmapPose,Pavlakos17icra-semanticKeypoints,Peng19CVPR-PVNet,Mo19iros-orcVIO,Hou20-mobilepose}.
Depending on the input modalities, back-end solvers can be roughly divided into {2D-3D} or {3D-3D} solvers,
where the former use 2D inputs only, and the latter incorporate additional depth information.
A number of related works investigate {\bf 2D-3D back-end solvers}~\cite{Pavlakos17icra-semanticKeypoints, Tome17cvpr-liftfromdeep, Kolotouros19iccv, Newell16-stackedHourglass,Zhou15cvpr,Zhou17pami-shapeEstimationConvex, Schmeckpeper22arxiv-singleRGBpose}.
Hou\setal~\cite{Hou20-mobilepose} defer the task of shape estimation to a neural network, and use EPnP~\cite{Lepetit09-epnp} to solve for the object bounding box's pose only.
Zhou\setal~\cite{Zhou15cvpr,Zhou17pami-shapeEstimationConvex} propose a convex relaxation to jointly optimize 3D shape parameters and object pose from 2D keypoints.
However, the relaxation assumes a weak perspective camera model, which might lead to poor results if the object is close to the camera.
Yang and Carlone~\cite{Yang20cvpr-shapeStar} apply the moment/sums-of-squares hierarchy~\cite{Blekherman12Book-sdpandConvexAlgebraicGeometry,Lasserre01siopt-LasserreHierarchy} to develop tighter relaxations than~\cite{Zhou15cvpr}, still under the assumption of a weak perspective model. 
Schmeckpeper\setal~\cite{Schmeckpeper22arxiv-singleRGBpose} use a local solver with a full perspective camera model.
Our work differs from~\cite{Zhou15cvpr,Zhou17pami-shapeEstimationConvex,Yang20cvpr-shapeStar,Schmeckpeper22arxiv-singleRGBpose} since we propose a \emph{certifiably optimal solver for the full perspective case}, using an algebraic point-to-line cost. 

{\bf 3D-3D back-end solvers} have been investigated in the 
robotics literature~\cite{Wen21iros-bundletrack, Slavcheva16eccv-sdf, Wang19icra-objScan}.
In robotics applications such as manipulation and self-driving,
depth information %
is readily available either via direct sensing (\eg RGB-D or stereo)
or algorithms (\eg mono depth techniques~\cite{Eigen14nips-monodepth,Lasinger19arXiv-robustMonocularDepthEstimation}),
so the requirements of depth is not too constraining.
Wang\setal~\cite{Wang20icra-6pack} decouple shape from pose estimation by predicting category specific keypoints and use Arun's method~\cite{Arun87pami} for estimating frame-by-frame relative pose.
Wen\setal~\cite{Wen21iros-bundletrack} view category-level object detection and tracking as a pose graph optimization problem,
solving 3D registration of keypoints across frames and then jointly optimizing the pose graph online.
Deng\setal~\cite{Deng22ral-icaps} use nonlinear optimization, and alternate between optimizing shape size and pose. 
In this and our previous work~\cite{Shi21rss-pace}, we propose the first 3D-3D certifiably optimal solver  that runs in a fraction of a second even in the presence of thousands of CAD models.
\vspace{-2mm}

\subsection{Robust Estimation}
 We review three robust estimation paradigms:
 \emph{M-estimation}, \emph{consensus maximization}, and \emph{graph-based outlier pruning}.

{\bf M-Estimation} 
 performs estimation by optimizing a robust cost function that reduces the influence of outliers. 
The resulting problems are typically optimized using iterative local solvers.
MacTavish and Barfoot~\cite{MacTavish15crv-robustEstimation} compare several robust costs using  iterative re-weighted least squares solvers. The downside of local solvers is that they need a good initial guess, 
which is often unavailable in practice.
A popular approach to circumvent the need for an initial guess is \emph{Graduated Non-Convexity} (\GNC)~\cite{Blake1987book-visualReconstruction,Black96ijcv-unification}. 
Zhou~\etal~\cite{Zhou16eccv-fastGlobalRegistration} use \GNC for point cloud registration.
Yang~\etal~\cite{Yang20ral-GNC} and Antonante~\etal~\cite{Antonante21tro-outlierRobustEstimation} combine \GNC with global non-minimal solvers and show their general applicability to problems with up to $80\%$ outliers.

For certain low-dimensional geometric problems, fast global solvers exist.
Enqvist~\etal~\cite{Enqvist12eccv-robustFitting} use a \emph{truncated least squares} (\TLS) cost to solve triangulation in polynomial time in the number of measurements, but exponential time in the {dimension of the to-be-estimated state $\vxx$}.
Ask~\etal~\cite{Ask13cvpr-optimalTruncatedL2} use a \TLS cost for image registration. %
\JS{
Recently,
{\bf certifiable algorithms} have been developed to solve outlier-robust estimation problems with a posteriori optimality guarantees~\cite{Yang20neurips-certifiablePerception,Lajoie19ral-DCGM,Yang20tro-teaser,Yang19iccv-QUASAR,Yang22pami-certifiablePerception}. 
They rely on Lasserre's hierarchy of moment relaxations to obtain convex relaxations of robust estimation problems~\cite{Yang20neurips-certifiablePerception,Yang22pami-certifiablePerception}.
They compute an estimate together with a certificate of optimality (or a bound on the suboptimality gap), based on the rank of the SDP solution or the duality gap. Brynte~\etal~\cite{Brynte22jmiv-rotationTightness} categorize cases where such relaxations are always tight and study the failure cases.
Unfortunately, current SDP solvers have poor scalability, and such methods are mostly viable
to \emph{check} optimality~\cite{Yang20tro-teaser,Yang20neurips-certifiablePerception}.
}

{\bf Consensus Maximization} is a framework for robust estimation that aims to find {the largest} set of measurements with errors {below} a user-defined threshold.
Consensus maximization is NP-hard~\cite{Chin18eccv-robustFitting,Antonante21tro-outlierRobustEstimation},
hence the community has resorted to randomized approaches, such as \ransac~\cite{Fischler81}.
\ransac repeatedly draws a minimal subset of measurements from which a rough estimate is computed,
and the algorithm stops after finding an estimate that agrees with a large set of measurements.
While \ransac works well for problems where the minimal set is small and there are not many outliers,
the average number of iterations it requires %
increases exponentially with the percentage of outliers~\cite{Raguram08-RANSACcomparative}, making it impractical for problems with many outliers.
On the other hand, global solvers, such as branch-and-bound (\BnB)~\cite{Li09cvpr-robustFitting} and tree search~\cite{Cai19ICCV-CMtreeSearch}, exist but scale poorly with the problem size, with \BnB being exponential in the size of $\vxx$, and tree search being exponential in the number of outliers~\cite{Cai19ICCV-CMtreeSearch}.

{\bf Graph-based Outlier pruning methods} aim at discarding gross outliers from the set of measurements. 
These methods do not necessarily reject all the outliers, hence they are often used as a preprocessing for
 M-estimation or maximum consensus~\cite{Yang20tro-teaser,Shi21icra-robin}.
Outlier pruning methods detect outliers by analyzing a
\emph{compatibility graph},
where vertices represent data points and
edges represent pre-defined compatibility measures between data points~\cite{Shi21icra-robin}.
Bailey\setal~\cite{Bailey00icra-dataAssociation}~propose a Maximum Common Subgraph algorithm for feature matching in lidar scans.
Segundo and Artieda~\cite{San15-cliqueFeatureMatching} build an association graph and find the maximum clique for 2D image feature matching.
Perera and Barnes~\cite{Perera12-maxCliqueSegmentation} segment objects under rigid body motion with a clique formulation.
Leordeanu and Hebert~\cite{Leordeanu05-spectral} {establish image matches by finding strongly-connected clusters in the correspondence graph with an approximate spectral method.}
Enqvist~\etal~\cite{Enqvist09iccv} develop an outlier rejection algorithm for 3D-3D and 2D-3D registration based on approximate vertex cover. %
Recent progress in graph algorithms (\eg~\cite{Rossi15parallel, Parra19arXiv-practicalMaxClique})
has led to fast graph-theoretic outlier pruning algorithms that are robust to 
 extreme outlier rates, see, \eg \teaserpp~\cite{Yang20tro-teaser}.

In this work we generalize graph-based methods to use {\bf hypergraphs}: 
while a standard graph only contains edges connecting pairs of nodes (which represent compatibility tests in our context), each edge in a hypergraph may connect an arbitrary subset of vertices.
Hypergraphs have been studied in the context of network learning, robotics, and computer vision.
Torres-Jimenez~\cite{Torres17dmaa-hclique} develops an exact algorithm for finding maximum cliques in uniform hypergraphs.
Shun~\cite{Shun20-parallelHypergraph} develops a collection of fast parallel hypergraph algorithms for large-scale networks.
Srinivasan\setal~\cite{Srinivasan21sdm-hypergraphLearning} develop a framework for hypergraph representation learning.
Rueb\setal~\cite{Rueb87pami-hypergraphPathPlan} formulate free space as a hypergraph for robot path planning.
Du\setal~\cite{Du16tc-hypergraphTracking} represent humans as a hypergraph for visual tracking.
Yu\setal~\cite{Yu12ip-hypergraphImageClassification} treat image classification as a hypergraph edge weight optimization problem.
In our work, we build upon~\cite{Shi21icra-robin}
by extending the definition of \emph{compatibility graphs} from simple graphs to hypergraphs.

\vspace{-2mm}
\section{Conclusion}
\label{sec:conclusion}

We proposed \PACETwo and \PACEThree, the first certifiably optimal solvers for the estimation of the pose and shape of
3D objects from 2D and 3D keypoint detections, respectively.
While existing iterative methods get stuck in local minima corresponding to poor estimates,
\PACETwo and \PACEThree leverage tight SDP relaxations to compute certifiably optimal estimates.
We then design a general framework for graph-theoretic outlier pruning, named \robin, that extends our original proposal 
in~\cite{Shi21icra-robin} to operate on compatibility hypergraphs. 
We show that \robin can be effectively applied
to 2D and 3D category-level perception and is able to prune a large fraction of outliers. %
The combination of \robin and our optimal solvers (\PACETwo and \PACEThree), 
 leads to \PACErobustTwo and \PACErobustThree, which are outlier-robust algorithms for 3D-3D and 2D-3D pose and shape estimation.
While \PACErobustTwo is currently slow and is sensitive to the quality of the  keypoint detections, 
\PACErobustThree largely outperforms the state of the art and a non-optimized implementation runs in a fraction of a second.

{%
\tiny
\bibliographystyle{IEEETran}
\bibliography{myRefs,refs}
}

\begin{IEEEbiography}[{\includegraphics[width=1in,height=1.25in,clip,keepaspectratio]{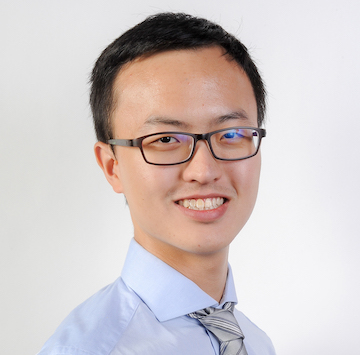}}]{Jingnan Shi} is a PhD candidate
in the Department of Aeronautics and Astronautics at the Massachusetts Institute of Technology.
He obtained his M.S. (2021) from MIT and B.S. (2019) from Harvey Mudd College.
His research interests include robust perception and self-supervised learning with applications to robotics.
He is a Best Paper Finalist at RSS 2021, and a recipient of the MathWorks Fellowship.
\end{IEEEbiography}
\begin{IEEEbiography}[{\includegraphics[width=1in,height=1.25in,clip,keepaspectratio]{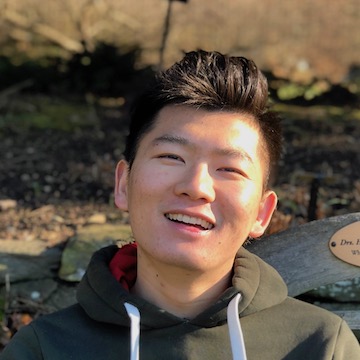}}]{Heng Yang} is an (incoming) Assistant Professor of Electrical Engineering in the School of Engineering and Applied Sciences at Harvard University. He obtained his Ph.D. from MIT in 2022, M.S. from MIT in 2017, and B.Eng. from Tsinghua University in 2015. He is interested in the algorithmic foundations of robot perception, decision-making, and learning, with focus on bringing large-scale convex optimization, semidefinite relaxation, statistics, and machine learning to safe and trustworthy autonomy. He is a recipient of the Best Paper Award in Robot Vision at ICRA 2020, a Best Paper Award Honorable Mention from RA-L in 2020, and a Best Paper Finalist at RSS 2021. He is a Class of 2021 RSS Pioneer.
\end{IEEEbiography}
\begin{IEEEbiography}[{\includegraphics[width=1in,height=1.25in,clip,keepaspectratio]{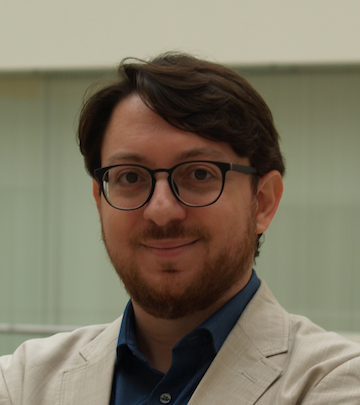}}]{Luca Carlone}
  is the Leonardo Career Development Associate Professor in the Department of Aeronautics and Astronautics at MIT, and a Principal Investigator in the Laboratory for Information \& Decision Systems (LIDS).
  He obtained a B.S. (2006) and an S.M. (2008) in mechatronics, and a Ph.D. (2012) in robotics from the Polytechnic University of Turin, Italy; and an S.M. (2008) in automation engineering from the Polytechnic University of Milan, Italy.
  He joined LIDS as a postdoctoral associate (2015) and later as a Research Scientist (2016), after two years as a postdoctoral fellow at the Georgia Institute of Technology (2013-2015). His research interests include nonlinear estimation, numerical and distributed optimization, and probabilistic inference, applied to sensing, perception, and decision-making in single and multi-robot systems. His work includes seminal results on certifiably correct algorithms for localization and mapping, as well as approaches for visual-inertial navigation and distributed mapping.
    He is a recipient of the 2022 and 2017 Transactions on Robotics King-Sun Fu Memorial Best Paper Award,
    the Best Student Paper Award at IROS 2021,
    the Best Paper Award in Robot Vision at ICRA 2020,
    a 2020 Honorable Mention from the IEEE Robotics and Automation Letters,
    a Track Best Paper award at the 2021 IEEE Aerospace Conference,
    the Best Paper Award at WAFR 2016,
    the Best Student Paper Award at the 2018 Symposium on VLSI Circuits,
    and he was best paper finalist at RSS 2015, RSS 2021, and WACV 2023.
    He is also a recipient of the AIAA Aeronautics and Astronautics Advising Award (2022), the NSF CAREER Award (2021), the RSS Early Career Award (2020), the Google Daydream Award (2019), the Amazon Research Award (2020, 2022), and the MIT AeroAstro Vickie Kerrebrock Faculty Award (2020). He is an IEEE senior member and an AIAA associate fellow.
\end{IEEEbiography}

\isExtended{

\renewcommand{\thesection}{A\arabic{section}}
\renewcommand{\theequation}{A\arabic{equation}}
\renewcommand{\thetheorem}{A\arabic{theorem}}
\renewcommand{\thefigure}{A\arabic{figure}}
\renewcommand{\thetable}{A\arabic{table}}
\renewcommand{\thealgocf}{A\arabic{algocf}}

\setcounter{equation}{0}
\setcounter{section}{0}
\setcounter{theorem}{0}
\setcounter{figure}{0}
\setcounter{algocf}{0}

\appendices

\section{Proof of Theorem~\ref{thm:inliers-form-clique}: Inliers Belong to a Clique}
\label{sec:app-proofGraph}

\begin{proof} 
By definition, \compatibility tests are designed to pass as long as the subset of $n$ nodes under test includes all inliers.
Therefore, the set of nodes corresponding to inliers, say $\calI$, will be such that any subset of $n$ nodes in $\calI$ will be connected 
by a hyperedge, and therefore will form a hyperclique in the compatibility hypergraph.
\end{proof}
\section{Example Showing the Difference Between Clique-expanded Graphs and Hypergraphs For \robin} \label{sec:app-hypergraph-v-graph}

Consider a $3$-invariant and its compatibility test, with 5 measurements, denoted as Nodes 1 to 5.
Node 5 is an outlier, whereas Node 1 through 4 are inliers.
Assume Table~\ref{tbl:app-hypergraph-comp-test-results} contains the results of running the compatibility test with the $3$-invariant.
Note that $(1,2,5)$, $(2,3,5)$, and $(2,4,5)$ pass the compatibility test despite Node $5$ being an outlier.
Fig.~\ref{fig:app-hypergraph-v-graph} shows the compatibility hypergraph constructed according to the results, with the maximum hyperclique being $(1,2,3,4)$.
Fig.~\ref{fig:app-hypergraph-clique-expanded} shows the clique-expanded hypergraph.
In this case, the maximum clique is $(1,2,3,4,5)$, which is the entire measurement set including the outlier.
Thus, for this example, \robin described in this paper will successfully reject Node $5$ as an outlier, while \robin described in~\cite{Shi21icra-robin} will not.

\begin{table}[htbp!]
  \centering
\begin{tabular}{cc}
  \toprule
Triplet & Pass the Compatibility Test? \\ \midrule
1,2,3   & True                         \\
1,2,4   & True                         \\
1,2,5   & True                         \\
1,3,4   & True                         \\
1,3,5   & False                        \\
1,4,5   & False                        \\
2,3,4   & True                         \\
2,3,5   & True                         \\
2,4,5   & True                         \\
3,4,5   & False                        \\ \bottomrule
\end{tabular}
\caption{Results of running the compatibility test on all possible triplets.}
\label{tbl:app-hypergraph-comp-test-results}
\end{table}

 While this is only a toy example, %
we empirically observe that using hypergraphs improves the effectiveness of outlier pruning in real problems, such as the ones in Section~\ref{sec:2d3d-comp-test}.

\begin{figure}[htbp!]
  \centering%
  \includegraphics[width=0.7\columnwidth]{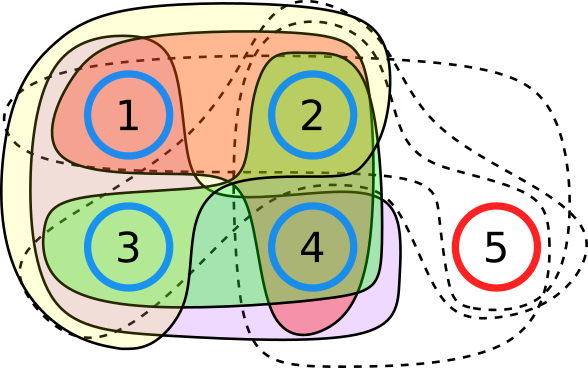}
  \caption{Compatibility hypergraph with hyperedges including the outlier (outlier is Node 5 circled in red;
    hyperedges containing the outlier are represented with dashed edges).
    The maximum hyperclique is $(1, 2, 3, 4)$.
  }
  \label{fig:app-hypergraph-v-graph}
\end{figure}

\begin{figure}[htbp!]
  \centering%
  \includegraphics[width=0.6\columnwidth]{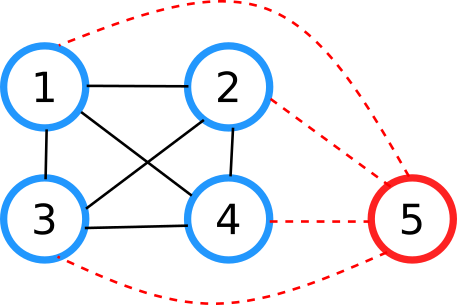}
  \caption{Clique-expanded compatibility graph. Outlier is Node 5; all edges including the outlier are marked in red, with dashed lines. Comparing to the hypergraph where the maximum hyperclique is $(1,2,3,4)$, in this graph the maximum clique is $(1,2,3,4,5)$.
  }
  \label{fig:app-hypergraph-clique-expanded}
\end{figure}

\section{Finding Maximum \Hypercliques \\ in \Compatibility Hypergraphs}
\label{sec:app-milp}

While there exist fast parallel implementations for finding the maximum clique in ordinary graphs (see, \eg~\cite{Rossi15parallel}), 
to the best of our knowledge there is no such implementation for hypergraphs.
Thus, in this paper we use a simple mixed-integer linear programming (MILP) formulation for finding the maximum \hyperclique within an $n$-uniform hypergraph.

Assume an arbitrary $n$-uniform hypergraph $\calG(\calV,\calE)$, where $|\calV| = N$.
Consider the following MILP:
\begin{equation}
\hspace{-3mm}
  \label{eq:max-hyperclique}\tag{MILP}
\begin{array}{rl@{}ll}
\displaystyle\max_{\vb \in \{0,1\}^N}  & \sum_{i=1}^{N} b_{i} &\\
\subject& \sum_{i \in \subMeas{n}} b_{i} \leq n-1,   & \quad & \forall 
\subMeas{n} \subset \calV, |\subMeas{n}| = n, {\subMeas{n} \notin \calE}
\end{array}
\end{equation}
where
$\calE$ is the set of hyperedges,
$\subMeas{n} \subset \calV$ are subsets of $n$ nodes, and
$b_{i}$, $i=1, \ldots, N$, are binary variables that indicate whether node $i$ belongs to the maximum \hyperclique.

Intuitively,~\eqref{eq:max-hyperclique} looks for a set of nodes such that ---within that set--- there is no subset of $n$ nodes 
that are not connected by a hyperedge. This matches the definition of a \hyperclique, as formalized below.

\begin{theorem}[MILP Finds Maximum Hypercliques]\label{thm:hyperclique-mip}
  For any feasible solution $\vb$ to~\eqref{eq:max-hyperclique}, 
  the set of nodes $\setdef{i \in \calV}{b_i = 1}$ forms a \hyperclique in $\calG$;
  moreover, the optimal solution $b^\star$ is such that
  the set $\setdef{i \in \calV}{b^\star_i = 1}$ is a \emph{maximum} \hyperclique.
\end{theorem}

\begin{proof}
  Assume by contradiction that there exists a feasible solution of the MILP that does \emph{not} form a \hyperclique.
  Without loss of generality, assume that ---for some $m$--- such a solution includes nodes 
  $\hat{\calV} \doteq \{v_{1}, v_{2}, \ldots, v_{m} \}$, \ie $b_i=1$ for $i=1,\ldots,m$, and zero otherwise.
  Since this set of nodes does not form a \hyperclique, there is a subset $\subMeas{n}$ of $n$ nodes
  in $\hat{\calV}$ that does not belong to the hyperedge set $\calE$, and for such subset $\sum_{i \in \subMeas{n}} b_{i}  = n$.
  which violates the constraint in the MILP. Since $\hat{\calV}$ violates a constraint, it cannot be feasible, leading to a contradiction.
  Since the objective maximizes the count of non-zero $b_{i}$ (\ie the number of nodes in the \hyperclique), the optimal solution is a maximum \hyperclique.
\end{proof} %
\section{Minimum and Maximum Distances \\ between Convex Hulls for Eq.~\eqref{eq:definebminbmax}}
\label{sec:app-bmin-bmax}
Recall from eq.~\eqref{eq:definebminbmax} the definitions of $b_{ij}^{\min}$ and $b_{ij}^{\max}$:
\bea
b_{ij}^{\min} = \min_{\vc \geq 0, \ones\tran\vc=1} \norm{\sum_{k=1}^K c_k (\basis{k}{j} - \basis{k}{i})}, \label{eq:definebmin}\\
b_{ij}^{\max} = \max_{\vc \geq 0, \ones\tran\vc=1} \norm{\sum_{k=1}^K c_k (\basis{k}{j} - \basis{k}{i})}, \label{eq:definebmax}
\eea
and let us use the following shorthand:
\bea
\vb^{k}_{ij} \triangleq \basis{k}{j} - \basis{k}{i}, \\
\MB_{ij} \triangleq \bmat{ccc} \vb^{1}_{ij} & \cdots & \vb^{K}_{ij} \emat \in \Real{3 \times K},
\eea
to write problems~\eqref{eq:definebmin} and~\eqref{eq:definebmax} compactly as:
\bea
\hspace{-4mm} b_{ij}^{\min} = \min_{\vc \geq 0, \ones\tran\vc=1} \norm{\MB_{ij} \vc}, \quad b_{ij}^{\max} = \max_{\vc \geq 0, \ones\tran\vc=1} \norm{\MB_{ij} \vc}.
\eea

{\bf Compute $b^{\max}_{ij}$}. Because $\norm{\MB_{ij} \vc}$ is a convex function of $\vc$, and the maximum of a convex function over a polyhedral set (in our case, the standard simplex $\Delta_K \triangleq \{\vc \in \Real{K}: \vc \geq 0, \ones\tran \vc=1 \}$) is always obtained at one of the vertices of the polyhedron~\cite[Corollary 32.3.4]{Rockafellar70book-convexanalysis}, we have:
\bea
b_{ij}^{\max} = \max_{k} \norm{\vb^{k}_{ij}},
\eea
since the vertices of $\Delta_K$ are the vectors $\ve_k,k=1,\dots,K$, where $\ve_k$ is one at its $k$-th entry and zero anywhere else.

{\bf Compute $b^{\min}_{ij}$}. Observe that computing the minimum of $\norm{\MB_{ij} \vc}$ is equivalent to computing 
the minimum of $\norm{\MB_{ij} \vc}^2 = \vc\tran (\MB_{ij}\tran \MB_{ij}) \vc$
because the quadratic function $f(x) = x^2$ is monotonically increasing in the interval $[0,\infty]$, and hence we first solve the following convex quadratic program (QP):
\bea \label{eq:QPofbmin}
\min_{\vc \in \Real{K}} & \vc\tran (\MB_{ij}\tran \MB_{ij}) \vc \\
\subject & \vc \geq 0, \quad \ones\tran \vc = 1
\eea
and then compute $b_{ij}^{\min} = \norm{\MB_{ij} \vc^\star}$ from the solution $\vc^\star$ of the QP. Note that the QP~\eqref{eq:QPofbmin} can be solved in milliseconds for large $K$, so pre-computing $b_{ij}^{\min}$ for all $1\leq i < j \leq N$ is still tractable even when $N$ is large.

\section{Proof of Proposition~\ref{prop:compute-winding-order}: 2D Winding Orders} \label{sec:app-compute-winding-order}

\begin{proof}
Consider the 2D image points $\measTwo{i}$, $\measTwo{j}$, and $\measTwo{m}$
and their representation in homogeneous coordinates:
\begin{align*}
\bmeasTwo{i} = \begin{bmatrix}
  \measTwo{i} \\ 1
  \end{bmatrix} \quad \bmeasTwo{j} = \begin{bmatrix}
  \measTwo{j} \\ 1
  \end{bmatrix} \quad \bmeasTwo{m} = \begin{bmatrix}
  \measTwo{m} \\ 1
  \end{bmatrix}.
\end{align*}
Recall these 3D vectors are expressed in a standard right-handed image coordinate frame with origin at the 
center of the image (irrelevant for the derivation below), and where the $\boldsymbol{x}$ axis points towards the right in the image 
plane, the  $\boldsymbol{y}$ axis points down in the image plane, 
and the $\boldsymbol{z}$ axis points into the image plane (to ensure right-handedness).

Now the triplet of points is arranged in clockwise order if their crossproduct is aligned with the 
$\boldsymbol{z}$ axis:
\begin{align}
  &(\bmeasTwo{j} - \bmeasTwo{i}) \times (\bmeasTwo{m} - \bmeasTwo{i}) \\
 =& \det \left( \begin{bmatrix}
   \boldsymbol{x} & \boldsymbol{y} & \boldsymbol{z} \\
   \multicolumn{2}{c}{(\measTwo{j} - \measTwo{i})\tran} & 0 \\
   \multicolumn{2}{c}{(\measTwo{m} - \measTwo{i})\tran} & 0
 \end{bmatrix} \right) \\
 =& \det \left( \begin{bmatrix}
        (\measTwo{j} - \measTwo{i})\tran \\
        (\measTwo{m} - \measTwo{i})\tran
   \end{bmatrix}   \right) \boldsymbol{z}
   \\
   =&  \det \left( \begin{bmatrix}
        \measTwo{j} - \measTwo{i} \;\;
        \measTwo{m} - \measTwo{i}
   \end{bmatrix}   \right)   \boldsymbol{z}
\end{align}
where we used the standard relation between the cross product and the determinant and 
then developed the expression of the determinant.
Therefore, we have
\begin{align}
\det \left( \begin{bmatrix}
        \measTwo{j} - \measTwo{i} \;\;
        \measTwo{m} - \measTwo{i}
   \end{bmatrix}   \right) > 0
\end{align}
if the points are arranged clockwise. And
\begin{align}
\det \left( \begin{bmatrix}
        \measTwo{j} - \measTwo{i} \;\;
        \measTwo{m} - \measTwo{i}
   \end{bmatrix}   \right) < 0
\end{align}
if the points are arranged counter-clockwise.
\end{proof}

\section{Proof for Theorem~\ref{thm:winding-order-half-space}: 3D Winding Orders}\label{sec:app-winding-order-half-space}

We first recall a simple fact about the scalar triple product.
\begin{lemma}[Scalar Triple Product As Determinant]\label{lem:triple-product-det}
  For arbitrary $\va$, $\vb$, $\vc \in \Real{3}$,
  \begin{align}
\va \cdot \left( \vb \times \vc \right)=\det\left[\begin{array}{lll}
\va & \vb & \vc
\end{array}\right] \label{eq:app-triple-product-as-det}
 \end{align}
\end{lemma}
This is a well-known property and can be proven by inspection.
We are now ready to prove Theorem~\ref{thm:winding-order-half-space}.

\begin{proof}
Recall that $\vo$ is the optical center of the camera in the CAD model's frame.
 Since $(\MR,\vt)$ is the pose of the object in the coordinate frame of the camera, it follows 
 that $\vo = -\MR\tran \vt$.

For an arbitrary camera center $\vo$, it follows that
\begin{align}
  &(\vo - \basis{k}{i}) \cdot \normal{k}{i,j,m} \nonumber \\
  & \grayout{\text{(using the definition of } \normal{k}{i,j,m} 
  \text{ in~\eqref{eq:single-shape-triplet-normal} and Lemma~\ref{lem:triple-product-det})}}
  \nonumber \\
  =&
  \det \left(
  \begin{bmatrix}
    \vo - \basis{k}{i} & \basis{k}{j} - \basis{k}{i} & \basis{k}{m} - \basis{k}{i}
  \end{bmatrix}
  \right)  \label{eq:app-triple-product} \\
  & \grayout{\text{(subtracting the first column from the second and third)}}
  \nonumber \\
  =&
  \det \left(
  \begin{bmatrix}
    \vo - \basis{k}{i} & \basis{k}{j} - \vo & \basis{k}{m} - \vo
  \end{bmatrix}
  \right)  \label{eq:app-subtract-col} \\
  & \grayout{\text{(flipping the sign of the first column)}}
  \nonumber \\
  =&
  - \det \left(
  \begin{bmatrix}
   \basis{k}{i} - \vo & \basis{k}{j} - \vo & \basis{k}{m} - \vo
  \end{bmatrix}
  \right)  \label{eq:app-flip-sign}
  \\
  & \grayout{\text{(multiplying the matrix by a rotation matrix $\MR$)}}
  \nonumber \\
  =&
  - \det \left(
  \begin{bmatrix}
   \MR (\basis{k}{i} - \vo) & \MR(\basis{k}{j} - \vo) & \MR(\basis{k}{m} - \vo)
  \end{bmatrix}
  \right)  \label{eq:app-mul-R}
  \\
  & \grayout{\text{(using $\vo = -\MR\tran \vt$ and $\MR \MR\tran = \MI$)}}
  \nonumber \\
  =&
  - \det \left(
  \begin{bmatrix}
   \MR \basis{k}{i} + \vt & \MR\basis{k}{j} + \vt & \MR\basis{k}{m} + \vt
  \end{bmatrix}
  \right) \label{eq:app-expand-R} \\
  & \grayout{\text{(dividing each column by its third coordinate)}}
  \nonumber \\
  =&
   - (\measThree{i})_{z} (\measThree{j})_{z} (\measThree{m})_{z} \nonumber \\
  &
\scalemath{0.8}{
    \det \left(
  \begin{bmatrix}
   (\measThree{i})_{x} / (\measThree{i})_{z} & (\measThree{j})_{x} / (\measThree{j})_{z} &  (\measThree{m})_{x} / (\measThree{m})_{z} \\
   (\measThree{i})_{y} / (\measThree{i})_{z} & (\measThree{j})_{y} / (\measThree{j})_{z} &  (\measThree{m})_{y} / (\measThree{m})_{z} \\
   1 & 1 & 1
  \end{bmatrix}
  \right)\label{eq:app-divide-z}}
\end{align}
In 
\eqref{eq:app-subtract-col} we used the fact that adding a scalar multiple of one column to another column does not change the value of the determinant, while in 
\eqref{eq:app-flip-sign} we observed that flipping the sign of the first column flips the sign of the determinant~\cite[pp. 249-252]{Strang16book-introLinAlg}.
In \eqref{eq:app-mul-R} we observed that left multiplying by $\MR$ does not change the determinant because $\det(\MR) = 1$.
Finally, in \eqref{eq:app-divide-z} we divided each column by its third coordinate and multiplied the determinant by the same coordinate to keep it constant~\cite[pp. 249-252]{Strang16book-introLinAlg}.

Now recall that the canonical perspective projection of 3D points $\measThree{i}$, $\measThree{j}$, $\measThree{m}$ is
\begin{gather*}
\measTwo{i} = \begin{bmatrix}
 (\measThree{i})_{x} / (\measThree{i})_{z} \\
 (\measThree{i})_{y} / (\measThree{i})_{z} \\
  \end{bmatrix}
   \quad
\measTwo{j} = \begin{bmatrix}
 (\measThree{j})_{x} / (\measThree{j})_{z} \\
 (\measThree{j})_{y} / (\measThree{j})_{z} \\
  \end{bmatrix} \\
\measTwo{m} = \begin{bmatrix}
 (\measThree{m})_{x} / (\measThree{m})_{z} \\
 (\measThree{m})_{y} / (\measThree{m})_{z} \\
  \end{bmatrix}
\end{gather*}
Substituting the projections back into~\eqref{eq:app-divide-z}:
\begin{align}
  \text{\eqref{eq:app-divide-z}} =&
 - (\measThree{i})_{z} (\measThree{j})_{z} (\measThree{m})_{z} \nonumber \\
  &\det \begin{bmatrix}
     \measTwo{i} & \measTwo{j} & \measTwo{m} \\
     1 & 1 & 1
   \end{bmatrix} \nonumber \\
 =&
- (\measThree{i})_{z} (\measThree{j})_{z} (\measThree{m})_{z} \nonumber \\
  &\det \begin{bmatrix}
     \measTwo{i} & \measTwo{j} - \measTwo{i}  & \measTwo{m} - \measTwo{i}  \\
     1 & 0 & 0
   \end{bmatrix} \label{eq:app-2d-det}
\end{align}
because subtracting a column from another does not change the determinant.
By cofactor expansion along the last row:
\begin{align}
  \text{\eqref{eq:app-2d-det}} =&
    - (\measThree{i})_{z} (\measThree{j})_{z} (\measThree{m})_{z} \nonumber \\
  &\det \begin{bmatrix}
      \measTwo{j} - \measTwo{i}  & \measTwo{m} - \measTwo{i}
   \end{bmatrix} \label{eq:app-2d-det-simplified}
\end{align}
Note that we assume the object always stays in front of the camera (\ie in the direction of the positive $z$-axis of the camera frame).
Applying the signum function, we finally have
\begin{align*}
  &\sgn{\left((\vo - \basis{k}{i}) \cdot \normal{k}{i,j,m} \right)} \\
  &= -\sgn{\left( \det \begin{bmatrix}
      \measTwo{j} - \measTwo{i}  & \measTwo{m} - \measTwo{i}
   \end{bmatrix}\right)} %
\end{align*}
which proves the claim.
\end{proof}

One way to interpret Theorem~\ref{thm:winding-order-half-space} is
that for all $\vo$ in the positive half-space (resp. negative half-space) of the triplet plane,
the observed winding order of the projected points is counter-clockwise (resp. clockwise) following Proposition~\ref{prop:compute-winding-order}.

\section{Extensions of Proposition~\ref{prop:2d3d-comp-test}} 
\label{sec:app-2d3d-invariant-valid-condition}

This appendix discusses conditions under which Proposition~\ref{prop:2d3d-comp-test} can be extended to the case where the 2D keypoints are generated from a convex combination of shapes ---as in our generative model~\eqref{eq:generativeModel-2d}--- rather than being generated by a single shape (as currently assumed in Proposition~\ref{prop:2d3d-comp-test}).

{
Consider $K$ shapes, with 2D measurements $\measTwo{i}$, $i=1,\ldots,N$, generated according to eq.~\eqref{eq:generativeModel-2d},
and such that $\|\epsTwo{i}\| < \beta$ and $\beta$ is small enough for Theorem~\ref{thm:winding-order-half-space} to hold.
Assume for each shape, we have obtained the feasible winding order dictionary $\calW_{k}$ a priori.
Let $\vc = \begin{bmatrix} c_{1} c_{2} \ldots c_{K}\end{bmatrix}$ be an arbitrary ground-truth shape coefficient vector.

Given an arbitrary camera center $\vo$, let the visible keypoints' index set be $\calM_\vo$.
In other words, $\vo$ is able to observe $\measTwo{i}$, $i \in \calM_\vo$, and therefore 
$\vo$  lies within the non-empty covisibility region of $\sum_{k=1}^{\nrShapes}\shapeParam{k} \basis{k}{i}$, $i \in \calM_\vo$.
In the following, we assume that whenever a camera center $\vo$ observes a 3D point $\sum_{k=1}^{\nrShapes}\shapeParam{k} \basis{k}{i}$, it is also in the visibility region of $\basis{k}{i}$ for $i \in \calM_\vo$ and for $k = 1, \ldots K$. 

Given an arbitrary triplet of keypoints $\measTwo{i}$, $\measTwo{j}$, and $\measTwo{m}$
(corresponding to $\sum_{k=1}^{\nrShapes}\shapeParam{k} \basis{k}{i}$, $\sum_{k=1}^{\nrShapes}\shapeParam{k} \basis{k}{j}$, and $\sum_{k=1}^{\nrShapes}\shapeParam{k} \basis{k}{m}$),
 we want to prove (\cf Proposition~\ref{prop:2d3d-comp-test})
 \begin{align}
 \invfunTwo(\measTwo{i}, \measTwo{j}, \measTwo{m}) \in \invFunTwo(\vtheta_{i}, \vtheta_{j}, \vtheta_{m})
  \end{align}
  with
  \begin{align}
    \invfunTwo(&\measTwo{i}, \measTwo{j}, \measTwo{m}) \nonumber \\
               &\doteq \det \left( \begin{bmatrix} \measTwo{j} - \measTwo{i} & \measTwo{m} - \measTwo{i} \end{bmatrix}\right) \\
    \invFunTwo(&\vtheta_{i}, \vtheta_{j}, \vtheta_{m}) = \bigcup\limits_{k=1}^{K} \calW_{k}(i, j, m)
  \end{align}

We have the following cases:
}

\emph{Case 1.} $\bigcup\limits_{k=1}^{K} \calW_{k}(i, j, m) = \{+1, -1\}$.
In this case, $\invfunTwo(\measTwo{i}, \measTwo{j}, \measTwo{m}) \in \invFunTwo(\vtheta_{i}, \vtheta_{j}, \vtheta_{m})$ trivially since the points are assumed in generic position and hence the determinant is non-zero (\ie either positive or negative).

\emph{Case 2.} $\bigcup\limits_{k=1}^{K} \calW_{k}(i, j, m) = \{+1\}$ or $\bigcup\limits_{k=1}^{K} \calW_{k}(i, j, m) = \{-1\}$.
In other words, the winding order dictionaries $\calW_{k}(i, j, m)$ for all $k$ have the same winding order.
To have $\invfunTwo(\measTwo{i}, \measTwo{j}, \measTwo{m}) \in \invFunTwo(\vtheta_{i}, \vtheta_{j}, \vtheta_{m})$,
we need to find a condition under which the convex combination of keypoints
has the same winding order as any of the corresponding keypoints in each individual shape.
In other words, we need:
\begin{align} \label{eq:app-weighted-winding-order}
  \sgn{\left(( \vo - \textstyle\sum_{k=1}^{K} \shapeParam{k}\basis{k}{i} ) \cdot \vn \right)} = \calW_{k} (i,j,m)
\end{align}
where $\vn = (\sum_{k=1}^{K} c_{k}(\basis{k}{j} - \basis{k}{i})) \times (\sum_{k=1}^{K} \shapeParam{k}(\basis{k}{m} - \basis{k}{i}))$.

We can rewrite eq.~\eqref{eq:app-weighted-winding-order} in a way that eliminates $\vn$.
Towards this goal, we start by using Lemma~\ref{lem:triple-product-det} to establish the following equality for a 
triplet of points $\basis{k}{i}, \basis{k}{j}, \basis{k}{m}$ and for each shape:
\begin{align}
  \sgn \left( (\vo - \basis{k}{i}) \cdot \normal{k}{i,j,m}\right)
  =
    \sgn \left(
\det
\begin{bmatrix}
\vv^{k}_{i} & \basis{k}{ji} & \basis{k}{mi}
\end{bmatrix} \right)  \nonumber \\
  \forall k=1,\ldots,K
\end{align}
where $\vv^{k}_{i} = \vo - \basis{k}{i}$, $\basis{k}{ji} = \basis{k}{j} - \basis{k}{i}$, and $\basis{k}{mi} = \basis{k}{m} - \basis{k}{i}$ (recall that $\normal{k}{i,j,m} = (\basis{k}{j} - \basis{k}{i}) \times (\basis{k}{m} - \basis{k}{i})$ from eq.~\eqref{eq:single-shape-triplet-normal}).

Let $\MA^{k} = \begin{bmatrix}
\vv^{k}_{i} & \basis{k}{ji} & \basis{k}{mi}
\end{bmatrix}$ for $k=1,\ldots,K$.
It follows
\begin{align}
  &\det \left (  \sum_{k=1}^{K} c_{k} \MA^{k} \right) & & \nonumber \\
  = & \det \begin{bmatrix}
\sum_{k=1}^{K}c_{k} \vo - \sum_{k=1}^{K}c_{k}\basis{k}{i} & \sum_{k=1}^{K} c_{k}\basis{k}{ji} & \sum_{k=1}^{K} c_{k}\basis{k}{mi}
    \end{bmatrix} \nonumber \\
=& \det  \begin{bmatrix}
 \vo - \sum_{k=1}^{K}c_{k}\basis{k}{i} & \sum_{k=1}^{K} c_{k} \basis{k}{ji} & \sum_{k=1}^{K} c_{k} \basis{k}{mi}
         \end{bmatrix} \nonumber \\
  =& ( \vo - \sum_{k=1}^{K} \vc_{k}\basis{k}{i} ) \cdot \vn
\end{align}
Therefore, as long as:
\begin{align} 
  \sgn \left( \det \left (  \sum_{k=1}^{K} c_{k} \MA^{k} \right) \right) = 
  \sgn \left( \det \left ( \MA^{k} \right) \right),
    \quad \forall \vc \in \Delta_K \label{eq:condA}
\end{align}
then, we have
\begin{align} 
\label{eq:app-winding-order-valid-condition}
  \sgn \left( \det \left (  \sum_{k=1}^{K} c_{k} \MA^{k} \right) \right) &= \sgn  \left(( \vo - \sum_{k=1}^{K} \vc_{k}\basis{k}{i} ) \cdot \vn \right) \nonumber \\
  &= \calW_{k} (i,j,m) 
\end{align}
which matches the claim in Proposition~\ref{prop:2d3d-comp-test} since it implies 
$\invfunTwo(\measTwo{i}, \measTwo{j}, \measTwo{m}) \in \invFunTwo(\vtheta_{i}, \vtheta_{j}, \vtheta_{m})$.

While condition~\eqref{eq:condA} might not be satisfied in general,  
with simulated data, we check empirically whether Proposition~\ref{prop:2d3d-comp-test} and eq.~\eqref{eq:app-winding-order-valid-condition} hold under scenarios where the ground-truth shape coefficients are randomly sampled in the probability simplex $\Delta_K$.
Fig.~\ref{fig:app-winding-order-validity} shows the percentage of test instances where eq.~\eqref{eq:app-winding-order-valid-condition} holds true under the same synthetic data generation procedure for the robustness experiments described in Section~\ref{sec:exp-optimality-robustness-2d}, with octahedra as shapes.
We set the outlier rate to $0\%$ (the invariant is only expected to hold for inliers), while changing the number of shapes from $2$ to $5$.
At all shape counts, eq.~\eqref{eq:app-winding-order-valid-condition} holds true for $100\%$ of the generated test instances.

\begin{figure}[htbp!]
  \centering%
  \includegraphics[width=0.8\columnwidth]{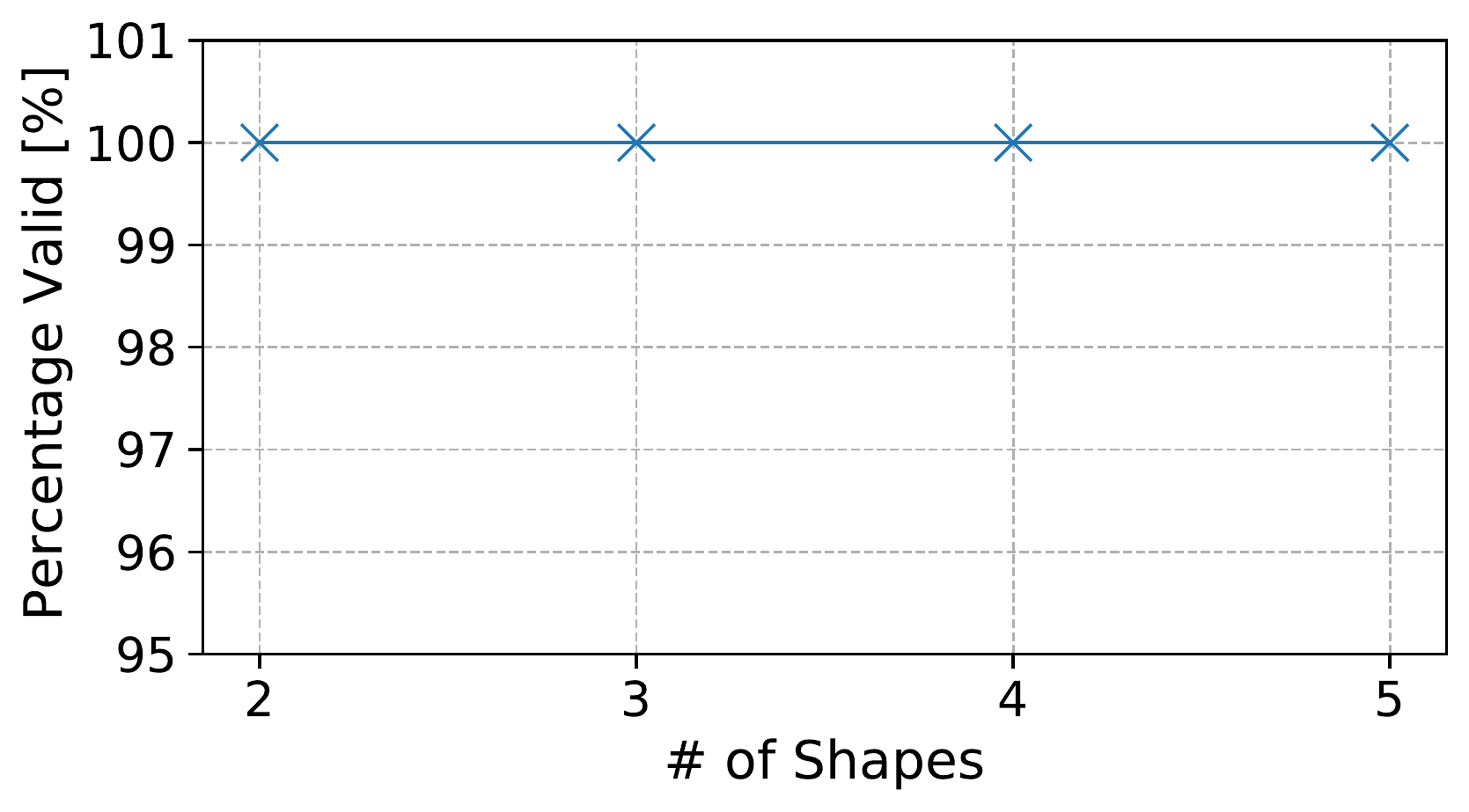}
  \caption{Percentage of simulated test instances where eq.~\eqref{eq:app-winding-order-valid-condition} and  Proposition~\ref{prop:2d3d-comp-test} hold even when the keypoints are generated by linear combinations of shapes (rather than by one of the shapes in the CAD library).}
  \label{fig:app-winding-order-validity}
\end{figure}

\section{Generating Winding Order Dictionaries}\label{sec:app-winding-order-dictionary}

In this section, we discuss three different methods for
constructing winding order dictionaries.
The first method uses linear programming to
solve eq.~\eqref{eq:ijm_order_feasible} and eq.~\eqref{eq:imj_order_feasible} directly.
It is suitable for cases where analytical equations of the faces of the CAD models are known,
and the shapes are convex.
The second method, which uses ray tracing, is applicable to complex non-convex shapes.
The last method constructs the winding order dictionaries by learning from ground-truth 2D annotations.
Our simulated experiments for \PACErobustTwo uses the first method.
Our experiments on \apollo uses a combination of the last two methods.

\subsection{Using Linear Programs}
Eq.~\eqref{eq:ijm_order_feasible} and eq.~\eqref{eq:imj_order_feasible} are two feasibility problems
that can be solved using linear program solvers, if $\vo \in \CovisSet$ can be expressed as linear constraints.
Luckily, this can be achieved if the underlying shape is closed and convex, with keypoints lying strictly in the interior of the faces (excluding the vertices and edges).
In essence, for keypoints on an arbitrary face, their visibility regions are equivalent to the half-space outside the shape.
Assume we have $L$ faces with known plane equations
\begin{equation}
  \vn_{l} \cdot \vo + b_{l} = 0, \quad \forall l = 1, \ldots, L,
\end{equation}
where $\vn_{l}$ is the normal vector of face $l$ pointing \emph{away} from the model, and $b_{l}$ is a constant.
Then, the visibility region of all keypoints on face $l$ is equivalent to checking whether there exists a $\vo$ such that
\begin{equation}
  \vn_{l} \cdot \vo + b_{l} > 0
\end{equation}
The \covis region of a keypoint triplet $i, j, m$ can therefore be found by combining the above constraints for all three keypoints
\begin{align}
\vn_{l_{1}} \cdot \vo + b_{l_{1}} > 0 \\
\vn_{l_{2}} \cdot \vo + b_{l_{2}} > 0 \\
\vn_{l_{3}} \cdot \vo + b_{l_{3}} > 0
\end{align}
where $l_{1}$, $l_{2}$, and $l_{3}$ are the indices of the faces keypoints $i$, $j$, and $m$ belong to.
Fig.~\ref{fig:app-cube-covis} shows an example of a \covis region of three keypoints on a cube, which is the intersection of two half-spaces.

\begin{figure}[htbp!]
  \centering
\includegraphics[width=0.7\columnwidth]{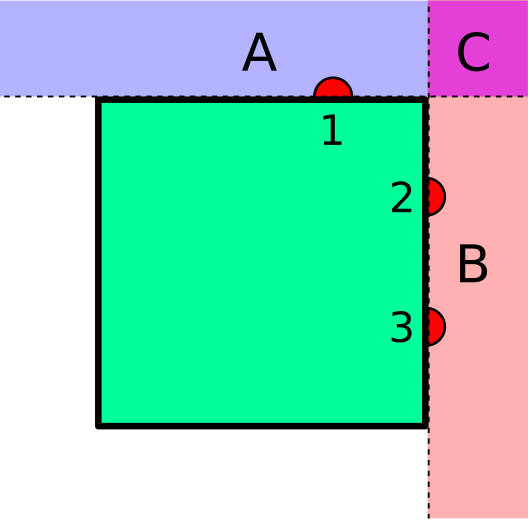}
\caption{A top down view of a cube, with 3 keypoints (1, 2, and 3) on two different faces. Shaded region A (\emph{blue}) is the visibility region of keypoint 1.
  Shaded region B (\emph{red}) is the \vis region of 2 and 3.
  Shaded region C (\emph{purple}) is the \covis region of 1, 2, and 3.}
\label{fig:app-cube-covis}
\end{figure}

The feasibility linear program for solving eq.~\eqref{eq:ijm_order_feasible}
is then
\begin{equation}
  \label{eq:app-winding-order-lp-1}
\begin{array}{ll@{}l}
\text{find}  & \vo&\\
\text{subject to} &\vn_{l_{1}} \cdot \vo + b_{l_{1}} &\geq C \\
                  &\vn_{l_{2}} \cdot \vo + b_{l_{2}} &\geq C \\
                  &\vn_{l_{3}} \cdot \vo + b_{l_{3}} &\geq C \\
                  &(\vo - \basis{k}{i}) \cdot \normal{k}{i,j,m}  &\geq C
\end{array}
\end{equation}
where $C$ is a small positive constant.
If \eqref{eq:app-winding-order-lp-1} has a solution, that means keypoints $i,j,m$ can be viewed in counterclockwise winding order.
And the feasibility linear program for solving eq.~\eqref{eq:ijm_order_feasible} is
\begin{equation}
  \label{eq:app-winding-order-lp-2}
\begin{array}{ll@{}l}
\text{find}  & \vo&\\
\text{subject to} &\vn_{l_{1}} \cdot \vo + b_{l_{1}} &\geq C \\
                  &\vn_{l_{2}} \cdot \vo + b_{l_{2}} &\geq C \\
                  &\vn_{l_{3}} \cdot \vo + b_{l_{3}} &\geq C \\
                  &(\vo - \basis{k}{i}) \cdot \normal{k}{i,j,m}  &\leq -C
\end{array}
\end{equation}
If \eqref{eq:app-winding-order-lp-2} has a solution, that means keypoints $i,j,m$ can be viewed in clockwise winding order.
To construct the winding order dictionary, we simply solve \eqref{eq:app-winding-order-lp-1} and \eqref{eq:app-winding-order-lp-1}
for all triplets, and record the results.

\subsection{Using Ray Tracing}
For shapes where we do not have access to analytical face equations, or non-convex shapes, ray tracing can be used instead.
The main insight is to use ray tracing to check visibilities of keypoints from a set of sampled camera locations around the model,
instead of using linear programs.

Given a shape model, we first discretize the volume surrounding it into voxels,
with each voxel center outside the model representing a potential optical center for ray tracing.
For each optical center, we perform ray tracing with all keypoints to determine the set of visible keypoints.
We store the results of ray tracing into a dictionary of 3D boolean arrays, with keypoints as keys.
{For any keypoint, its boolean array represents the visibility of it from each optical center.}

{
After we have the dictionary, for an arbitrary triplet of keypoints $i$, $j$ and $m$,
we can check whether the constraint $\vo \in \CovisSet$ can be met by performing an \texttt{AND}
operation on the boolean visibility arrays of the three keypoints.
If the resulting array after \texttt{AND} has at least one \texttt{True} value,
that means there exists one optical center such that all three keypoints are visible.
Hence $\vo \in \CovisSet$ can be satisfied.
}
The feasibility problems of eq.~\eqref{eq:ijm_order_feasible} and eq.~\eqref{eq:imj_order_feasible}
can then be checked by calculating $(\vo - \basis{k}{i}) \cdot \normal{k}{i,j,m}$ for all optical centers in $\CovisSet$.

\subsection{Learning From 2D Annotations}
In cases where ground-truth 2D keypoints annotations are available, we may also construct winding order dictionaries by learning.
We first construct a dictionary $D$ with keypoint triplets as keys and empty sets as values.
For each ground-truth annotated training image,
we calculate the winding order following Proposition~\ref{prop:compute-winding-order} for each triplet $(i,j,m)$ in the image,
and push the result into the $(i,j,m)$ entry of $D$.
Since the values are sets, only unique winding orders will be preserved.
After enumerating through all the training images, if there are triplets that were not found during training, we use ray tracing to estimate the feasible winding orders.
We use this approach to construct the winding order dictionary for our experiments on the \apollo dataset.

\section{Problem~\eqref{eq:probOutFree-3D3Dcatlevel} is a MAP Estimator when the Measurement Noise is Gaussian}
\label{sec:app-mapOutlierFree}

Here we prove that the optimization in eq.~\eqref{eq:probOutFree-3D3Dcatlevel} is a 
\emph{maximum a posteriori} (MAP) estimator when the measurement noise $\epsThree{i}$ in~\eqref{eq:generativeModel-3d}
follows a zero-mean Gaussian with covariance $\frac{1}{w_i} \eye_3$ (where $\eye_3$ is the 3-by-3 identity matrix) and we have a zero-mean Gaussian prior with covariance $\frac{1}{\lambda} \eye_K $ over the shape parameters $\vc$. 
Mathematically:
\bea
\label{eq:distMAP1a}
\!\!\prob{\epsThree{i}} = \kappa_\epsilon \exp\left(\!-\frac{w_i}{2}
\|\epsThree{i}\|^2\!
\right), \\
\label{eq:distMAP1b}
\prob{\vc} = 
\kappa_c \exp\left( \!-\frac{\lambda}{2} 
\|\vc\|^2\!
\right),
\eea
where $\kappa_\epsilon$ and $\kappa_c$ are suitable normalization constants that are irrelevant for the following derivation.

A MAP estimator for the unknown parameters $\vxx \triangleq \{\MR, \vt, \vc\}$ (belonging to a suitable domain $\domainX$) given measurements~$\measThree{i}$ ($i=1,\ldots,N$) is defined as the maximum of the posterior distribution
$\prob{ \vxx | \measThree{1}\;\ldots\; \measThree{N}}$:
\beq
\label{eq:MAP}
\argmax_{\vxx \in \domainX} \prob{ \vxx | \measThree{1}\;\ldots\; \measThree{N}}
= \\ 
\argmax_{\vxx \in \domainX} \prod_{i=1}^N \prob{ \measThree{i} | \vxx} \prob{ \vxx }
\eeq 
where on the right we applied Bayes rule and used the standard assumption of independent measurements.
Using~\eqref{eq:distMAP1a} and~\eqref{eq:generativeModel-3d} we obtain:
\bea
\label{eq:pyx}
\prob{ \measThree{i} | \vxx}  = \kappa_\epsilon \exp\left( -\frac{w_i}{2}
\left\| \measThree{i} \!-\! \MR \sum_{k=1}^{\nrShapes} c_{k} \basis{k}{i} \!-\! \vt \right\|^{2}
\right).
\eea
Moreover, assuming we only have a prior on $\vc$:
\bea
\label{eq:px}
\prob{ \vxx } = \prob{\vc} = 
\kappa_c \exp\left( -\frac{\lambda}{2} 
\|\vc\|^2
\right).
\eea
Substituting~\eqref{eq:pyx} and~\eqref{eq:px} back into~\eqref{eq:MAP} and 
observing that the maximum of the posterior is the same as the minimum of the negative logarithm of the 
posterior:
\bea
\argmax_{\vxx \in \domainX} \prod_{i=1}^N \prob{ \measThree{i} | \vxx} \prob{ \vxx } =  \\
\argmin_{\vxx \in \domainX} \sum_{i=1}^N -\log\prob{ \measThree{i} | \vxx} -\log \prob{ \vxx } = \\
\argmin_{\substack{\MR \in \SOthree, \\ \vt \in \Real{3}, \vc \in \Real{\nrShapes}, \\ \ones\tran \vc  = 1 } } \sum_{i=1}^N \frac{w_i}{2} 
\left\| \measThree{i} - \MR \sum_{k=1}^{\nrShapes} c_{k} \basis{k}{i} - \vt \right\|^{2} \\
 +\frac{\lambda}{2} 
\|\vc\|^2 + \text{constants}
\eea 
which, after dropping constant multiplicative and additive factors, can be seen to match eq.~\eqref{eq:probOutFree-3D3Dcatlevel}, proving the claim.

\section{Certifiably Optimal Rotation Estimation: \\ Proof of Proposition~\ref{prop:optRotation}}
\label{sec:app-rotEst-details}

Let us first develop the cost function of problem~\eqref{eq:nonconvexR} as a quadratic function of $\vr \triangleq \vectorize{\MR}$:
\bea
\norm{\MM (\eye_N \kron \MR\tran)\barvy + \vh}^2 \\
= \norm{\MM \vectorize{\MR\tran \MY} + \vh}^2 \\
= \norm{\MM (\MY\tran \kron \eye_3) \vectorize{\MR\tran} + \vh}^2 \\
= \norm{\MM (\MY\tran \kron \eye_3) \MP \vr + \vh}^2 \\
= \vrhomo \tran \MQ \vrhomo
\eea
where $\MP \in \Real{9 \times 9}$ is the following permutation matrix
\bea
(1,1,1), 
(2,4,1),
(3,7,1), \\
(4,2,1),
(5,5,1),
(6,8,1), \\
(7,3,1),
(8,6,1),
(9,9,1),
\eea
with the triplet $(i,j,v)$ defining the nonzero entries of $\MP$ (\ie~$\MP_{ij} = v$), such that:
\bea
\vectorize{\MR\tran} \equiv \MP \vectorize{\MR}
\eea 
always holds, $\MY$ and $\vrhomo$ are defined as: 
\bea
\MY \triangleq \bmat{ccc}
\barvy_{1} & \cdots & \barvy_{N}
\emat \in \Real{3 \times N}, \\
\vrhomo \triangleq \bmat{cc} 1 & \vr\tran \emat\tran \in \Real{10},
\eea
and $\MQ \in \sym{10}$ can be assembled as follows:
\bea
 \MQ \triangleq 
\bmat{cc}
\vh\tran \vh & \vh\tran\MM(\MY\tran\kron\eye_3)\MP \\
\star & \MP\tran(\MY \kron \eye_3)\MM\tran \MM (\MY\tran \kron \eye_3) \MP
\emat.
\eea

Now that the objective function of~\eqref{eq:nonconvexR} is quadratic in $\vr$ ($\MR$), we can write problem~\eqref{eq:nonconvexR} equivalently as the \emph{quadratically constrained quadratic program} (QCQP) in~\eqref{eq:categoryQCQP},
where $\MA_i \in \sym{10}, i=1,\dots,15$, are the constant matrices that define the quadratic constraints associated with $\MR \in \SOthree$~\cite[Lemma 5]{Yang20cvpr-shapeStar}. For completeness, we give the expressions for $\MA_i$'s:
\bea
\MA_0: (1,1,1) \nonumber \\
\MA_1-\MA_3: \text{ columns have unit norm} \nonumber \\
\MA_1: (1,1,1),(2,2,-1),(3,3,-1),(4,4,-1) \nonumber \\
\MA_2: (1,1,1),(5,5,-1),(6,6,-1),(7,7,-1) \nonumber \\
\MA_3: (1,1,1),(8,8,-1),(9,9,-1),(10,10,-1) \nonumber %
\eea
\bea
\MA_4-\MA_6: \text{ columns are mutually orthogonal} \nonumber \\
\MA_4: (2,5,1),(3,6,1),(4,7,1) \nonumber \\
\MA_5: (2,8,1),(3,9,1),(4,10,1) \nonumber \\
\MA_6: (5,8,1),(6,9,1),(7,10,1) \nonumber %
\eea
\bea
\MA_7-\MA_{15}: \text{ columns form right-handed frame} \nonumber \\
\MA_7: (3,7,1),(4,6,-1),(1,8,-1) \nonumber \\
\MA_8: (4,5,1),(2,7,-1),(1,9,-1) \nonumber \\
\MA_9: (2,6,1),(1,10,-1),(3,5,-1) \nonumber \\
\MA_{10}: (6,10,1),(1,2,-1),(7,9,-1) \nonumber \\
\MA_{11}: (7,8,1),(5,10,-1),(1,3,-1) \nonumber \\
\MA_{12}: (5,9,1),(1,4,-1),(6,8,-1) \nonumber \\
\MA_{13}: (4,9,1),(3,10,-1),(1,5,-1) \nonumber \\
\MA_{14}: (2,10,1),(1,6,-1),(4,8,-1) \nonumber \\
\MA_{15}: (3,8,1),(2,9,-1),(1,7,-1) \nonumber
\eea
where the triplets $(i,j,v)$ define the \emph{diagonal and upper triangular} nonzero entries of a symmetric matrix (\ie~$\MA_{ij} = \MA_{ji} = v$ with $i \leq j$).

\section{Shor's Semidefinite Relaxation: \\ Proof of Proposition~\ref{prop:optRotation} }
\label{sec:app-sdp-relaxation-gap}

\begin{proof}
To see why problem~\eqref{eq:categoryQCQPrelax} is a convex relaxation for problem~\eqref{eq:categoryQCQP}, let us first create a matrix variable
\bea \label{eq:MXfactor}
\MX = \vrhomo \vrhomo\tran \in \sym{10},
\eea
and notice that $\MX$ satisfies
\bea
\MX \succeq 0, \quad \rank{\MX} = 1.
\eea
Moreover, if $\MX \succeq 0, \rank{\MX} = 1$ then $\MX$ must have a factorization of the form~\eqref{eq:MXfactor}. Therefore, the non-convex QCQP~\eqref{eq:categoryQCQP} is equivalent to the following rank-constrained matrix optimization problem:
\bea \label{eq:rankconstrained}
\min_{\MX \in \sym{10}} & \trace{\MQ \MX} \\
\subject & \trace{\MA_0 \MX} = 1,\\
&  \trace{\MA_i \MX} = 0, \forall i=1,\dots,15,\\
& \MX \succeq 0, \\
& \rank{\MX} = 1, \label{eq:rankoneconstraint}
\eea
where $\MA_0 \in \sym{10}$ is an all-zero matrix except the top-left entry being 1 (to enforce that the first entry of $\vrhomo$ is 1), and we have used the fact that
\bea
\vrhomo\tran \MA \vrhomo = \trace{\vrhomo\tran \MA \vrhomo} = \trace{\MA \vrhomo \vrhomo\tran} = \trace{\MA \MX}.
\eea 
Now observe that the only nonconvex constraint in problem~\eqref{eq:rankconstrained} is the rank constraint~\eqref{eq:rankoneconstraint}, and the SDP relaxation~\eqref{eq:categoryQCQPrelax} is obtained by simply removing the rank constraint. 
\end{proof}

In practice, we solve the convex problem~\eqref{eq:categoryQCQPrelax} and obtain an optimal solution $\MX^\star$;
 if $\rank{\MX^\star} = 1$, then the optimal solution of problem~\eqref{eq:categoryQCQPrelax} is unique (the rationale behind this is that interior-point methods converge to a maximum rank solution~\cite{DeKlerk06book-IPMSDP}) and it actually satisfies the rank constraint that has been dropped. Therefore, in this situation, we say the convex relaxation is tight and the global optimal solution to the nonconvex problem~\eqref{eq:categoryQCQP} can be obtained from the rank-one factorization of $\MX^\star$.

\section{Rounding and Suboptimality Gap for~\eqref{eq:categoryQCQPrelax}}
\label{sec:app-roundingAndSuboptimality3D3D}

This appendix discusses how to compute a feasible rotation estimate and the corresponding suboptimality gap from the solution $\MXstar$  of the SDP relaxation~\eqref{eq:categoryQCQPrelax}.
Let $\fstar$ be the optimal objective value of the SDP \eqref{eq:categoryQCQPrelax}, and $\MXstar = \sum_{i=1}^{10} \gamma_i \vu_i \vu_i\tran$ be the spectral decomposition of $\MXstar$ with $\gamma_1 \geq \dots \geq \gamma_{10}$.
 We define a rounding procedure
\bea\label{eq:roundingshor}
\vu_i \leftarrow \frac{\vu_1}{\vu_1(1)}, \quad \hatMR = \proj_{\SOthree}(\vu_1(\vr))
\eea
that extracts a feasible point $\hatMR$ to \eqref{eq:categoryQCQP} from the leading eigenvector $\vu_1$ of $\MXstar$. In \eqref{eq:roundingshor}, $\vu_1(\cdot)$ extracts the entries of $\vu_1$ using the indices of ``$\cdot$'' in $\vrhomo$, and $\proj_{\SOthree}$ represents the projection onto $\SOthree$. Given $\hatMR$, denote the objective value of \eqref{eq:categoryQCQP} at $\hatMR$ as $\hatp$; we compute a \emph{relative suboptimality} as
\bea\label{eq:relativesubopt}
\eta = \abs{\hatp - \fstar} / (1 + \abs{\hatp} + \abs{\fstar})
\eea
to evaluate the quality of the feasible solution. Apparently, $\eta = 0$ certifies the global optimality of $\hatMR$.

\section{Alternation Approach}
\label{sec:app-alternationApproach}

In Sections~\ref{sec:shape-cat-level}-\ref{sec:rotation-cat-level} of the main paper, we presented a certifiably optimal solver to solve the shape and rotation $(\vc,\MR)$ problem~\eqref{eq:translation-free-problem} (after eliminating the translation $\vt$). Here we describe a baseline method that solves problem~\eqref{eq:translation-free-problem} using \emph{alternating minimization} (\altern), a heuristic that is popular in related works on 3D shape reconstruction from 2D landmarks~\cite{Lin14eccv-modelFitting,Gu06cvpr-faceAlignment,Ramakrishna12eccv-humanPose}, but offers no optimality guarantees. Towards this goal, let us denote the cost function of~\eqref{eq:translation-free-problem} as $f(\MR,\vc)$; the \altern method starts with an initial guess $(\MR^{(0)}, \vc^{(0)})$ (default $\MR^{(0)} = \eye_3, \vc^{(0)} = \zero$), and performs the following two steps at each iteration $\tau$:
\begin{enumerate}
	\item Optimize $\vc$:
	\bea
	\vc^{(\tau)} = \argmin_{\vc \in \Real{\nrShapes}, \ones\tran \vc = 1} f(\MR^{(\tau-1)},\vc),
	\eea
	which is a linearly constrained linear least squares problem and can be solved by the closed-form solution~\eqref{eq:optimalvcofR}.

	\item Optimize $\MR$:
	\bea
	\MR^{(\tau)} = \argmin_{\MR \in \SOthree} f(\MR, \vc^{(\tau)}),
	\eea
	which can be cast as an instance of Wahba's problem~\cite{Yang19iccv-QUASAR} and can be solved in closed form using singular value decomposition~\cite{markley1988jas-svdAttitudeDeter}.
\end{enumerate}
The \altern method stops when the cost function converges,~\ie~$|f(\MR^{(\tau)},\vc^{(\tau)}) - f(\MR^{(\tau-1)},\vc^{(\tau-1)})| < \epsilon$ for some small threshold $\epsilon > 0$, or when $\tau$ exceeds the maximum number of iterations (\eg~$1000$). %

\section{Rounding and Suboptimality Gap for~\eqref{eq:denserelax}}
\label{sec:app-roundingAndSuboptimality2D3D}

This section provides extra results related to Lasserre's Hierarchy of semidefinite relaxations and provides a 
rounding procedure to obtain a rotation and shape parameters estimate from the solution of the SDP~\eqref{eq:denserelax}.

\begin{corollary}[Optimality Certification from Lasserre's Hierarchy~\cite{Lasserre01siopt-LasserreHierarchy}]
\label{cor:certificateLasserre}
Let $\pstar$ and $\fstar$ be the optimal objectives of \eqref{eq:2d3dp2ltf} and \eqref{eq:denserelax}, respectively, and let $\MXstar = (\MXstar_0,\dots,\MXstar_K)$ be an optimal solution of \eqref{eq:denserelax}, we have
\begin{enumerate}[label=(\roman*)]
\item $\fstar \leq \pstar$,
\item if $\rank{\MXstar_0} = 1$, then $\fstar = \pstar$, and $\MXstar_0$ can be factorized as $\MXstar_0 = [\vxxstar]_2[\vxxstar]_2\tran$, where $\vxxstar = [\vectorize{\MRstar}\tran,(\vcstar)\tran]\tran$ is a globally optimal for \eqref{eq:2d3dp2ltf}.
\end{enumerate}
\end{corollary}

\myParagraph{Rounding}
{Empirically, we observe that solving the SDP \eqref{eq:denserelax} empirically yields a rank-one optimal solution, and hence it typically allows retrieving the global solution of the non-convex problem \eqref{eq:2d3dp2ltf}} per Corollary~\ref{cor:certificateLasserre}.
Even when the SDP solution is not rank-one, we can ``round'' a feasible solution to \eqref{eq:2d3dp2ltf} from $\MXstar_0$. To do so, let $\MXstar_0 = \sum_{i=1}^{n_0}\gamma_i \vu_i \vu_i\tran$ be the spectral decomposition of $\MXstar_0$ with $\gamma_1 \geq \gamma_2 \geq \dots \geq \gamma_{n_0}$. 
To extract a feasible point $(\hatMR,\hatvc)$ for \eqref{eq:2d3dp2ltf} from the leading eigenvector $\vu_1$,
we follow
\bea \label{eq:rounding}
\vu_1 \leftarrow \frac{\vu_1}{\vu_1(1)},
\hatMR = \proj_{\SOthree}\parentheses{\vu_1(\vr)},
\hatvc = \proj_{\Delta_K}\parentheses{\vu_1(\vc)} 
\nonumber
\eea
 where $\proj_{\Delta_K}$ represents the projection onto $\Delta_K$~\cite{Wang13arxiv-projectionSimplex}. We can then evaluate the relative suboptimality of the rotation and shape estimate $(\hatMR,\hatvc)$ using \eqref{eq:relativesubopt}.

\section{Graduated Non-Convexity for \\ Category-level Perception}
\label{sec:app-gnc-category}

\myParagraph{Robust 3D-3D Category-level Perception}
As prescribed by standard robust estimation,
we can re-gain robustness to outliers by replacing
the squared $\ell_2$ norm in~\eqref{eq:probOutFree-3D3Dcatlevel} with a robust loss function $\rho$, leading to
\bea
\label{eq:robust-3D3Dcatlevel}
  \hspace{-5mm} \min_{\substack{\MR \in \SOthree, \\ \vt \in \Real{3}, \vc \in \Delta_\nrShapes } } &
  \displaystyle 
   \hspace{-3mm} \sum_{i=1}^{N} \rho \left( \left\| \measThree{i} \!-\! \MR \sum_{k=1}^{\nrShapes} c_{k} \basis{k}{i} \!-\! \vt \right\| \right) \!+\! \lambda \norm{\vc}^2 \hspace{-3mm}
\eea
While GNC can be applied to a broad class of loss functions~\cite{Yang20ral-GNC}, here 
we consider a truncated least square loss $\rho(r) = \min(r^2, \barcsqThree)$ which minimizes the squared
residuals whenever they are below $\barcsqThree$  or becomes constant otherwise (note: the constant $\inthrThree$ is the same inlier threshold of Section~\ref{sec:3d3d-comp-test}). Such cost function can be written by
using auxiliary slack variables $\rho(r) = \min(r^2, \barcsqThree) = \min_{\weight\in\{0,1\}} \weight r^2 + (1-\weight)\barcsqThree$~\cite{Yang20ral-GNC},
hence allowing to rewrite~\eqref{eq:robust-3D3Dcatlevel} as
\bea \label{eq:robust-3D3Dcatlevel2}
  \hspace{-3mm} \min_{\substack{\MR \in \SOthree, \\ \vt \in \Real{3}, \vc \in \Delta_{\nrShapes} \\ \weight_i\in\{0,1\} \forall i } } &
  \displaystyle 
   \hspace{-3mm} \sum_{i=1}^{N} \weight_i \left\| \measThree{i} \!-\! \MR \sum_{k=1}^{\nrShapes} c_{k} \basis{k}{i} \!-\! \vt \right\|^2
   \!\! \nonumber \\[-10pt]
   &\qquad \qquad +
   \!(1\!-\!\weight_i)\barcsqThree  \!+\! \lambda \norm{\vc}^2 \hspace{-3mm} \hspace{-3mm}
\eea
In~\eqref{eq:robust-3D3Dcatlevel2}, when $\weight_i = 1$, the $i$-th measurement is considered an inlier and the 
cost minimizes the corresponding squared residual; when $\weight_i = 0$, the cost becomes independent of $\measThree{i}$ hence the
corresponding measurement is rejected as an outlier. Therefore, equation~\eqref{eq:robust-3D3Dcatlevel2} simultaneously
estimates pose and shape variables $(\MR,\vt,\vc)$ while classifying inliers/outliers via the binary weights $\weight_i$ ($i=1,\ldots, N$).
Now the advantage is that we can minimize~\eqref{eq:robust-3D3Dcatlevel2} with an alternation 
scheme where we iteratively optimize  (i) over $(\MR,\vt,\vc)$ with fixed weights $\weight_i$ 
and (ii) over the weights $\weight_i$ with fixed $(\MR,\vt,\vc)$. 
This approach is convenient since the optimization over   $(\MR,\vt,\vc)$ can be solved to optimality with \PACEThree (see Section~\ref{sec:optimalSolver-3d3d}),
while the optimization of the weights can be solved in closed form~\cite{Yang20ral-GNC}.
To improve convergence of this alternation scheme, we adopt graduated non-convexity~\cite{Blake1987book-visualReconstruction,Yang20ral-GNC}, which starts with a convex approximation of the loss function in~\eqref{eq:robust-3D3Dcatlevel2} and then gradually increases the non-convexity until the  loss $\rho$ in~\eqref{eq:robust-3D3Dcatlevel2} is recovered.

\myParagraph{Robust 2D-3D Category-level Perception}
Similarly, we adopt the graduate non-convexity scheme to robustify our solver for Problem~\ref{prob:2d3d-statement}.
We reformulate~\eqref{eq:probOutFree-2D3Dcatlevel} with a truncated least square loss cost function $\rho(r) = \min(r^2, \barcsqTwo)$ into
\begin{align}\label{eq:robust-2D3Dcatlevel2}
\min_{\substack{\MR \in \SOthree, \\ \vt \in \Real{3}, \vc \in \Delta_\nrShapes \\ \weight_i\in\{0,1\} \forall i } } \sum_{i=1}^{N} \weight_i  &\left\| \measTwo{i} - \pi\left( \MR \textstyle\sum_{k=1}^{\nrShapes}\shapeParam{k} \basis{k}{i} + \vt \right) \right\|^{2} \nonumber \\[-20pt]
   &\qquad + (1-\weight_i)\barcsqTwo  + \lambda \norm{\vc}^2 %
\end{align}
To minimize~\eqref{eq:robust-2D3Dcatlevel2}, we adopt the same alternation scheme as for~\eqref{eq:robust-3D3Dcatlevel2},
where we iteratively optimize (i) over $(\MR,\vt,\vc)$ with fixed weights $\weight_i$
and (ii) over the weights $\weight_i$ with fixed $(\MR,\vt,\vc)$.
To tackle the optimization over $(\MR,\vt,\vc)$, we generate an initial solution by solving~\eqref{eq:2d3dp2l} using \PACETwo,
  and then locally refine the solution using the geometric reprojection error in~\eqref{eq:probOutFree-2D3Dcatlevel}. 
  {In \supplementary{sec:app-experiments-pace-2d} we show the local refinement improves the quality of the resulting estimates.}

\renewcommand{\mpwfour}{4.6cm}
\newcommand{\mpwtwo}{9.2cm}
\renewcommand{\myhspace}{\hspace{-3.5mm}}
\begin{figure*}[!t]
	\begin{center}
	\begin{minipage}{\textwidth}
	\begin{tabular}{cccc}%
		\myhspace \hspace{-3mm}
			\begin{minipage}{\mpwtwo}%
			\centering%
			\includegraphics[width=\columnwidth]{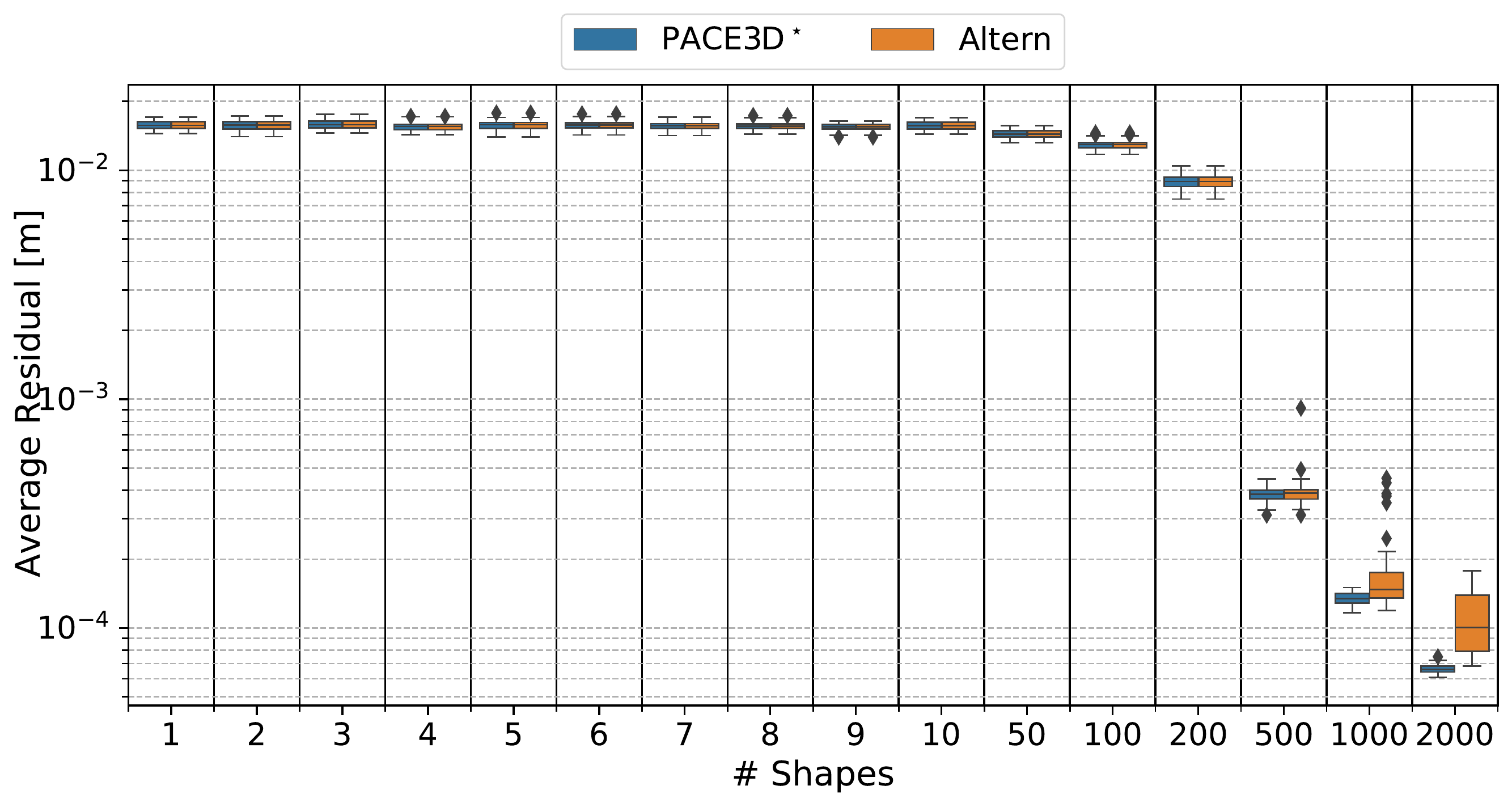}
			\end{minipage}
		&   \myhspace
			\begin{minipage}{\mpwtwo}%
			\centering%
			\includegraphics[width=\columnwidth]{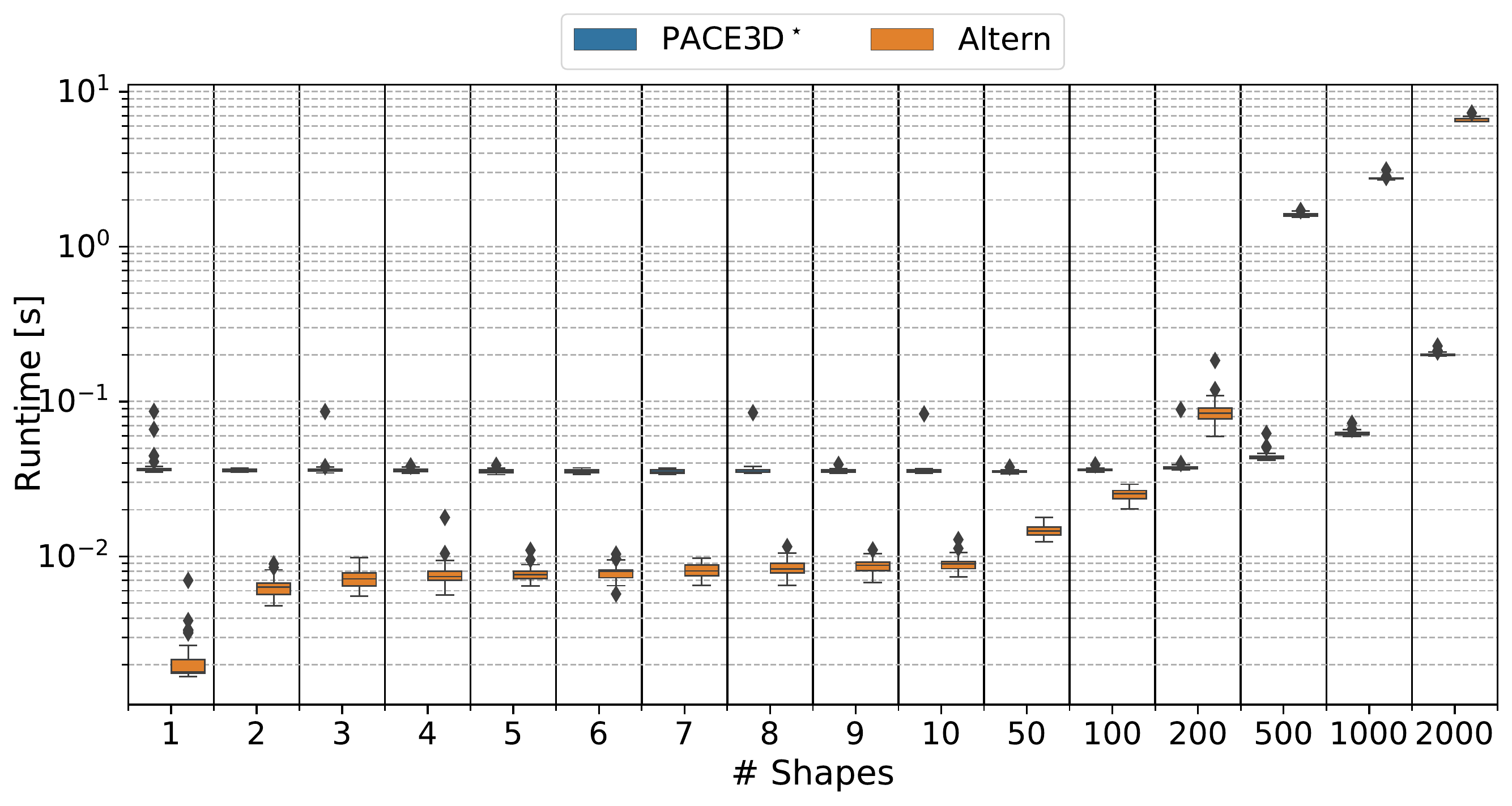}
			\end{minipage}
		\\
		\multicolumn{4}{c}{\smaller (a) Average residuals and runtimes of \PACEThree~on outlier-free synthetic data with varying number of shapes: $N=100$, $\sigma=0.01$. }
		\\
		\myhspace \hspace{-3mm}
			\begin{minipage}{\mpwtwo}%
			\centering%
			\includegraphics[width=\columnwidth]{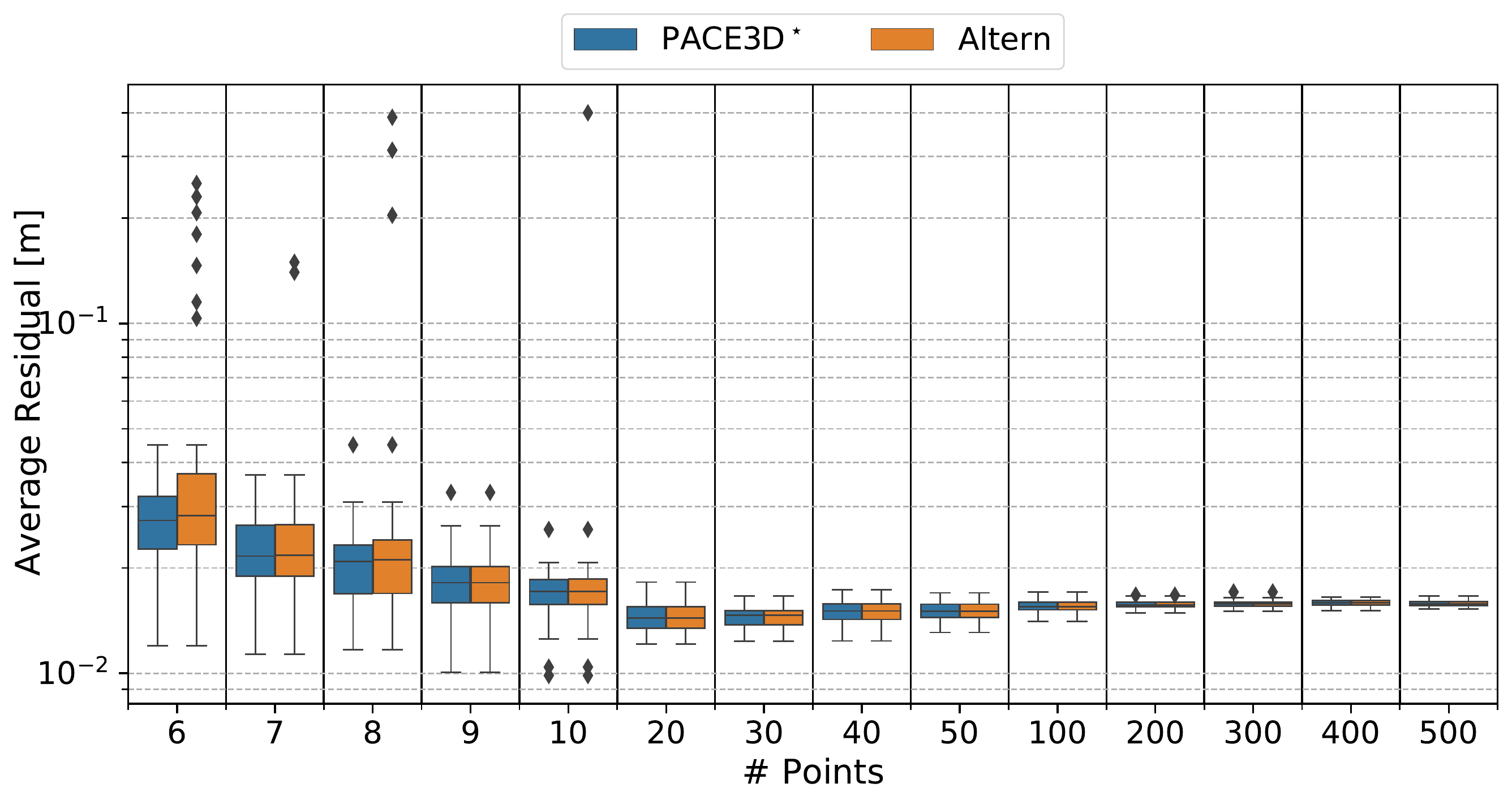}
			\end{minipage}
		&   \myhspace
			\begin{minipage}{\mpwtwo}%
			\centering%
			\includegraphics[width=\columnwidth]{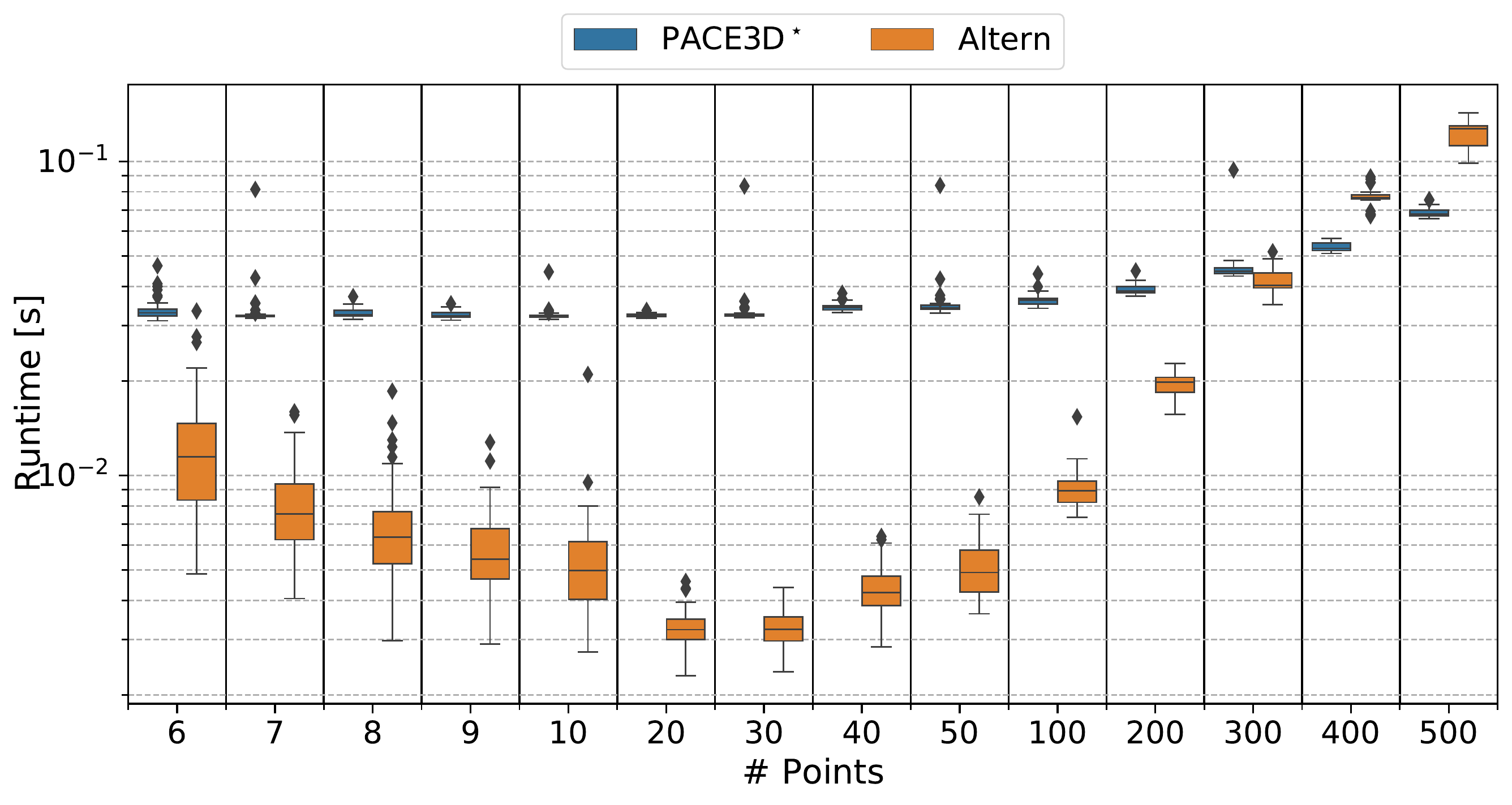}
			\end{minipage}
		\\
		\multicolumn{4}{c}{\smaller (b) Average residuals and runtimes of \PACEThree~on outlier-free synthetic data with varying number of points: $K=10$, $\sigma=0.01$. }
		\\
		\myhspace \hspace{-3mm}
			\begin{minipage}{\mpwtwo}%
			\centering%
			\includegraphics[width=\columnwidth]{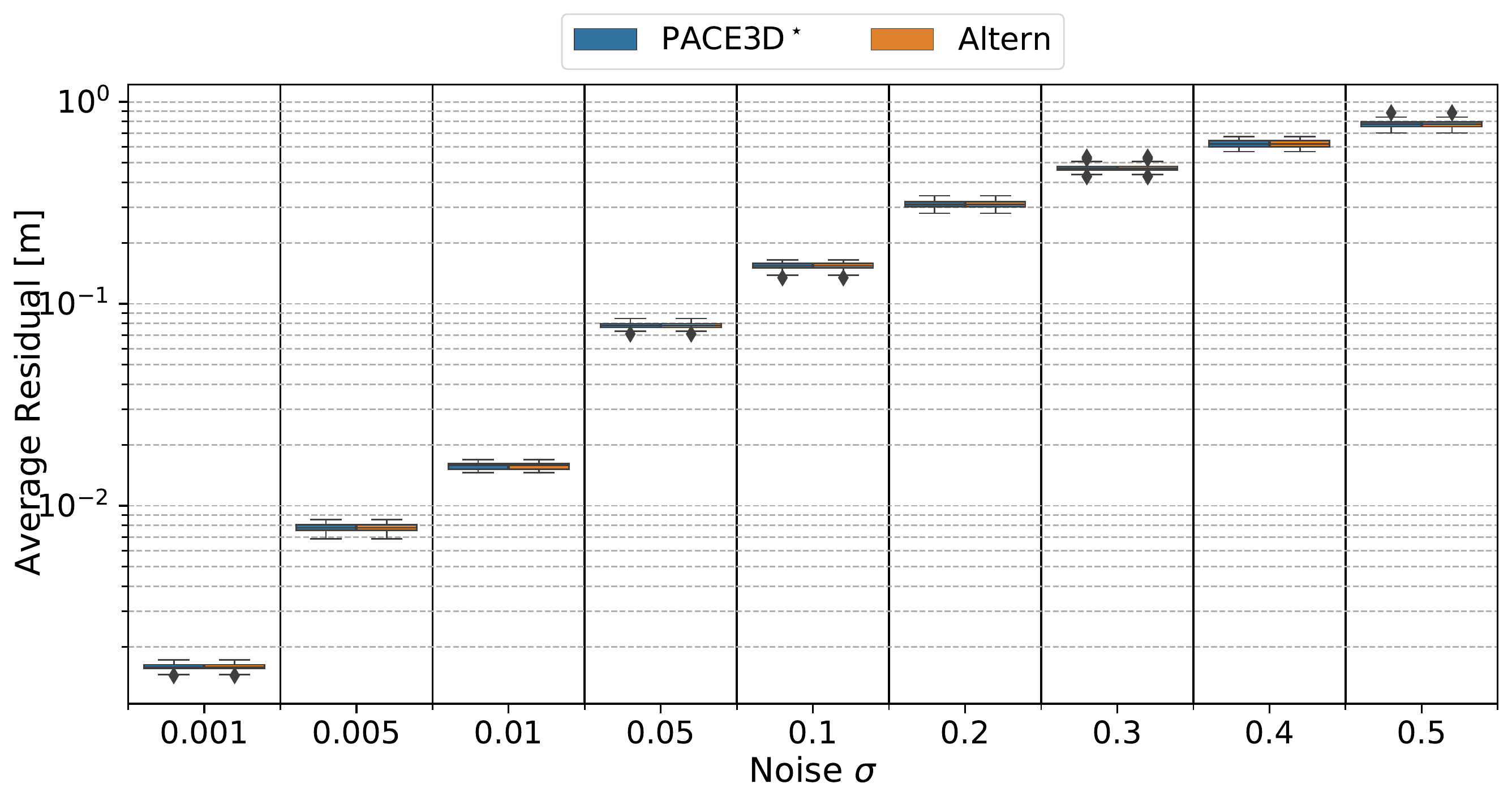}
			\end{minipage}
		&   \myhspace
			\begin{minipage}{\mpwtwo}%
			\centering%
			\includegraphics[width=\columnwidth]{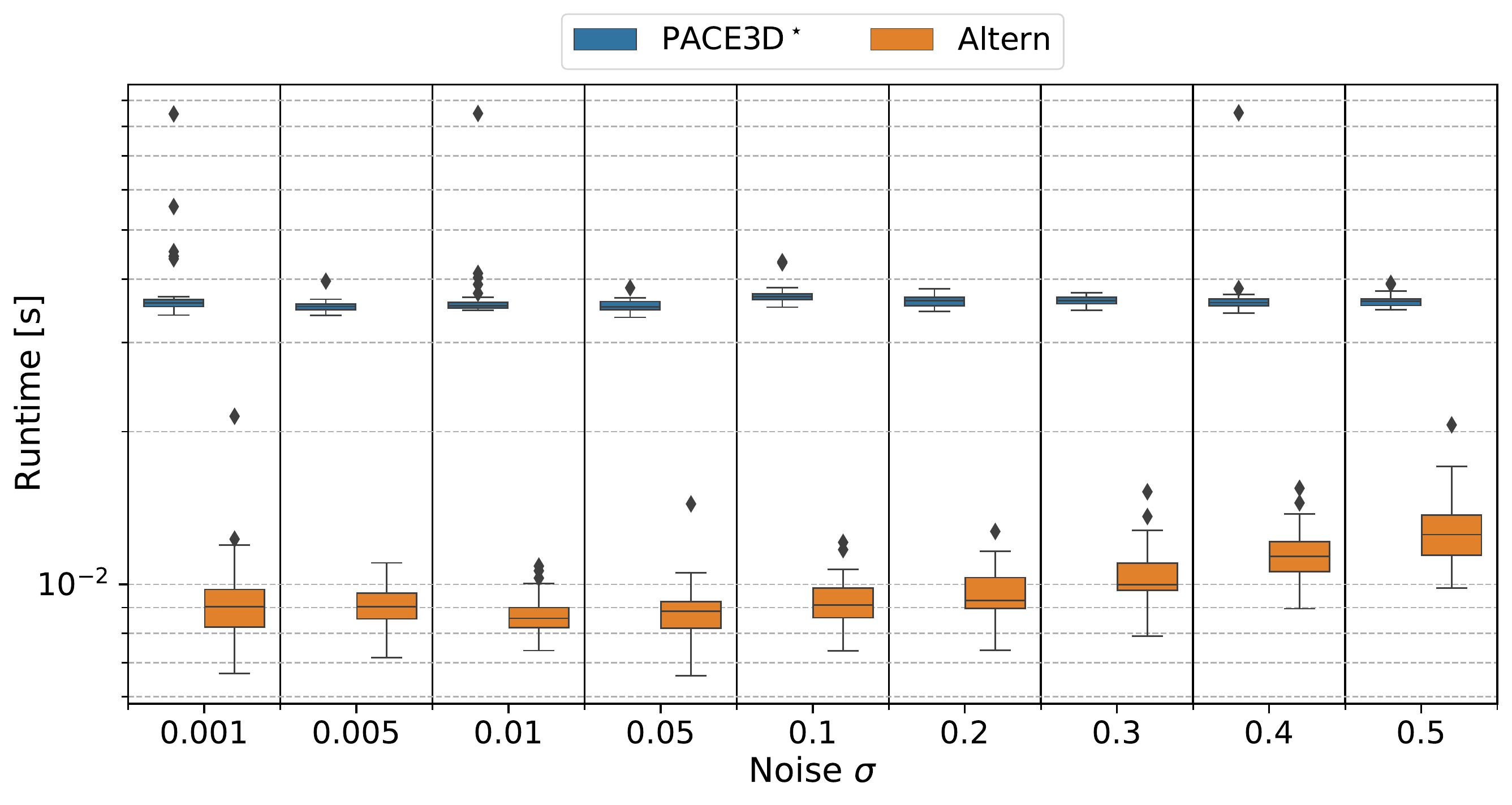}
			\end{minipage}
		\\
		\multicolumn{4}{c}{\smaller (c) Average residuals and runtimes of \PACEThree~on outlier-free synthetic data with varying noise levels: $N=100$, $K=10$. }
	\end{tabular}
	\end{minipage} 
	\caption{\JS{Residuals and runtimes of \PACEThree and \altern with (a) varying number of shapes, (b) varying number of points and (c) varying level of noise. Each boxplot reports statistics computed over 50 Monte Carlo runs.}
	\label{fig:app-simulation-optimality-residuals-runtime-3d}}
	\vspace{-4mm} 
	\end{center}
\end{figure*}

\begin{figure*}[!t]
	\begin{center}
	\begin{minipage}{\textwidth}
	\begin{tabular}{cccc}%
		\myhspace \hspace{-3mm}
			\begin{minipage}{\mpwfour}%
			\centering%
			\includegraphics[width=\columnwidth]{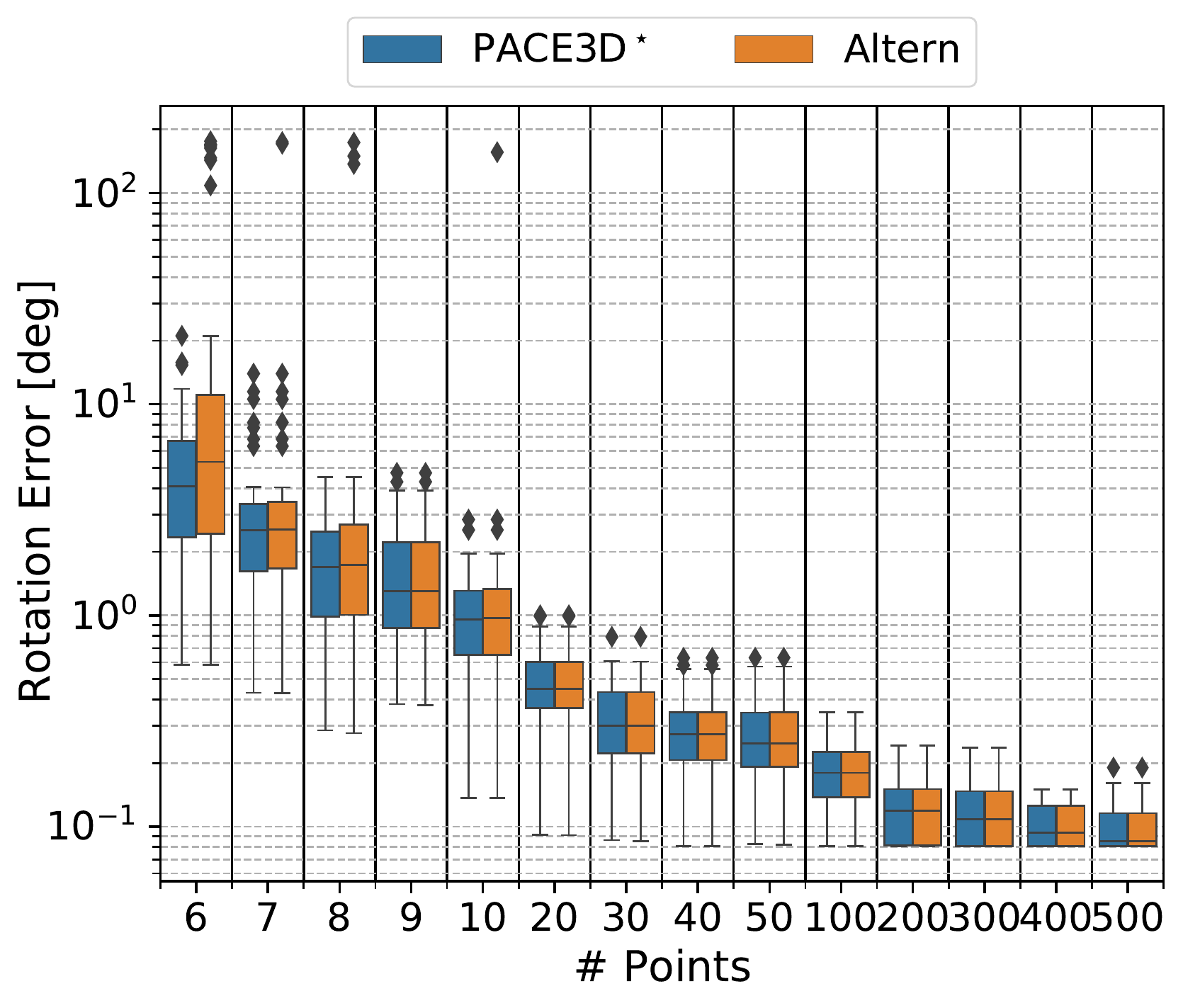}
			\end{minipage}
		&   \myhspace
			\begin{minipage}{\mpwfour}%
			\centering%
			\includegraphics[width=\columnwidth]{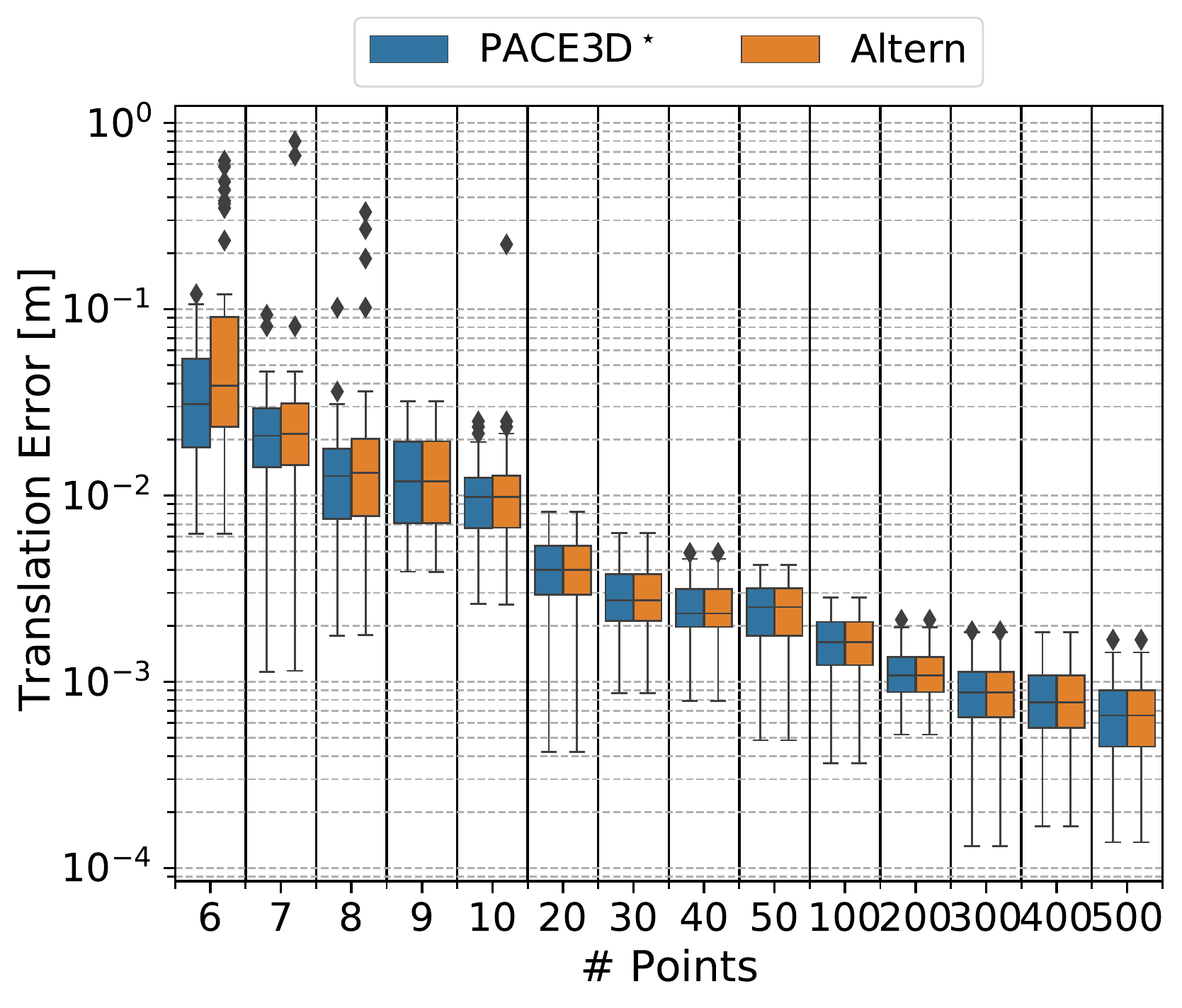}
			\end{minipage}
		&   \myhspace
			\begin{minipage}{\mpwfour}%
			\centering%
			\includegraphics[width=\columnwidth]{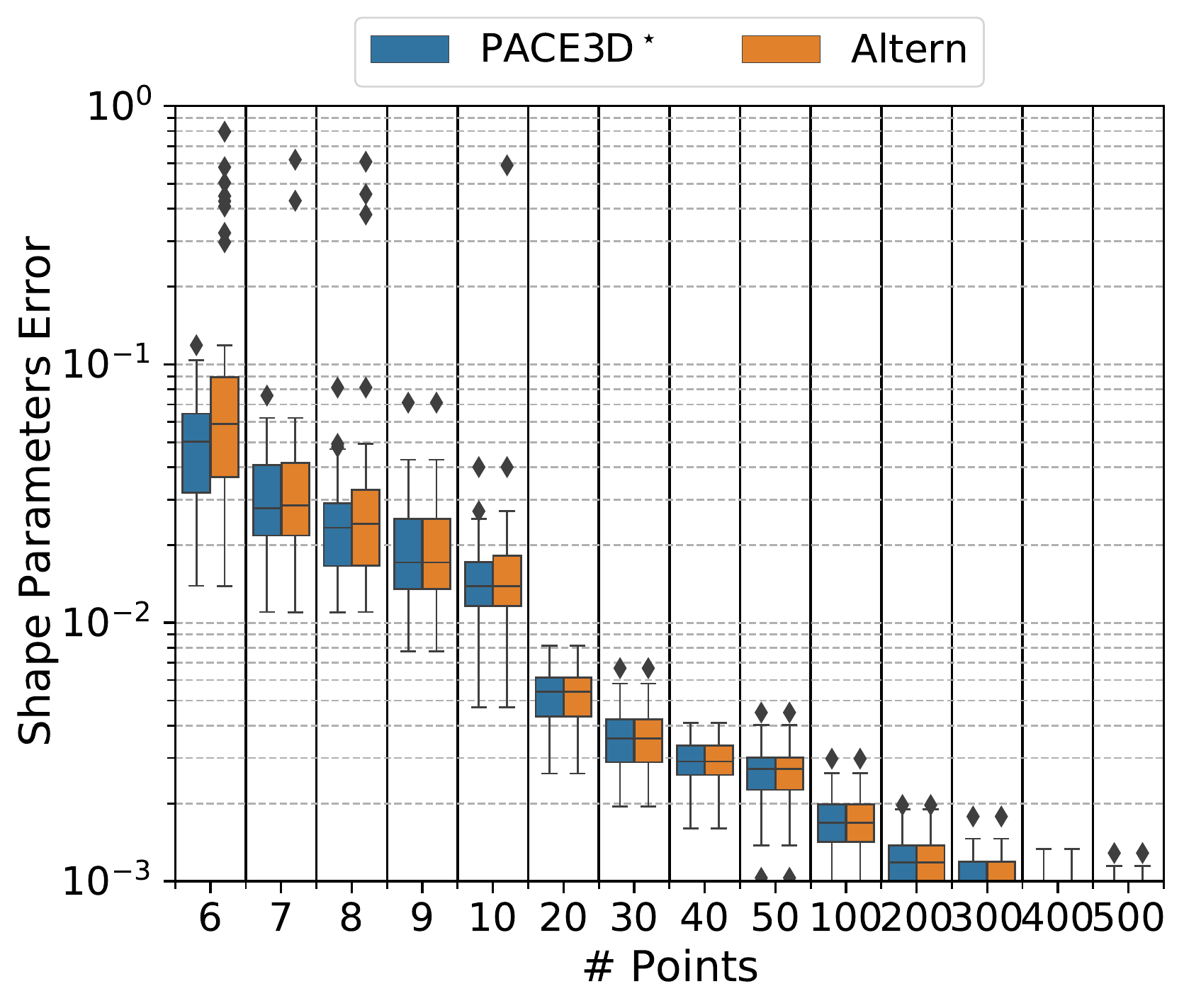}
			\end{minipage}
		&   \myhspace
			\begin{minipage}{\mpwfour}%
			\centering%
			\includegraphics[width=\columnwidth]{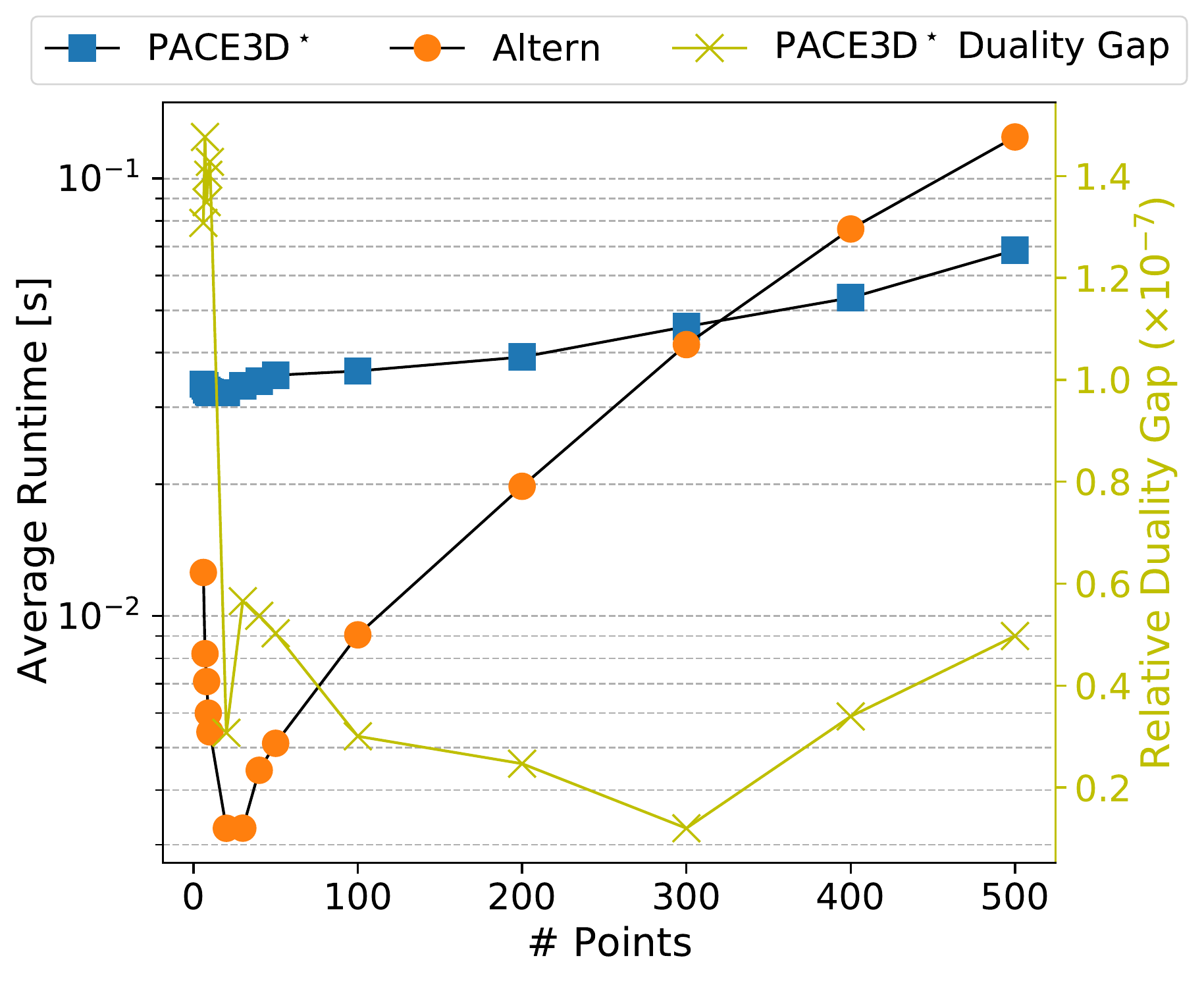}
			\end{minipage}
		\\
	  \multicolumn{4}{c}{\smaller (a) Performance of \PACEThree~on outlier-free synthetic data with varying number of shapes: $N=100$, $\sigma=0.01$. }
	  \\
		\myhspace \hspace{-3mm}
			\begin{minipage}{\mpwfour}%
			\centering%
			\includegraphics[width=\columnwidth]{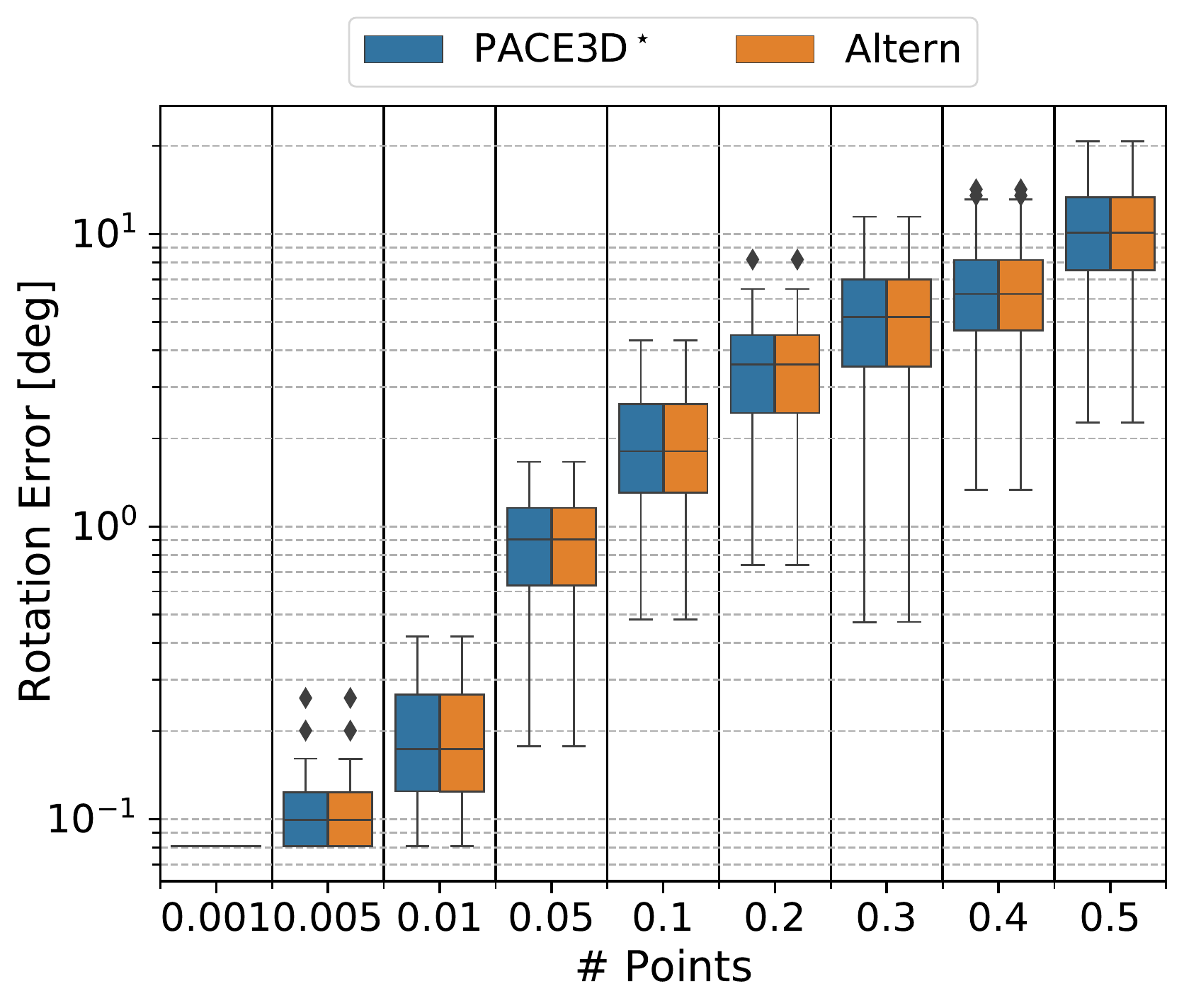}
			\end{minipage}
		&   \myhspace
			\begin{minipage}{\mpwfour}%
			\centering%
			\includegraphics[width=\columnwidth]{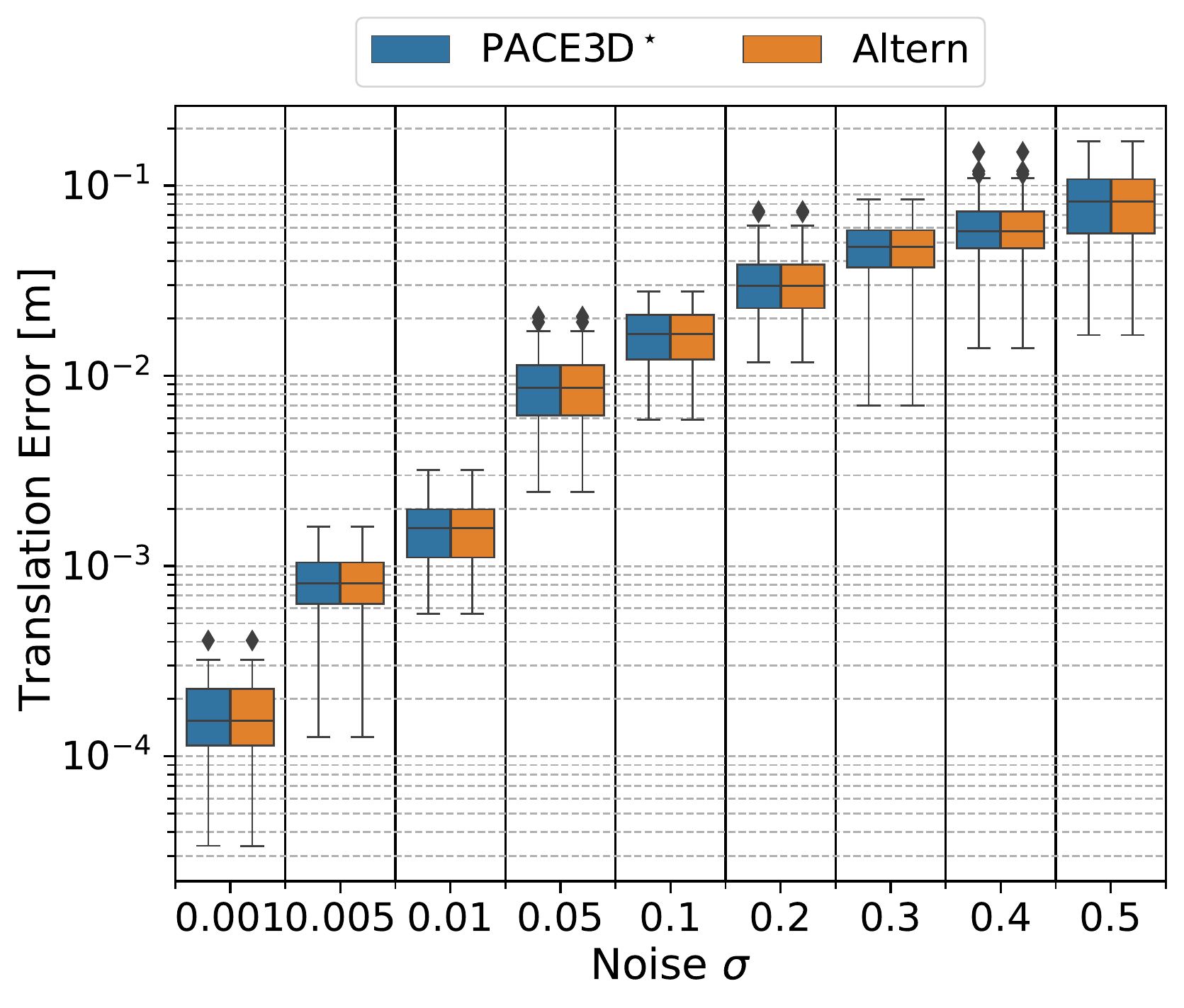}
			\end{minipage}
		&   \myhspace
			\begin{minipage}{\mpwfour}%
			\centering%
			\includegraphics[width=\columnwidth]{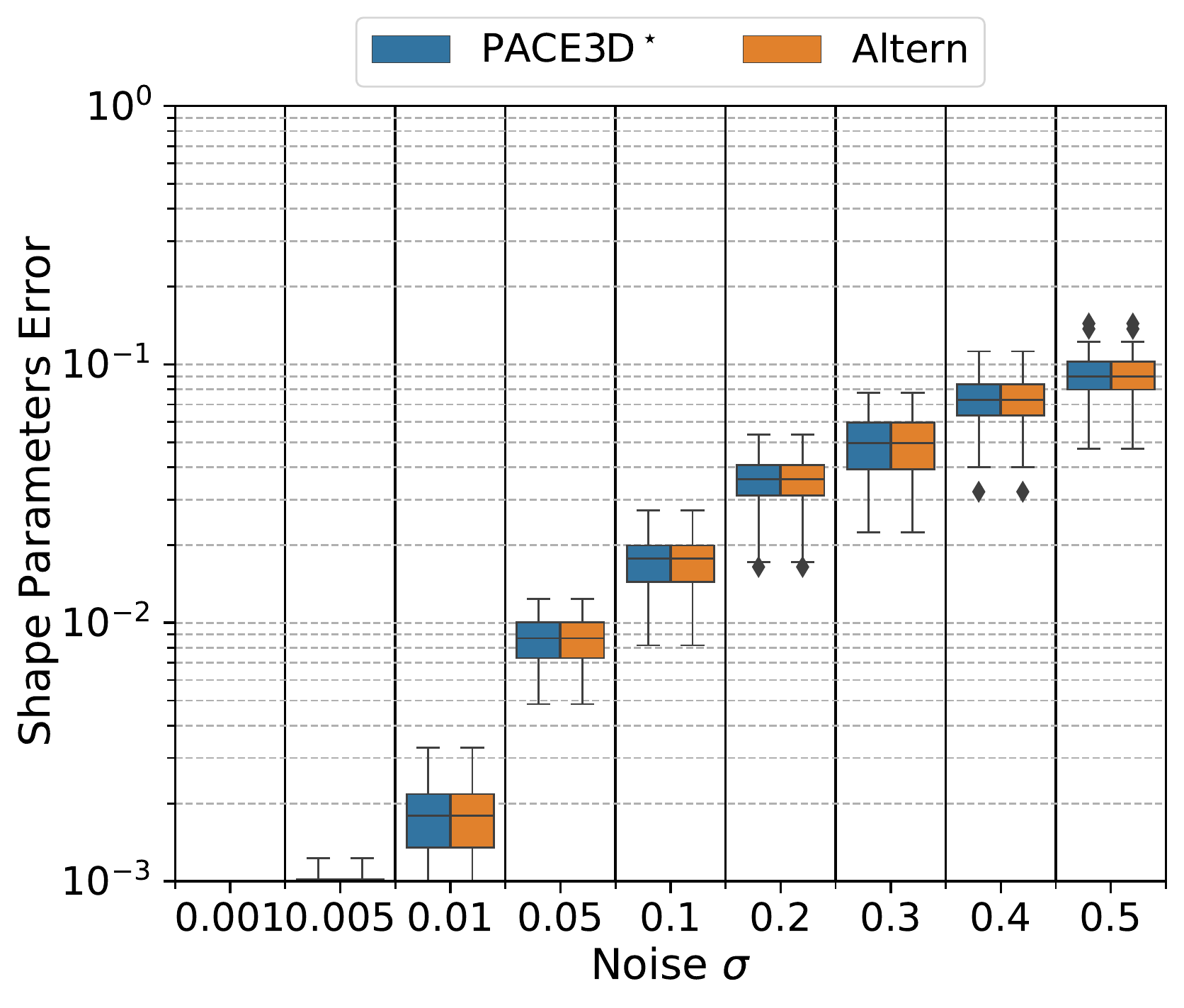}
			\end{minipage}
		&   \myhspace
			\begin{minipage}{\mpwfour}%
			\centering%
			\includegraphics[width=\columnwidth]{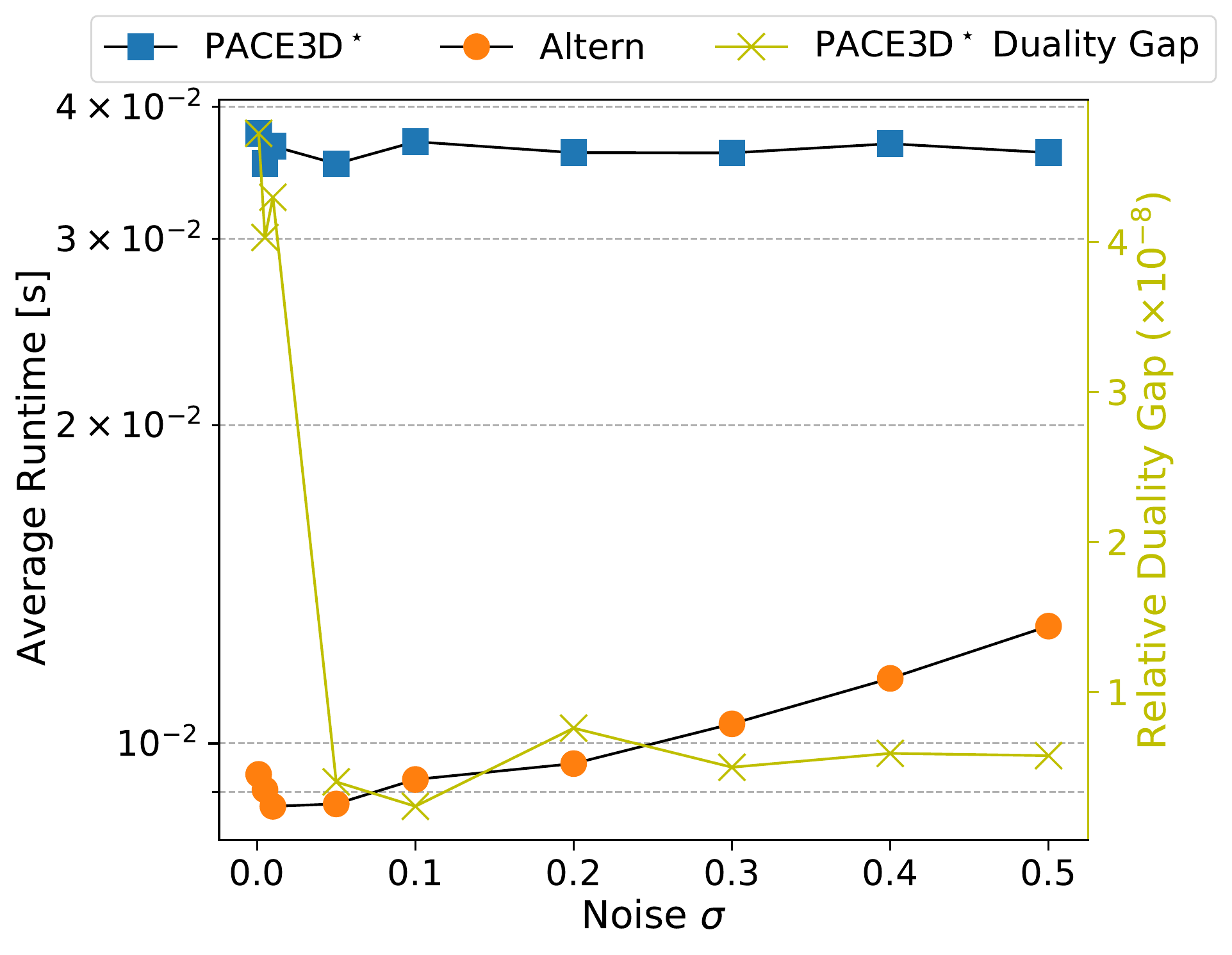}
			\end{minipage}
		\\
		\multicolumn{4}{c}{\smaller (b) Performance of \PACEThree~on outlier-free synthetic data with varying noise level: $N=100$, $K=10$.}
	\end{tabular}
	\end{minipage}
	\caption{Performance of \PACEThree on outlier-free synthetic data, with (a) varying number of points and (b) varying level of noise. Each boxplot reports statistics computed over 50 Monte Carlo runs.
	\label{fig:app-simulation-optimality-varying}}
	\vspace{-4mm}
	\end{center}
\end{figure*}

\renewcommand{\mpwfour}{4.6cm}
\renewcommand{\myhspace}{\hspace{-3.5mm}}
\begin{figure*}[!t]
	\begin{center}
	\begin{minipage}{\textwidth}
	\begin{tabular}{cccc}%
		\myhspace \hspace{-3mm}
			\begin{minipage}{\mpwfour}%
			\centering%
			\includegraphics[width=\columnwidth]{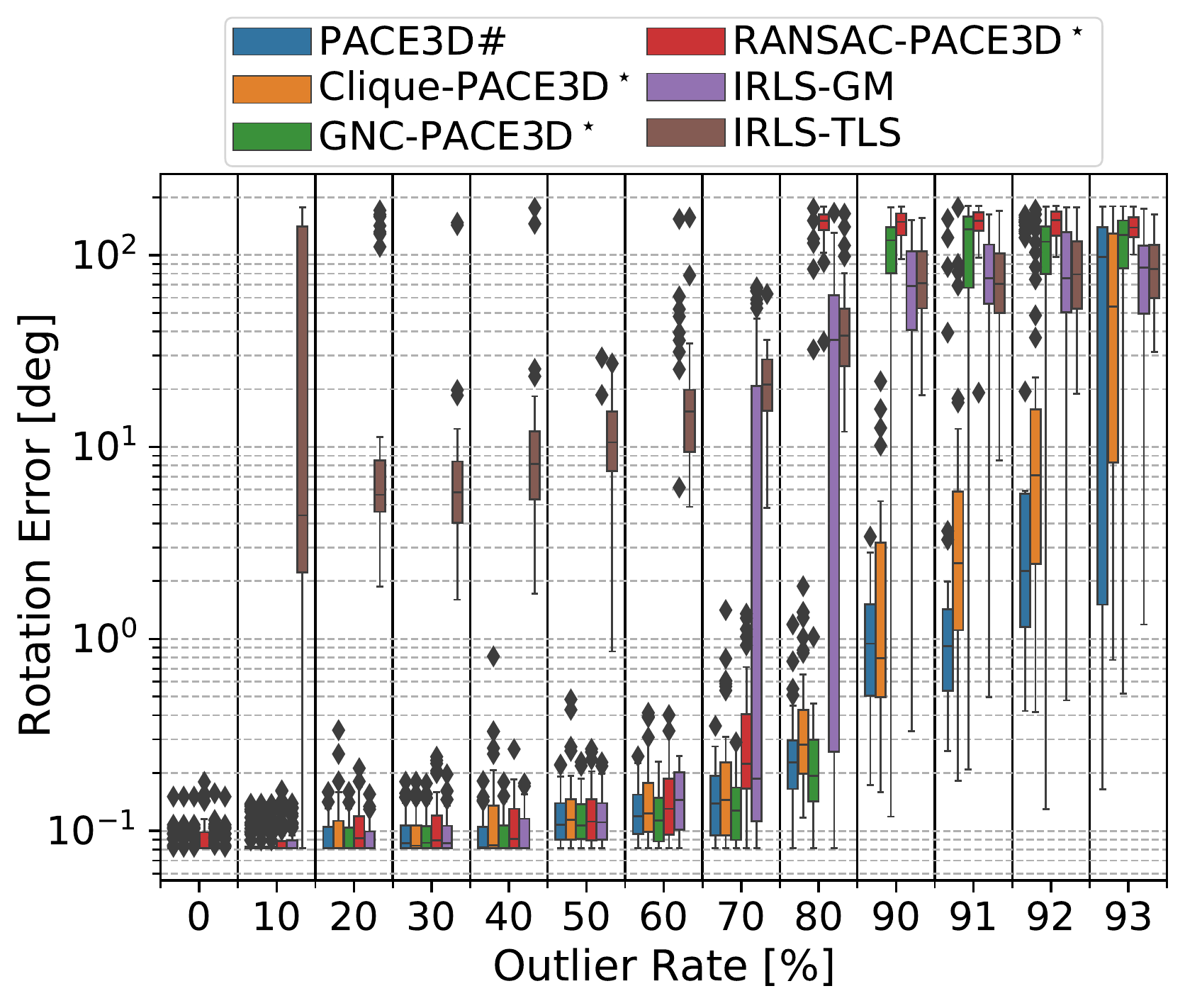}
			\end{minipage}
		&   \myhspace
			\begin{minipage}{\mpwfour}%
			\centering%
			\includegraphics[width=\columnwidth]{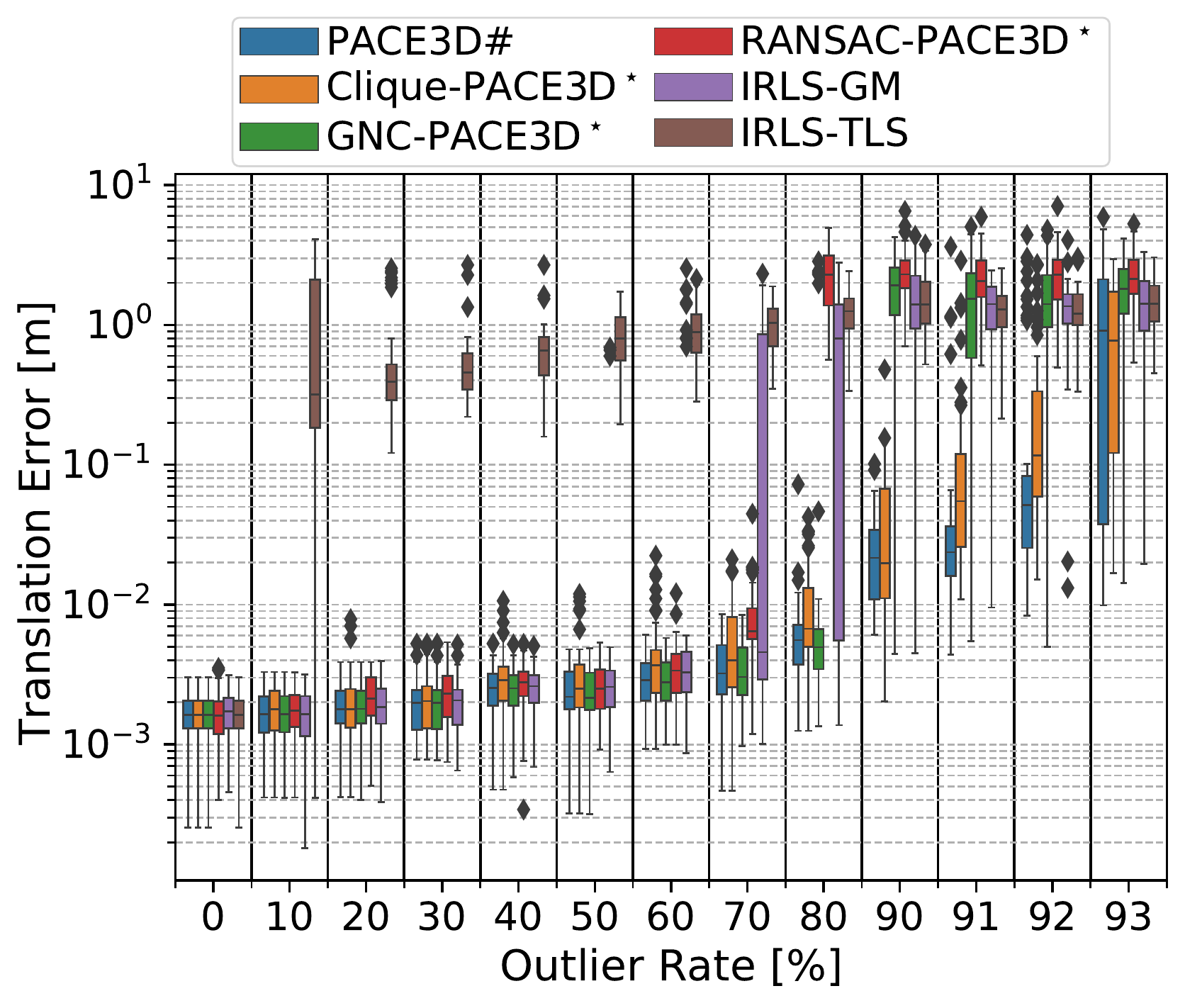}
			\end{minipage}
		&   \myhspace
			\begin{minipage}{\mpwfour}%
			\centering%
			\includegraphics[width=\columnwidth]{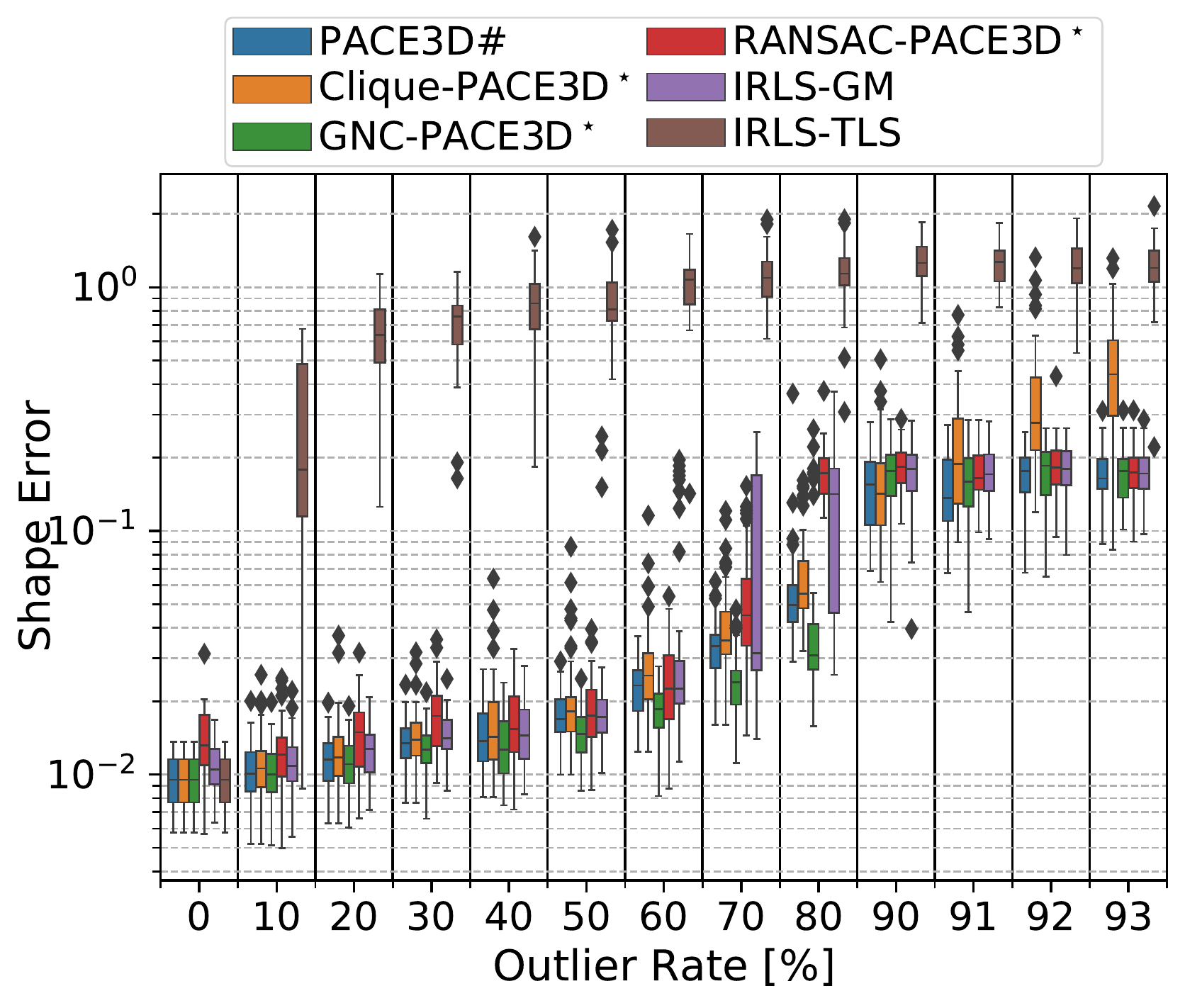}
			\end{minipage}
		&   \myhspace
			\begin{minipage}{\mpwfour}%
			\centering%
			\includegraphics[width=\columnwidth]{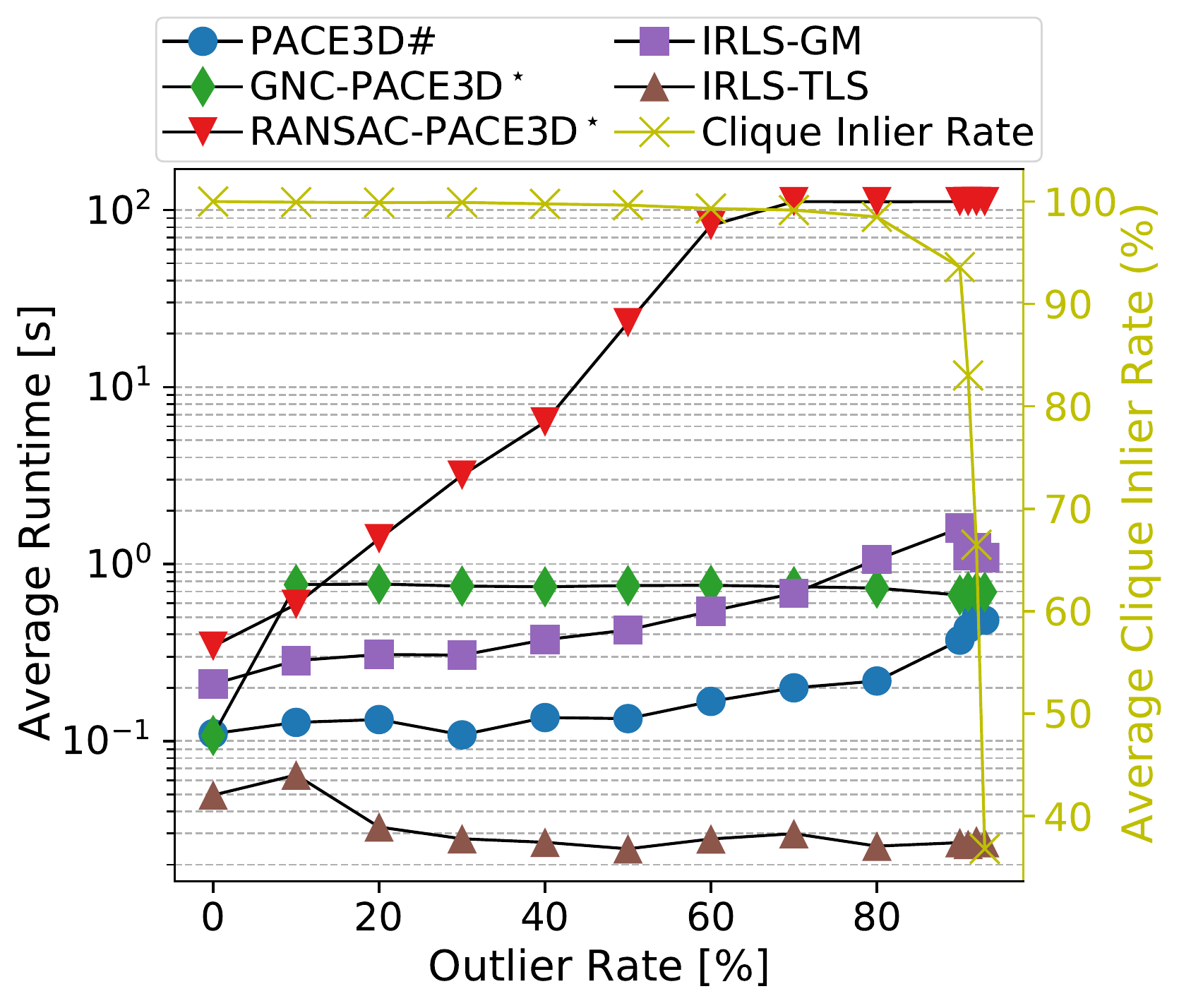}
			\end{minipage}
		\\
	\end{tabular}
	\end{minipage} 
	\caption{Performance of \PACErobustThree compared to baselines in simulated experiments
	 with
	the intra-class variation radius increased to $r=0.2$.
	Each boxplot (and lineplot) reports statistics computed over  50 Monte Carlo runs.
	\label{fig:app-simulation-robust-3d}}
	\vspace{-4mm} 
	\end{center}
\end{figure*}

\section{Additional Experimental Results}\label{sec:app-experiments-pace}

\subsection{Results on \PACEThree and \PACErobustThree}\label{sec:app-experiments-pace-3d}
\JS{
In Section~\ref{sec:exp-optimality-robustness-3d}, we demonstrate the optimality of \PACErobustThree when $N=100$, $\sigma=0.01$ with varying number of shapes ($K$).
Here we show extra results reporting the performance with different $N$ and $\sigma$.
In addition, we show more detailed runtimes and the final average residuals achieved.
Fig.~\ref{fig:app-simulation-optimality-varying}(a) shows the performance of \PACEThree with $N$ ranging from 6 to 500.
One can see that \altern shows a small number of failures at lower $N$, while the performance of  both \PACEThree and \altern improves  with
the increase of $N$.
Fig.~\ref{fig:app-simulation-optimality-varying}(b) shows the performance of \PACEThree with $\sigma$ ranging from 0.001 to 0.5.
Both \PACEThree and \altern suffer degradation of performance with increase in $\sigma$.
Fig.~\ref{fig:app-simulation-optimality-residuals-runtime-3d} shows the final residuals and runtimes of \PACEThree and \altern for three scenarios: (a) varying $K$, (b) varying $N$, (c) varying $\sigma$.

In Section~\ref{sec:exp-optimality-robustness-3d}, we also demonstrate the robustness of \PACErobustThree to $92\%$ outlier rates when $N=100$, $K=10$, and $r=0.1$. Here we show extra results when $r$ is increased.
Fig.~\ref{fig:app-simulation-robust-3d} shows the results for $N=100$, $K=10$, and $r=0.2$. One can see that as the intra-class variation radius $r$ is increased, the compatibility checks become less effective, leading to a slight decrease in the robustness of \PACErobustThree against outliers; \PACErobustThree is still robust to up to $90\%$ outliers while has two failures at $91\%$ outliers. However, \PACErobustThree still outperforms \irlstls and \irlsgm by a large margin.

}

\renewcommand{\mpwfour}{4.33cm}
\renewcommand{\myhspace}{\hspace{-3.5mm}}

\begin{figure}[hbtp!]
  \centering
  \includegraphics[width=0.7\columnwidth]{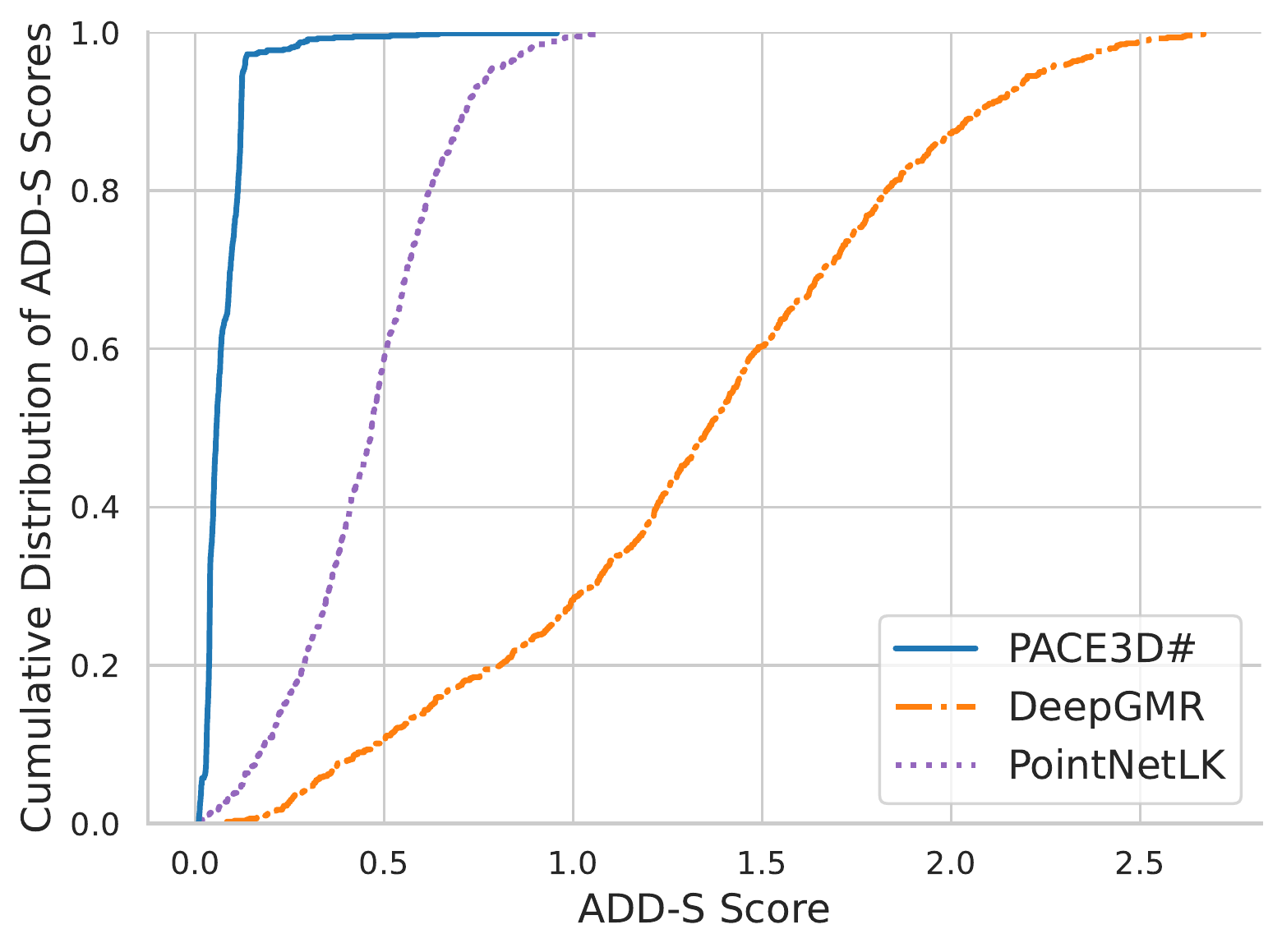}
  \caption{Cumulative distribution of ADD-S score for \PACErobustThree, PointNetLK~\cite{Li21cvpr-PointNetLKRevisited} and DeepGMR~\cite{Yuan20arxiv-DeepGMRLearning} in the \keypointnet experiment,
    averaged across all 16 categories.\label{fig:app-keypointnet-avg}}
\end{figure}

\begin{figure*}[hbtp!]
	\begin{center}{}
	\begin{minipage}{\textwidth}
	\begin{tabular}{cccc}%
\myhspace
	\begin{minipage}{\mpwfour}%
	\centering%
	\includegraphics[width=\columnwidth]{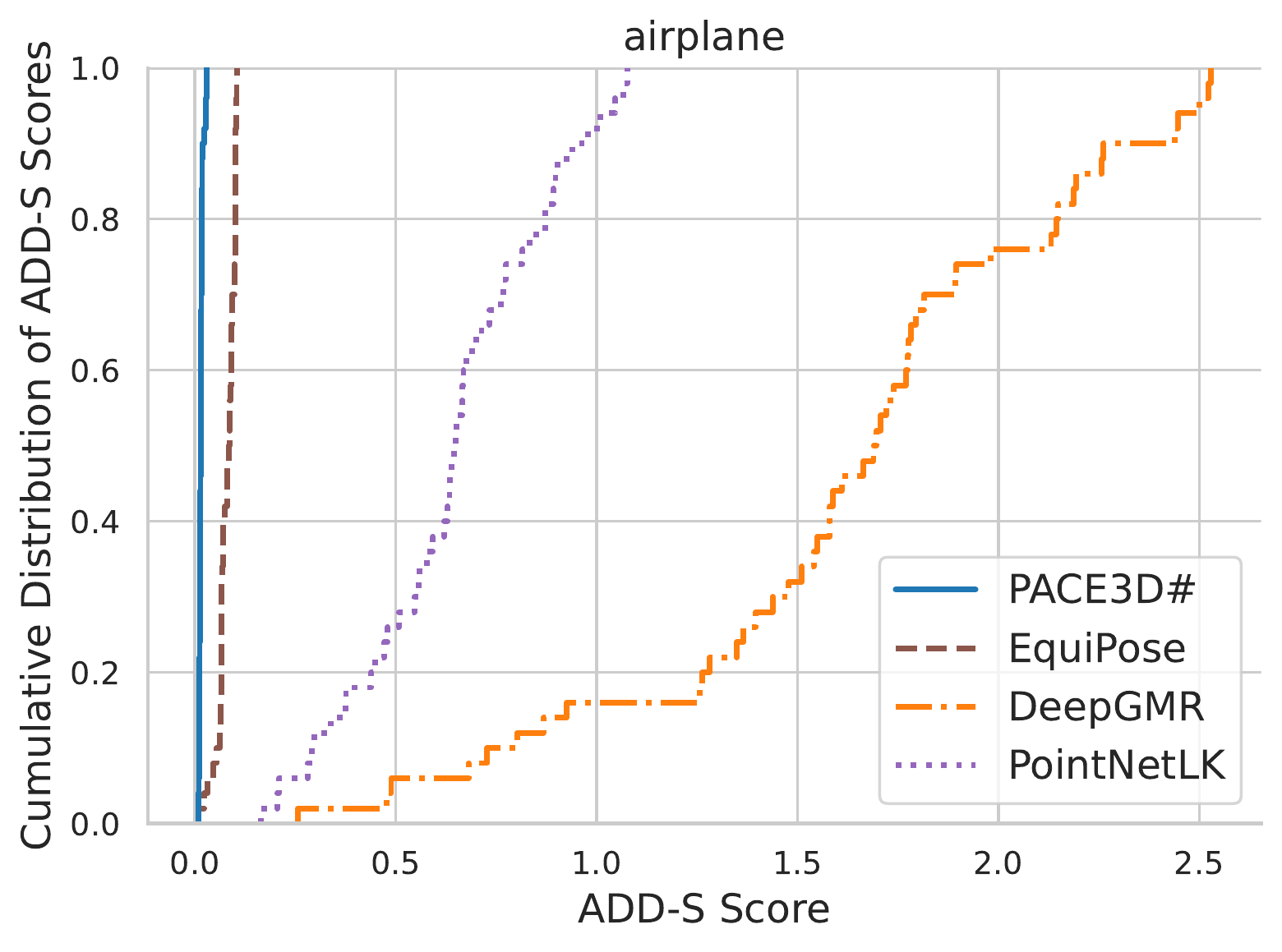} \\
	\vspace{1mm}
	\end{minipage}
& \myhspace
	\begin{minipage}{\mpwfour}%
	\centering%
	\includegraphics[width=\columnwidth]{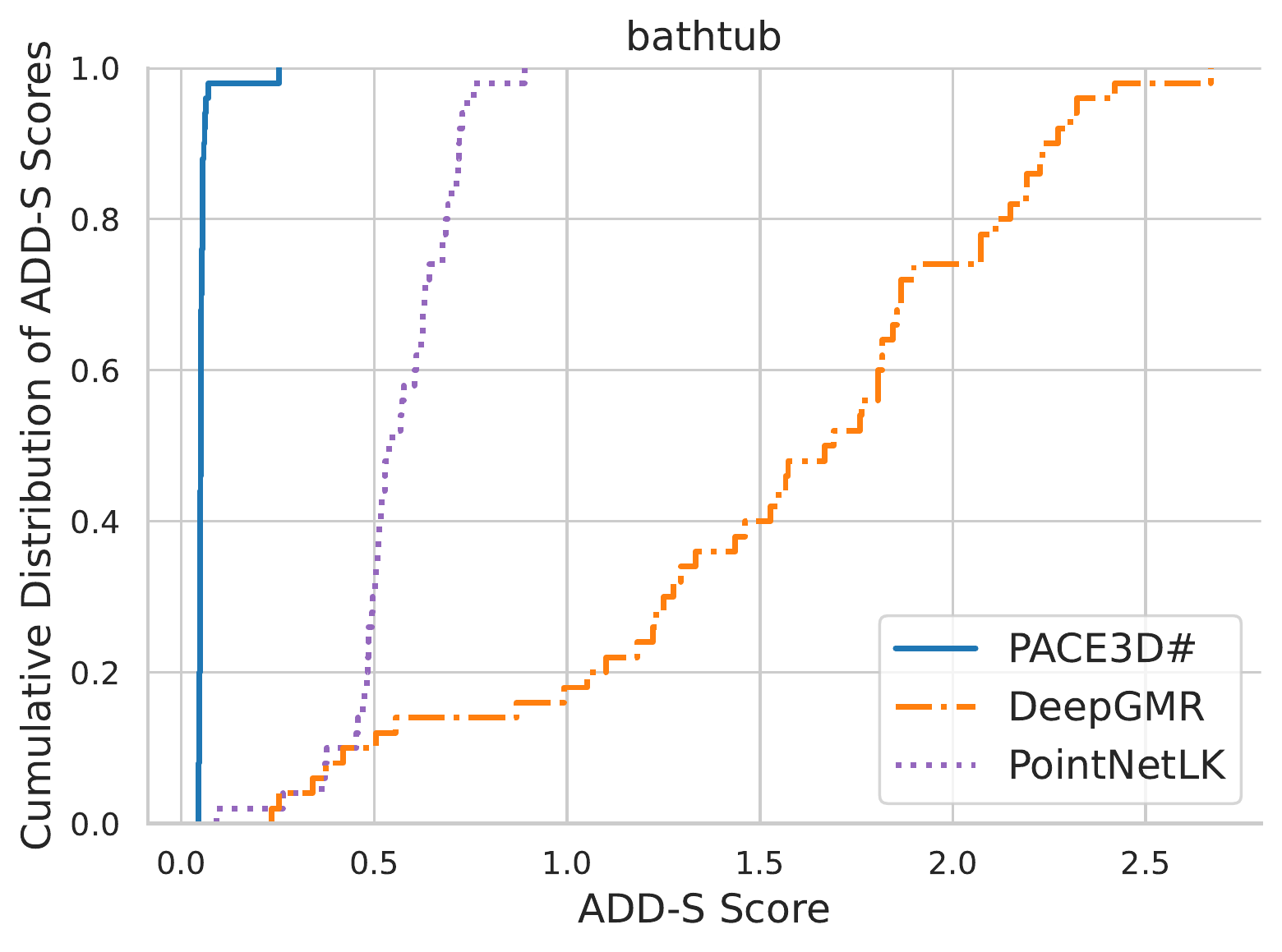} \\
	\vspace{1mm}
	\end{minipage}
& \myhspace
	\begin{minipage}{\mpwfour}%
	\centering%
	\includegraphics[width=\columnwidth]{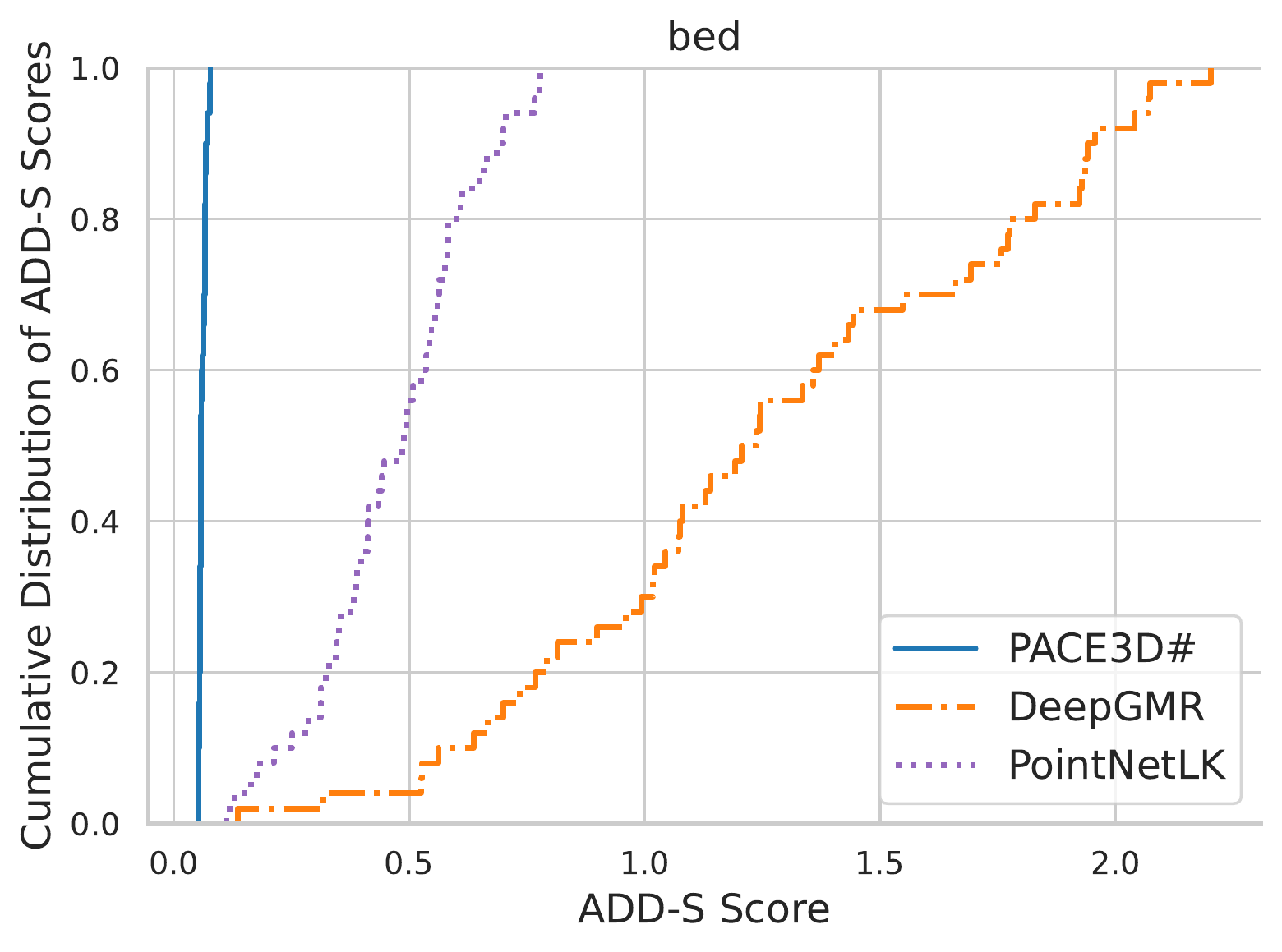} \\
	\vspace{1mm}
	\end{minipage}
& \myhspace
	\begin{minipage}{\mpwfour}%
	\centering%
	\includegraphics[width=\columnwidth]{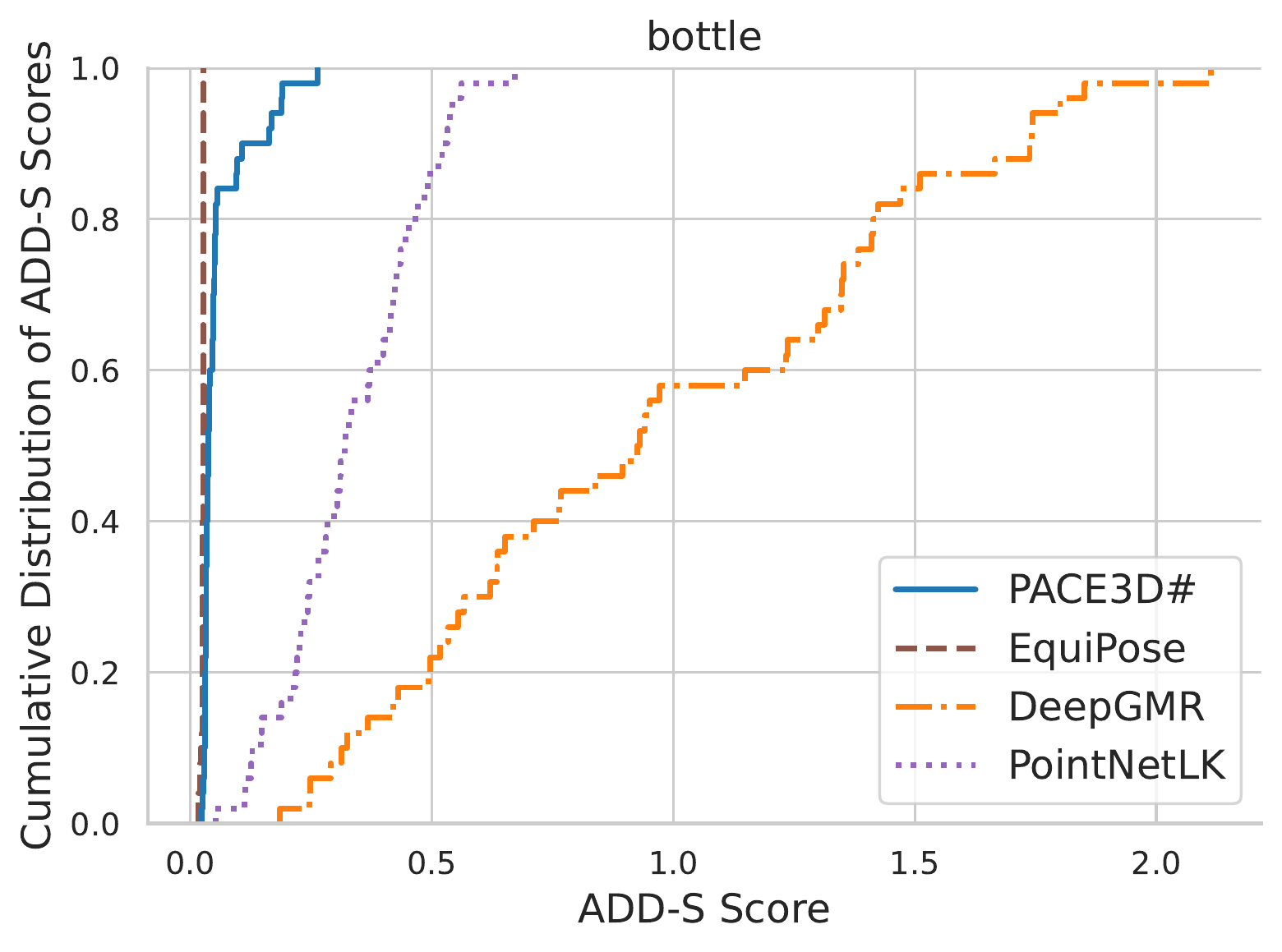} \\
	\vspace{1mm}
	\end{minipage} \\
\myhspace
	\begin{minipage}{\mpwfour}%
	\centering%
	\includegraphics[width=\columnwidth]{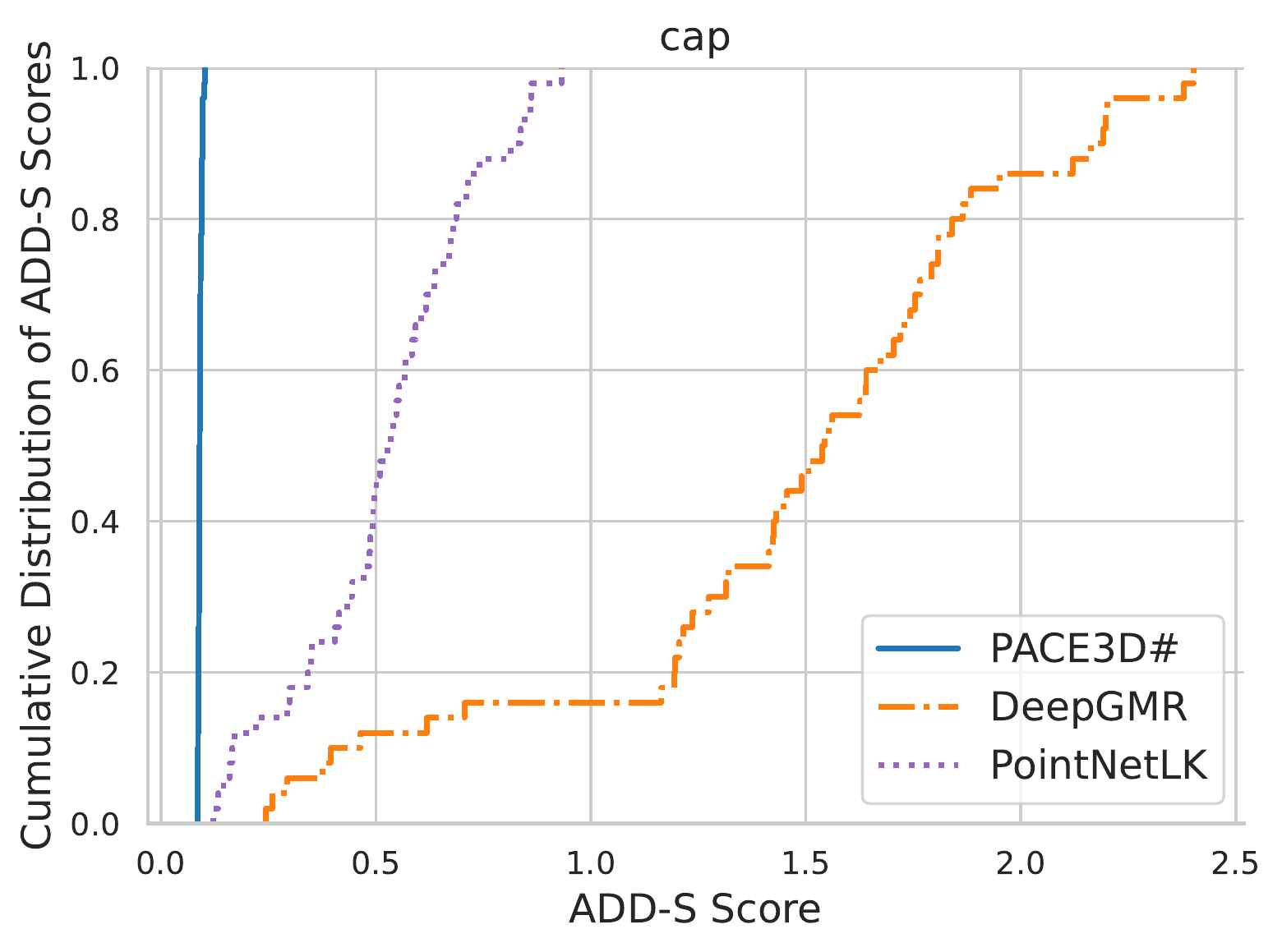} \\
	\vspace{1mm}
	\end{minipage}
& \myhspace
	\begin{minipage}{\mpwfour}%
	\centering%
	\includegraphics[width=\columnwidth]{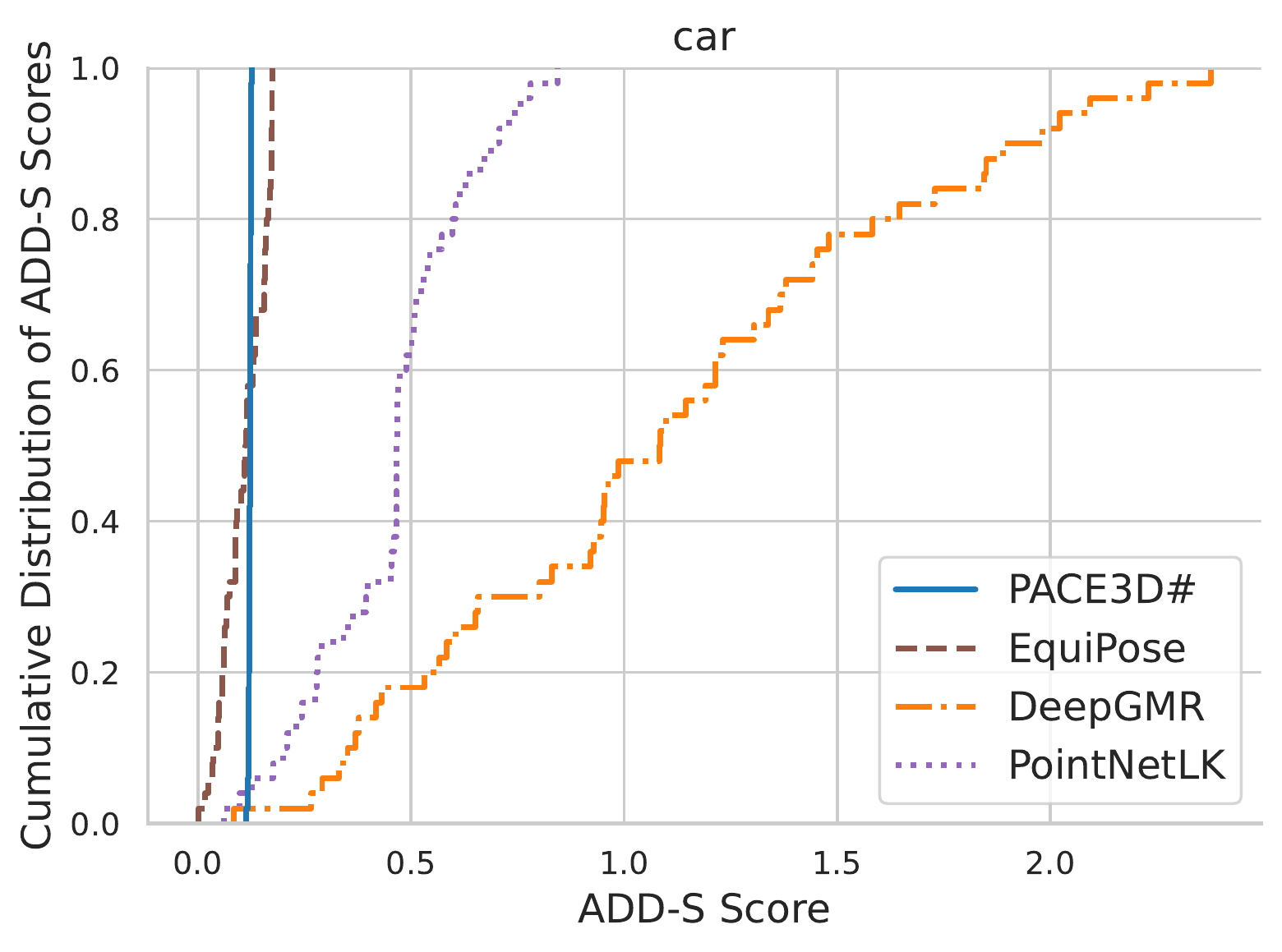} \\
	\vspace{1mm}
	\end{minipage}
& \myhspace
	\begin{minipage}{\mpwfour}%
	\centering%
	\includegraphics[width=\columnwidth]{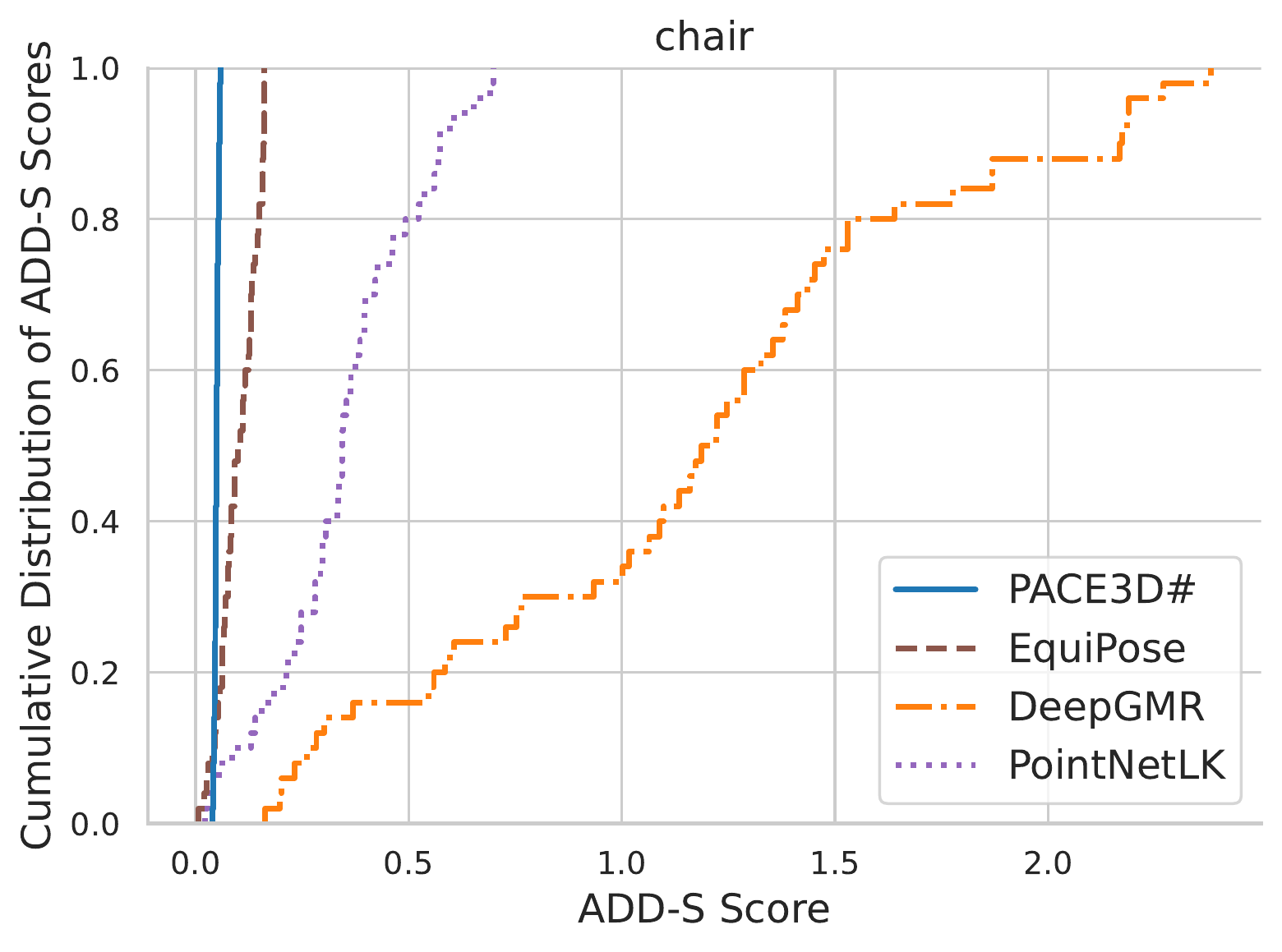} \\
	\vspace{1mm}
	\end{minipage}
& \myhspace
	\begin{minipage}{\mpwfour}%
	\centering%
	\includegraphics[width=\columnwidth]{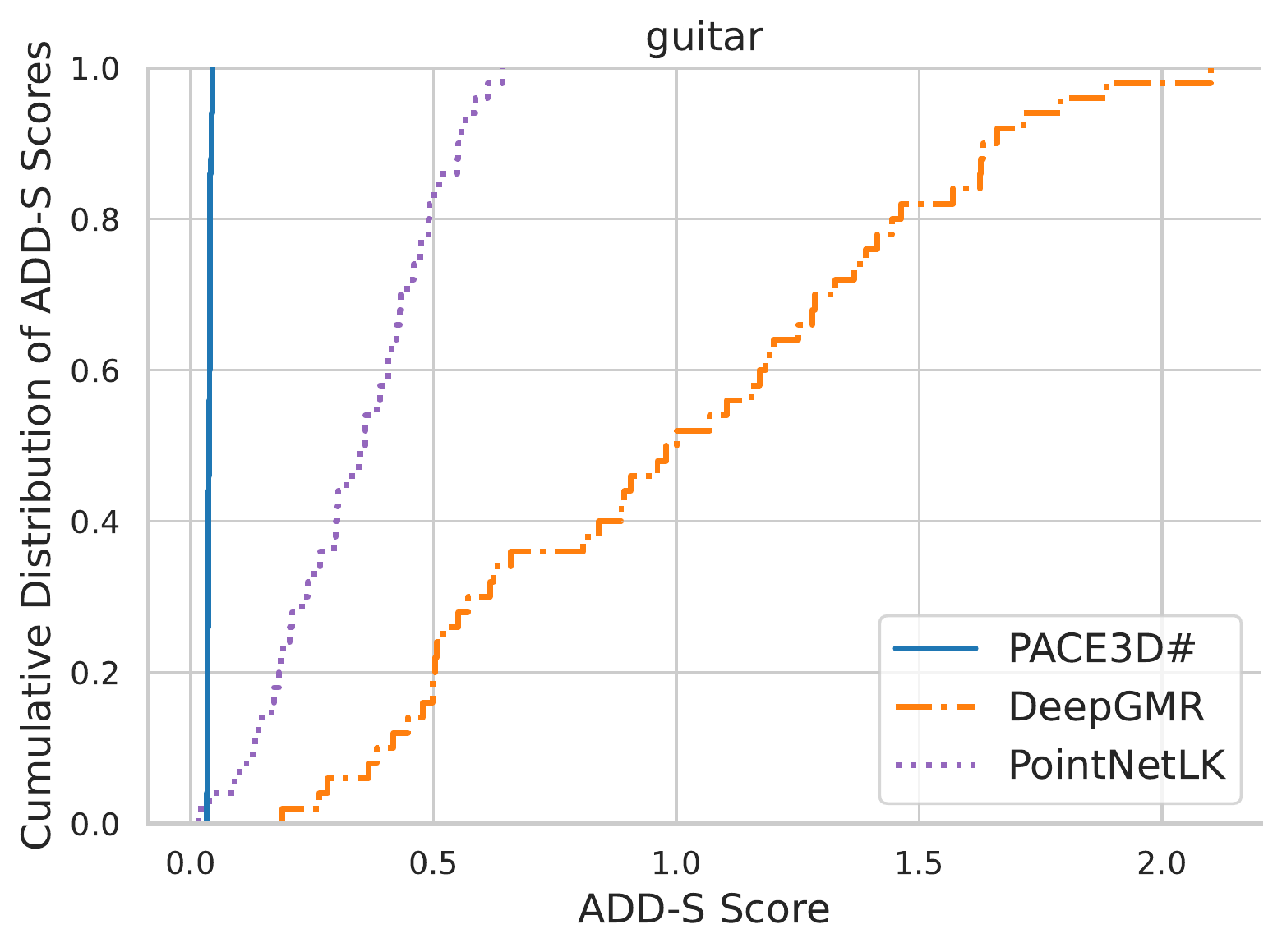} \\
	\vspace{1mm}
	\end{minipage} \\
\myhspace
	\begin{minipage}{\mpwfour}%
	\centering%
	\includegraphics[width=\columnwidth]{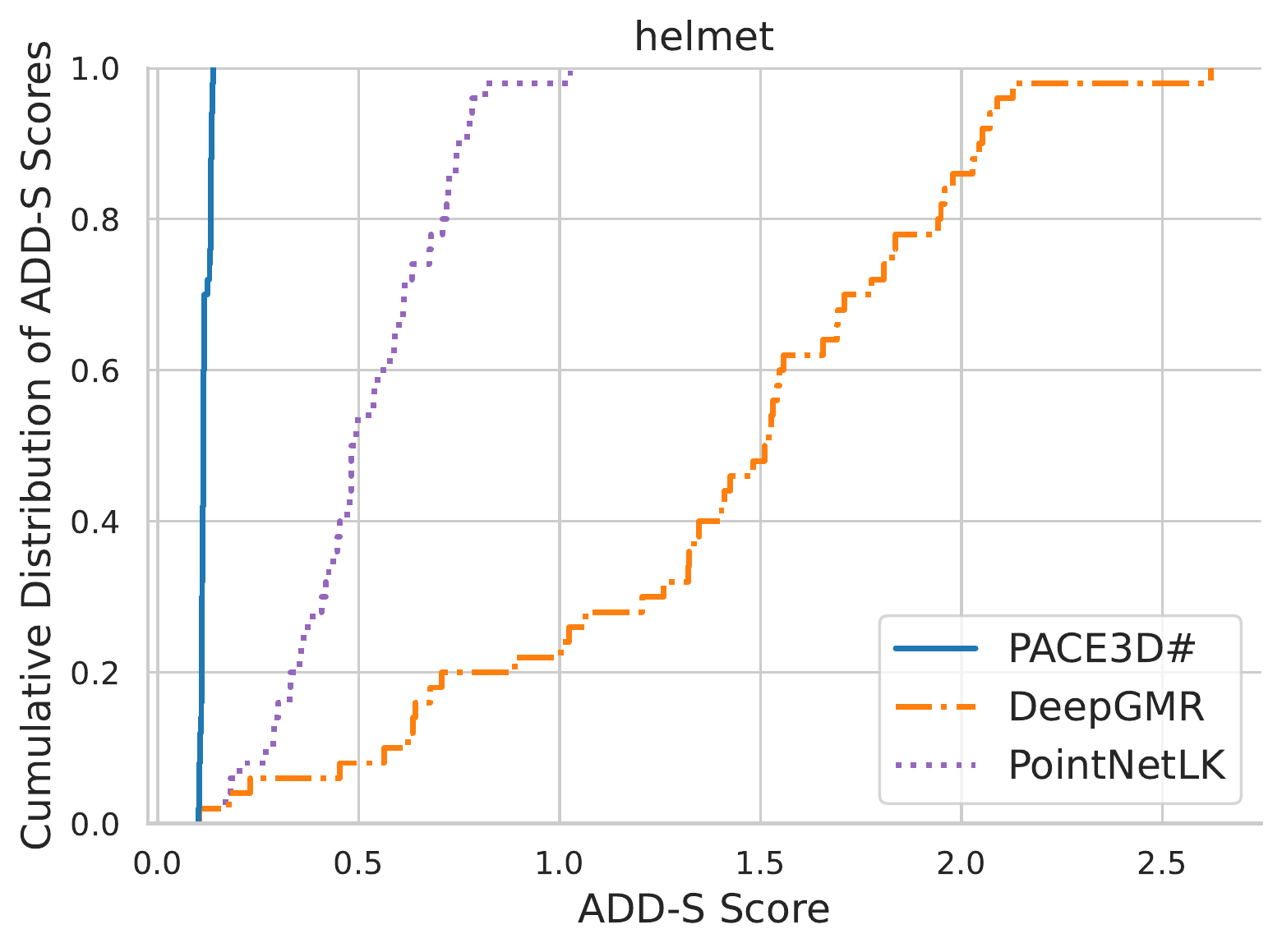} \\
	\vspace{1mm}
	\end{minipage}
& \myhspace
	\begin{minipage}{\mpwfour}%
	\centering%
	\includegraphics[width=\columnwidth]{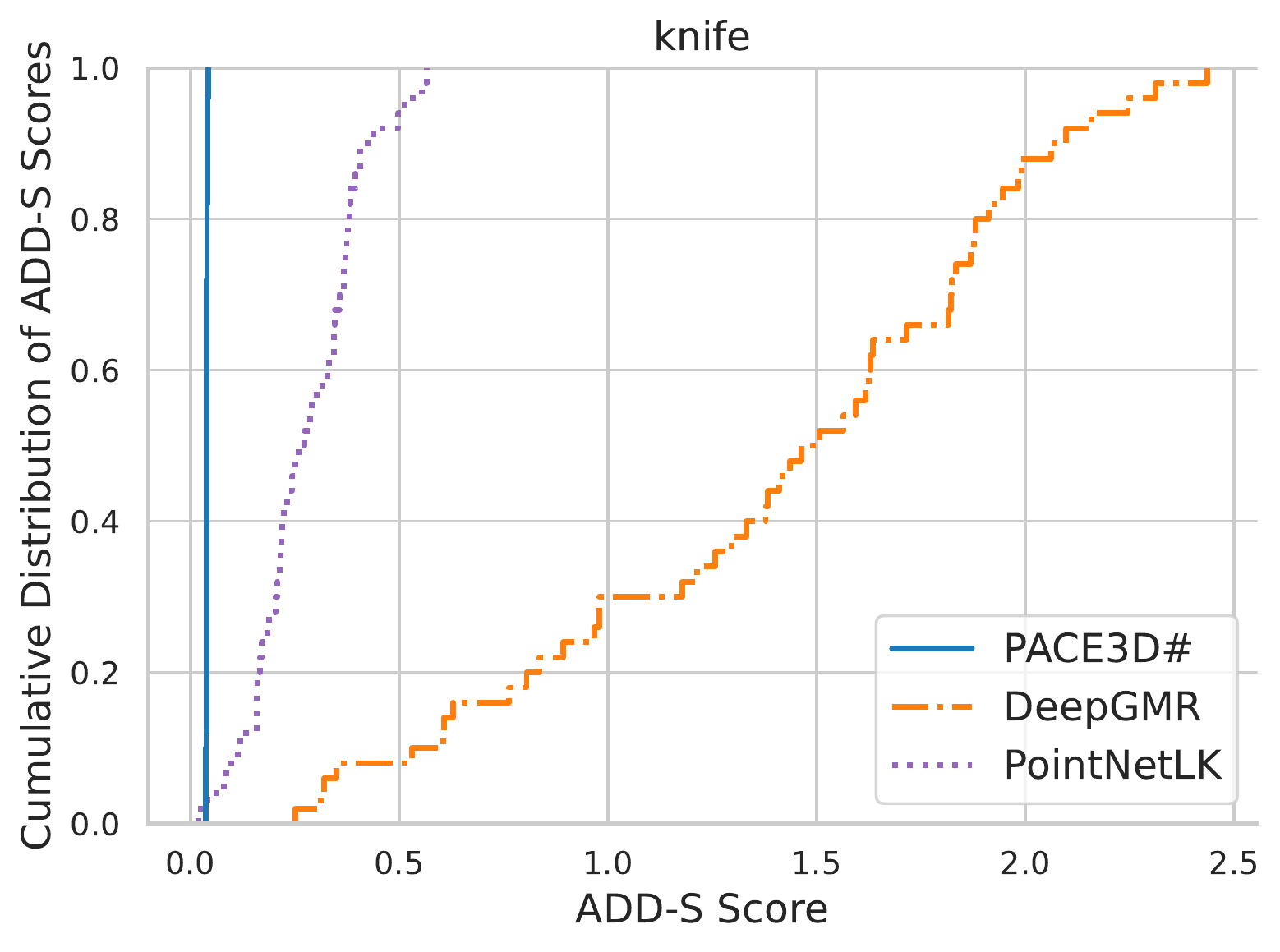} \\
	\vspace{1mm}
	\end{minipage}
& \myhspace
	\begin{minipage}{\mpwfour}%
	\centering%
	\includegraphics[width=\columnwidth]{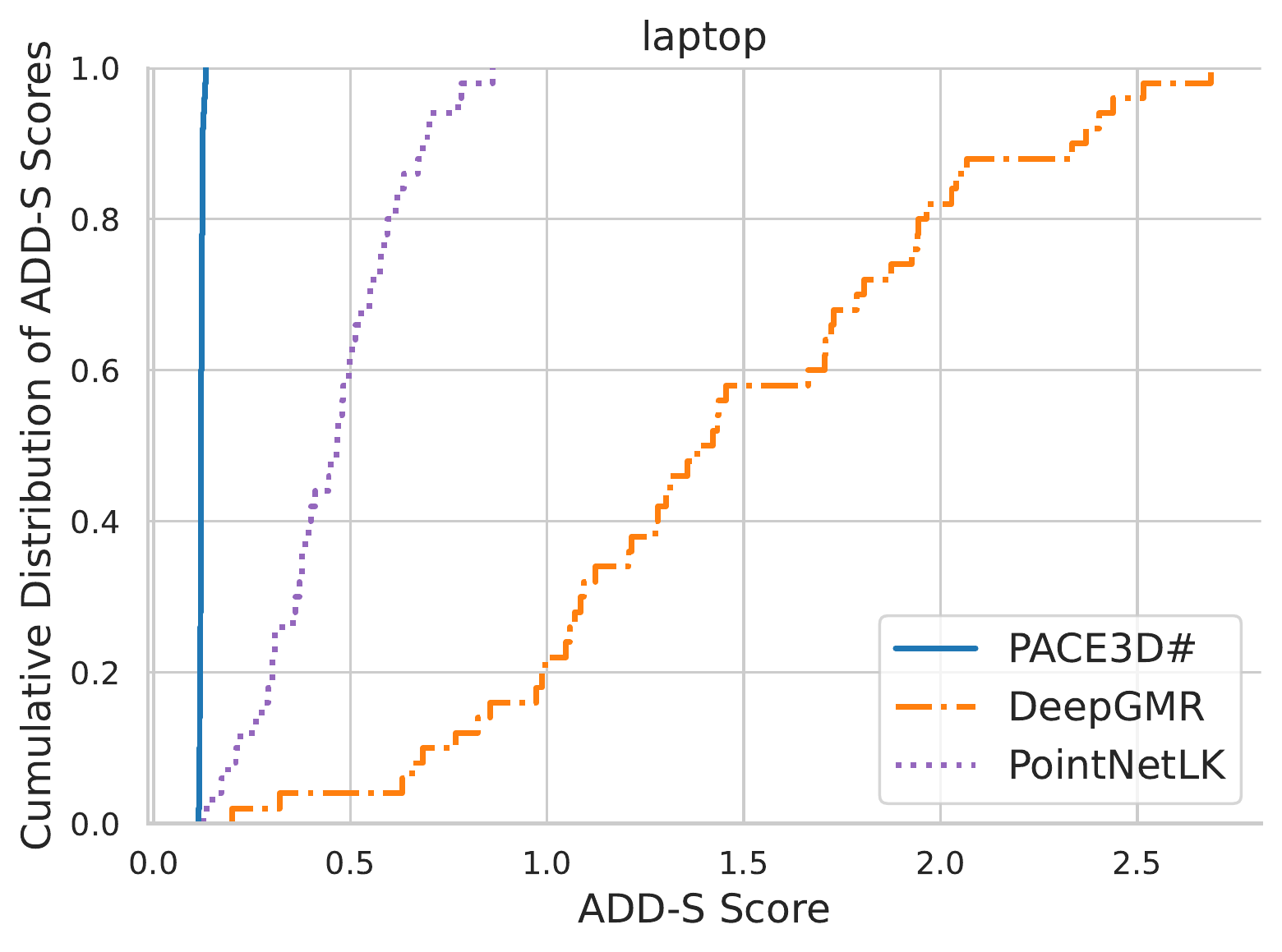} \\
	\vspace{1mm}
	\end{minipage}
& \myhspace
	\begin{minipage}{\mpwfour}%
	\centering%
	\includegraphics[width=\columnwidth]{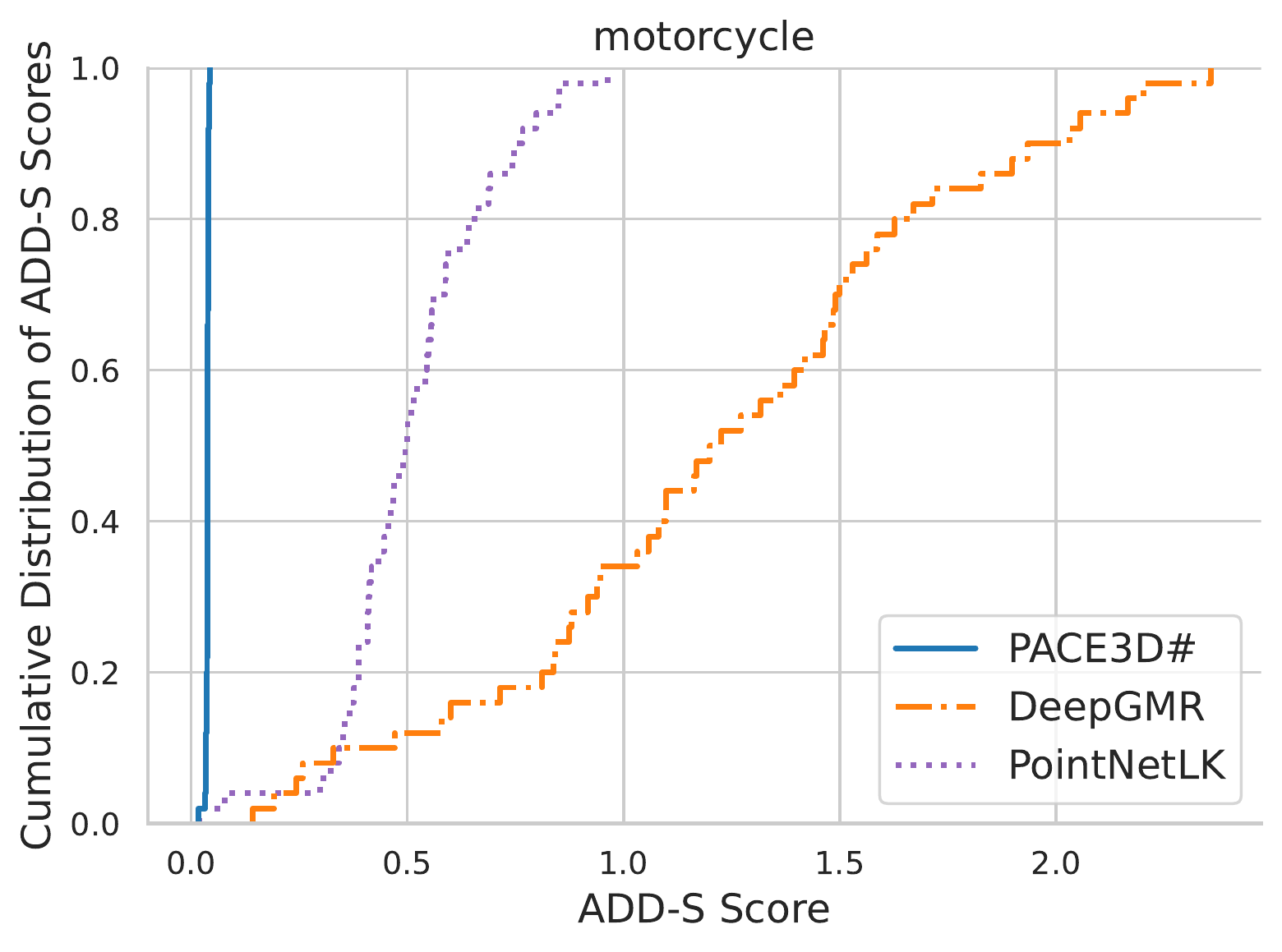} \\
	\vspace{1mm}
	\end{minipage} \\
\myhspace
	\begin{minipage}{\mpwfour}%
	\centering%
	\includegraphics[width=\columnwidth]{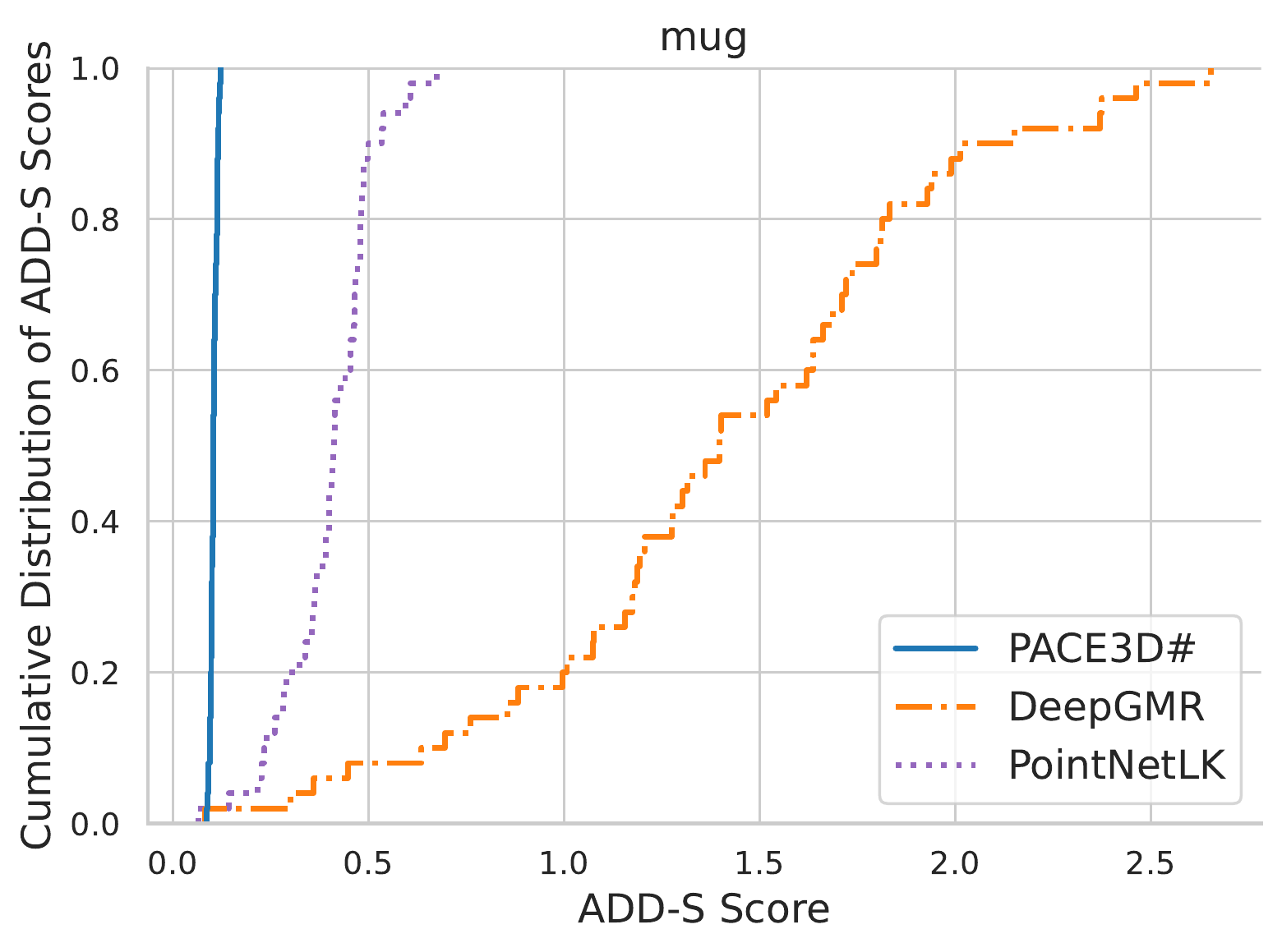} \\
	\vspace{1mm}
	\end{minipage}
& \myhspace
	\begin{minipage}{\mpwfour}%
	\centering%
	\includegraphics[width=\columnwidth]{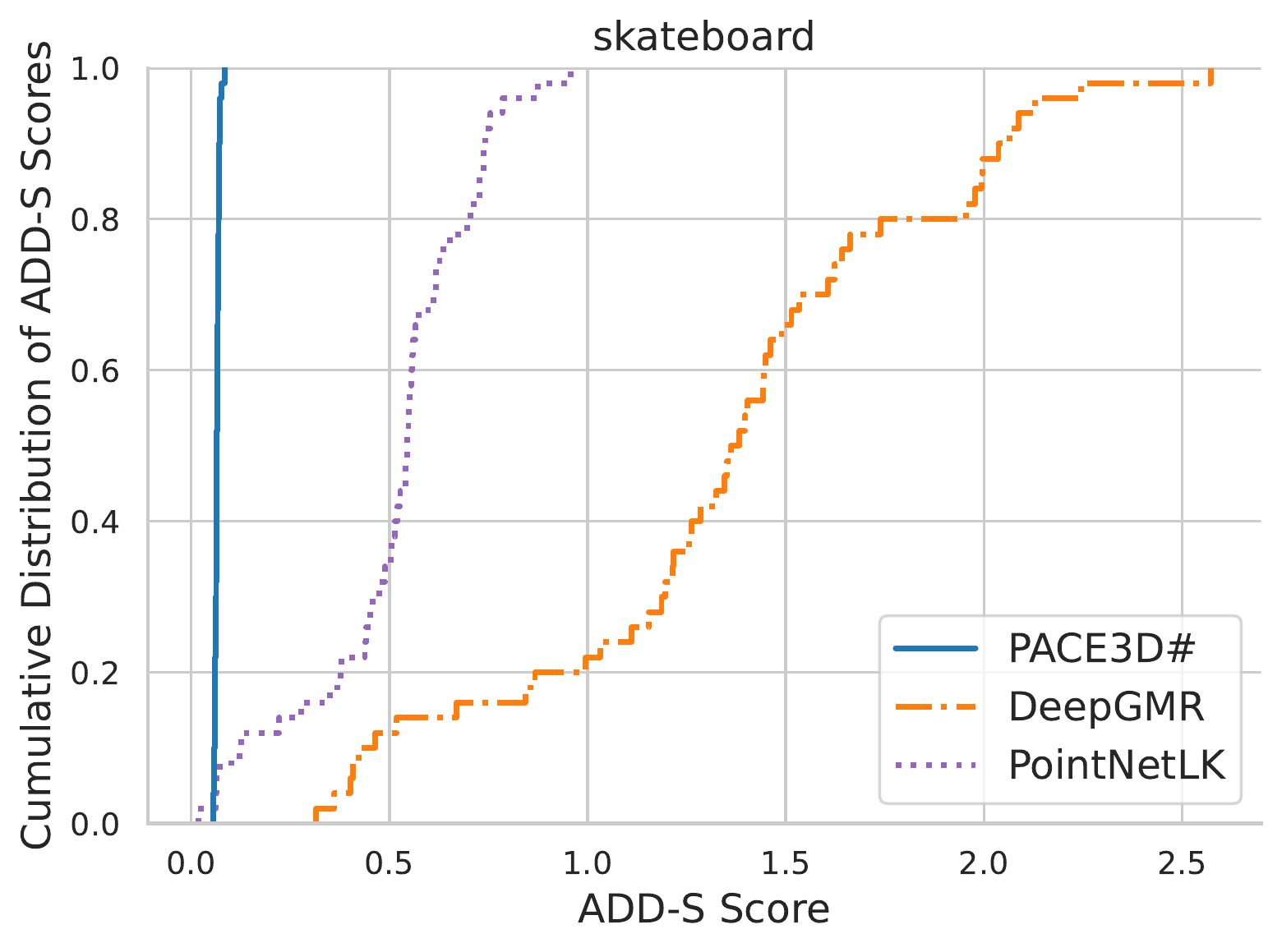} \\
	\vspace{1mm}
	\end{minipage}
& \myhspace
	\begin{minipage}{\mpwfour}%
	\centering%
	\includegraphics[width=\columnwidth]{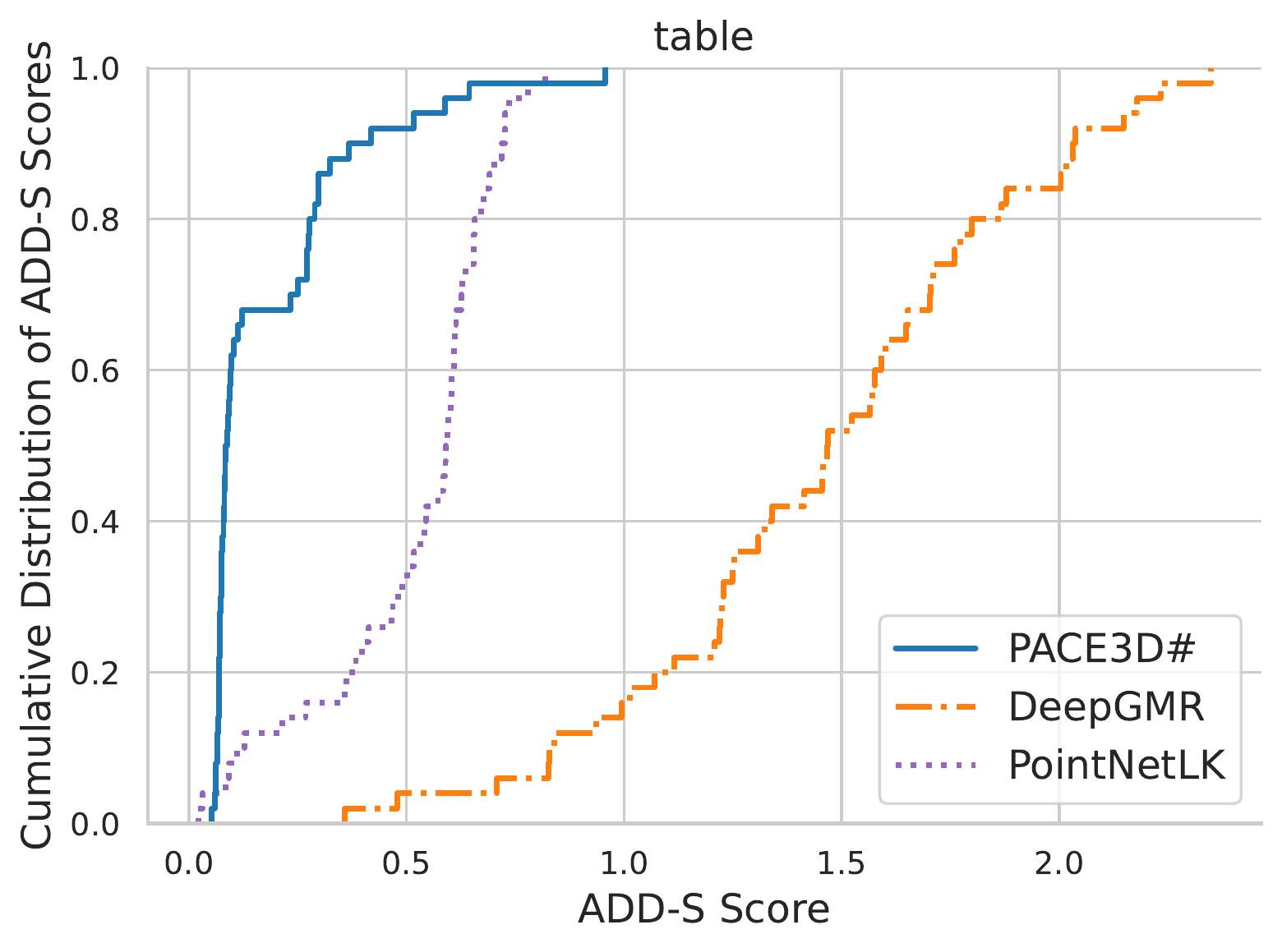} \\
	\vspace{1mm}
	\end{minipage}
& \myhspace
	\begin{minipage}{\mpwfour}%
	\centering%
	\includegraphics[width=\columnwidth]{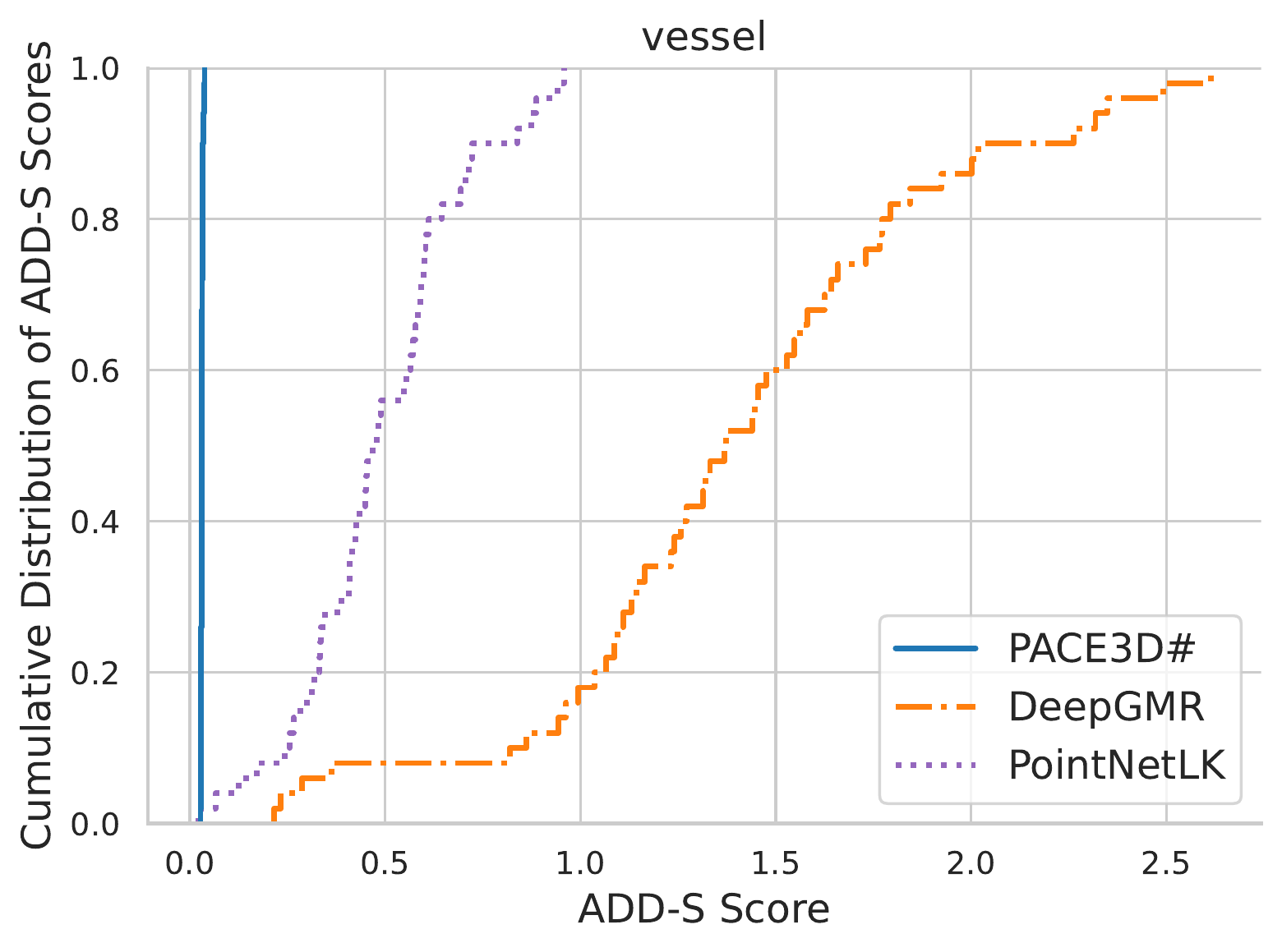} \\
	\vspace{1mm}
	\end{minipage}
	\end{tabular}
	\end{minipage}
	\vspace{-2mm}
	\caption{Performance of \PACErobustThree with a few baseline methods on \keypointnet for all 16 categories.}
  \label{fig:app-keypointnet-categories}
	\end{center}
\end{figure*}

\subsection{Details and Results on \keypointnet}~\label{sec:app-experiments-keypointnet}

\JS{

\myParagraph{Network architecture} The keypoint prediction network follows the architecture of the classification network described in~\cite{Zhao21iccv-PointTransformer}, with the last classification layer replaced with a linear layer with output dimension equal to $3 \times N$, where $N$ is the number of keypoints for one object.
We train for a maximum of 20 epochs for each object, with a batch size of 50.
We use the standard stochastic gradient descent optimizer with momentum, with learning rate of 0.02 and momentum of 0.9.

\myParagraph{Dataset Details} \keypointnet is a large-scale 3D keypoint dataset containing 8329 models from 16 object categories~\cite{You20cvpr-KeypointNetLargescale}.
The models are subset of the ShapeNetCore dataset~\cite{Chang15arxiv-shapenet}.
For each category, the authors of \keypointnet filtered out models that deviate significantly from the majority,
and establish a consistent canonical orientation across object models.
For each object category, we select one object and render depth point clouds using Open3D~\cite{Zhou18arxiv-open3D}.
We apply a random translation bounded within $[0, 1]^{3}$, normalized by the objects' diameters, and apply a uniformly sampled rotation to the object.
We generate 5000 samples as the training set, and 50 samples for test and validation sets each.

\keypointnet provides unique semantic IDs for all keypoints, and we use such IDs to construct shape libraries for each category (see Table~\ref{tbl:app-keypointnet-shapes}).
For each category, we first obtain the semantic IDs of the keypoint annotations for the object under test.
We then construct the shape library by selecting, among all available models, only the models that have keypoint annotations with the exact same set of semantic IDs.

\begin{table}[hbtp!]
  \centering
\begin{tabular}{@{}lll@{}}
\toprule
Category & N  & K    \\ \midrule
airplane   & 14 & 437  \\
bathtub    & 12 & 17   \\
bed        & 10 & 57   \\
bottle     & 17 & 291  \\
cap        & 6  & 36   \\
car        & 22 & 708  \\
chair      & 10 & 517  \\
guitar     & 9  & 587  \\
helmet     & 9  & 75   \\
knife      & 6  & 139  \\
laptop     & 6  & 439  \\
motorcycle & 14 & 139  \\
mug        & 11 & 183  \\
skateboard & 10 & 109  \\
table      & 8  & 1096 \\
vessel     & 16 & 7    \\ \bottomrule
\end{tabular}
\caption{$N$ and $K$ for each category's shape library used for \PACErobustThree.\label{tbl:app-keypointnet-shapes}}
\end{table}

\myParagraph{Additional Results}
Fig.~\ref{fig:app-keypointnet-avg} show cumulative ADD-S score distribution averaged over all 16 categories.
EquiPose is excluded as the authors do not provide pretrained models for all categories.
\PACErobustThree  outperforms both DeepGMR and PointNetLK significantly.
Fig.~\ref{fig:app-keypointnet-categories} shows the cumulative ADD-S score distribution for each of the 16 categories.
Among all categories, \PACErobustThree outperforms DeepGMR and PointNetLK, and achieves comparable or better performance comparing to EquiPose.
}

\renewcommand{\mpwfour}{4.6cm}
\renewcommand{\myhspace}{\hspace{-3.5mm}}
\begin{figure*}[hbtp!]
	\begin{center}
	\begin{minipage}{\textwidth}
	\begin{tabular}{cccc}%
		\myhspace \hspace{-3mm}
			\begin{minipage}{\mpwfour}%
			\centering%
			\includegraphics[width=\columnwidth]{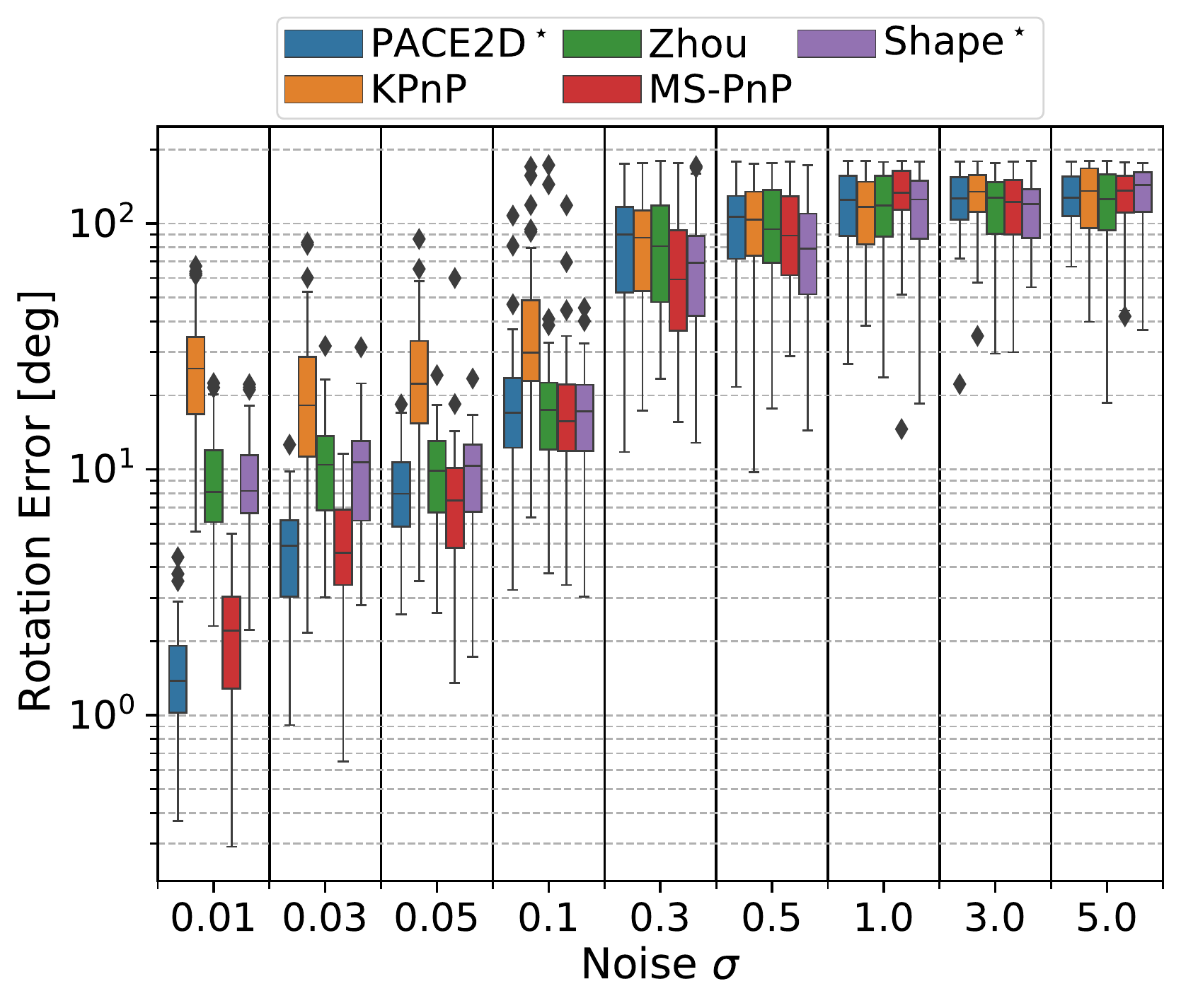}
			\end{minipage}
		&   \myhspace
			\begin{minipage}{\mpwfour}%
			\centering%
			\includegraphics[width=\columnwidth]{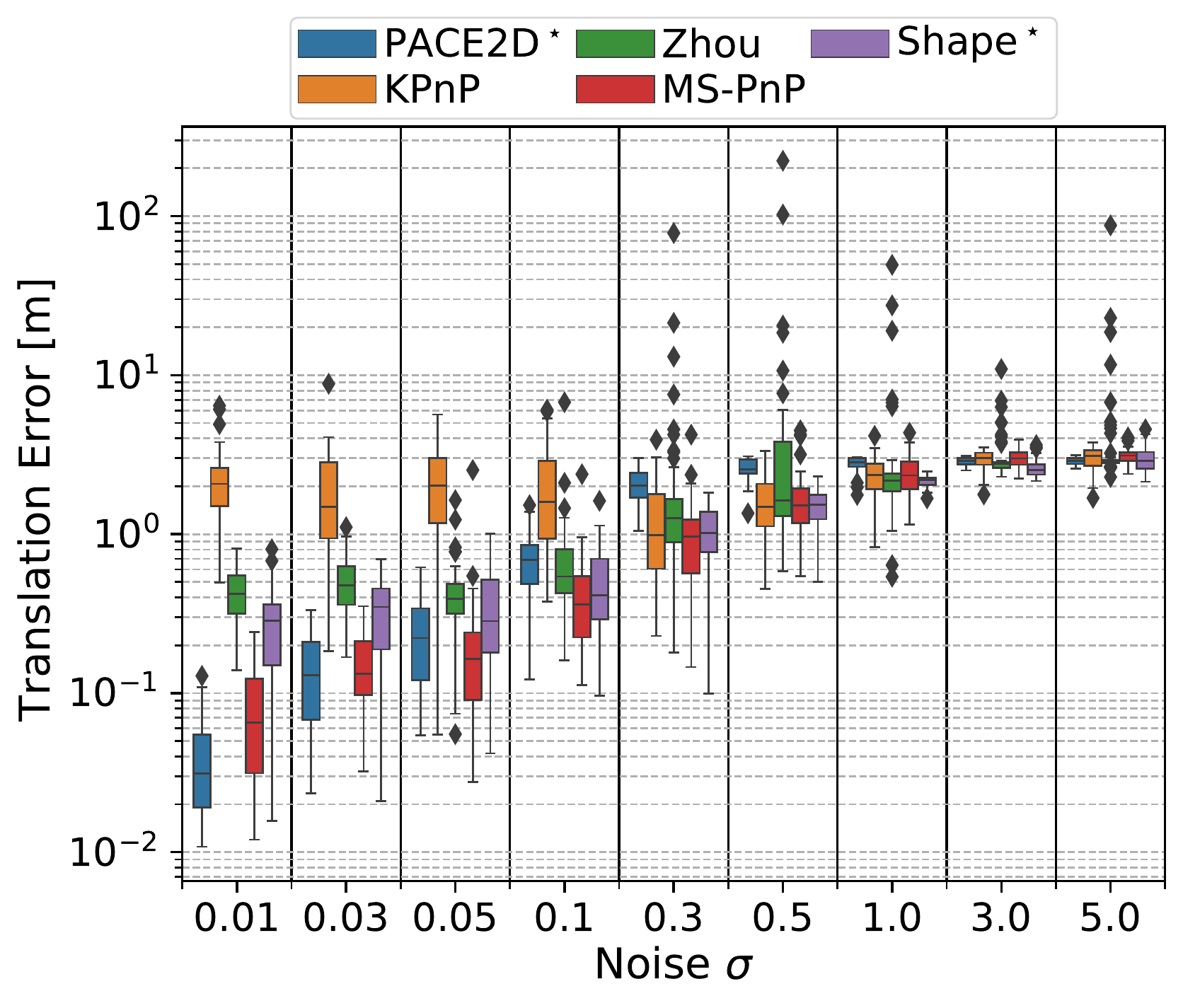}
			\end{minipage}
		&   \myhspace
			\begin{minipage}{\mpwfour}%
			\centering%
			\includegraphics[width=\columnwidth]{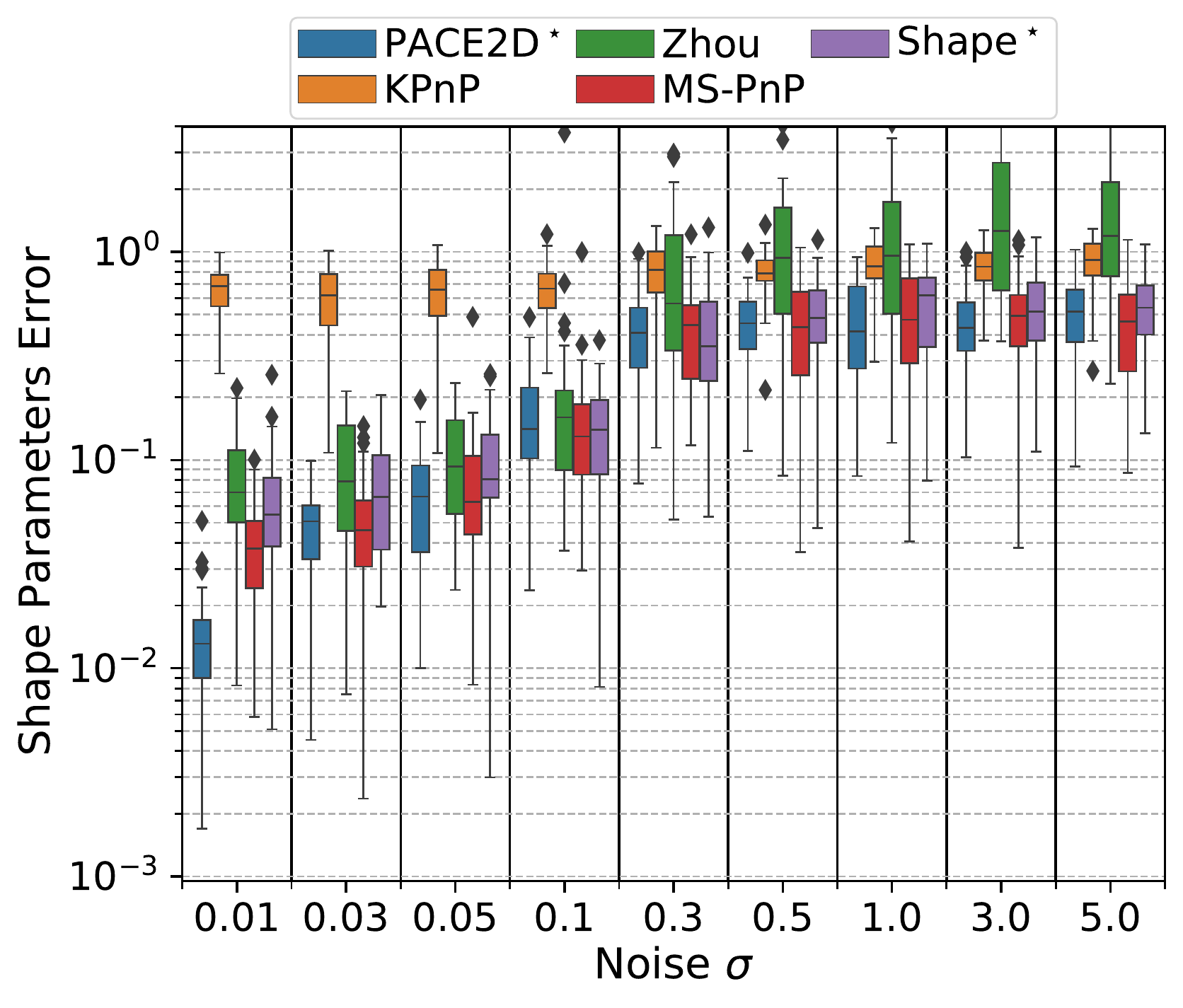}
			\end{minipage}
		&   \myhspace
			\begin{minipage}{\mpwfour}%
			\centering%
			\includegraphics[width=\columnwidth]{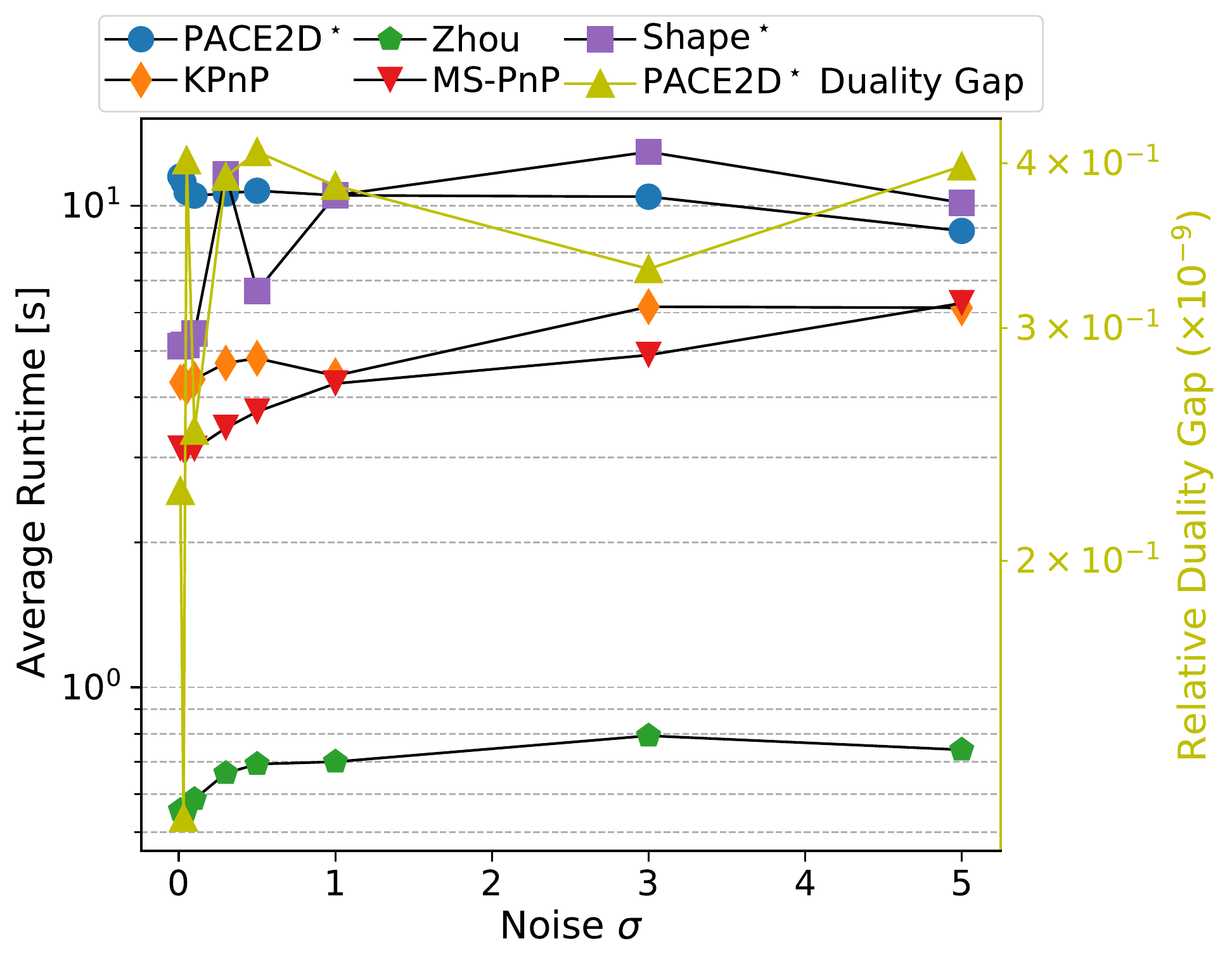}
			\end{minipage}
	\end{tabular}
	\end{minipage}
	\caption{\JS{Performance of \PACETwoLTwo on outlier-free synthetic data with varying noise level: $N=8$; $\vc$ sampled from $\Delta_{K}$ uniformly at random. \vspace{1mm}}
	  \label{fig:app-pace2d-noise}}
	\vspace{-5mm}
	\end{center}
\end{figure*}

\begin{figure*}[hbtp!]
	\begin{center}
	\begin{minipage}{\textwidth}
	\begin{tabular}{cccc}%
		\myhspace \hspace{-3mm}
			\begin{minipage}{\mpwfour}%
			\centering%
			\includegraphics[width=\columnwidth]{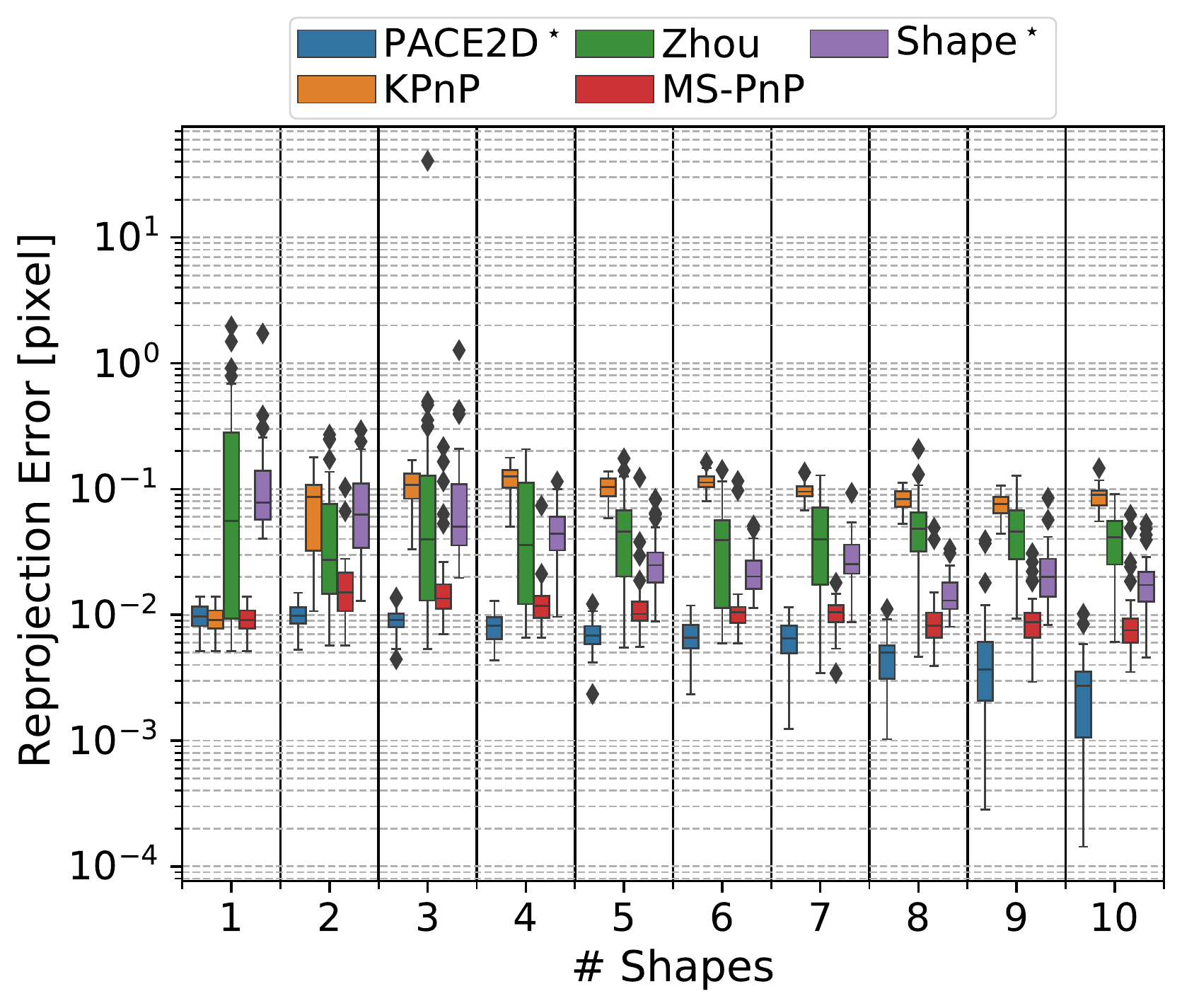}
			\end{minipage}
		&   \myhspace
			\begin{minipage}{\mpwfour}%
			\centering%
			\includegraphics[width=\columnwidth]{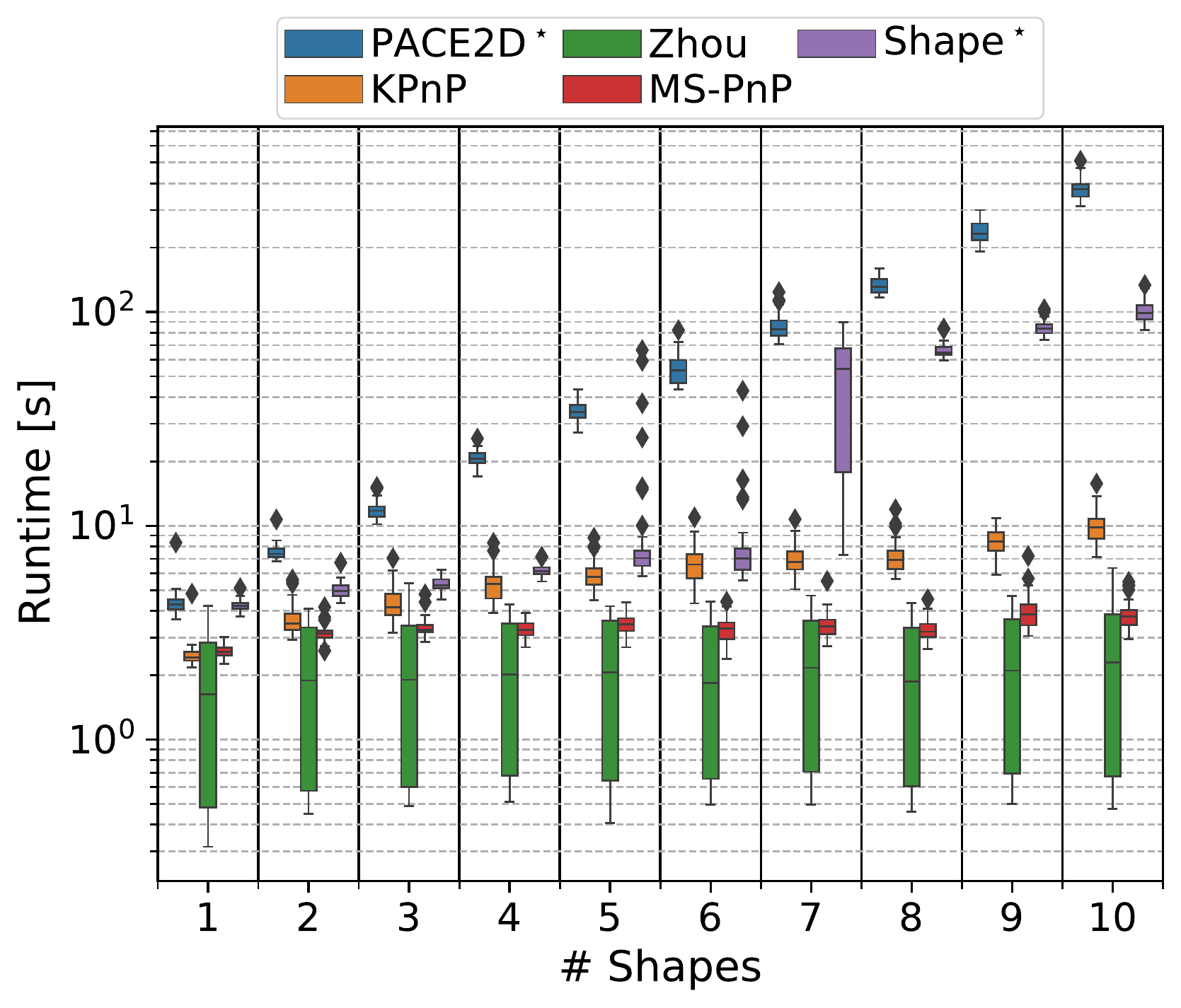}
			\end{minipage}
		&   \myhspace
			\begin{minipage}{\mpwfour}%
			\centering%
			\includegraphics[width=\columnwidth]{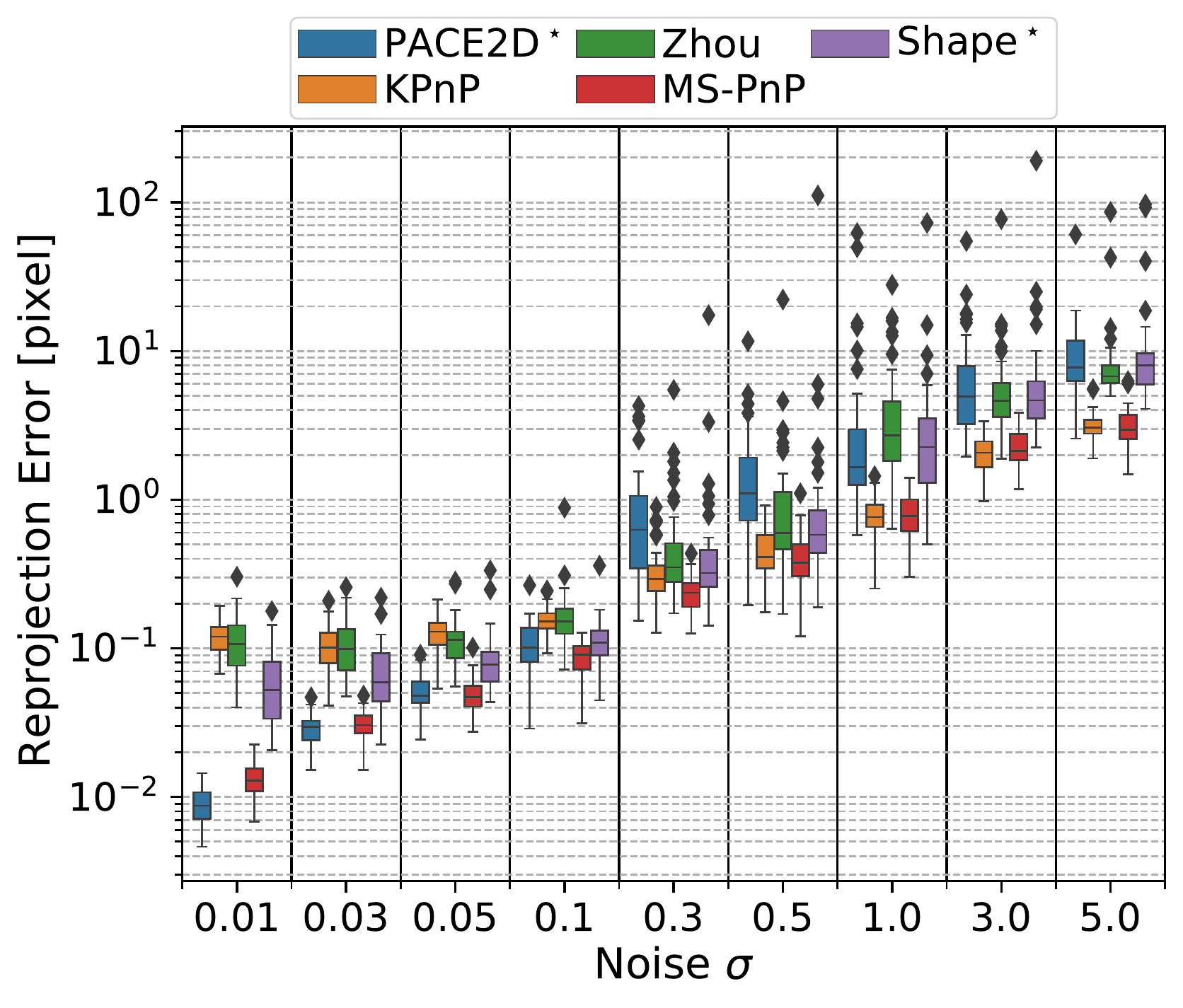}
			\end{minipage}
		&   \myhspace
			\begin{minipage}{\mpwfour}%
			\centering%
			\includegraphics[width=\columnwidth]{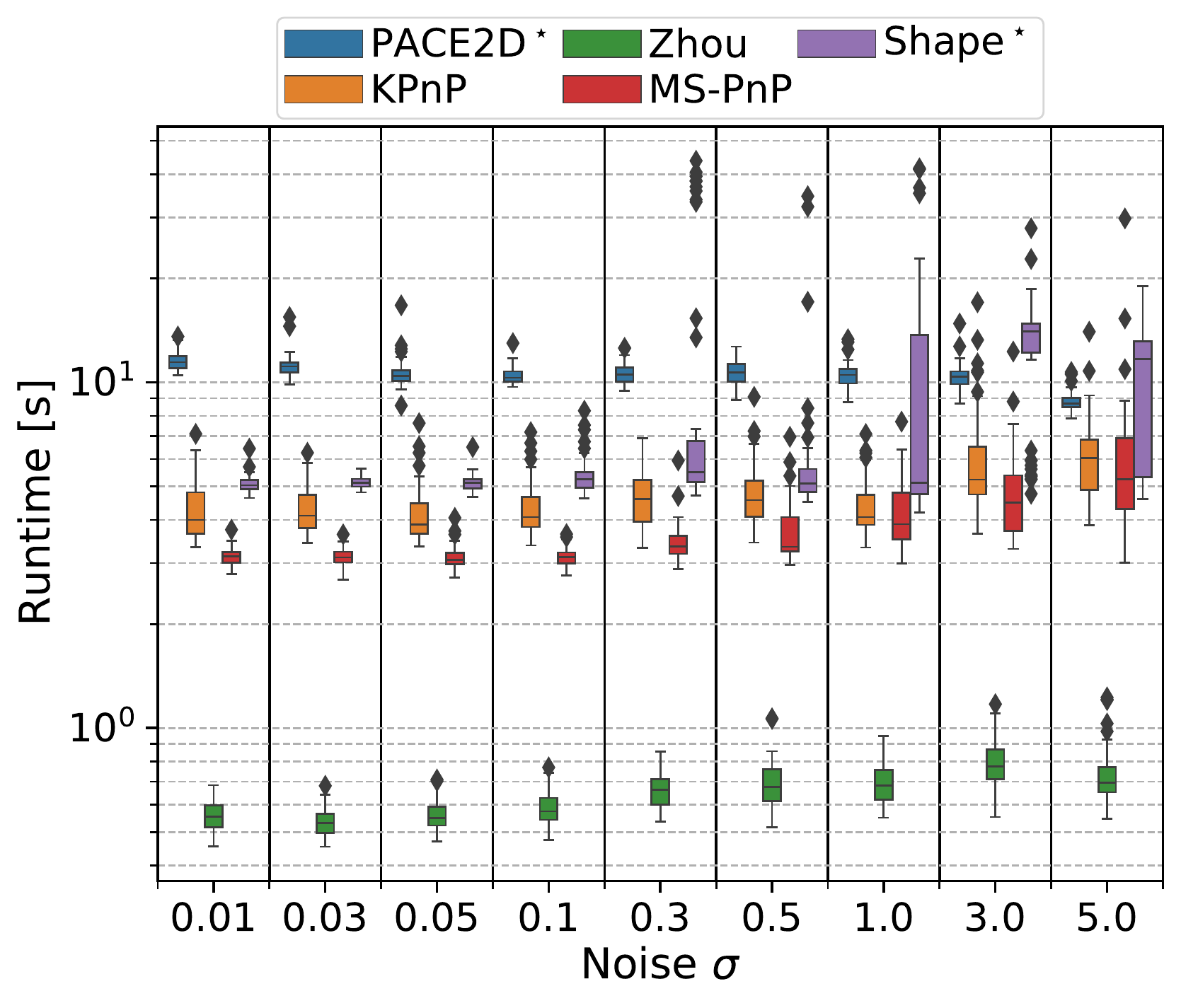}
			\end{minipage}
		\\
		\multicolumn{2}{c}{\smaller \JS{(a) Varying $K$; $\vc$ sampled from $\Delta_{K}$ uniformly at random.}} &
		\multicolumn{2}{c}{\smaller \JS{(b) Varying noise; $\vc$ sampled from $\Delta_{K}$ uniformly at random.}}
	\end{tabular}
	\end{minipage}
	\caption{\JS{Reprojection errors and runtimes of \PACETwoLTwo with respect to varying $K$ and noises. (a) and (b): $\vc$ sampled from $\Delta_{K}$ uniformly at random; (c) and (d): $\vc$ sampled as a random one-hot vector; $N=8$ for all runs. Each boxplot reports statistics computed over 50 Monte Carlo runs.}
	  \label{fig:app-simulation-optimality-residuals-runtime-2d}}
	\vspace{-5mm}
	\end{center}
\end{figure*}

\renewcommand{\mpwfour}{4.6cm}
\renewcommand{\myhspace}{\hspace{-3.5mm}}
\begin{figure*}[hbtp!]
	\begin{center}
	\begin{minipage}{\textwidth}
	\begin{tabular}{cccc}%
		\myhspace \hspace{-3mm}
			\begin{minipage}{\mpwfour}%
			\centering%
			\includegraphics[width=\columnwidth]{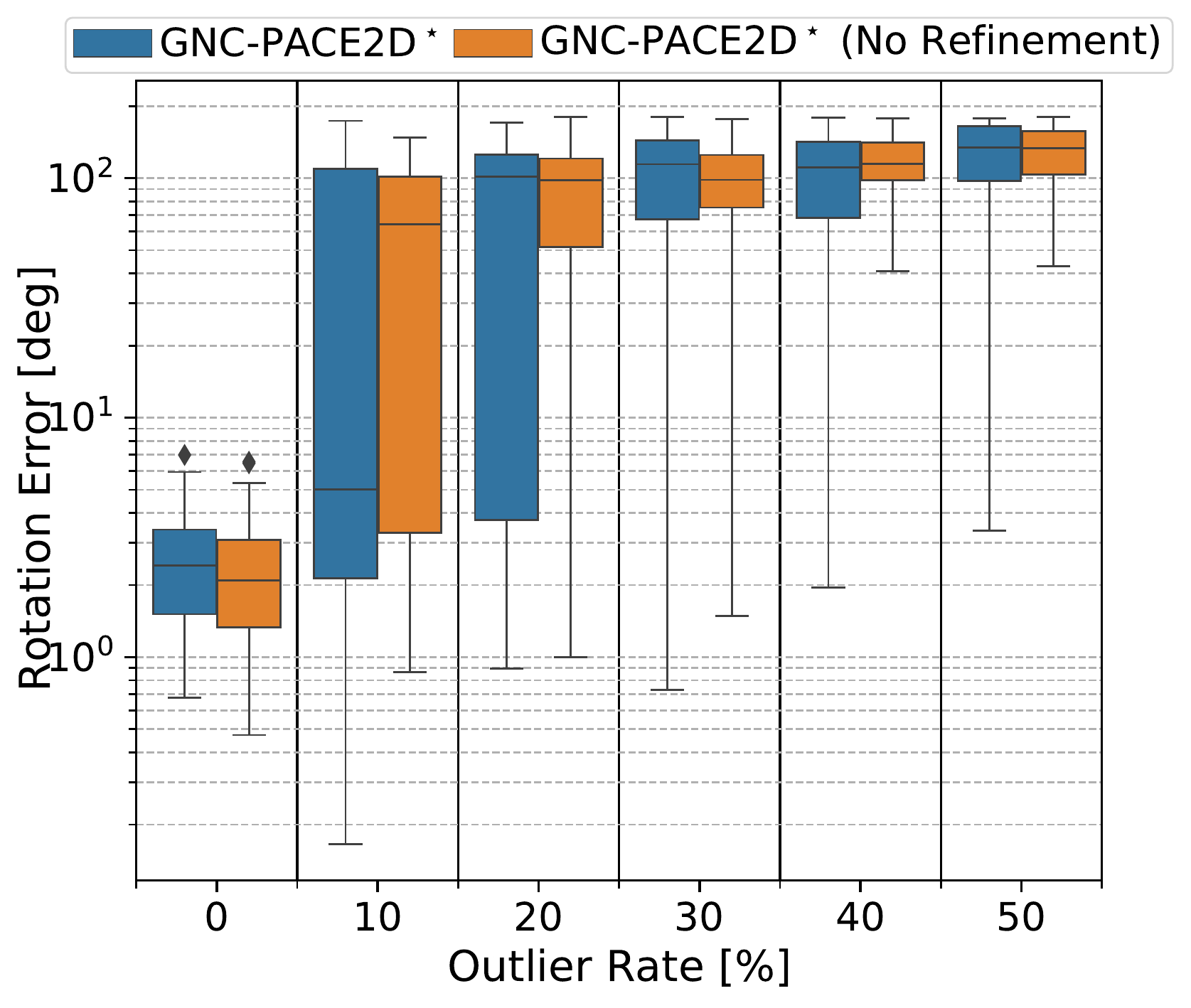}
			\end{minipage}
		&   \myhspace
			\begin{minipage}{\mpwfour}%
			\centering%
			\includegraphics[width=\columnwidth]{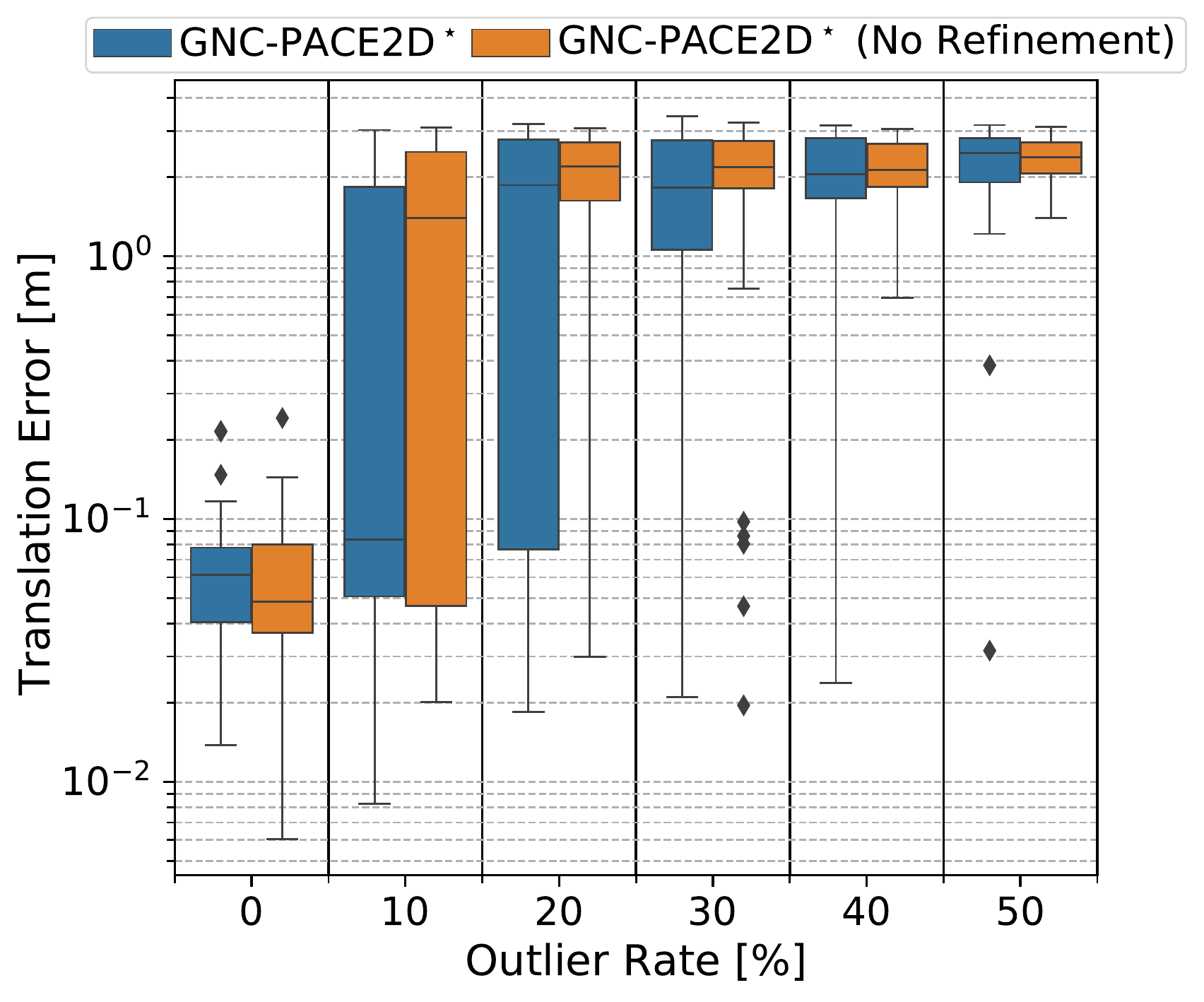}
			\end{minipage}
		&   \myhspace
			\begin{minipage}{\mpwfour}%
			\centering%
			\includegraphics[width=\columnwidth]{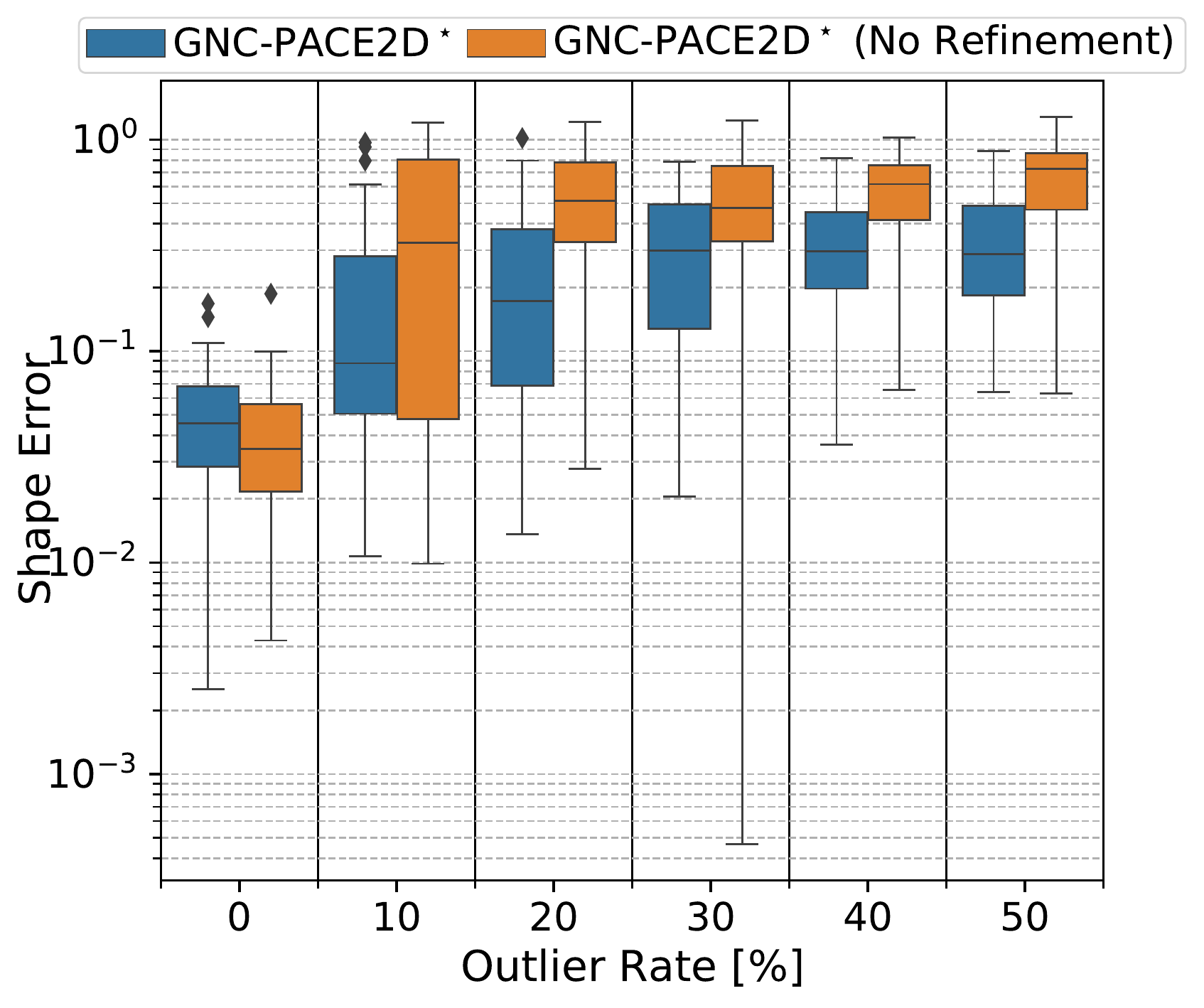}
			\end{minipage}
		&   \myhspace
			\begin{minipage}{\mpwfour}%
			\centering%
			\includegraphics[width=\columnwidth]{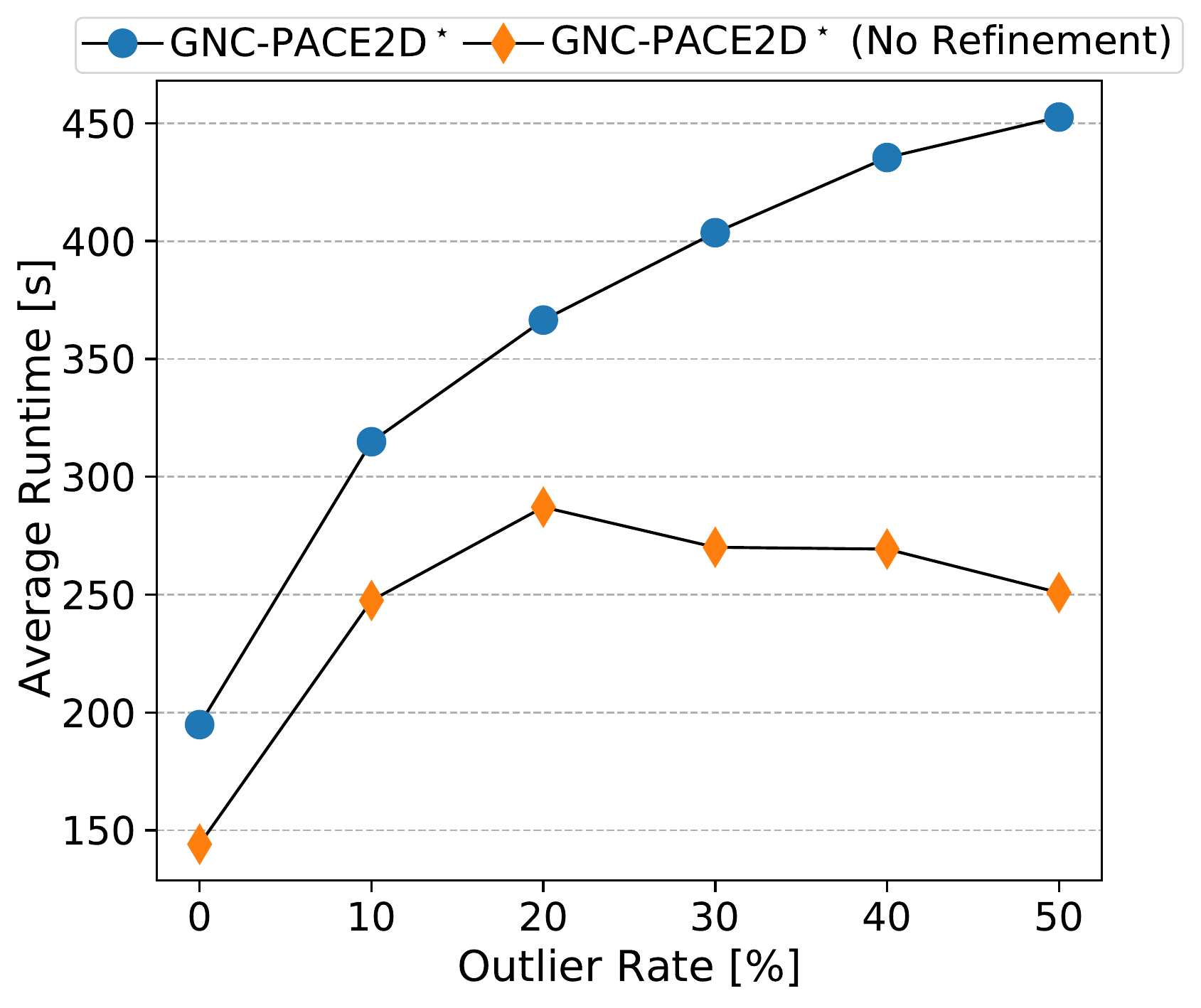}
			\end{minipage}
	\end{tabular}
	\end{minipage} 
	\caption{Performance of \PACETwoGNCLTwo with and without local refinement, see Appendix~\ref{sec:app-gnc-category}.	\label{fig:app-simulation-robust-2d}}
	\end{center}
\end{figure*}

\begin{figure*}[hbtp!]
	\begin{center}
	\begin{minipage}{\textwidth}
	\begin{tabular}{cccc}%
		\myhspace \hspace{-3mm}
			\begin{minipage}{\mpwfour}%
			\centering%
			\includegraphics[width=\columnwidth]{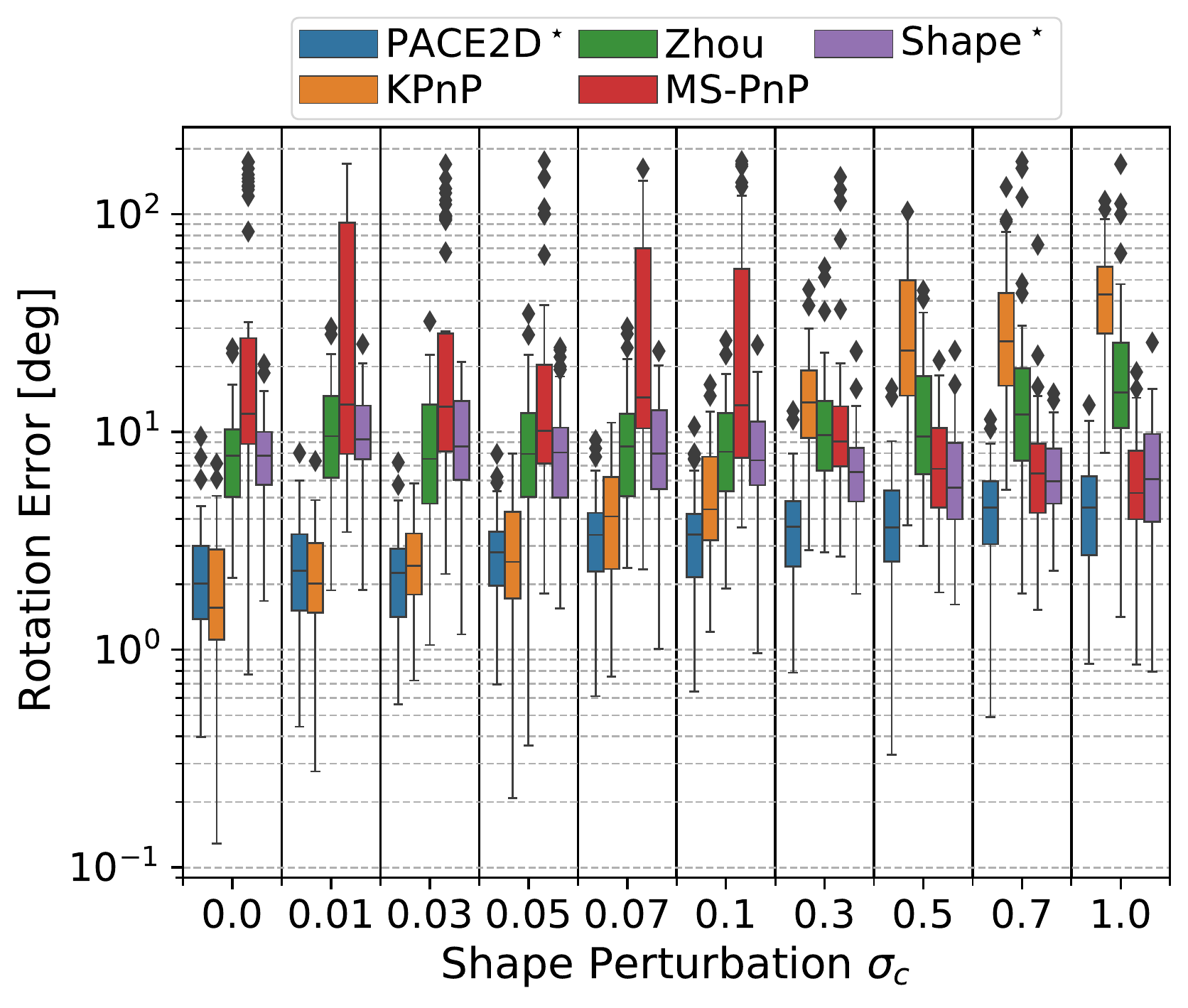}
			\end{minipage}
		&   \myhspace
			\begin{minipage}{\mpwfour}%
			\centering%
			\includegraphics[width=\columnwidth]{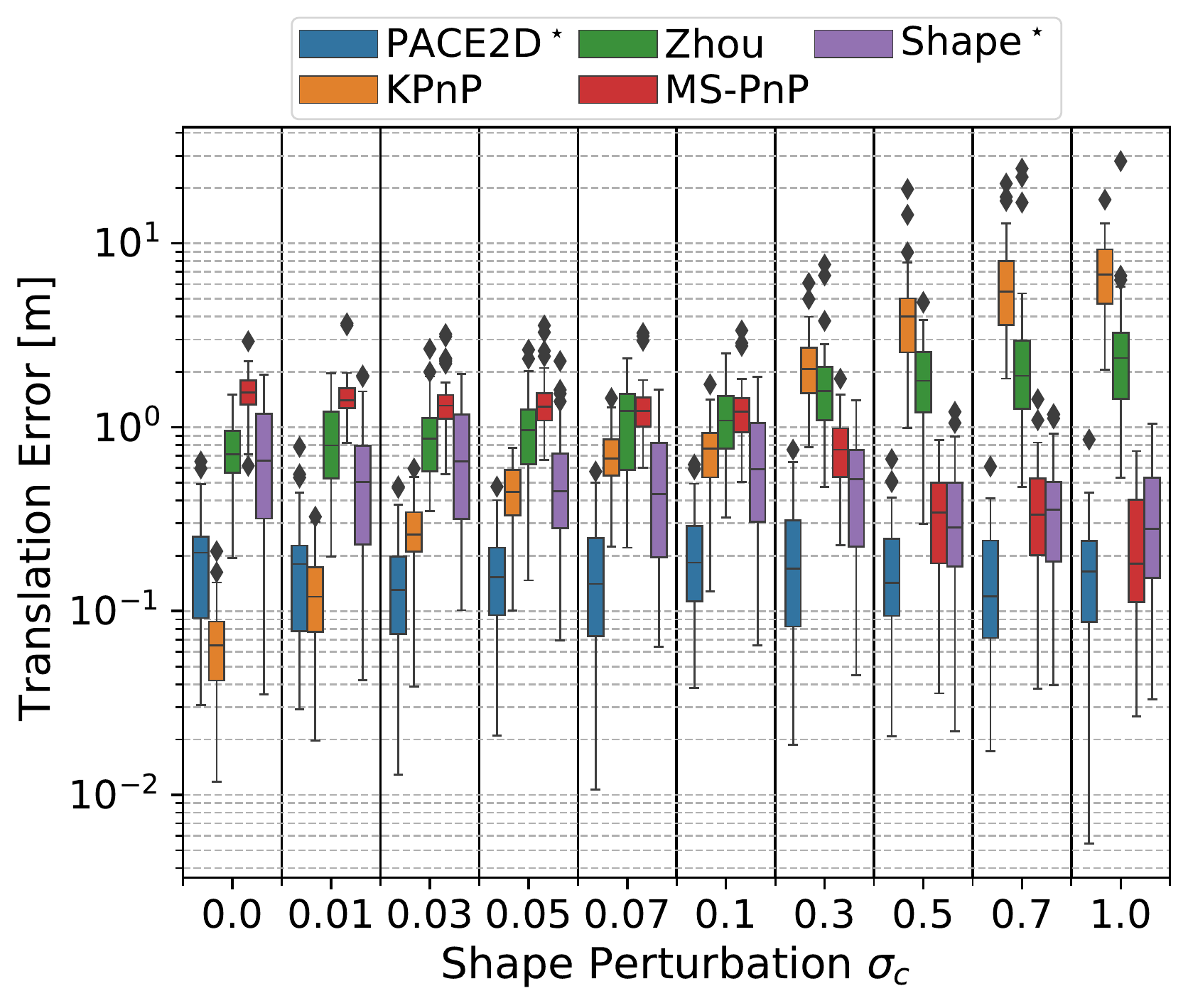}
			\end{minipage}
		&   \myhspace
			\begin{minipage}{\mpwfour}%
			\centering%
			\includegraphics[width=\columnwidth]{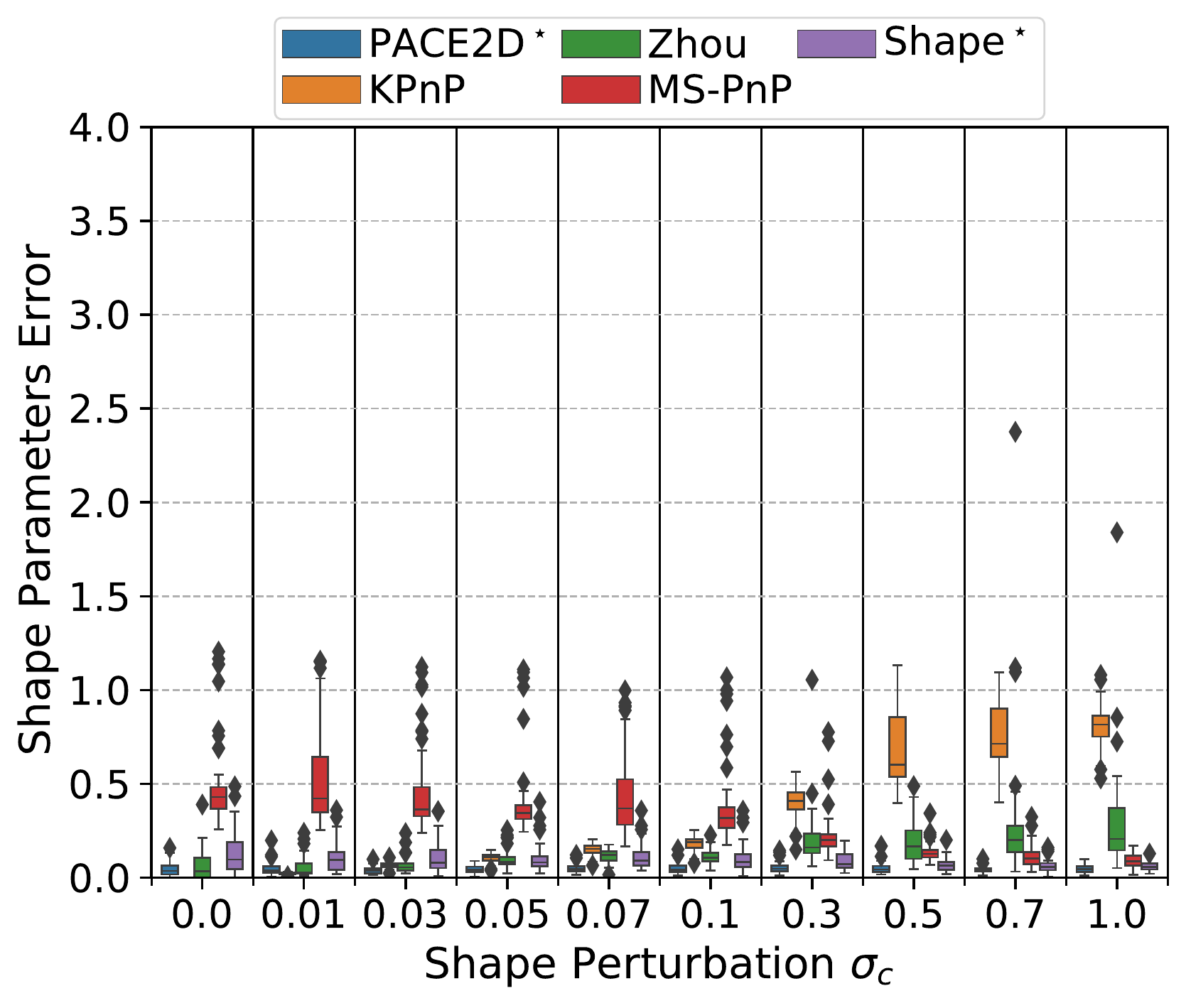}
			\end{minipage}
		&   \myhspace
			\begin{minipage}{\mpwfour}%
			\centering%
			\includegraphics[width=\columnwidth]{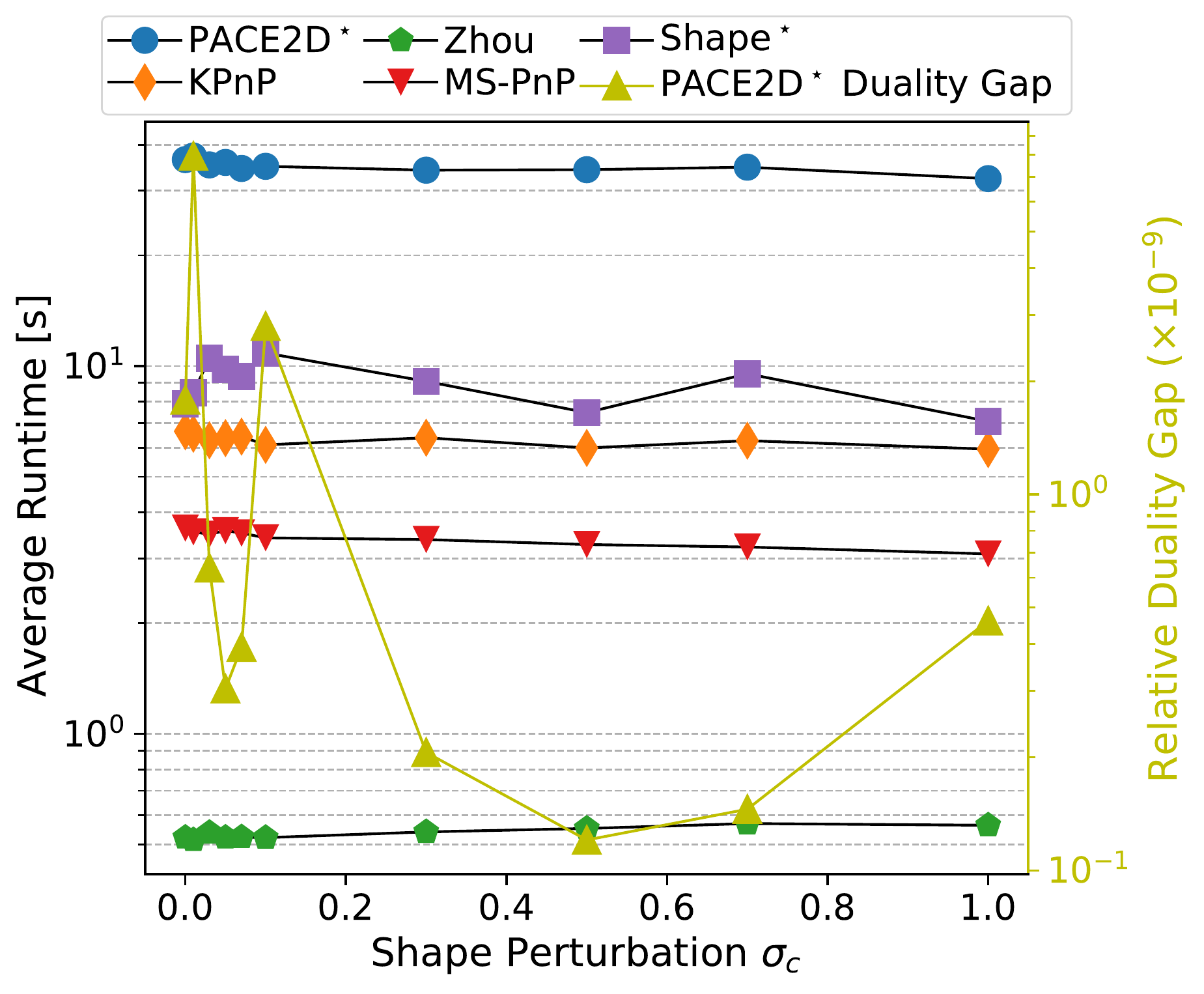}
			\end{minipage}
	\end{tabular}
	\end{minipage}
	\caption{\JS{Performance of \PACETwoLTwo with different levels of perturbation to an one-hot $\vc$.}
	  \label{fig:app-pace2d-shape-perturb}}
	\vspace{-5mm}
	\end{center}
\end{figure*}

\subsection{Results on \PACETwo and \PACErobustTwo}\label{sec:app-experiments-pace-2d}

\JS{
In Section~\ref{sec:exp-optimality-robustness-2d}, we demonstrated the optimality of \PACETwoLTwo under varying $K$.
Fig.~\ref{fig:app-pace2d-noise} shows the optimality of \PACETwoLTwo under varying noise.
Similarly to results shown in Section~\ref{sec:exp-optimality-robustness-2d}, \PACETwoLTwo achieves lower rotation, translation and shape errors when the shape vectors $\vc$ are sampled from $\Delta_{K}$ uniformly at random.

Fig.~\ref{fig:app-simulation-optimality-residuals-runtime-2d} shows the reprojection errors and runtime of \PACETwoLTwo under varying noise and $K$ for $\vc$ sampled from $\Delta_{K}$ uniformly at random.
Under fixed noise with varying $K$,
\PACETwoLTwo consistently achieves lowest reprojection errors across all tests.
Under varying noise with fixed $K$,
\PACETwoLTwo achieves lowest reprojection errors at noise $\sigma=0.01$,
and experiences degradation of performance starting from $\sigma=0.03$.
}

In addition, Fig.~\ref{fig:app-simulation-robust-2d} (b) shows  \PACETwoGNCLTwo running with and without local refinement minimizing the reprojection cost.
Evidently, \PACETwoGNCLTwo without refinement has consistently more failure cases,
especially in terms of shape errors.
Note that using refinement incurs a higher runtime cost, as in each \gnc iteration
we need to run additional iterations of the local solver minimizing the reprojection error.

\JS{
Last but not least, we conduct additional experiments to understand how the distribution of one-hot shape vectors $\vc$ affects the performance of \PACETwoLTwo.
For this experiment, we first sample $\vc$ as a random one-hot vector.
We then sample $K$ random values in the interval $[0, \sigma_{\vc}]$, and add them to $\vc$ component-wise.
$\vc$ is finally rescaled so that $\one\tran\vc =1$.
Fig.~\ref{fig:app-pace2d-shape-perturb} shows the performance of \PACETwoLTwo under different $\sigma_{\vc}$ values.
\PACETwoLTwo maintains low rotation, translation and rotation errors, together with low relative duality gap, across $\sigma_{\vc}$ values.
As expected, \meanPnP's performance increases steadily with the increase of $\sigma_{\vc}$, as the mean shape assumption gives better initialization with $\vc$ closer to being sampled uniformly at random.
\kpnp's performance degrades steadily with the increase of $\sigma_{\vc}$.
However, if $\vc$ are known to be one-hot vectors, \kpnp can achieve lower median errors than \PACETwoLTwo.
}

\subsection{Visualization of Results on \apollo}\label{sec:app-apollo-vis}

Fig.~\ref{fig:app-pace3d-vis} shows results obtained with \PACErobustThree on \apollo.
Fig.~\ref{fig:app-pace2d-vis} shows results obtained with \PACErobustTwo on \apollo.
In both figures, first four rows show successful results, while the last row demonstrates failure cases.

\renewcommand{\mpwthree}{5.8cm}
\renewcommand{\myhspace}{\hspace{-3.5mm}}

\begin{figure*}[h]\ContinuedFloat
	\begin{center}{}
	\begin{minipage}{\textwidth}
	\begin{tabular}{cccc}%
		
& 
\myhspace
	\begin{minipage}{\mpwthree}%
	\centering%
	\includegraphics[width=\columnwidth]{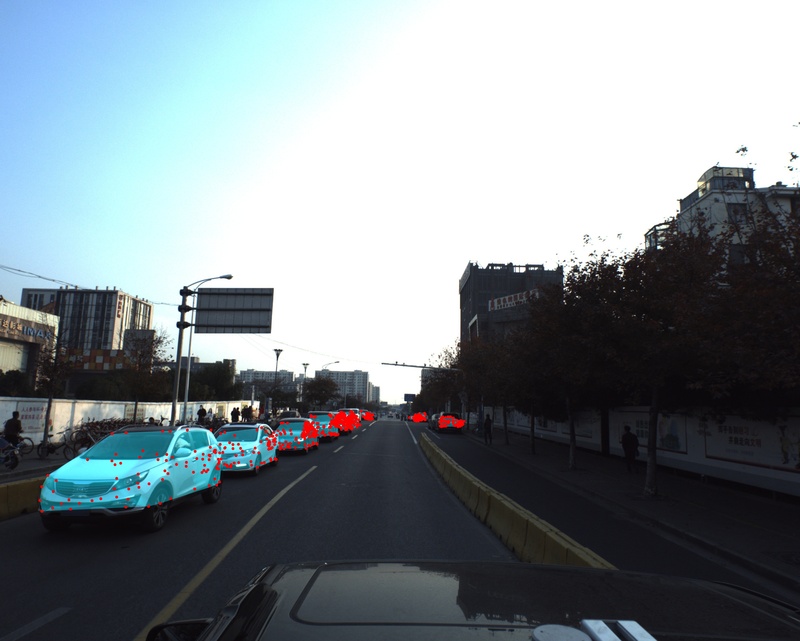} \\
	\vspace{1mm}
	\end{minipage}
& \myhspace
	\begin{minipage}{\mpwthree}%
	\centering%
	\includegraphics[width=\columnwidth]{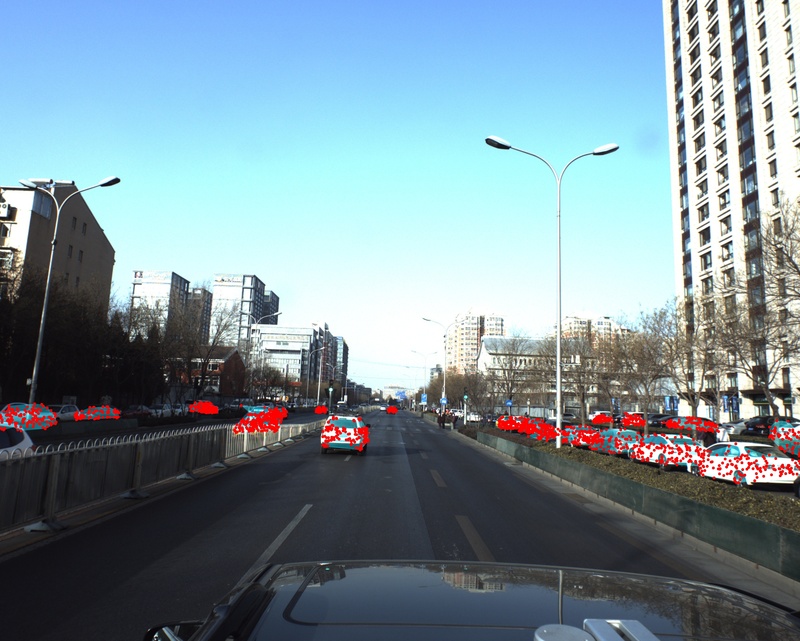} \\
	\vspace{1mm}
	\end{minipage}
& \myhspace
	\begin{minipage}{\mpwthree}%
	\centering%
	\includegraphics[width=\columnwidth]{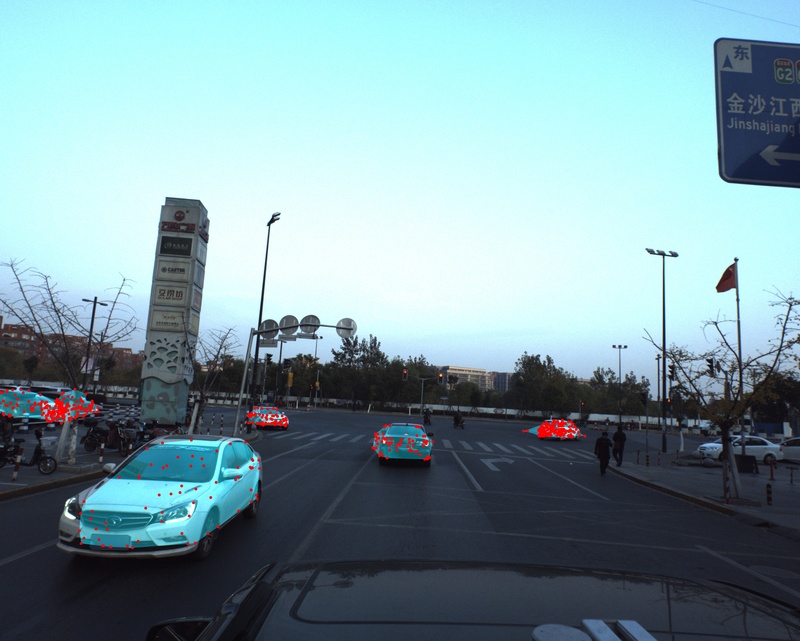} \\
	\vspace{1mm}
	\end{minipage} 
\\
&
\myhspace
	\begin{minipage}{\mpwthree}%
	\centering%
	\includegraphics[width=\columnwidth]{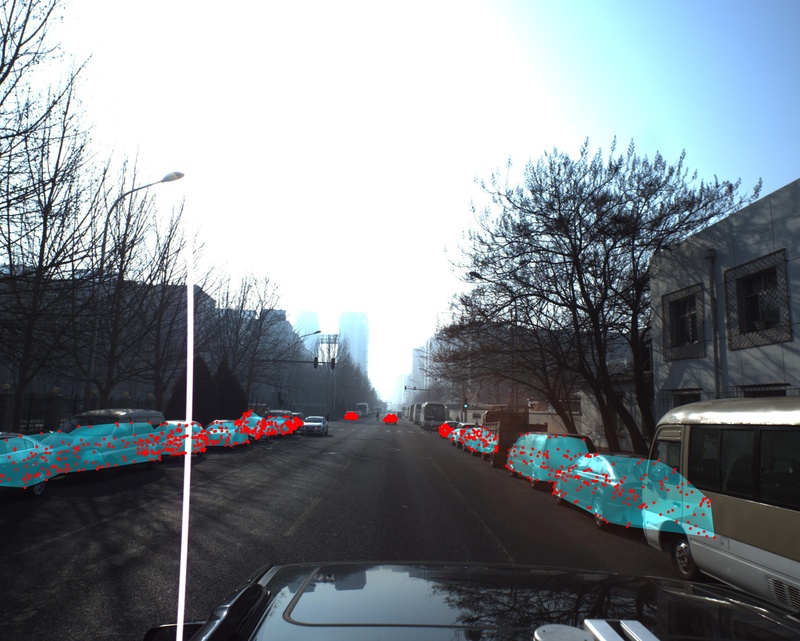} \\
	\vspace{1mm}
	\end{minipage}
& \myhspace
	\begin{minipage}{\mpwthree}%
	\centering%
	\includegraphics[width=\columnwidth]{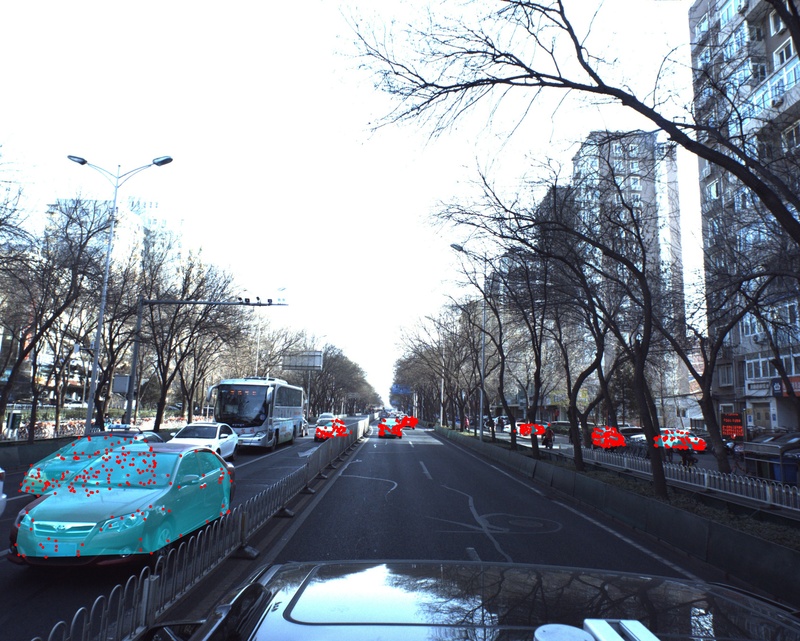} \\
	\vspace{1mm}
	\end{minipage}
& \myhspace
	\begin{minipage}{\mpwthree}%
	\centering%
	\includegraphics[width=\columnwidth]{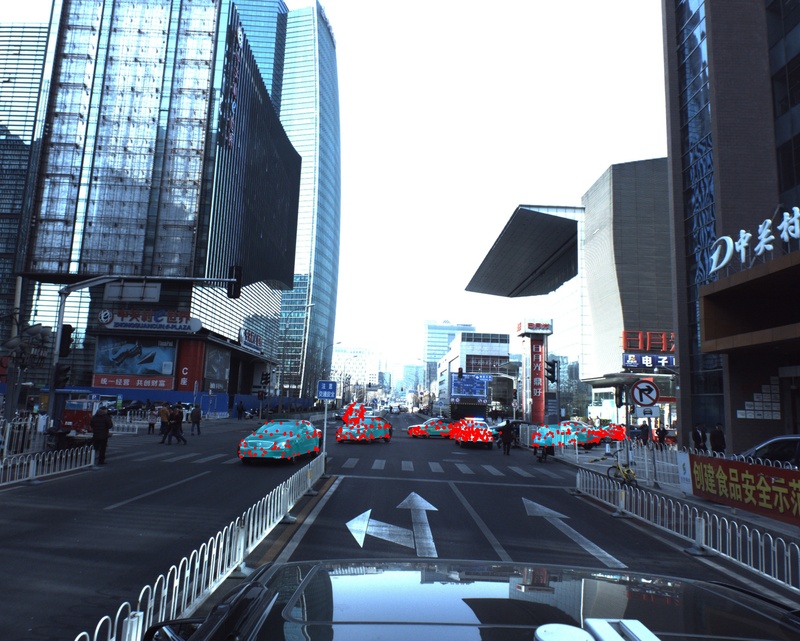} \\
	\vspace{1mm}
	\end{minipage} 
\\
&
\myhspace
	\begin{minipage}{\mpwthree}%
	\centering%
	\includegraphics[width=\columnwidth]{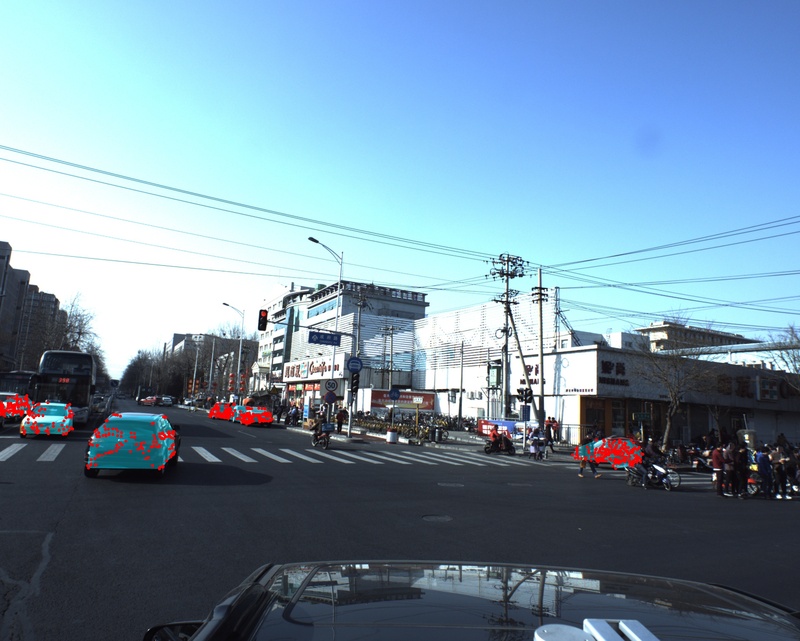} \\
	\vspace{1mm}
	\end{minipage}
& \myhspace
	\begin{minipage}{\mpwthree}%
	\centering%
	\includegraphics[width=\columnwidth]{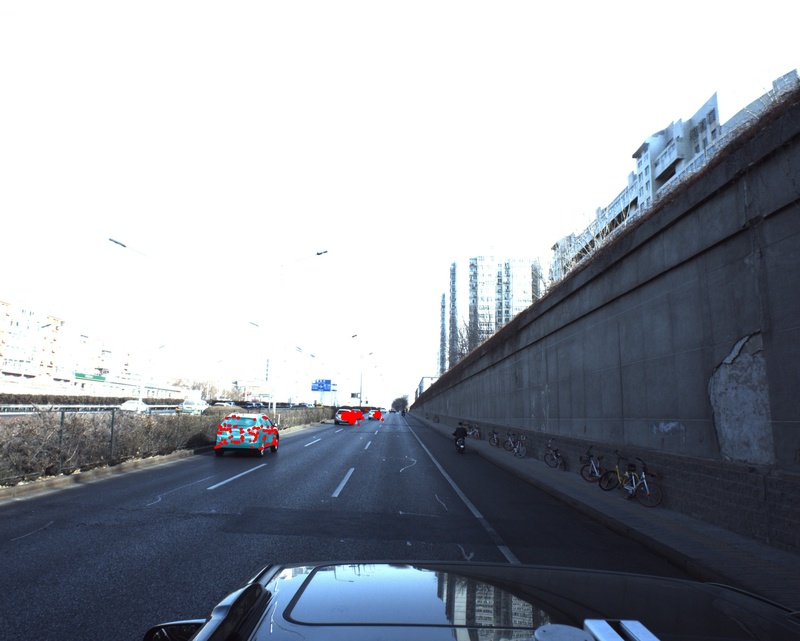} \\
	\vspace{1mm}
	\end{minipage}
& \myhspace
	\begin{minipage}{\mpwthree}%
	\centering%
	\includegraphics[width=\columnwidth]{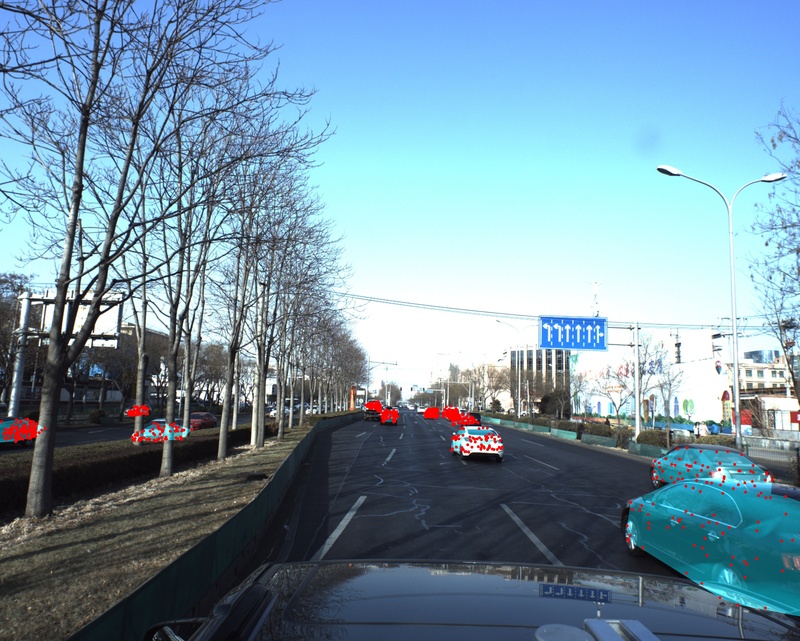} \\
	\vspace{1mm}
	\end{minipage} 
\\
&
\myhspace
	\begin{minipage}{\mpwthree}%
	\centering%
	\includegraphics[width=\columnwidth]{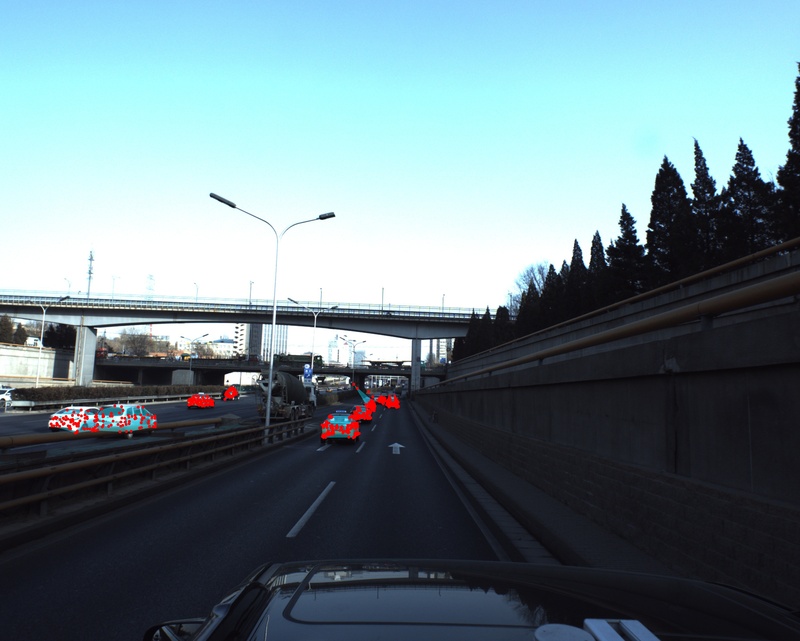} \\
	\vspace{1mm}
	\end{minipage}
& \myhspace
	\begin{minipage}{\mpwthree}%
	\centering%
	\includegraphics[width=\columnwidth]{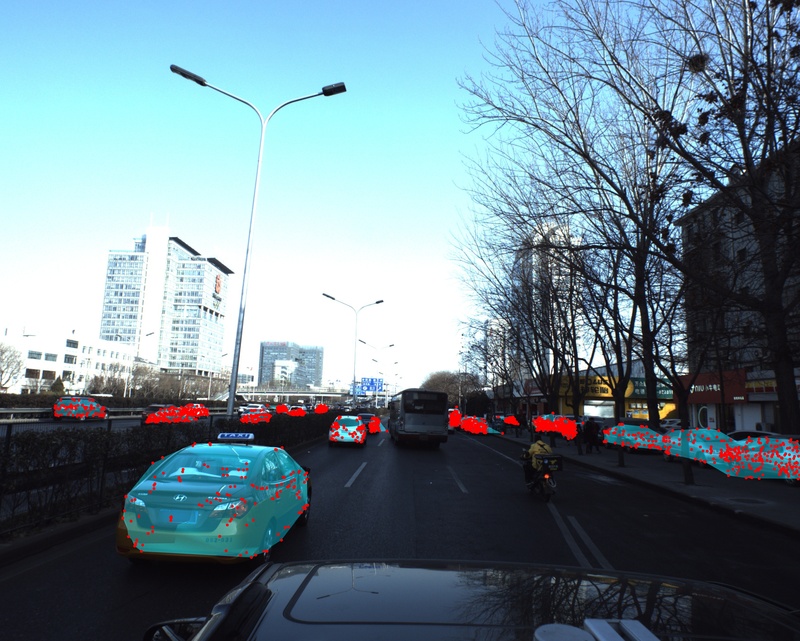} \\
	\vspace{1mm}
	\end{minipage}
& \myhspace
	\begin{minipage}{\mpwthree}%
	\centering%
	\includegraphics[width=\columnwidth]{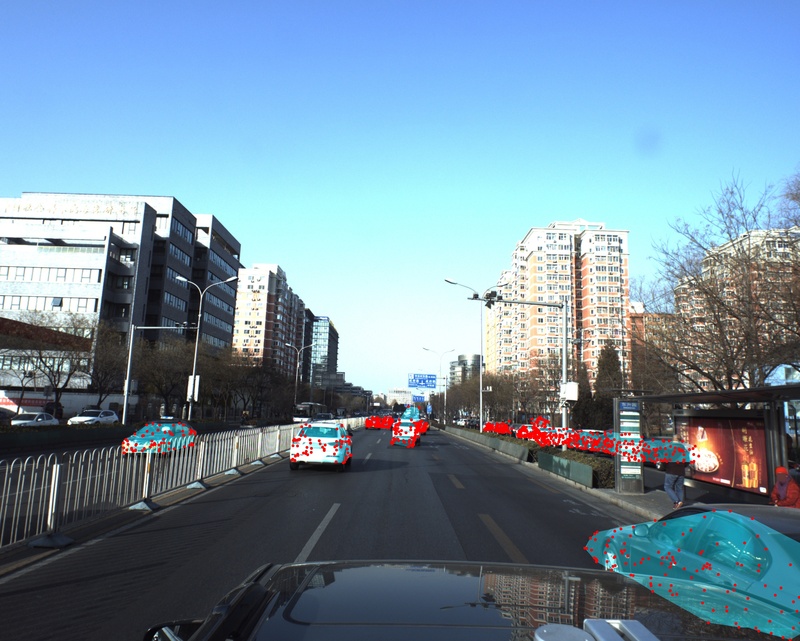} \\
	\vspace{1mm}
	\end{minipage} 
\\
\hline
\vspace{-2mm}
\\
	\myhspace \myhspace %
	\rotatebox{90}{\hspace{-7mm} \red{failures} } 
	&
\myhspace
	\begin{minipage}{\mpwthree}%
	\centering%
	\includegraphics[width=\columnwidth]{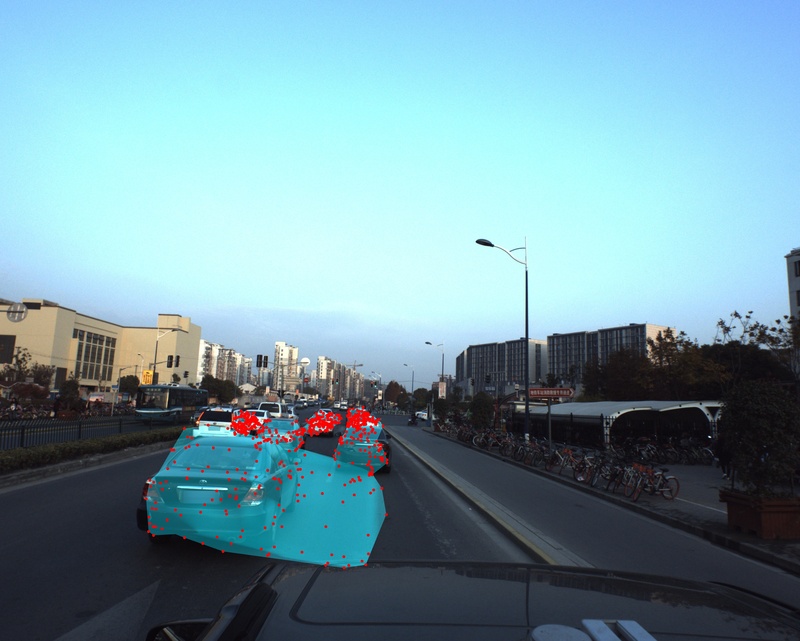} \\
	\vspace{1mm}
	\end{minipage}
& \myhspace
	\begin{minipage}{\mpwthree}%
	\centering%
	\includegraphics[width=\columnwidth]{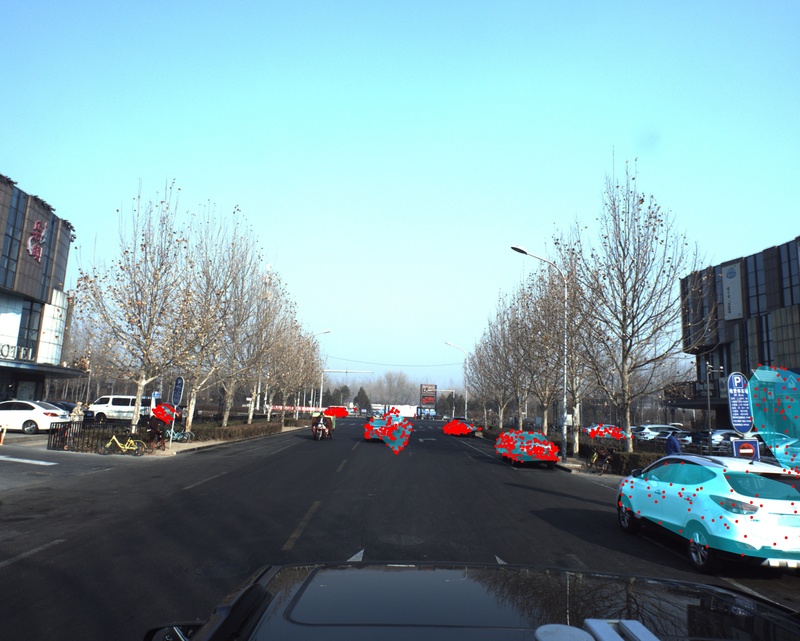} \\
	\vspace{1mm}
	\end{minipage}
& \myhspace
	\begin{minipage}{\mpwthree}%
	\centering%
	\includegraphics[width=\columnwidth]{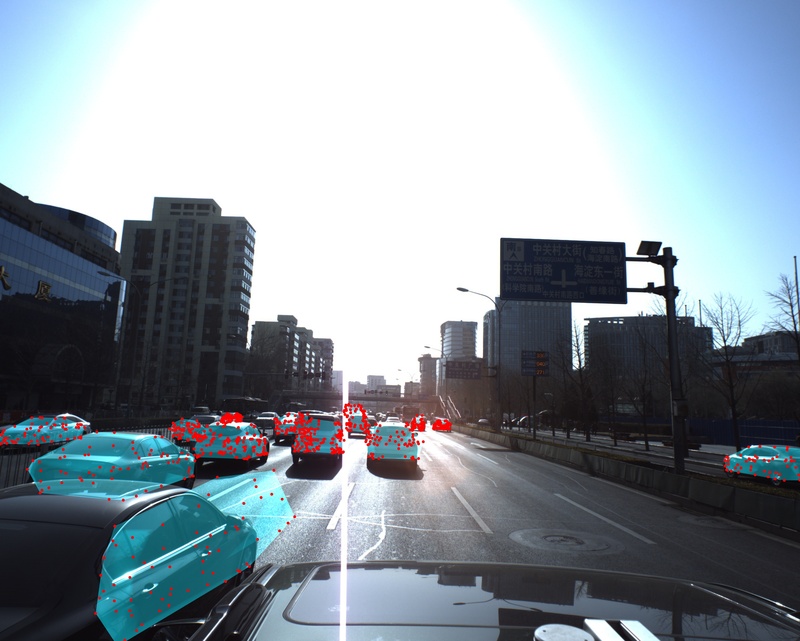} \\
	\vspace{1mm}
	\end{minipage} 
	\end{tabular}
	\end{minipage}
	\vspace{-2mm} 
	\caption{Qualitative results for \PACErobustThree: overlay of estimated vehicle pose and shape on the images from the \apollo dataset. 
	The images are manually selected out of the validation set in the dataset to showcase successful vehicle
	localization (top 4 rows) as well as failure cases (last row). }
  \label{fig:app-pace3d-vis}
	\end{center}
\end{figure*}

\renewcommand{\mpwthree}{5.8cm}
\renewcommand{\myhspace}{\hspace{-3.5mm}}

\begin{figure*}[h]\ContinuedFloat
	\begin{center}{}
	\begin{minipage}{\textwidth}
	\begin{tabular}{cccc}%
		
& 
\myhspace
	\begin{minipage}{\mpwthree}%
	\centering%
	\includegraphics[width=\columnwidth]{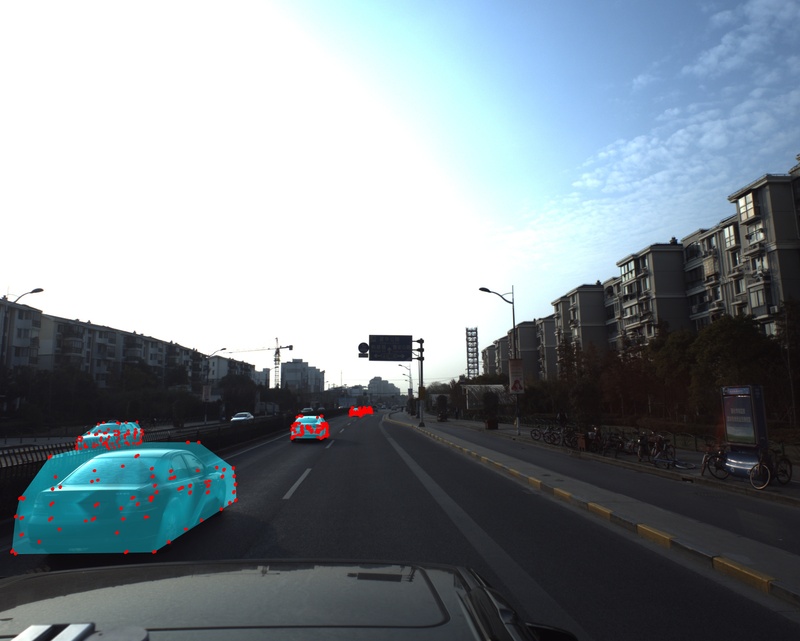} \\
	\vspace{1mm}
	\end{minipage}
& \myhspace
	\begin{minipage}{\mpwthree}%
	\centering%
	\includegraphics[width=\columnwidth]{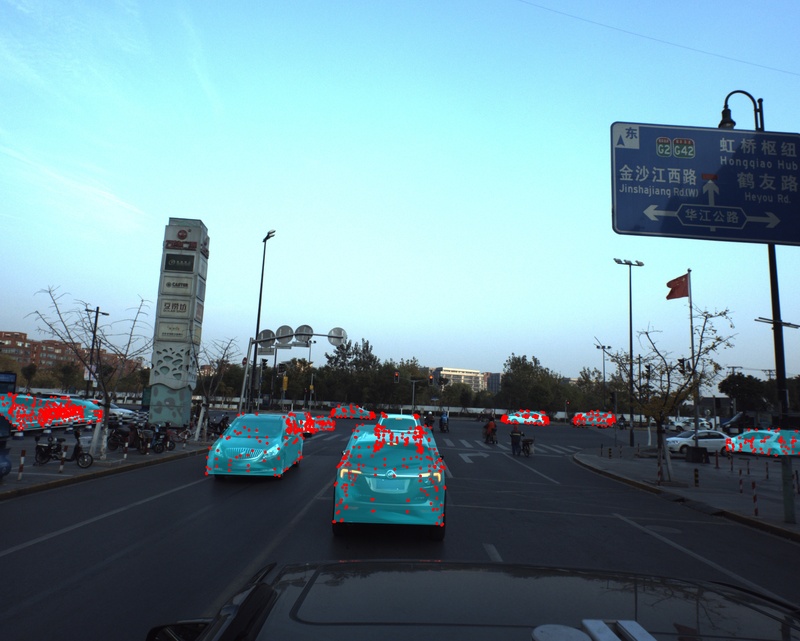} \\
	\vspace{1mm}
	\end{minipage}
& \myhspace
	\begin{minipage}{\mpwthree}%
	\centering%
	\includegraphics[width=\columnwidth]{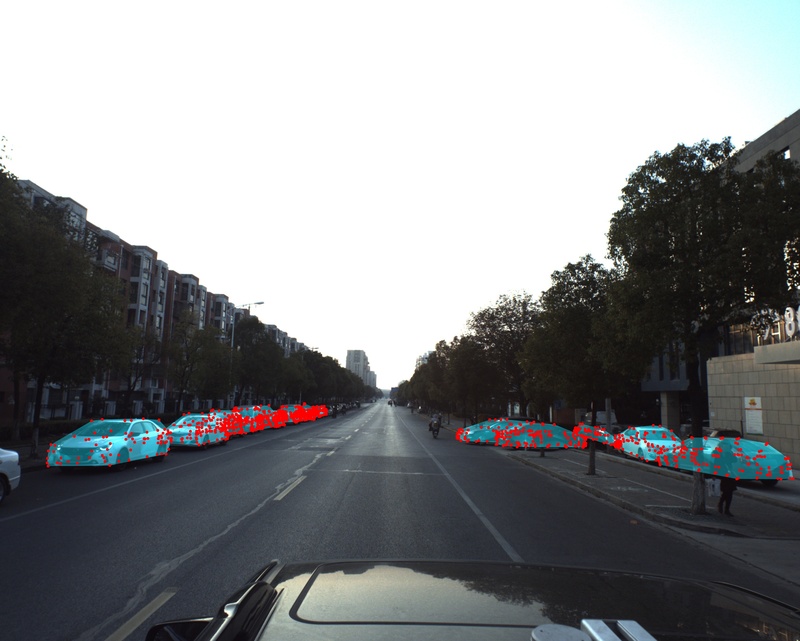} \\
	\vspace{1mm}
	\end{minipage} 
\\
&
\myhspace
	\begin{minipage}{\mpwthree}%
	\centering%
	\includegraphics[width=\columnwidth]{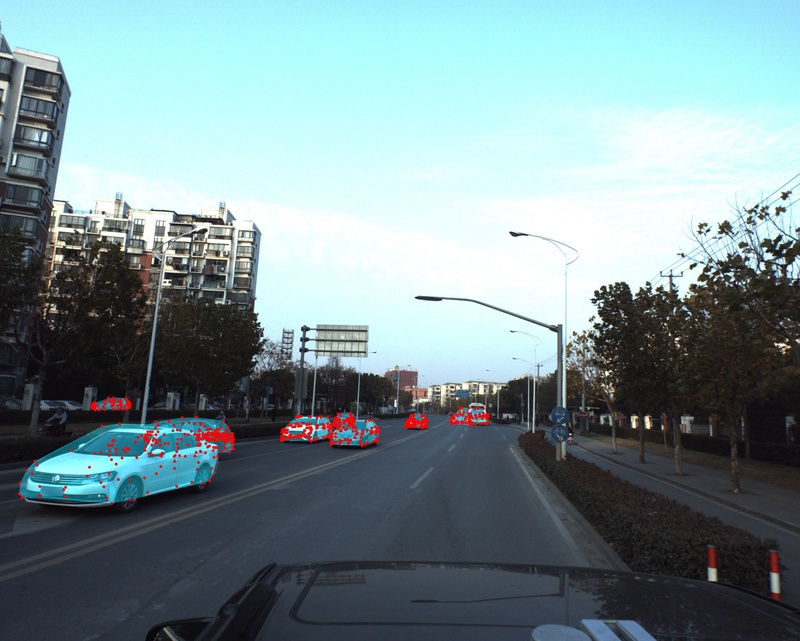} \\
	\vspace{1mm}
	\end{minipage}
& \myhspace
	\begin{minipage}{\mpwthree}%
	\centering%
	\includegraphics[width=\columnwidth]{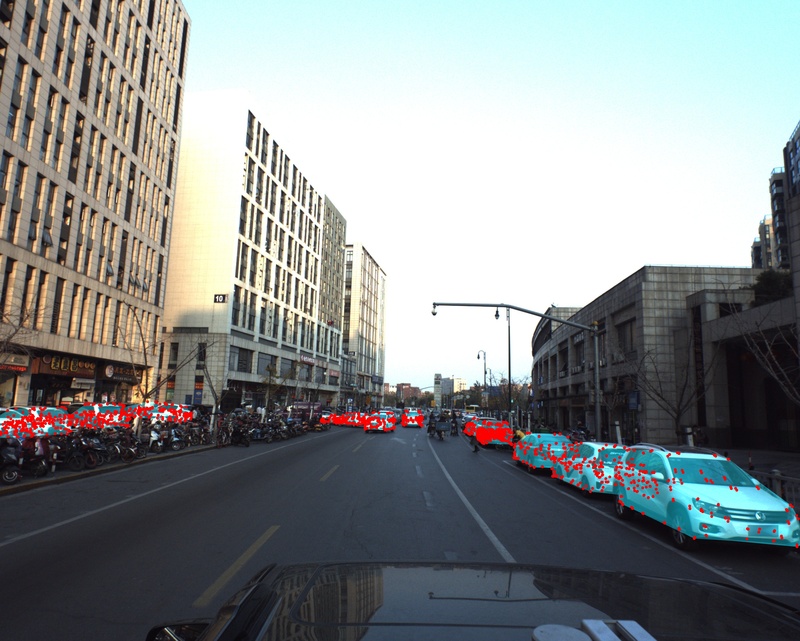} \\
	\vspace{1mm}
	\end{minipage}
& \myhspace
	\begin{minipage}{\mpwthree}%
	\centering%
	\includegraphics[width=\columnwidth]{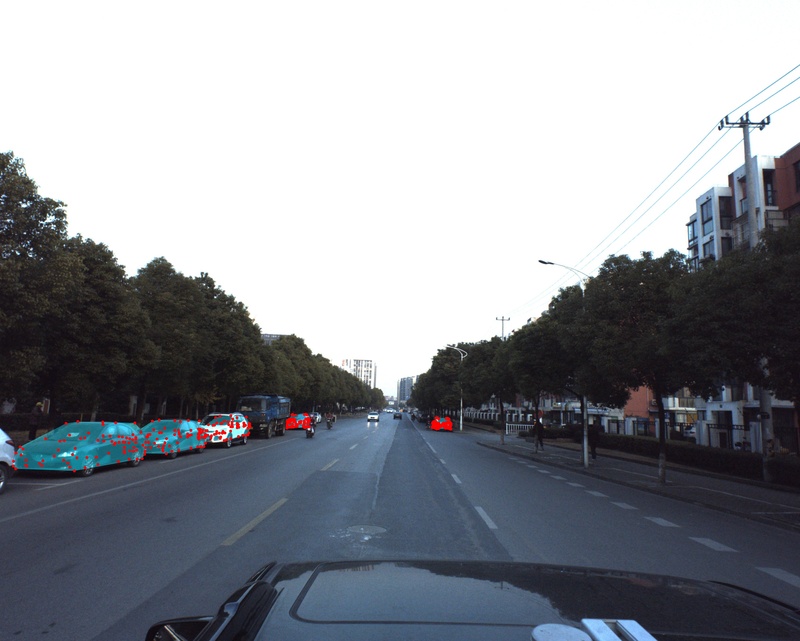} \\
	\vspace{1mm}
	\end{minipage} 
\\
&
\myhspace
	\begin{minipage}{\mpwthree}%
	\centering%
	\includegraphics[width=\columnwidth]{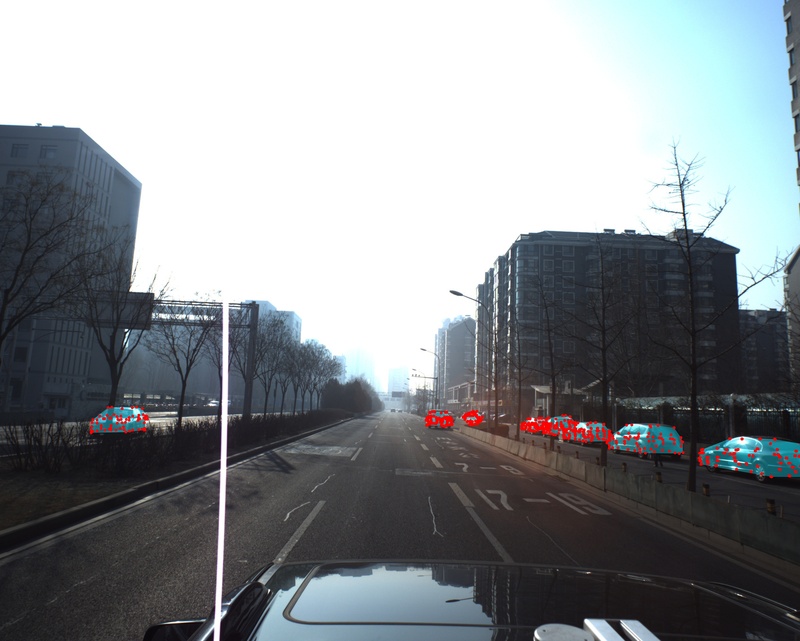} \\
	\vspace{1mm}
	\end{minipage}
& \myhspace
	\begin{minipage}{\mpwthree}%
	\centering%
	\includegraphics[width=\columnwidth]{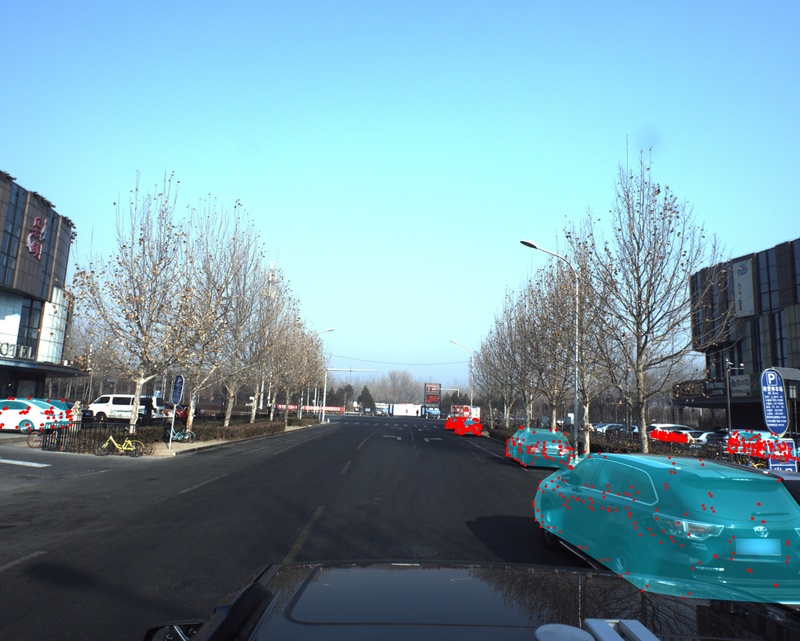} \\
	\vspace{1mm}
	\end{minipage}
& \myhspace
	\begin{minipage}{\mpwthree}%
	\centering%
	\includegraphics[width=\columnwidth]{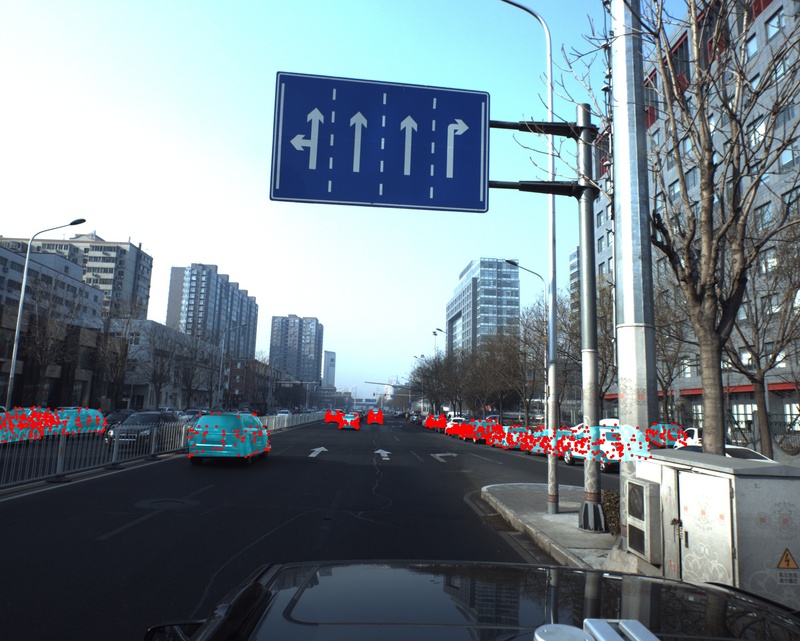} \\
	\vspace{1mm}
	\end{minipage} 
\\
&
\myhspace
	\begin{minipage}{\mpwthree}%
	\centering%
	\includegraphics[width=\columnwidth]{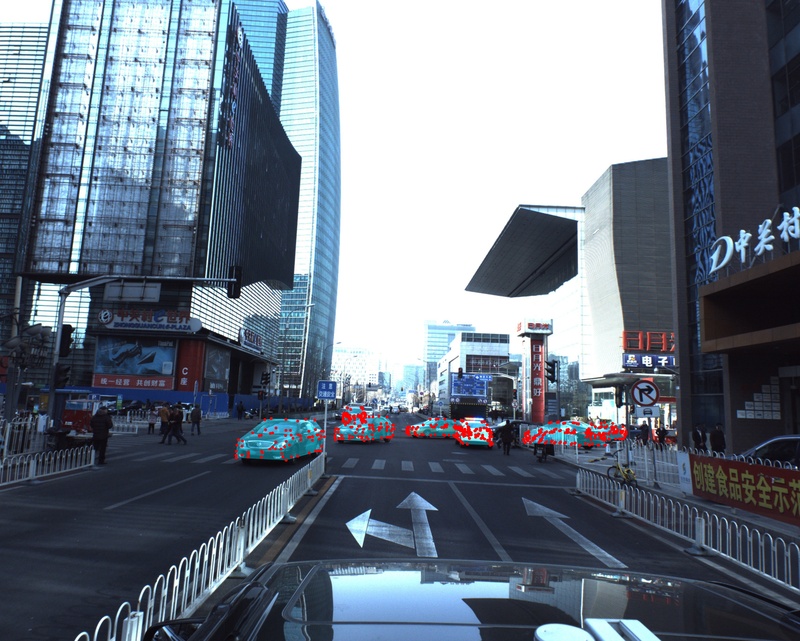} \\
	\vspace{1mm}
	\end{minipage}
& \myhspace
	\begin{minipage}{\mpwthree}%
	\centering%
	\includegraphics[width=\columnwidth]{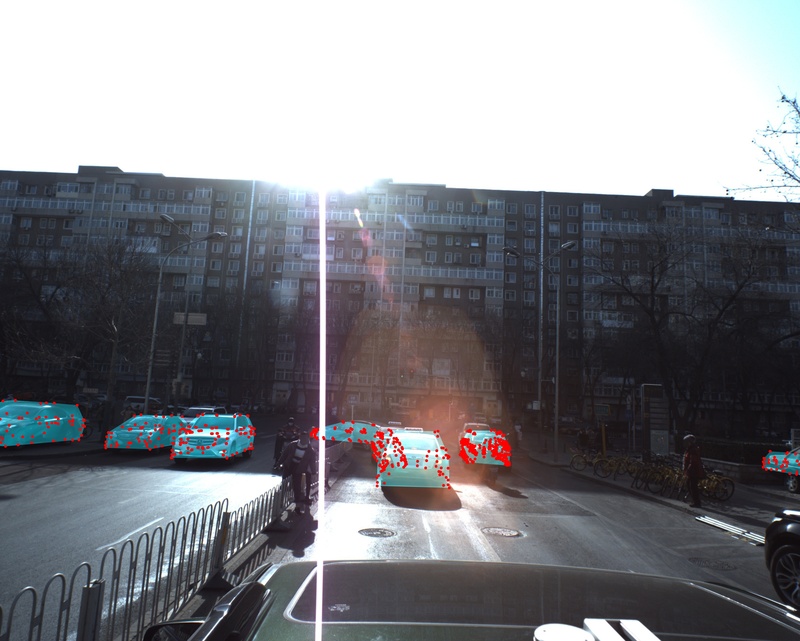} \\
	\vspace{1mm}
	\end{minipage}
& \myhspace
	\begin{minipage}{\mpwthree}%
	\centering%
	\includegraphics[width=\columnwidth]{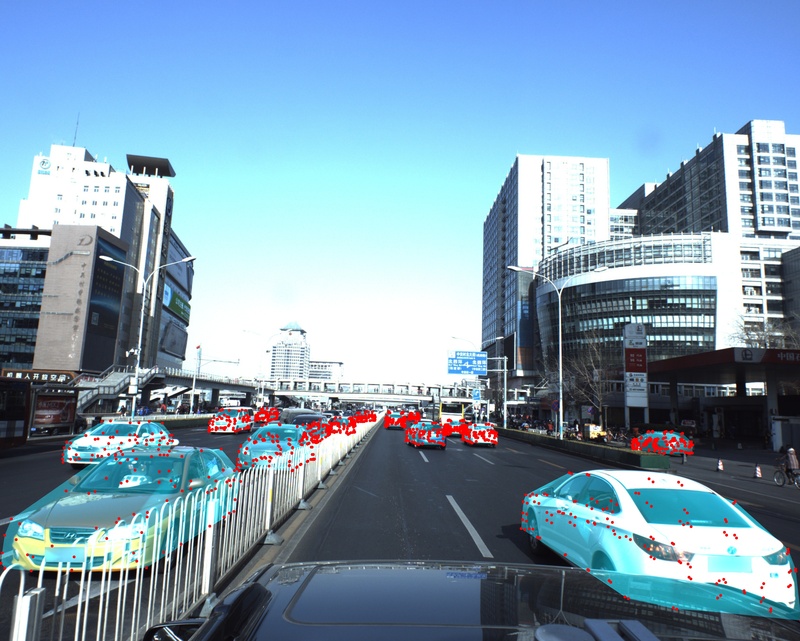} \\
	\vspace{1mm}
	\end{minipage} 
\\
\hline
\vspace{-2mm}
\\
	\myhspace \myhspace %
	\rotatebox{90}{\hspace{-7mm} \red{failures} } 
	&
\myhspace
	\begin{minipage}{\mpwthree}%
	\centering%
	\includegraphics[width=\columnwidth]{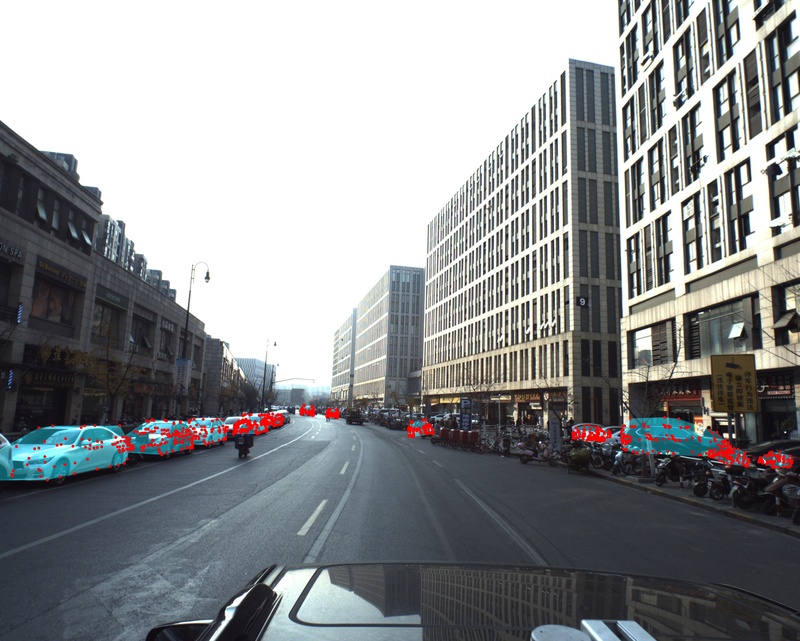} \\
	\vspace{1mm}
	\end{minipage}
& \myhspace
	\begin{minipage}{\mpwthree}%
	\centering%
	\includegraphics[width=\columnwidth]{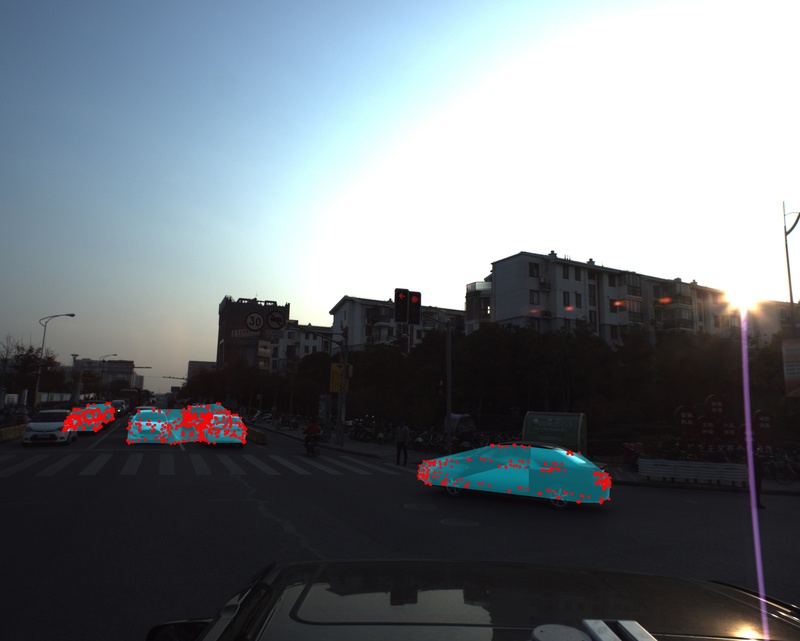} \\
	\vspace{1mm}
	\end{minipage}
& \myhspace
	\begin{minipage}{\mpwthree}%
	\centering%
	\includegraphics[width=\columnwidth]{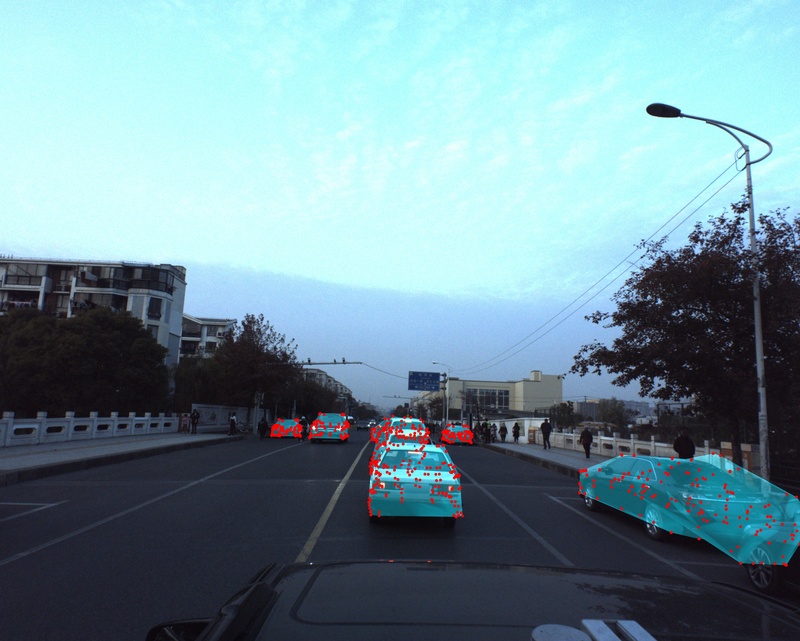} \\
	\vspace{1mm}
	\end{minipage} 
	\end{tabular}
	\end{minipage}
	\vspace{-2mm} 
	\caption{Qualitative results for \PACErobustTwo: overlay of estimated vehicle pose and shape on the images from the \apollo dataset. 
	The images are manually selected out of the validation set in the dataset to showcase successful vehicle
	localization (top 4 rows) as well as failure cases (last row).}
  \label{fig:app-pace2d-vis}
	\end{center}
\end{figure*}

\clearpage

\section{Using \robin to Uncover Mislabeled Ground-truth Annotations in \apollo}
\label{sec:apollo-gt-mislabel}

This section shows that \robin is also able to detect incorrect ground-truth keypoint labels in the  \apollo dataset.
In particular, we tested \robin on the ground-truth keypoints provided by \apollo for training/validation/testing and manually inspected the points discarded as outliers by \robin.
Intuitively, if the ground-truth annotations are correct (hence consistent with the CAD models),
then all ground-truth keypoints must be included in the maximum hyperclique of the compatibility hypergraph,
 and \robin would detect zero outliers.
On the other hand, if there are keypoints not included in the maximum hyperclique, then they might be potentially mislabeled keypoints.

We ran \robin on the ground-truth annotated keypoints in the validation set, which consists of 200 images with a total of 2674 car instances.
For each instance in which \robin detected outliers (1566 counts), we then manually checked the keypoints against their semantic descriptions provided in~\cite{Song19-apollocar3d}.
We uncovered a total of 31 instances of incorrect ground-truth annotations, which is about 1.9\% of the number of instances where \robin detects outliers.\footnote{Since in this case we build the winding order dictionary via ray-tracing and due to the approximations induced by the discretization, \robin also returns false positives, where the ground-truth is correct while \robin detects outliers.}

\begin{table}[h!]
  \scriptsize
  \centering
\begin{tabular}{@{}lll@{}}
\toprule
Image Name &  \begin{tabular}{@{}l@{}}Instance \\ Index\end{tabular} & Description \\ \midrule
\texttt{171206_034625454_Camera_5} & 3 & \#42 should be \#38.  \\ \hline
\texttt{171206_034625454_Camera_5} & 11& \#65 should be \#22.  \\ \hline
\texttt{171206_035907482_Camera_5} & 3 & \#15 should be \#19.  \\ \hline
\texttt{171206_065841148_Camera_6} & 0 & \#7 should be \#50. \\ \hline
\texttt{171206_070959313_Camera_5} & 0 & \#54 should be \#8. \\ \hline
\texttt{171206_071422746_Camera_5} & 4 & \begin{tabular}{@{}l@{}}\#45 should be \#41; \\ \#38 should be \#37; \\ \#37 should be 36.\end{tabular}  \\ \hline
\texttt{171206_074637632_Camera_5} & 3 & \#34 should be \#26. \\ \hline
\texttt{171206_080929510_Camera_5} & 5 & \#49 and \#8 should be switched. \\ \hline
\texttt{171206_082047084_Camera_5} & 0 & \#61 and \#60 should be switched. \\ \hline
\texttt{180114_023107908_Camera_5} & 4 & \begin{tabular}{@{}l@{}}\#29 should be placed at the\\ right exhaust below \#31.\end{tabular} \\ \hline
\texttt{180114_025215740_Camera_5} & 5 & \#36 should be \#21. \\ \hline
\texttt{180114_025714151_Camera_5} & 2 & \begin{tabular}{@{}l@{}}\#62, \#63, \#64, \#65 should \\ be \#59, \#58, \#60, \#61. \end{tabular}\\ \hline
\texttt{180114_030620413_Camera_5} & 18& \#19 should be \#24. \\ \hline
\texttt{180116_031340200_Camera_5} & 18& \#11 should be \#9. \\ \hline
\texttt{180116_032115845_Camera_5} & 4 & \#46 should be \#33. \\ \hline
\texttt{180116_040654119_Camera_5} & 6 & \#10 and \#15 misplaced. \\ \hline
\texttt{180116_041204649_Camera_5} & 0 & \#35 should be \#22. \\ \hline
\texttt{180116_041204649_Camera_5} & 5 & \#35 should be \#22. \\ \hline
\texttt{180116_053637003_Camera_5} & 4 & \#16 should be \#41. \\ \hline
\texttt{180116_053637003_Camera_5} & 10& \#30 should be \#24. \\ \hline
\texttt{180116_053637003_Camera_5} & 11& \#9 should be \#48. \\ \hline
\texttt{180116_053945714_Camera_5} & 6 & \#46 should be \#11. \\ \hline
\texttt{180116_064533064_Camera_5} & 17 & \#9 should be \#12. \\ \hline
\texttt{180116_065033600_Camera_5} & 11 & \#30 should be \#33. \\ \hline
\texttt{180118_070001729_Camera_5} & 11 & Keypoints labeled are on two cars. \\ \hline
\texttt{180118_070451170_Camera_5} & 5 & \#8 and \#49 should be switched. \\ \hline
\texttt{180118_070451170_Camera_5} & 13& \#28 should be \#33. \\ \hline
\texttt{180118_071218867_Camera_5} & 5 & \#11 should be \#29. \\ \hline
\texttt{180118_071316772_Camera_5} & 4 & \#9 is mislabeled on another car. \\ \hline
\texttt{180118_072006080_Camera_5} & 2 & \#42 should be \#10. \\ \hline
\texttt{180118_072006080_Camera_5} & 6 & \#11 should be \#29. \\
\bottomrule
\end{tabular}
\caption{Details on all 31 mislabeled instances where \robin returns non-zero outliers.} \label{tbl:app-apollo-gt-mislabel}
\end{table}

\begin{table}[h!]
  \scriptsize
  \centering
\begin{tabular}{@{}lll@{}}
\toprule
Image Name &  \begin{tabular}{@{}l@{}}Instance \\ Index\end{tabular} & Description \\ \midrule
\texttt{171206_072104589_Camera_5} & 8 & \#3 misplaced.  \\ \hline
\texttt{171206_074637632_Camera_5} & 5 & \#25 should be \#35.  \\ \hline
\texttt{171206_080255599_Camera_5} & 1 & \#42 misplaced.  \\ \hline
\texttt{171206_080827230_Camera_5} & 3 & \#31 and \#34 placed on another car. \\ \hline
\texttt{180114_031156951_Camera_5} & 3 & \#9 should be \#11. \\ \hline
\texttt{180116_061243174_Camera_5} & 4 & \#1 should be \#3. \\ 
\bottomrule
\end{tabular}
\caption{Details on 6 mislabeled instances where \robin returns zero outliers.} \label{tbl:app-apollo-gt-mislabel-not-robin}
\end{table}

Table~\ref{tbl:app-apollo-gt-mislabel} details them with the specific frames in which they are found.
A common failure mode is the misplacement of keypoints between the two sides of the car.
For example, placing \#35 (defined as the top right corner of right rear car light in~\cite{Song19-apollocar3d}) instead of \#22 (defined as the top left corner of left rear car light).
Fig.~\ref{fig:app-apolloscape-wrong-gt-another-car} and~\ref{fig:app-apolloscape-wrong-gt-collage} show cropped images of the 31 mislabeled instances, with keypoints annotated.
In particular, Fig.~\ref{fig:app-apolloscape-wrong-gt-another-car} shows an egregious case, with  keypoint \#9 being mislabeled as belonging to a completely different car on the other side of the frame (in this case, \robin successfully detects \#9 as the sole outlier).
We want to point out that due to the inaccuracies in the winding order dictionary, outliers determined by \robin might contain false positives.
For example, in image \texttt{171206_071422746_Camera_5} (see Fig.~\ref{fig:app-apolloscape-wrong-gt-collage}), 6 keypoints are flagged by \robin~as outliers, yet only 3 of them are actual outliers.
Nevertheless, \robin still provides a viable way to validate the quality of 2D keypoints annotations based on 3D models.

A natural question to ask is how many mislabled keypoints are there in instances where \robin returns zero outliers (see Fig.~\ref{fig:app-apolloscape-wrong-gt-collage-not-detected} for some examples).
After all, it is possible for outliers to exist in degenerate configurations that pass \robin's compatibility checks consistently,
making their ways into the maximum hyperclique.
In addition, the dictionary of feasible winding orders obtained through ray-tracing is inherently only an approximation,
due to the use of discretization.
There are a total of 993 instances where \robin returns zero outliers.
We manually checked all of them, and 6 of them contains mislabled keypoints, which is about 0.6\%.
This is 3 times lower than that with non-zero \robin outliers,
suggesting that \robin is much more effective than a baseline method that randomly samples instances. 

We want to stress that we are not trying to discredit the work behind \apollo.
We sympathize with the authors and workers behind~\cite{Song19-apollocar3d}, as labeling 66 keypoints for all cars in the dataset is a herculean task.
Our intent is to demonstrate that \robin can effectively detects outliers, and that it may also be used to validate human labels.
In addition, exploring whether \robin can be used to enable self-supervision for detecting semantic keypoints will be an interesting topic for future research.
For consistency with prior works, when calculating metrics in Section~\ref{sec:experiments}, we used the original annotations, including mislabeled ground-truth keypoints.

\clearpage

\begin{figure*}[th!]
    \centering%
    \includegraphics[width=\textwidth]{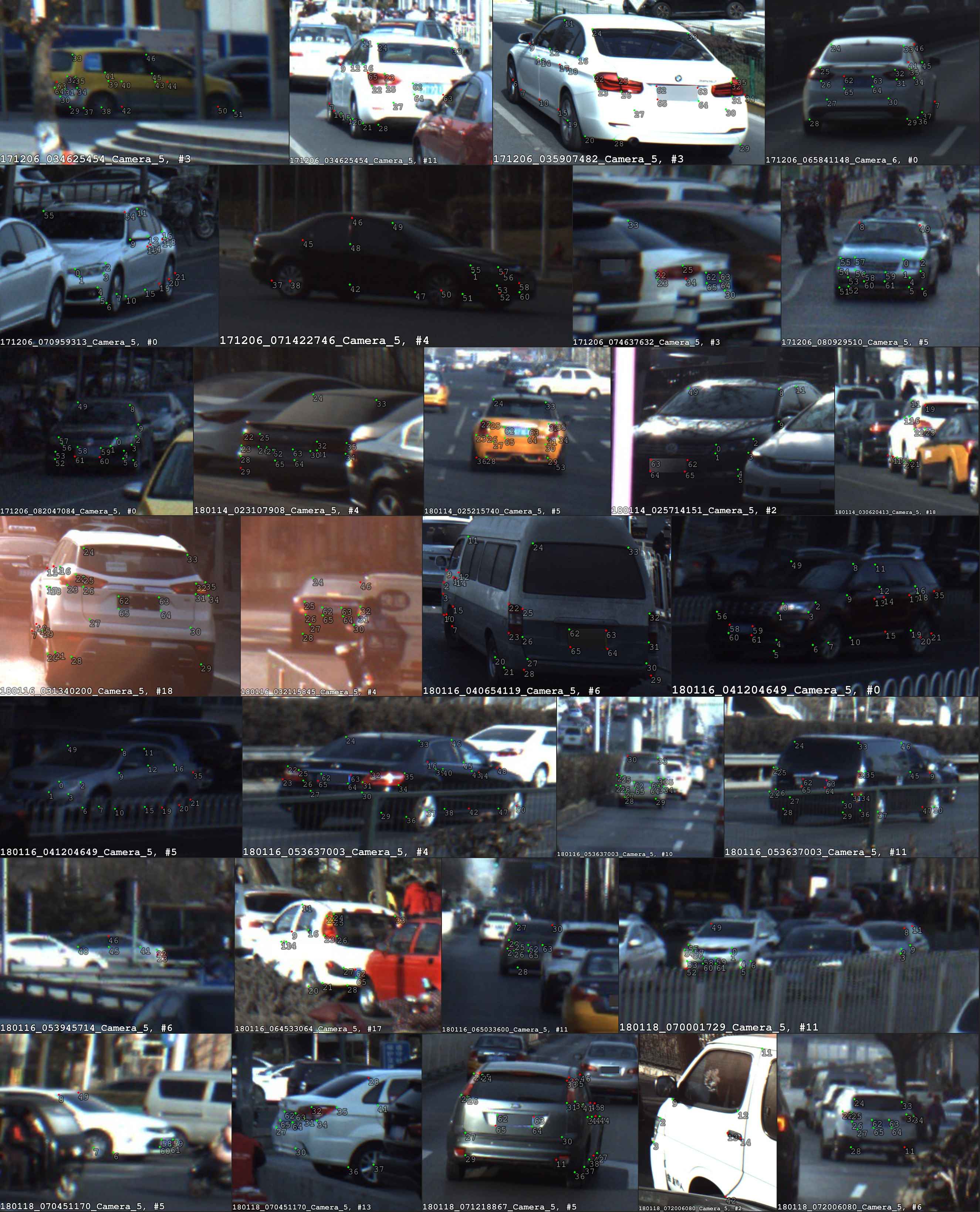}
  \caption{
      Mislabeled ground-truth keypoint annotations in \apollo detected by \robin. Green and red dots represent inliers and outliers determined by \robin, respectively.
      For a detailed description of the incorrect keypoint annotations for each image, we refer the reader to Table~\ref{tbl:app-apollo-gt-mislabel}.
    }
\label{fig:app-apolloscape-wrong-gt-collage}
\end{figure*}

\begin{figure*}[t!]
    \centering%
    \includegraphics[width=\textwidth]{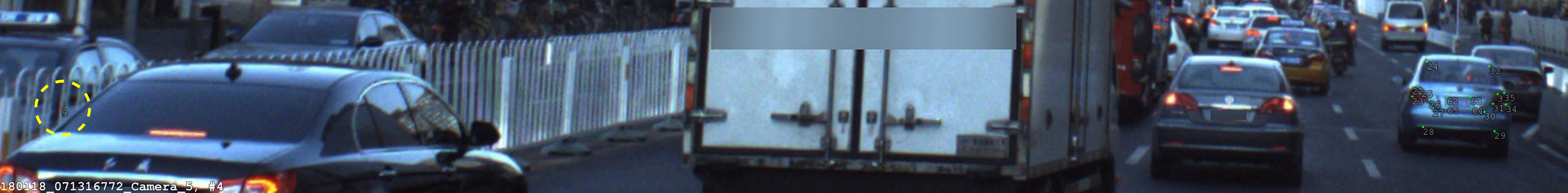}
    \caption{Keypoint \#9  (circled in yellow) is mislabeled in the 4th labeled instance in image 180118_071316772_Camera_5.
      Instead of being on the same car as the rest of the keypoints, \#9 is misplaced as belonging to  another car on the far left side.
    }
\label{fig:app-apolloscape-wrong-gt-another-car}
\end{figure*}

\begin{figure*}[th!]
    \centering%
    \includegraphics[width=0.5\textwidth]{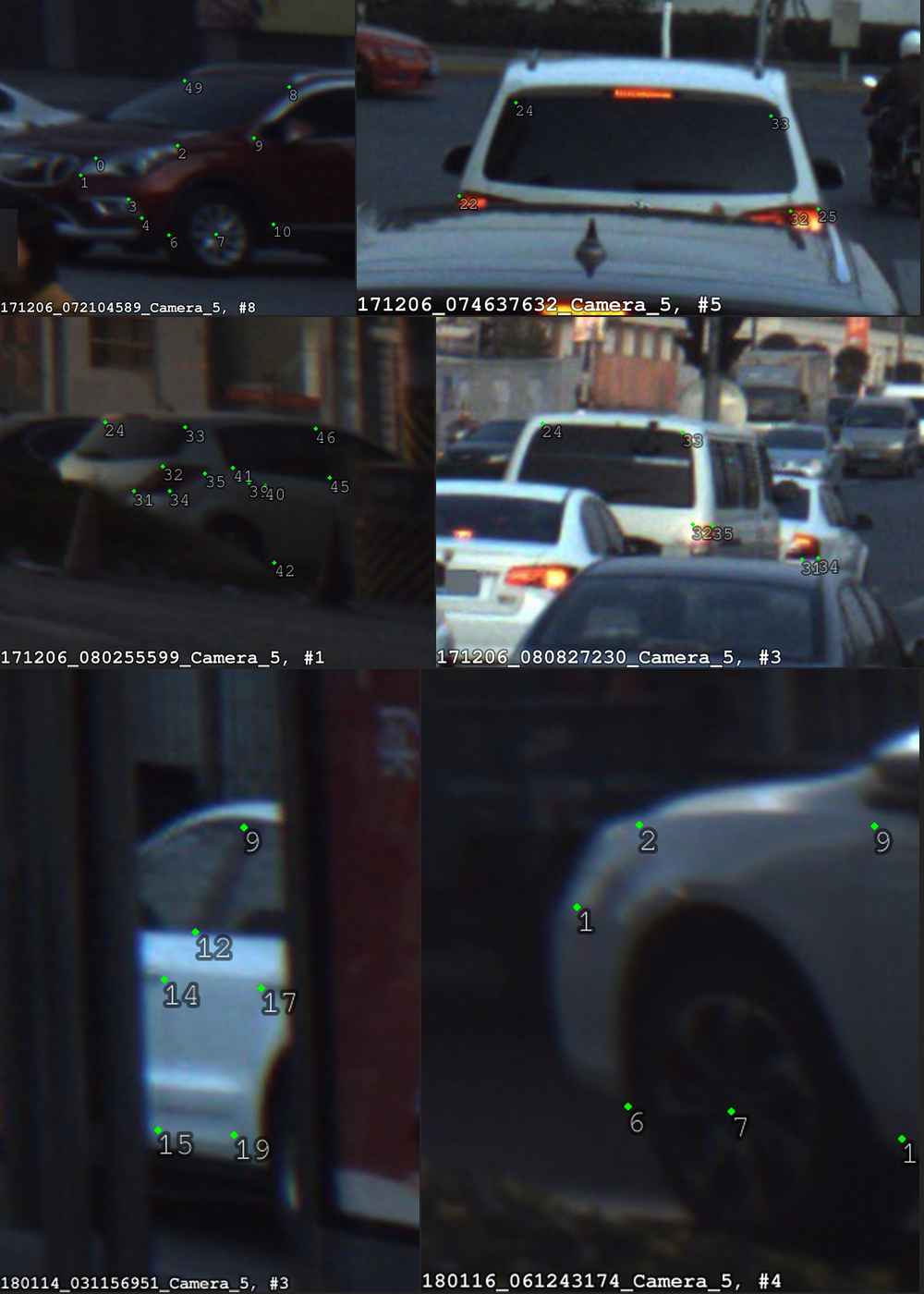}
  \caption{
      Mislabeled ground-truth keypoint annotations in \apollo that were not detected by \robin. Green dots represent inliers determined by \robin.
      For details, we refer the reader to Table~\ref{tbl:app-apollo-gt-mislabel-not-robin}.
    }
\label{fig:app-apolloscape-wrong-gt-collage-not-detected}
\end{figure*}

}{
}

\end{document}